\documentclass{article}

\PassOptionsToPackage{numbers,compress}{natbib}

\usepackage[final,main]{neurips_2026}  

\usepackage[utf8]{inputenc}
\usepackage[T1]{fontenc}
\usepackage{microtype}
\usepackage{graphicx}
\usepackage{subcaption}
\usepackage{booktabs}
\usepackage{algorithm}
\usepackage{algorithmic}
\usepackage{amsmath}
\usepackage{amssymb}
\usepackage{mathtools}
\usepackage{amsthm}
\usepackage{float}
\usepackage{wrapfig}
\usepackage{multirow}
\usepackage{makecell}
\usepackage{array}
\usepackage{pifont}
\usepackage{placeins}
\usepackage{dblfloatfix}
\usepackage{siunitx}
\usepackage{etoc}
\usepackage{xcolor}
\usepackage{colortbl}
\usepackage{tabularx}
\usepackage[font=small,skip=2pt]{caption}

\definecolor{alg_shape_color}{RGB}{130,130,255}
\definecolor{alg_comment_color}{RGB}{185,185,185}
\definecolor{lightblue}{RGB}{170,170,170}
\definecolor{darkblue}{rgb}{0,0.08,0.45}

\usepackage[colorlinks=true,linkcolor=darkblue,citecolor=darkblue,urlcolor=darkblue]{hyperref}
\makeatletter
\hypersetup{pdftitle={On the Recall Scaling Laws in Mamba: A Theoretical and Mechanistic Study via Hashing}}
\if@anonymous\else
\hypersetup{pdfauthor={Yuval Koren, Assaf Ben-Kish, Raja Giryes, Lior Wolf, Itamar Zimerman}}
\fi
\makeatother
\usepackage[capitalize,noabbrev]{cleveref}

\newcommand{\shape}[1]{{\textcolor{lightblue}{$\;\;$(#1)}}}
\theoremstyle{plain}
\newtheorem{theorem}{Theorem}[section]

\newtheorem{lemma}[theorem]{Lemma}

\theoremstyle{definition}
\newtheorem{remark}[theorem]{Remark}

\newcommand{\grayline}{\arrayrulecolor{gray!30}\hline\arrayrulecolor{black}}


\AtBeginDocument{%
    \setlength{\abovedisplayskip}{4pt}
    \setlength{\belowdisplayskip}{4pt}
    \setlength{\abovedisplayshortskip}{2pt}
    \setlength{\belowdisplayshortskip}{2pt}
    \setlength{\jot}{2pt}
    \setlength{\parskip}{0pt}
    \setlength{\topsep}{2pt}
    \setlength{\partopsep}{0pt}
    \setlength{\itemsep}{0pt}
    \setlength{\parsep}{0pt}
    \captionsetup{skip=2pt}
    \setlength{\abovecaptionskip}{2pt}
    \setlength{\belowcaptionskip}{0pt}
    \setlength{\textfloatsep}{10pt}
    \setlength{\intextsep}{10pt}
    \setlength{\floatsep}{10pt}
}

\makeatletter
\newcommand{\startappendixtoc}{%
  \let\app@oldaddcontentsline\addcontentsline
  \renewcommand{\addcontentsline}[3]{%
    \app@oldaddcontentsline{##1}{##2}{##3}%
    \def\app@type{##2}%
    \def\app@sec{section}%
    \def\app@subsec{subsection}%
    \ifx\app@type\app@sec
      \app@oldaddcontentsline{apptoc}{##2}{##3}%
    \else\ifx\app@type\app@subsec
      \app@oldaddcontentsline{apptoc}{##2}{##3}%
    \fi\fi
  }%
}
\newcommand{\printappendixtoc}{%
  \begingroup
  \parindent=0pt
  \renewcommand{\addvspace}[1]{\vskip 8pt}
  \@starttoc{apptoc}%
  \endgroup
}
\makeatother

\setcitestyle{aysep={,},yysep={;}}
\title{
On the Recall Scaling Laws in Mamba:\\
A Theoretical and Mechanistic Study via Hashing
}

\author{%
 Yuval Koren \quad Assaf Ben-Kish \quad Raja Giryes \quad Lior Wolf \quad Itamar Zimerman \\[3pt]
Blavatnik School of Computer Science and AI, Tel Aviv University
}

\begin{document}
\maketitle
\begingroup\makeatletter\renewcommand{\thefootnote}{}\long\def\@makefntext#1{\noindent\fontsize{9pt}{10.5pt}\selectfont #1}\makeatother\footnotetext{Correspondence: \texttt{yuvalkoren4@mail.tau.ac.il}}\endgroup
\begin{abstract}
Associative Recall (AR) is the cognitive ability to learn and retrieve links between items in memory. In NLP, AR is used as a benchmark for evaluating the in-context memory capacity of architectures such as Mamba, and has been found to strongly correlate with language modeling performance. This paper explores AR from the perspective of mechanistic interpretability, aiming to reverse-engineer the exact internal algorithm used by Mamba to perform recall. Our key insight is that Mamba performs recall by implicitly learning linear hash functions, and we identify the low-level circuit that enables this behavior. Building on these findings and inspired by theoretical tools in similarity-preserving hashing, such as the Johnson–Lindenstrauss lemma, we develop a theoretical framework for analyzing AR, which we term \textit{Recall Scaling Laws}. Given the vocabulary size and the number of facts in context, this framework allows us to (1) predict the embedding and state dimensions required for Mamba to achieve perfect recall, (2) predict recall success probability given the model dimensions, and (3) analyze multi-layer models and multi-head SSM patterns. Empirical results show that our theoretical findings are accurate and predictive, offering insights into how AR capacity scales with vocabulary, state, embedding size, and architecture.
\end{abstract}

\begin{center}\vspace{-2pt}
\raisebox{-0.23em}{\includegraphics[height=1.1em]{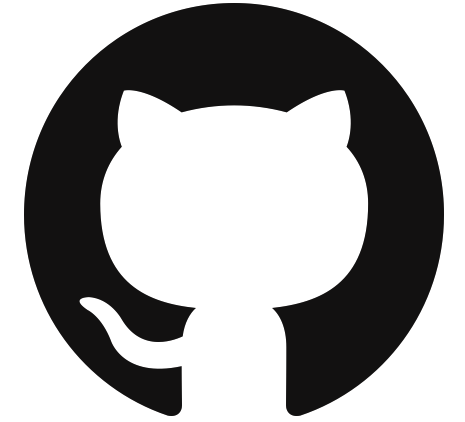}}\hspace{.55em}%
\raisebox{0.1em}{%
\href{https://github.com/yuvalko1/mamba-recall-scaling-laws}
{\texttt{github.com/yuvalko1/mamba-recall-scaling-laws}}}
\vspace{-6pt}\end{center}

\section{Introduction}
Although Transformers have achieved remarkable success in sequence modeling, they suffer from a key inefficiency when handling long sequences: during decoding, memory complexity grows linearly with sequence length due to key-value caching. Recent architectures such as Mamba~\citep{gu2023mamba}, RWKV~\citep{peng2023rwkv}, and other linear RNNs~\citep{de2024griffin} address this by utilizing a fixed-size recurrent state, enabling constant memory complexity. However, this fixed-size state is a double-edged sword: from an information perspective, compressing the entire context into a single vector imposes inherent limitations and may lead to information loss~\citep{wen2024rnns,ben2025overflow}. Hence, improving memory utilization in %
recurrent architectures has become an active area of research%
~\citep{parnichkun2025quantifying,arora2024simple}. To better understand the memory bottlenecks of sequence models, researchers have proposed synthetic benchmarks such as Associative Recall (AR)~\citep{ba2016using}, Multi-Query Associative Recall (MQAR)~\citep{arora2023zoology} and others~\citep{arora2025mechanistic,you2024revealing}, which are highly associated with language modeling capabilities~\citep{arora2023zoology}. These tasks are designed to diagnose a model's recall abilities in controlled environments and have been widely studied across various architectures~\citep{wang2025test,bick2025understanding}, including Transformers~\citep{olsson2022context}, linear RNNs~\citep{fu2022hungry,trockman2024mimetic}, and others~\citep{poli2023hyena}.

In this paper, we take a step toward a deeper understanding of Mamba's recall abilities by adopting a mechanistic interpretability perspective~\citep{elhage2021mathematical,olah2020zoom}, aiming to reverse-engineer the internal algorithm employed by linear RNNs to solve AR tasks. Through empirical analysis, we identify the circuit-level mechanism that enables Mamba to perform AR and interpret it as implicitly computing a similarity-preserving linear hash function~\citep{andoni2008near,datar2004locality,salakhutdinov2009semantic}. Building on this connection, and drawing on theoretical tools from linear hashing such as the Johnson–Lindenstrauss lemma~\citep{johnson1984extensions}, we develop the \textit{Recall Scaling Laws} theoretical framework, which allows us to analyze how the state size and embedding dimension must scale with the vocabulary size and the number of facts in order to achieve perfect recall. Empirical analysis shows that these scaling laws are relatively accurate and reflect the model's behavior in practice.
\begin{figure*}[t]
\small
\centering
\begin{tabular}{@{\hskip 0.06in}c@{\hskip 0.06in}c@{\hskip 0.06in}c@{\hskip 0.05in}c}
   AR & MQAR & MQAR & \\
   (a) Linear Model & (b) Linear Model & (c) Full Model & \\
  \cmidrule(lr){1-3}

    \raisebox{-0.5\height}{\includegraphics[width=0.31\textwidth]{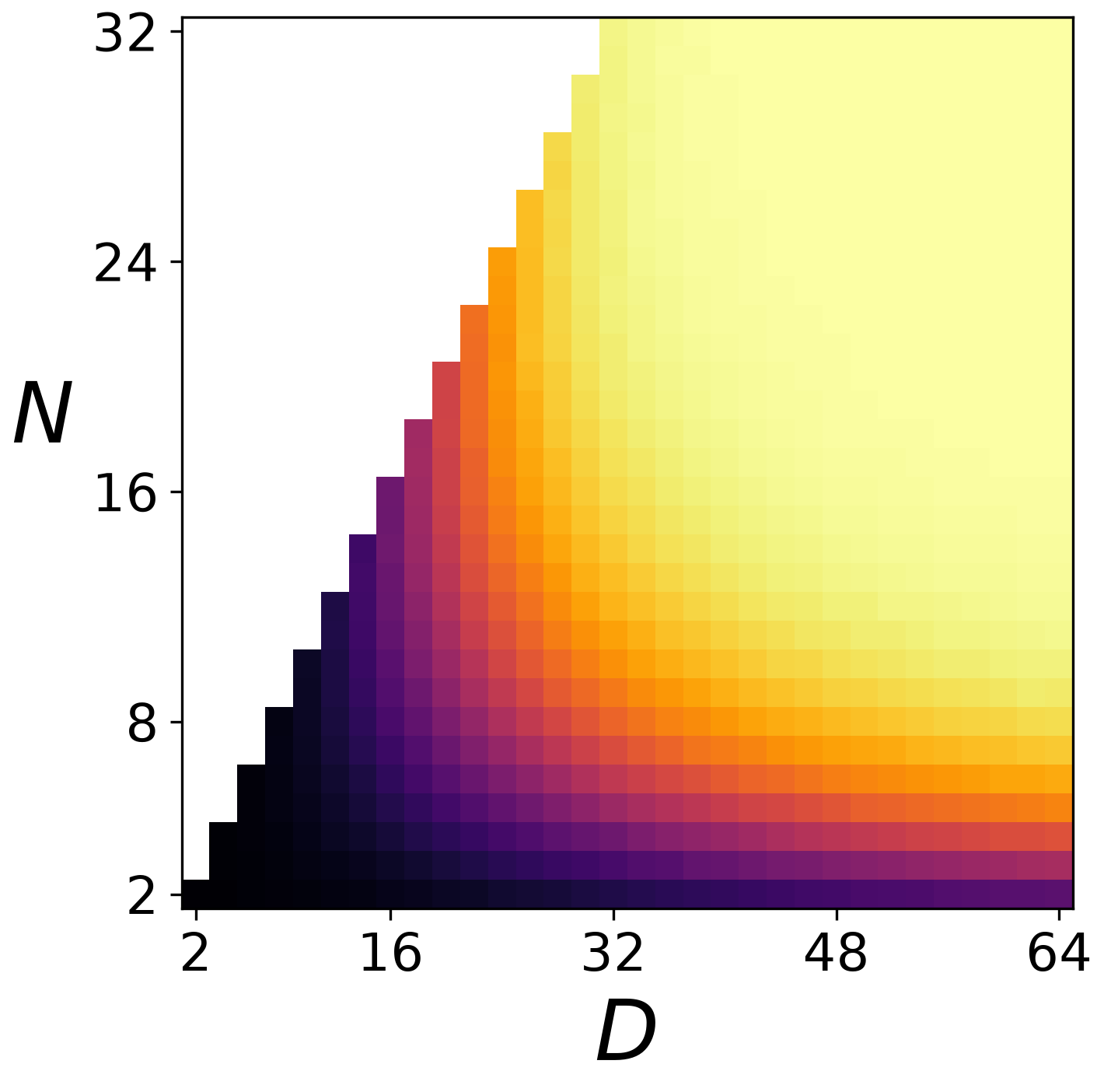}} &
    \raisebox{-0.5\height}{\includegraphics[width=0.31\textwidth]{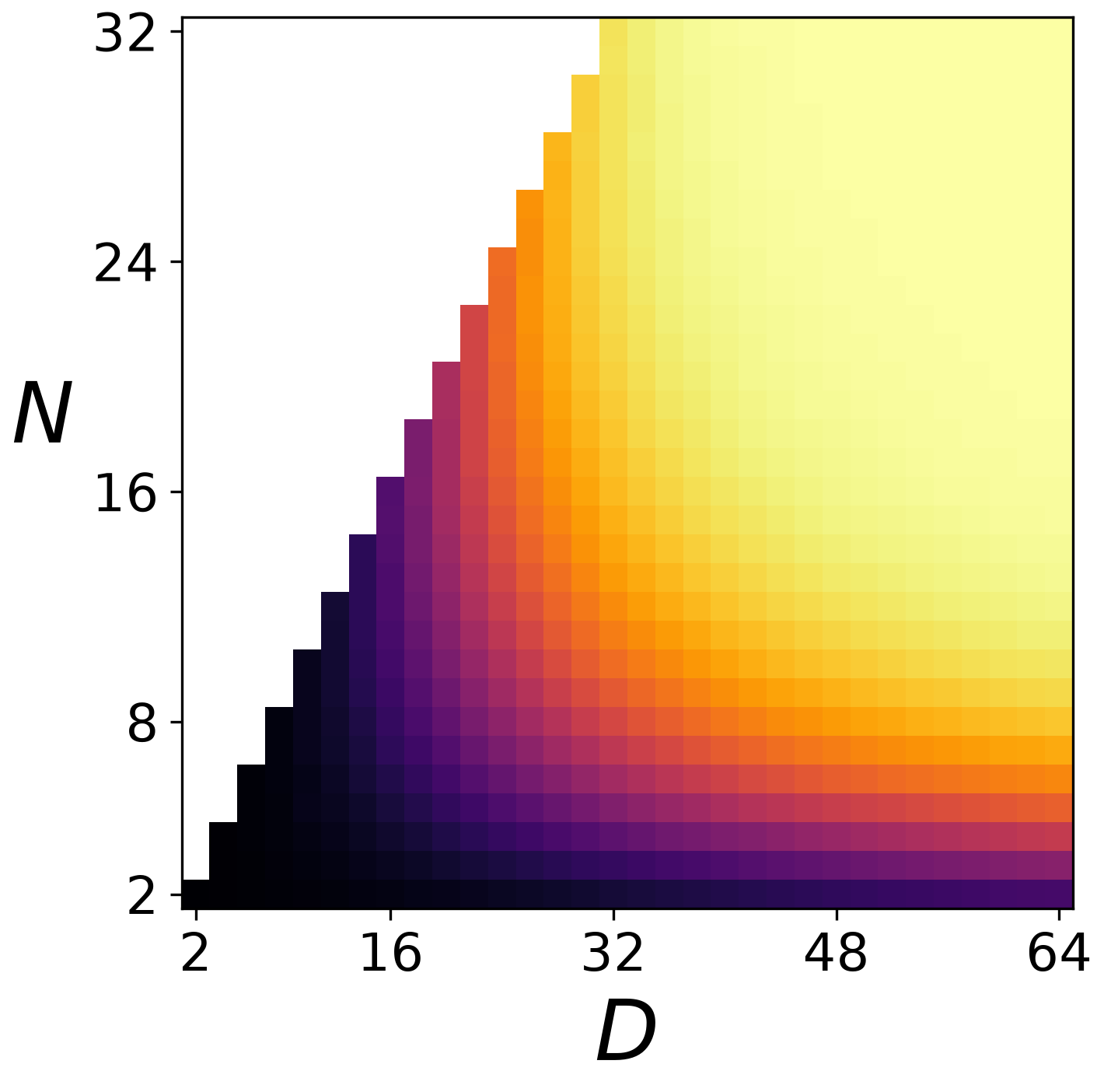}} &
    \raisebox{-0.5\height}{\includegraphics[width=0.31\textwidth]{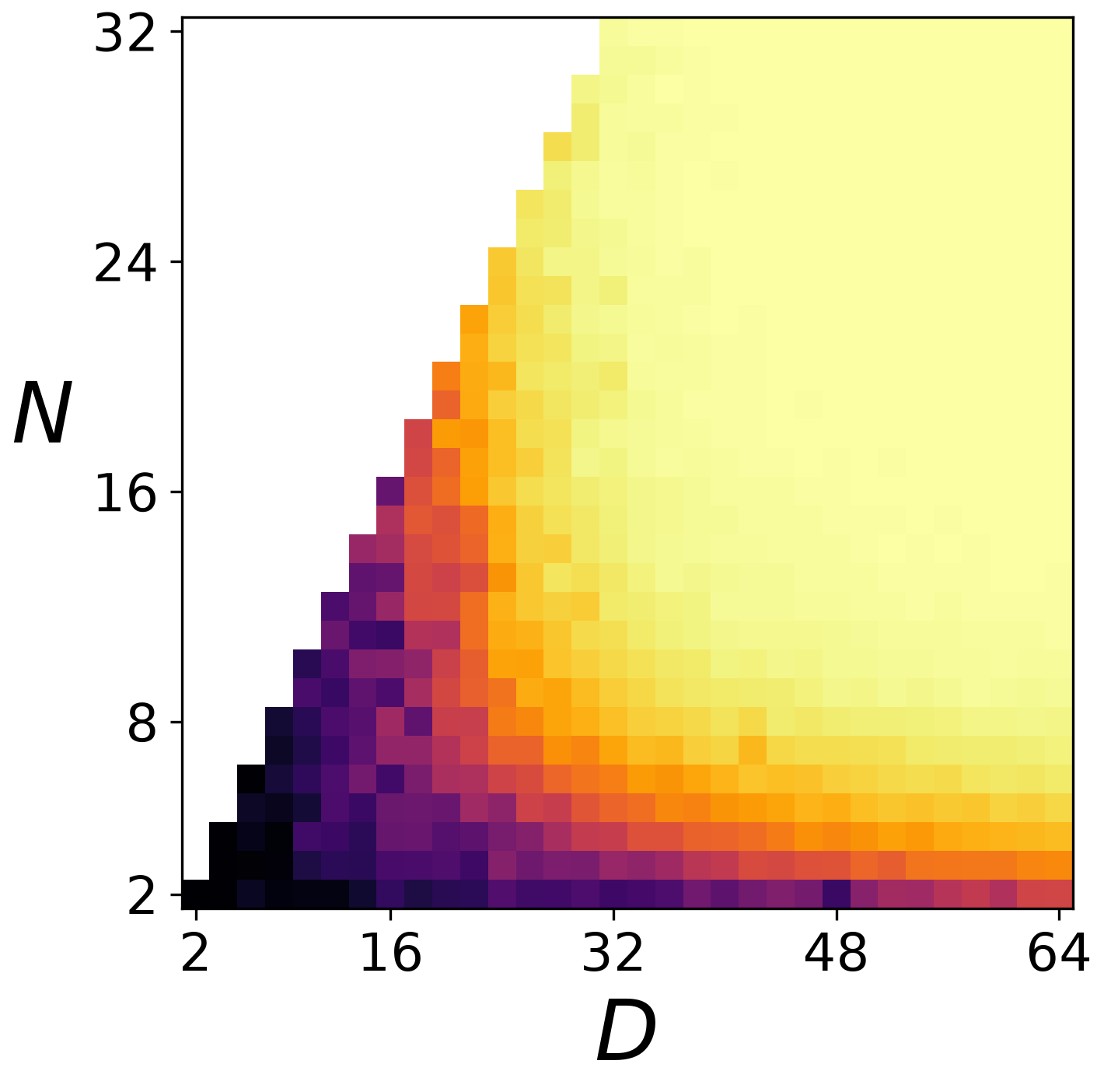}} &
    \raisebox{-0.5\height}{\includegraphics[width=0.045\textwidth]{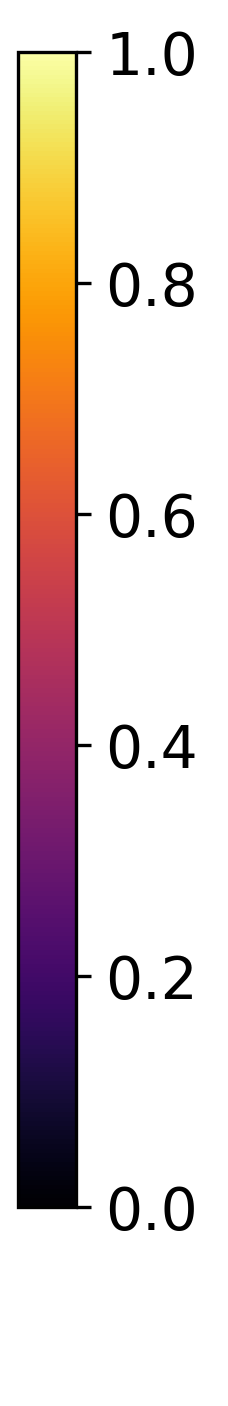}}
 \\

    \raisebox{-0.5\height}{\includegraphics[width=0.31\textwidth]{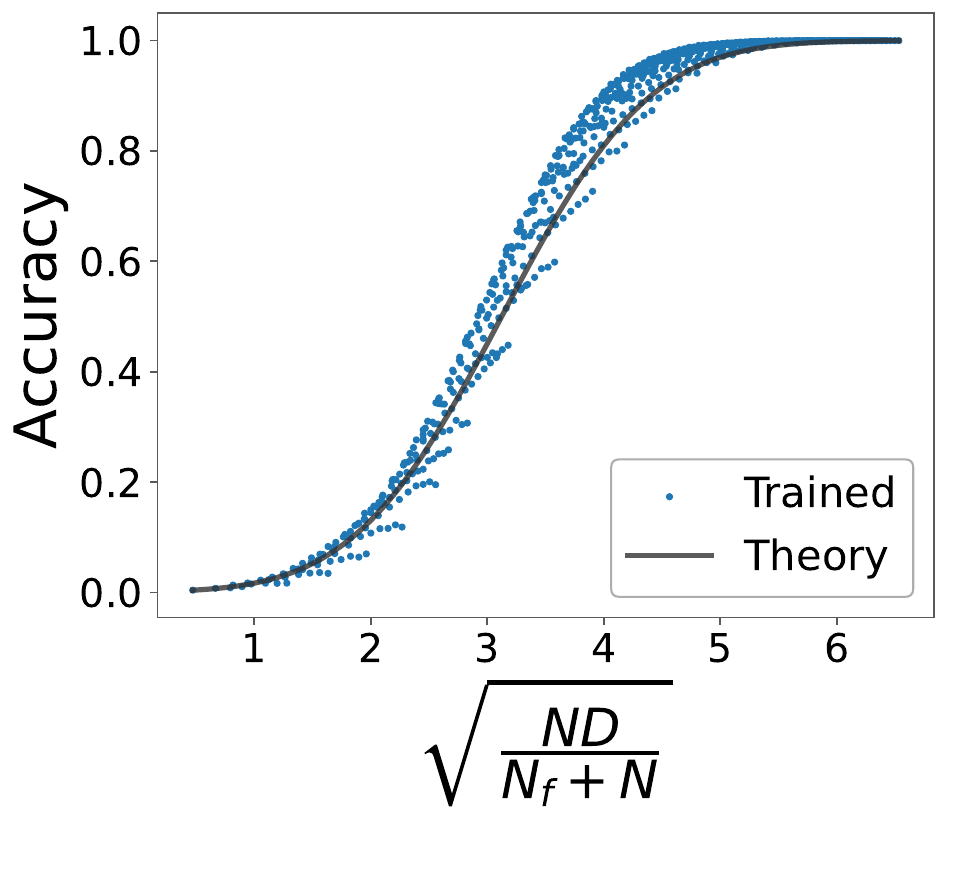}} &
    \raisebox{-0.5\height}{\includegraphics[width=0.31\textwidth]{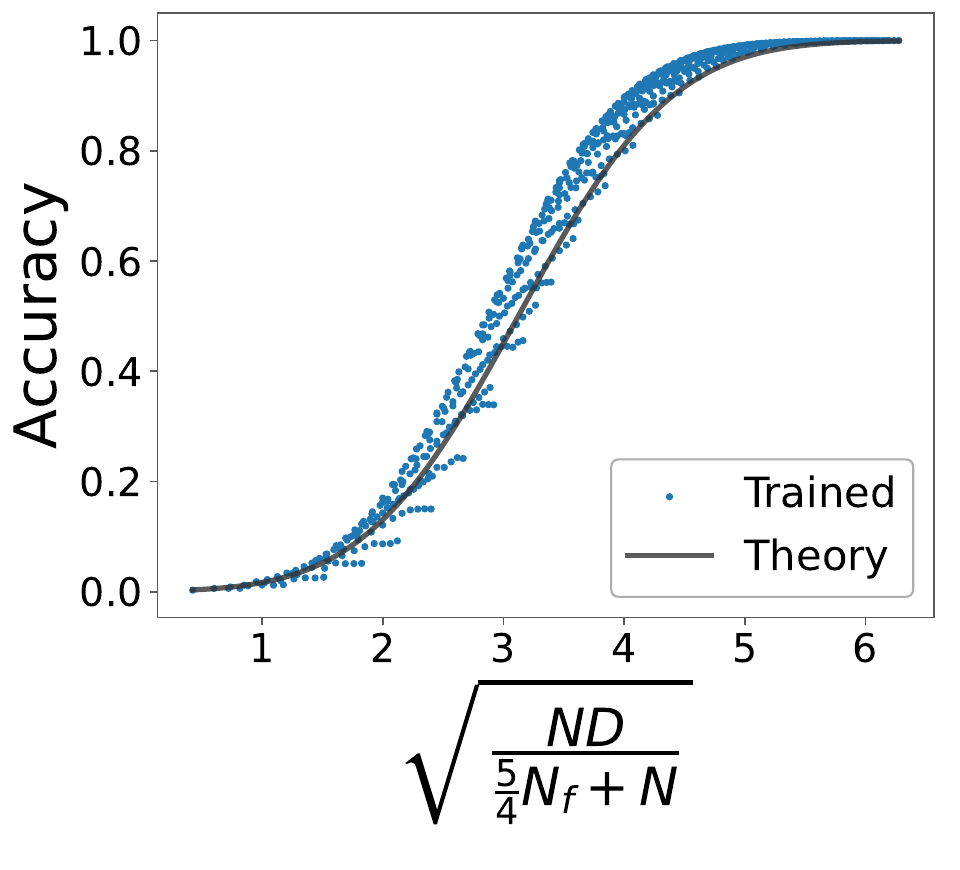}} &
    \raisebox{-0.5\height}{\includegraphics[width=0.31\textwidth]{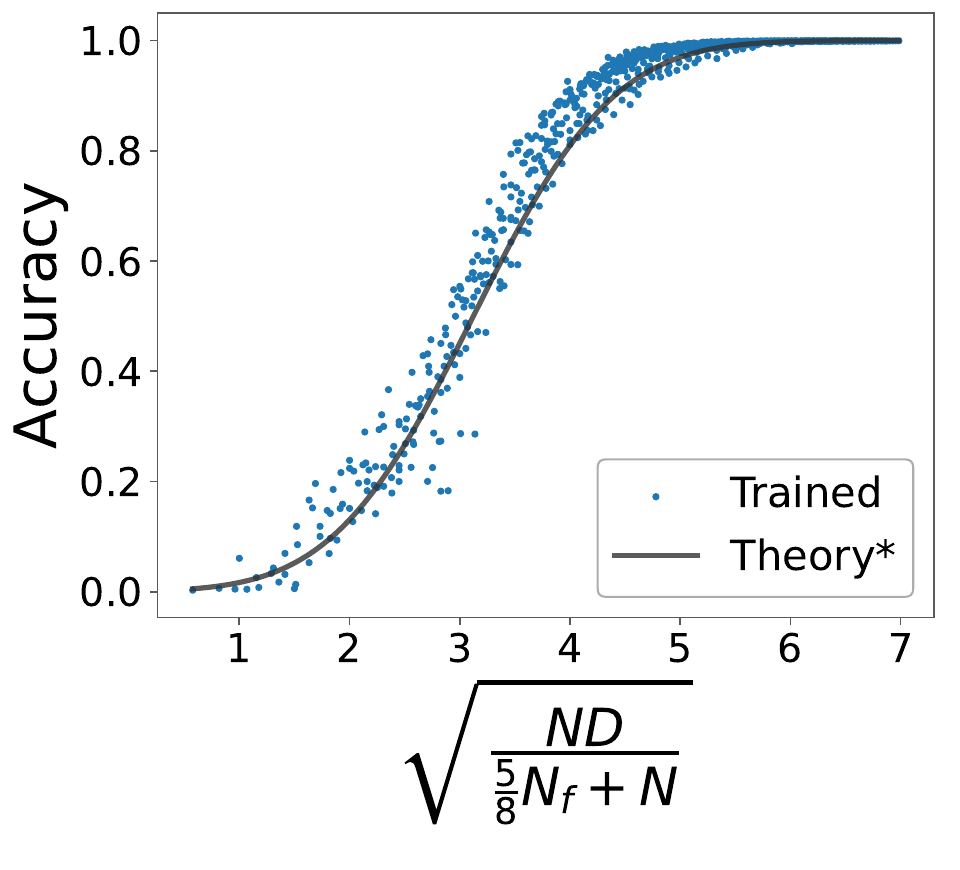}} &

 \\
 \end{tabular}
 \caption{%
 \textbf{Mamba recall scaling laws: theory versus empirical results.} \textbf{Top:} Recall accuracy of trained Mamba, represented by pixel color, against model embedding dimension $D$ and state size $N$. \textbf{Bottom:} the same accuracies collapse onto the theoretical curve $p_\mathrm{recall}(x)\approx\Phi(x-\sqrt{2\log V})$ of Thm.~\ref{thm:prob-scaling-law} and Rem.~\ref{rem:practical-scaling-law}, with $x=\sqrt{\tfrac{ND}{aN_f+N}}$. Left to right: (a) linear model on AR, (b) linear model on MQAR, (c) full model, theory fits up to a heuristic constant; see Sec.~\ref{subsec:single-layer-curves}. (Here $V=1024$, $N_f=16$; see more regimes in Figs.~\ref{app-fig:scaling-laws-2d-dims} and~\ref{app-fig:scaling-curves-regimes}, Appendix.)\label{fig:scaling-laws-2d-dims}
 }
\end{figure*}

\textbf{Our main contributions} encompass the following aspects: %
(1) We study the associative recall capabilities of Mamba through the lens of mechanistic interpretability. By progressively removing components from the model, we isolate the circuit-level mechanism responsible for performing recall. %
(2) Motivated by this finding, we interpret the learned weights as implicitly computing similarity-preserving linear hash functions, and the entire circuit as exhibiting hash-like behavior. This interpretation allows us to leverage the JL lemma to better understand the theoretical recall capacity of Mamba. %
(3) We propose a theoretical framework, which we term \textit{Recall Scaling Laws}, that provides an upper bound on the state size, context length and embedding size required to perform near-perfect recall for a given vocabulary size and number of facts, as well as a prediction of recall probability. (4) We extend our framework to multi-head and multi-layer Mamba models. (5) We provide extensive empirical validation of the scaling laws.

\paragraph{Paper roadmap.}
Sec.~\ref{sec:reverse} suggests theoretical recall circuits (Sec.~\ref{subsec:consturt}) and validates them in trained models via mechanistic-interpretability experiments (Sec.~\ref{subsec:mech-interp}); Sec.~\ref{sec:recall-scaling-laws} develops the corresponding scaling laws, which are validated empirically in Sec.~\ref{sec:scaling-laws-exp}. Given recall tasks of $N_f$ facts over a vocabulary of size $V$, the requirements on Mamba embedding dimension $D$ and state dimension $N$ form the following ladder of guarantees:
{\setlength{\leftmargini}{1.35em}%
\begin{enumerate}
\item Exact recall on every context is achievable by a (non-feasible) non-compressive model with dimensions $N{=}D{=}V$ (worst-case; Thm.~\ref{thm:non-compressive-recall}).
\item Exact recall is similarly achievable by a compressive, hash-based model with smaller dimensions, satisfying $\sqrt{ND}=O(N_f\log V)$ (worst-case; Lem.~\ref{lem:jl-scaling-law}).
\item High-probability recall is achievable with even smaller model dimensions, satisfying $ND=O(N_f\log{V})$ (mean-case; Thm.~\ref{thm:compressive-recall} qualitatively, Thm.~\ref{thm:prob-scaling-law} and Rem.~\ref{rem:practical-scaling-law} quantitatively).
\end{enumerate}
The last level is essentially optimal, since a matching lower bound exists as well:
\begin{enumerate}
\setcounter{enumi}{3}
\item High-probability recall \textit{requires} model dimensions satisfying $ND=\Omega(N_f\log V)$ (mean-case; Rem.~\ref{rem:dim-scaling-law} Mamba-specific, Lem.~\ref{lem:info-lower-bound} general information-theoretic).
\end{enumerate}}
In short, the required state memory is linear in the number of facts and logarithmic in vocabulary size.

\section{Background \& Related Work}
\paragraph{Associative Recall.}
The AR task was first introduced by~\citet{ba2016using}, and has proven to be an effective tool for evaluating recall and retrieval capabilities. More recently, AR has gained popularity due to its strong correlation with language modeling performance~\citep{arora2023zoology}. A growing body of work has leveraged AR to guide architectural choices~\citep{poli2023hyena,fu2022hungry,arora2024simple,lutati2023focus}, improve initialization strategies~\citep{trockman2024mimetic}, %
analyze optimization dynamics~\citep{okpekpe2025revisiting}, and provide both theoretical and empirical insights into the limitations of recurrent LLMs~[\citealp{wen2024rnns,ben2025overflow}, \citealp{behrouz2026atlas}]. The work most closely related to ours is by \citet{bick2025understanding}, who use mechanistic interpretability to analyze in-context retrieval through a gather-and-aggregate mechanism. However, it does not study AR or provide theoretical analysis. \citet{lecorre2025how} probe a simplified Mamba on small-vocabulary MQAR, identifying write, read and forget behaviours, but do not address capacity. \citet{jelassi2024repeat} investigate the copying capabilities of Transformers and SSMs, offering theoretical insights into these mechanisms. While relevant, their work does not directly examine AR and does not employ mechanistic interpretability to understand what models learn in practice. %
\citet{nichani2025understanding} bound in-weight fact storage in Transformers, linear in parameter count; we characterize in-context storage in a fixed-size state. Finally, \citet{huang2025understanding} explore the recall capabilities of Mamba by leveraging the JL lemma to derive theoretical bounds, but do not adopt a mechanistic interpretability perspective nor investigate the mechanisms learned in practice. In contrast, our analysis culminates in tighter recall scaling laws showing how both state size and embedding size must scale with vocabulary, and extends to multi-layer and multi-head architectures. (See detailed comparison in App.~\ref{app:sim_and_diff_huang}.)

\hypertarget{mqar-notation}{}%
\definecolor{myblue}{RGB}{100,100,255}
\definecolor{mygreen}{RGB}{0,200,0}
\newcommand{\f}[1]{\textcolor{myblue}{#1}}
\newcommand{\w}[1]{\textcolor{red}{#1}}
\newcommand{\q}[1]{\textcolor{mygreen}{#1}}
\newcommand{\z}[1]{\textcolor{gray}{#1}}
\newcommand{\h}[1]{\textcolor{brown}{#1}}

\paragraph{Multi-Query Associative Recall (MQAR).}\label{par:mqar-definition}
MQAR by~\citet{arora2023zoology} is an extension of the classical AR. It is constructed by concatenating a factual context section with a multiple queries section to form a single prompt. The vocabulary $\mathcal{V}$ is partitioned into two disjoint subsets: the key vocabulary $\mathcal{V}_k$ and the value vocabulary $\mathcal{V}_v$ (see Table~\ref{tab:mqar_notation}, Appendix). The vocabulary size is $V=\lvert \mathcal{V} \rvert$, which is assumed to be even so that half of the tokens correspond to keys, and the other half correspond to values 
$ ( V_k=\lvert \mathcal{V}_k \rvert= V_v=\lvert \mathcal{V}_v \rvert = \tfrac{V}{2} )$. Each prompt has a total sequence length $L$. The context is composed of $N_f$ non-repeating key-value pairs $(k_i, v_i)$, resulting in $2N_f$ tokens. The remaining $L-2N_f$ tokens form the query section, which includes $N_f$ query tokens $q_i$, each duplicating a key token that appears in the context. The rest of the query section is filled with padding. Padding tokens may be set to zero or chosen at random, and following the original implementation, we assume they are randomly sampled from the full vocabulary $\mathcal{V}$. For each query \q{$q_i$}, the model's task is to retrieve and output the value \w{$v^*$} corresponding to the fact $(\f{k^*},\w{v^*})$ with a matching key $\f{k^*}=\q{q_i}$. An example \textit{MQAR} input \textbf{$x$} and ground-truth output \textbf{$y$} sequences can be found below, with highlighted \f{key}, \w{value}, \q{query} and \h{padding} tokens:
\smallskip

\begin{center}
\small
\textbf{$x$} \;\;
\texttt{\f{A} \w{6} \f{B} \w{3} \f{C} \w{7} \q{B} \h{2} \h{5} \h{0} \h{9} \q{C} \h{4} \q{A} \h{8} \h{1}}
\end{center}
\begin{center}
\small
\textbf{$y$} \;\;
\texttt{\z{*} \z{*} \z{*} \z{*} \z{*} \z{*} \w{3} \z{*} \z{*} \z{*} \z{*} \w{7} \z{*} \w{6} \z{*} \z{*}}
\end{center}
where \texttt{\z{*}} are \textit{ignored} output tokens (which does not affect the loss). Additionally, the term \textit{AR} is used in this paper to denote a single-query associative recall task, i.e. with $L=2N_f+1$. %
\begin{figure*}[t]
    \centering
    \includegraphics[width=1.00\linewidth]
    {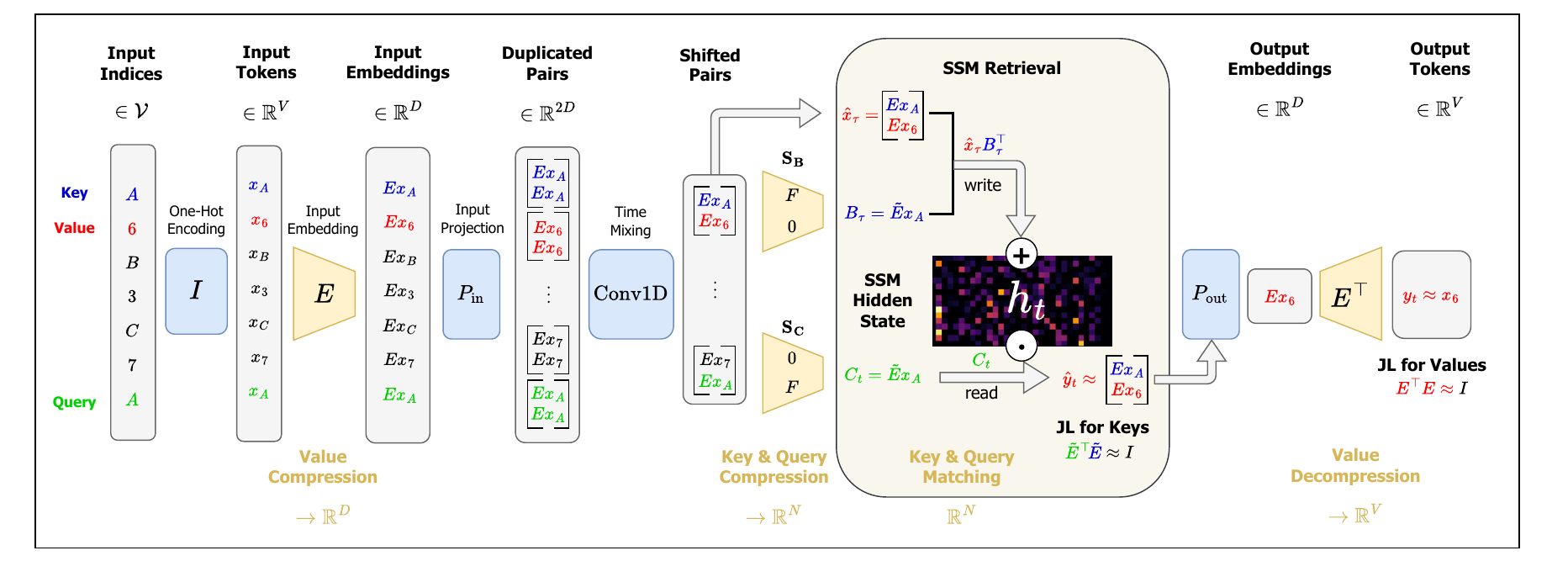}
    \caption{%
    \textbf{Minimal recall circuit.}
    An overview of the Mamba recall circuit in Thm.~\ref{thm:compressive-recall}. Left to right:
    (1) Tokens are embedded using $E\in \mathbb{R}^{D\times V}$.
    (2) $P_\mathrm{in}$ and Conv1D apply duplication, copy and shift, forming pairs $\bigl( \begin{smallmatrix}Ex_{t-1}\\Ex_{t}\end{smallmatrix} \bigr)$.
    (3) SSM projections $S_B$, $S_C$ extract keys and queries, then further compress them via $\tilde{E}=FE\in \mathbb{R}^{N\times D}$.
    (4) Embedding pairs are stored in the SSM using $h_t=\sum_\tau^t\hat{x}_\tau B_\tau^\top$.
    (5) JL matrix $\tilde{E}$ enables retrieval through key-query matching, $h_tC_t=\sum_\tau^t\hat{x}_\tau (B_\tau^\top C_t)\approx\hat{x}_\tau$.
    (6) $P_\mathrm{out}$ selects the value half $Ex_\tau$.
    (7) JL matrix $E$ enables value reconstruction, $y_t=E^\top Ex_\tau \approx {x}_\tau$.
    }
    \label{fig:recallcirc}
\end{figure*}
\paragraph{Mamba.}\label{par:mamba-background}
The recently presented \textit{Selective SSM}~\citep{gu2023mamba}, known as S6, outperforms previous SSMs and various other architectures in NLP~\citep{mambamoe1,wang2024mambabyte}, vision~\citep{mambaViT1,mambaViT2}, graph classification~\citep{mambaGraph1%
}, and more. S6 incorporates a dynamic input-dependent form of the discrete matrices $\bar{A},\bar{B},$ and $C$, such that for every time-step the SSM employs a different recurrent rule. This technique differs from the previous state-space layers, which use the same set of matrices %
for each timestep.

A \textit{Mamba} model of embedding size $D$ and state size $N$ is composed of  $\Lambda$ stacked Mamba blocks. The \textit{Mamba block} combines the S6 state space model (SSM), a Conv1D layer, and other elementwise operators; it borrows elements from Gated MLP and similar architectures. 
We denote the expanded dimension by $D_{\text{in}}=\mathrm{expand}\cdot D$ (where $\mathrm{expand}\in \mathbb{N}$), and the convolution filter size by $D_\mathrm{conv}$.

Given an embedded input sequence $x^e = (x^e_0,\; \dots, \;x^e_{L-1}) \in \mathbb{R}^{D \times L}$, the Mamba block operates as follows (see Tab.~\ref{tab:mamba_notation},~\ref{tab:mamba_dims} for notation details):
\begin{equation}
    \hat{x} = \text{SiLU}(\text{Conv1D}(\text{Linear}(x^e))),\,
    \hat{z} = \text{SiLU}(\text{Linear}(x^e)),\,
    \hat{y} = \text{SSM}(\hat{x}) \odot \hat{z},\,
    y^e = \text{Linear}(\hat{y})
    \label{eq:output_block}
\end{equation}
where $\odot$ denotes element-wise multiplication, and we omit normalization layers for simplification. For each timestep $t$, the SSM vectors $B_t, \, C_t \in \mathbb{R}^N$ and discrete matrices $\smash[t]{\bar{A_t}}, \smash[t]{\bar{B_t}} \in \mathbb{R}^{D_\mathrm{in} \times N}$ are computed as:
\begin{equation} \label{eq:TimeVariantMatrices1}
    B_t = S_B \, \hat{x}_t\,, \; C_t = S_C \, \hat{x}_t \,, \; \Delta_t = \text{SoftPlus}(S_{\Delta} \hat{x}_t)\,, \; \bar{A}_t = \exp (\Delta_t A)\,, \; \bar{B}_t = \Delta_t B_t
\end{equation}
where $S_B, S_C, S_{\Delta}$ are linear projection matrices, SoftPlus is a
smooth approximation of
ReLU, and biases are omitted for simplicity.
Then, each SSM channel $d=1,\dots,D_\mathrm{in}$ transforms an input $\hat{x}^d_t \in \mathbb{R}$ into an output $\hat{y}^d_t \in \mathbb{R}$, through updating a recurrent hidden state vector ${h}_t^d \in \mathbb{R}^N$:
\begin{equation}\label{eq:original-ssm}
    {h}^d_t = {\bar{A}}^d_t \odot {h}^d_{t-1} + \hat{x}^d_t \,{\bar{B}}^d_t \,, \quad
    \hat{y}^d_t = {h}^d_t \cdot {C}_t
\end{equation}
where
$\smash[t]{{\bar{A}}_t^d}, \smash[t]{{\bar{B}}_t^d}, {C}_t \in \mathbb{R}^N$, $\odot$
denotes an element-wise product, and $\cdot$ is a dot product. All channels form the hidden state matrix, $h_t\in \mathbb{R}^{D_\mathrm{in} \times N }$ (See App.~\ref{app-subsec:ssm-dims} for details).

\color{black}

The usage of time-variant layers adds to the expressivity of the layer~\citep{cohen2025expressivity}, allowing it to adapt to the input, and potentially captures more complex dependencies. While other input-dependent time-variant mechanisms have been proposed in previous works through gated RNNs, S6 also presents an efficient IO-aware implementation, which is parallelized on GPUs via work-efficient parallel scanners~\citep{blelloch1990prefix,%
smith2022simplified}, making it well-suited for long-context processing~\citep{benkish2024decimamba, ye2025longmamba}.
\paragraph{Mechanistic Interpretability.} Mechanistic interpretability seeks to reverse-engineer deep learning models so they can be trusted and controlled~\citep{elhage2021mathematical,olah2020zoom}. %
The approach works by linking observable behaviors to specific neural components, often by assigning semantic functions to individual neurons or weights and identifying compact ``circuits'' within the network. Representative examples include the induction-heads%
~\citep{olsson2022context} and arithmetic mechanisms %
\citep{kantamneni2025language,zhong2023the,chughtai2023toy}, in-context learning via task-vectors~\citep{hendel2023context} and thinking progress via progress-vectors~\citep{eisenstadt2025overclocking}. %
Our work follows this line of research and aims to uncover the intrinsic mechanisms responsible for recall in Mamba models and linear RNNs.
\section{Reverse Engineering Mamba on MQAR\label{sec:reverse}}

We aim to reverse-engineer the underlying algorithm that enables Mamba to perform MQAR. To do so, in Sec.~\ref{subsec:consturt} we present a minimal yet efficient Mamba model that successfully solves MQAR. Next, in Sec.~\ref{subsec:mech-interp}, through a mechanistic interpretability analysis, we verify that our suggested mechanism is indeed the one learned in practice.

\subsection{A Mamba Circuit for MQAR\label{subsec:consturt}}

\newcommand{\BEGINBLOCK}[1]{%
    \STATE {\bfseries #1:}
    \color{white} \FOR{$\;$} \vspace{-11.5pt} \color{black}
}

\newcommand{\ENDBLOCK}{%
    \color{white} \ENDFOR \vspace{-11.5pt} \color{black}
}

\renewcommand{\algorithmiccomment}[1]{
    $\;$
    \color{alg_comment_color} // #1 \color{black} 
}

\begin{figure}[t]
\begin{minipage}[t]{0.47\linewidth}
\begin{algorithm}[H]
\small\linespread{1.10}\selectfont
  \caption{Simplified Single-Layer Mamba}
  \label{alg:mamba}
  \begin{algorithmic}
    \STATE {\bfseries Input:} $x \;\color{alg_shape_color}{\in \mathbb{R}^{V \times L}}$
    \STATE {\bfseries Output:} $y \;\color{alg_shape_color}{\in \mathbb{R}^{V \times L}}$
    \STATE {\bfseries Parameters:} \COMMENT{$D_\mathrm{in}=\mathrm{expand}\cdot D$}
    \color{white} \FOR{$\;$} \vspace{-11.5pt} \color{black}
\STATE $E \;\color{alg_shape_color}{\in \mathbb{R}^{D \times V}}$, $P_{\mathrm{in}} \;\color{alg_shape_color}{\in \mathbb{R}^{D_\mathrm{in} \times D}}$,
\STATE $W_{\mathrm{conv}} \;\color{alg_shape_color}{\in \mathbb{R}^{D_\mathrm{in} \times D_\mathrm{conv}}}$,
\STATE $S_B$, $S_C \;\color{alg_shape_color}{\in \mathbb{R}^{N \times D_\mathrm{in}}}$, $P_{\mathrm{out}} \;\color{alg_shape_color}{\in \mathbb{R}^{D \times D_\mathrm{in}}}$
\ENDBLOCK

    \STATE $x_{\mathrm{e}} \leftarrow E\,x$ \COMMENT{input embedding}

    \BEGINBLOCK{Mixer}  %
    \STATE $x_{\mathrm{in}} \leftarrow P_{\mathrm{in}}\,x_{\mathrm{e}}$
    \STATE $\hat{x} \leftarrow \mathrm{Conv1D}(x_{\mathrm{in}}, W_{\mathrm{conv}})$
    \STATE $B \leftarrow S_B\,\hat{x}$
    \STATE $C \leftarrow S_C\,\hat{x}$
    \STATE $\hat{y} \leftarrow \mathrm{SSM}(\hat{x}, B, C)$
    \STATE $y_{\mathrm{out}} \leftarrow P_{\mathrm{out}}\,\hat{y}$
    \ENDBLOCK

    \STATE $y \leftarrow E^\top\,y_{\mathrm{out}}$ \COMMENT{output embedding}
    \STATE {\bfseries return} $y$
  \end{algorithmic}
\end{algorithm}
\end{minipage}
\hfill
\begin{minipage}[t]{0.47\linewidth}
\begin{algorithm}[H]
\small\linespread{1.10}\selectfont
  \caption{Hidden State Inversion}
  \label{alg:hidden_state_inversion}
  \begin{algorithmic}
    \STATE {\bfseries Input:} $h_t \;\color{alg_shape_color}{\in \mathbb{R}^{2D \times N}}$ \COMMENT{$D_\mathrm{in}=2D$}
    \STATE {\bfseries Output:} $h_t^{\mathrm{inv}} \;\color{alg_shape_color}{\in \mathbb{R}^{2V \times V}}$
    \STATE {\bfseries Parameters:}
    \color{white} \FOR{$\;$} \vspace{-11.5pt} \color{black}
    \STATE $E \;\color{alg_shape_color}{\in \mathbb{R}^{D \times V}}$, $P_{\mathrm{in}} \;\color{alg_shape_color}{\in \mathbb{R}^{2D \times D}}$,
    \STATE $W_{\mathrm{conv}} \;\color{alg_shape_color}{\in \mathbb{R}^{2D \times 2}}$ \COMMENT{$D_\mathrm{conv}=2$}
    \STATE $S_B$, $S_C \;\color{alg_shape_color}{\in \mathbb{R}^{N \times 2D}}$, $P_{\mathrm{out}} \;\color{alg_shape_color}{\in \mathbb{R}^{D \times 2D}}$
    \ENDBLOCK
    \BEGINBLOCK{Operators}
    \STATE $\hat E_{\mathrm{in}}
      \leftarrow
      \big(\,\mathrm{diag}(W^0_{\mathrm{conv}})\,P_{\mathrm{in}}\,E
      \;\big|\;
      \mathrm{diag}(W^1_{\mathrm{conv}})\,P_{\mathrm{in}}\,E\,\big)$ \vspace{-11.5pt}
    \STATE $\Pi_{q,\mathrm{in}} \leftarrow S_C \,\hat E_{\mathrm{in}}$ \COMMENT{input projection}
    \STATE $\Pi_{v,\mathrm{out}} \leftarrow E^\top P_{\mathrm{out}}$ \COMMENT{output projection}
    \ENDBLOCK

    \BEGINBLOCK{Inversion}
    \STATE $h_t^{\mathrm{inv}} \leftarrow \Pi_{v,\mathrm{out}}\, h_t\, \Pi_{q,\mathrm{in}}$
    \ENDBLOCK

    \STATE {\bfseries Visualization:} \COMMENT{$k_n,v_n \in \mathbb{R}^{V}$: facts key/value}
    \color{white} \FOR{$\;$} \vspace{-11.5pt} \color{black}
    \STATE $H_t \leftarrow \sum_{n=1}^{N_f} k_n^\top v_n \;\color{alg_shape_color}{\in \mathbb{R}^{V \times V}}$
    \STATE $H'_t \leftarrow (0 \mid I_D)\, h_t \;\color{alg_shape_color}{\in \mathbb{R}^{D \times N}}$
    \STATE $H''_t \leftarrow (0 \mid I_V)\, h_t^{\mathrm{inv}} \;\color{alg_shape_color}{\in \mathbb{R}^{V \times V}}$
    \ENDBLOCK
    \STATE {\bfseries return} $h_t^{\mathrm{inv}}$
  \end{algorithmic}
\end{algorithm}
\end{minipage}
\end{figure}

We describe a simplified, purely-linear, single-layer Mamba model, and show theoretically and empirically that it is able to solve MQAR with arbitrarily high probability. A schematic overview of how the minimal model performs recall is presented in Fig.~\ref{fig:recallcirc}.

\paragraph{Minimal model and simplified SSM.}

In \hyperref[alg:mamba]{Alg.~\ref*{alg:mamba}} we present a minimal single-layer Mamba model architecture%
, where gating, discretization, nonlinearities, biases, normalization layers, and residual connections are all removed. Thus, the simplified-Mamba layer is a linear attention preceded by a Conv1D. Our theorems are stated for this model, which actually fits a broader class of architectures. Dimensions are listed in Tab.~\ref{tab:mamba_dims}, Appendix. The model is wrapped with embedding and unembedding layers of a dimension $D$, which expands to size $D_{\text{in}}=\mathrm{expand}\cdot D$ in the input projection. We further note that without discretization, the SSM recurrent update step described in Eq.~\ref{eq:original-ssm} becomes rather simple. If we additionally set $A$ to identity, the SSM operation can now be written as (see App.~\ref{app-subsec:ssm-dims} for details):
\begin{equation}
\label{eq:simple-ssm}
  h_t = h_{t-1} + \hat{x}_t\,B_t^\top, \quad
  \hat{y}_t = h_t\, C_t \,
\end{equation}

\paragraph{Mamba recall circuits.}
We first describe an ideal (dimensionally inefficient) non-compressive circuit that solves MQAR. Next, we present a rather realistic, compressive circuit as a feasible approximation of the ideal one. 

\begin{theorem}[\textbf{Perfect non-compressive recall circuit}]
\label{thm:non-compressive-recall}
Given a vocabulary size $V$, a \hyperref[alg:mamba]{\textit{single-layer simplified Mamba}} with dimensions $D=N=V$, $\mathrm{expand}=2$, $D_\mathrm{conv}=2$ can \textbf{perfectly solve} an MQAR task (recall probability = 1).
\end{theorem}

\begin{proof}[\textbf{Proof sketch}]
(Full proof can be found in App.~\ref{app-thm-non-compressive-recall}.)
We set the model weights as follows:
\begin{equation}
    P_\mathrm{in} =  \bigl(\begin{smallmatrix} I_V \\ I_V \end{smallmatrix} \bigr), \;
    P_\mathrm{out} = \bigl(0\mid I_V), \;
    W_\mathrm{conv}=\bigl(\begin{smallmatrix}1_V\\0_V\end{smallmatrix}\mid \;\begin{smallmatrix}0_V\\1_V\end{smallmatrix}\bigr), \;
    E = I_V, \;
    S_B=\bigl(I_V\mid 0), \;
    S_C=\bigl(0\mid I_V) \,
\end{equation}
where $\mid$ denotes concatenation. $P_\mathrm{in}$ duplicates the input,  $x^p_t=P_\mathrm{in}\,E\,x_t=
\bigl( \begin{smallmatrix}x_{t}\\x_{t}\end{smallmatrix} \bigr)$. 
The Conv1D %
performs copy and shift; the SSM input is
$
\hat{x}_t =\mathrm{Conv1D}\bigl((x^p_{t-1}, x^p_t),\,W_\mathrm{conv}\bigr)= \bigl( \begin{smallmatrix}x_{t-1}\\x_{t}\end{smallmatrix} \bigr)    
$. 
Inside the SSM, we have $B_t=S_B\,\hat{x}_t=x_{t-1}$ and $C_t=S_C\,\hat{x}_t=x_t$. After unrolling the expression in Eq.~\ref{eq:simple-ssm} and plugging in the above values, the SSM operation becomes $\hat{y}_t = h_t\, x_t=(\sum_{\tau=0}^t 
  \hat{x}_{\tau}\,B_{\tau}^\top) \, x_t=\sum_{\tau=0}^t 
  \bigl( \begin{smallmatrix}x_{\tau-1}\\x_{\tau}\end{smallmatrix} \bigr)\,x_{\tau-1}^\top \, x_t$. Finally, given a query $x_t \equiv q_t \in \mathcal{V}_k$, the %
  output is:

\begin{equation}\label{eq:perfect-recall-out}
  y_t = 
  E^\top\,P_\mathrm{out}\, \hat{y}_t 
  =
\sum_{\tau=0}^t  {x_\tau\, \langle x_{\tau-1}, q_t \rangle} \,
\end{equation}

Since input vectors are one-hot (orthonormal) and keys are unique, $\langle x_{\tau-1}, q_t\rangle = \delta_{x_{\tau-1},\,q_t}$, so the sum collapses to $y_t = v^*$.
\end{proof}
The circuit described in Theorem~\ref{thm:non-compressive-recall} is %
impractical because it relies on large state and model dimensions (increases linearly with $V$), which are unrealistic in practice. Thus, the following theorem presents a more efficient circuit, based on an approximate solution rather than a perfect one.

\begin{figure*}[t]
\centering
\begin{minipage}[t]{\textwidth}
\centering
\begin{tabular}{@{\hskip 0.05in}c@{\hskip 0.05in}c@{\hskip 0.05in}c@{\hskip 0.05in}}
  $H_t$ & $H_t'$ & $H_t''$ \\
  \includegraphics[width=0.320\linewidth]
  {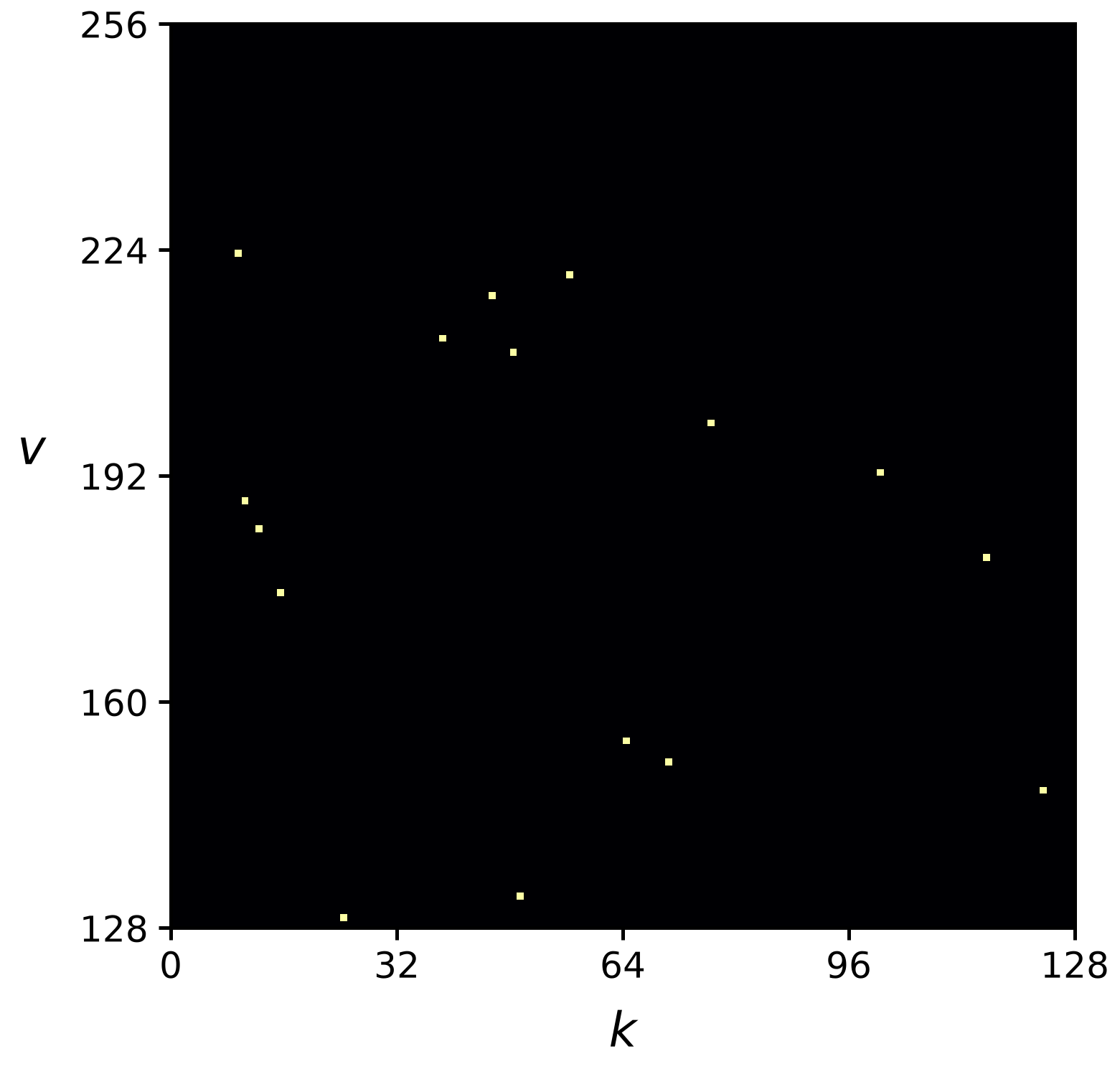} &
  \includegraphics[width=0.180\linewidth]
  {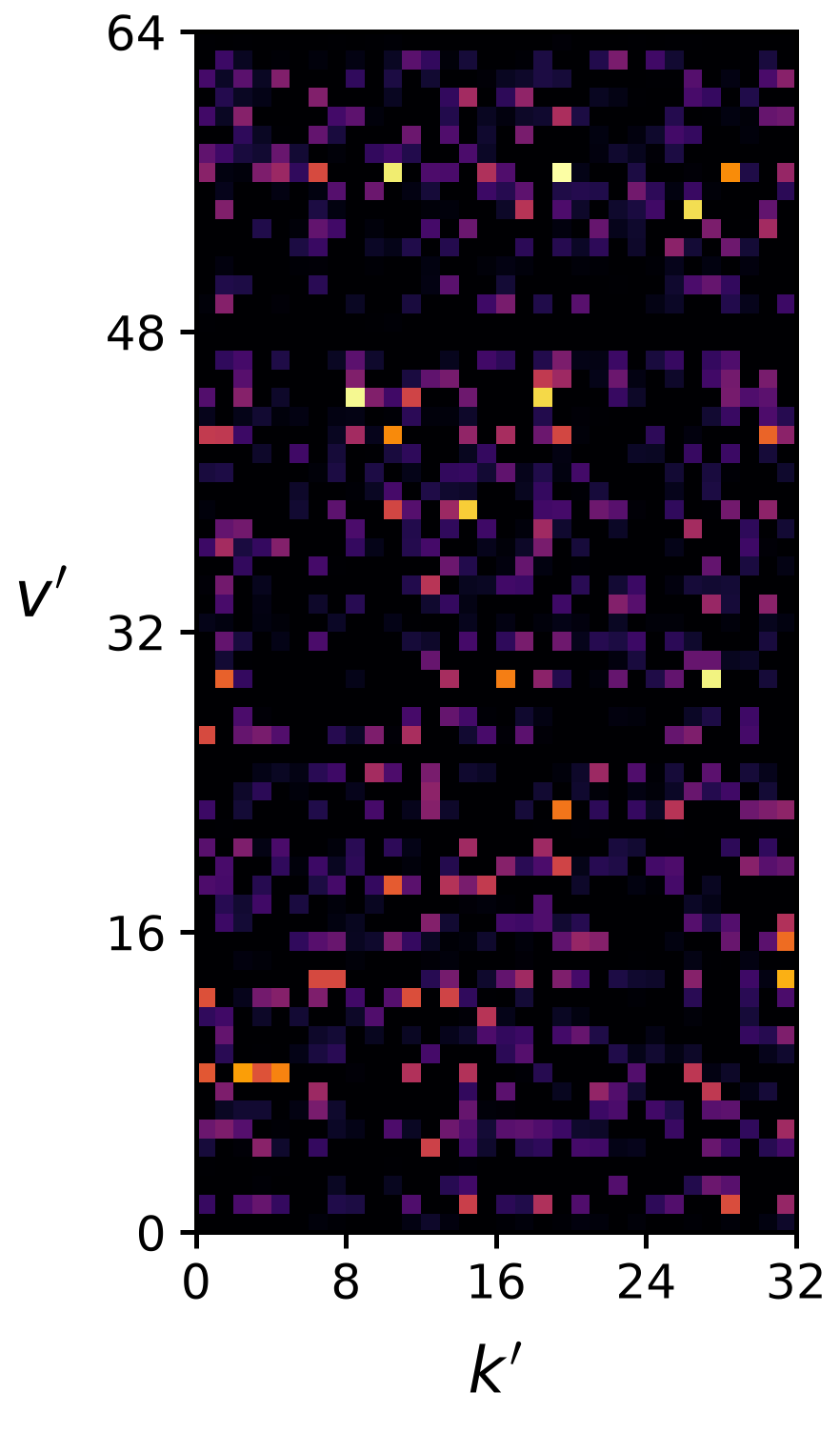} &
  \includegraphics[width=0.320\linewidth]
  {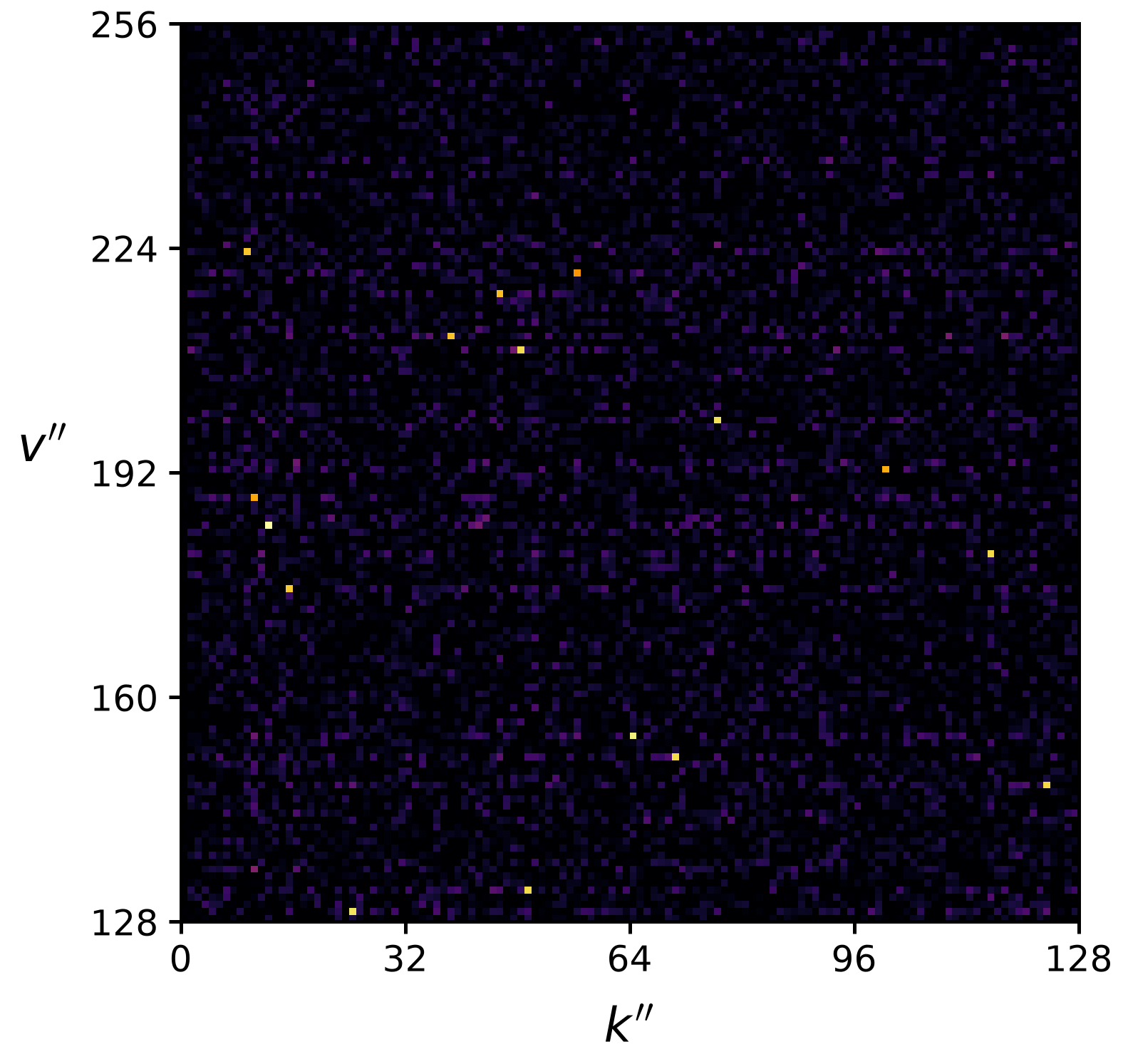}
\end{tabular}

\caption{
\textbf{Interpreting the hidden state.} The context 
key-value information $H_t$ (left) is compressed by the model into a hidden state $H'_t$ (middle), which lacks clear interpretability. As described in Alg.~\ref{alg:hidden_state_inversion}, projection onto vocabulary space (right) reveals an interpretable pattern: $H''_t \approx H_t$, indicating that a compression-decompression scheme is learned.}
\label{fig:reversing-hidden}
\end{minipage}

\end{figure*}
\begin{theorem}[\textbf{Efficient compressive recall circuit}]\label{thm:compressive-recall}
Given an MQAR task with vocabulary size $V$, $N_f$ facts and context length $L=4N_f$, a \hyperref[alg:mamba]{\textit{single-layer simplified Mamba}} with dimensions satisfying $ND=O(N_f\log{V})$,  $\mathrm{expand}=2$, $D_\mathrm{conv}=2$ can solve the task with high probability. (See Lem.~\ref{lem:jl-scaling-law} and Thm.~\ref{thm:prob-scaling-law} for the exact quantitative laws.)
\end{theorem}
\begin{proof}[\textbf{Proof sketch}]
(See App.~\ref{app-thm-compressive-recall} for full proof.) We adapt the circuit from Thm.~\ref{thm:non-compressive-recall} to enable compression. Given matrices $E\in \mathbb{R}^{D\times V}$ and $F\in \mathbb{R}^{N\times D}$, we now choose:
\begin{equation}
    \label{eq:compressive-recall-weights}
    P_\mathrm{in} =  \bigl(\begin{smallmatrix} I_D \\ I_D \end{smallmatrix} \bigr)\,, \;
    P_\mathrm{out} = \bigl(0\mid I_D)\,, \;
    W_\mathrm{conv}=\bigl(\begin{smallmatrix}1_D\\0_D\end{smallmatrix}\mid \;\begin{smallmatrix}0_D\\1_D\end{smallmatrix}\bigr)\,, \;
    S_B =\bigl(F\mid 0)\,, \;
    S_C=\bigl(0\mid F)
\end{equation}
These weights compress both SSM input $\hat{x}_t=\bigl(\begin{smallmatrix}
    E\,x_{t-1}\\E\,x_{t}\end{smallmatrix} \bigr)$ and operators $B_t=\tilde{E}\,x_{t-1}$, $C_t=\tilde{E}\,x_t$, where $\tilde{E}\,\equiv F E \in \mathbb{R}^{N \times V}$. Similarly to Eq.~\ref{eq:perfect-recall-out}, the overall model output is now:
\begin{equation}\label{eq:compressive-recall-out}
y_t = E^\top\,P_\mathrm{out}\,\hat{y}_t=\sum_{\tau=0} ^t E^\top (E\,x_{\tau}) \, (\tilde{E}\,x_{\tau-1})^\top (\tilde{E}\,x_t) \, 
\end{equation}
Applying \textit{\textbf{Johnson–Lindenstrauss lemma}}~\citep{johnson1984extensions} (Lem.~\ref{app-lem:jl}, Appendix) twice, for $D=O\!\bigl(\frac{\log V}{\varepsilon_v^2}\bigr)$, $N=O\!\bigl(\frac{\log V}{\varepsilon_k^2}\bigr)$, we can construct approximately-orthogonal JL matrices $E$, $\tilde{E}$ (with $\tilde{E}=FE$; see Lem.~\ref{app-lem:jl-consecutive}) such that $(\tilde{E}\,x_{\tau-1})^\top(\tilde{E}\,x_t)\approx\delta_{x_{\tau-1},\,x_t}+O(\varepsilon_k)$ and $E^\top(E\,x_\tau)\approx x_\tau+O(\varepsilon_v)$. Keys are unique, so the matching term contributes $v^*+O(\varepsilon_v+\varepsilon_k)$, while each of the $N_f-1$ non-matching terms leaves an $O(\varepsilon_v\varepsilon_k)$ residue. For large $N_f$, these $O(\varepsilon_v\varepsilon_k)$ terms dominate; in the worst case, the residues add up to $O(N_f\,\varepsilon_v\varepsilon_k)$, while in the mean case they cancel down to $O\big(\sqrt{N_f}\,\varepsilon_v\varepsilon_k\big)$. Thus, to solve the task, $\sqrt{ND}=O(N_f\log V)$ is required in the worst case, while $ND=O(N_f\log V)$ suffices with high probability. (See Lem.~\ref{lem:jl-scaling-law} for the exact worst-case conditions, and Thm.~\ref{thm:prob-scaling-law} for the exact high-probability law.) %
\end{proof} %

\subsection{Validating the Circuit: Mechanistic Interpretability Experiments\label{subsec:mech-interp}}
\paragraph{Equivalent recall circuits.} \label{par:equiv-recall}
When conducting a mechanistic examination of Mamba solution to MQAR, we are faced with a problem: even if the circuit from Thm.~\ref{thm:compressive-recall} is indeed the exact learned one, it is rather unlikely to find these specific weights naturally within a trained model, since the \textit{solution space} is wider. For example, if we suspect that weights $E$ and $ P_\mathrm{in}$ transform $x_t$ to $P_\mathrm{in}  E \, x_t$ through a specific operation, an equivalent operation would occur with $P'_\mathrm{in}=P_\mathrm{in}Q$ and $E'=Q^\top E$ for any orthogonal matrix $Q$. The evidence we seek for the circuit must be \textit{invariant} under these orthogonal linear transformations. Such invariant operators are presented next. Given the simplified model from Alg.~\ref{alg:mamba}, without any assumption on the weights, we define the following effective operators,
\begin{equation}\label{eq:proj-opertators}
\begin{gathered}
\hat E_{\mathrm{in}}
=\big(\,\text{diag}( W^0_\mathrm{conv})\,P_{\mathrm{in}}\,E\;\big|\; \text{diag}(W^1_\mathrm{conv})\,P_{\mathrm{in}}\,E\,\big) \,, \;\; %
\Pi_{v,\mathrm{in}}=\hat E_{\mathrm{in}}\,, \;\;
\Pi_{v,\mathrm{out}}=E^\top P_{\mathrm{out}}\,, \\
\Pi_{k,\mathrm{in}}=S_B\,\hat E_{\mathrm{in}}\,, \;\;
\Pi_{q,\mathrm{in}}=S_C\,\hat E_{\mathrm{in}}\,, \;\; %
G_{vv}=\Pi_{v,\mathrm{out}}\,\Pi_{v,\mathrm{in}}\,, \;\;
G_{kq}=\Pi_{k,\mathrm{in}}^{\top}\,\Pi_{q,\mathrm{in}}
\end{gathered}
\end{equation}
where we denote $W_\mathrm{conv}=\bigl(W^0_\mathrm{conv} \mid W^1_\mathrm{conv})$. Using these net operators, the model output is:
\begin{equation}
y_t = \sum_{\tau=0} ^t G_{vv}\,\xi_\tau \,\xi_\tau^\top G_{kq}\,\xi_t
\end{equation}
where $\xi_t \equiv \bigl(\begin{smallmatrix} x_{t-1} \\ x_t \end{smallmatrix} \bigr)$ are input token pairs. If we modify model weights such that $G_{vv}$, $G_{kq}$  are unchanged, we construct an \textit{equivalent recall circuit} with the same performance. Therefore, if the solutions learned in practice are equivalent to the one presented in Theorem~\ref{thm:compressive-recall}, we predict that the model weights will satisfy:
\begin{equation}\label{eq:Gvv-Gkq}
G_{vv} \approx \begin{pmatrix}0 \\ I_V \end{pmatrix} \,,\quad
G_{kq}  \approx  \begin{pmatrix}0 & 0 \\ I_V & 0 \end{pmatrix}\,,
\end{equation}
matching $G_{vv}$, $G_{kq}$ computed for the designed circuit weights. In Fig.~\ref{app-fig:reversing-operators} (Appendix) we can see that indeed, for a linear Mamba trained on MQAR (see App.~\ref{app:impl-details}  for implementation details), the solution weights satisfy this structure almost precisely. Specifically, we see that indeed $G_{kq}$ attends only to the query $x_t$ and the stored key $x_{\tau-1}$, while $G_{vv}$ attends only to the stored values $x_\tau$; all other matrix entries are $\approx 0$. Yet, note that the nonzero blocks $G_{vv}^{t,\tau}$ and $G_{kq}^{\tau-1,t}$ present more complex patterns than $I_V$. As discussed in App.~\ref{app-rem:key-value-selectivity}, these patterns are responsible for key-value selectivity. %

\paragraph{Hidden state as a hash table.}
Intuitively, for a perfect recall, the SSM hidden state must encode the full context \textit{information}. This fact-pairs information can be thought of as a simple key-value array, where each key entry contains $1$ in the corresponding fact value $v_\tau$ index, if exists, else $0$ anywhere.
The ideal, non-realistic, circuit of Thm.~\ref{thm:non-compressive-recall} suggests %
 such storage via an outer product key-value memory (see definitions and dimensions in App.~\ref{app-subsec:ssm-key-value-tables}):
\begin{equation} \label{eq:h}
H_t\equiv P_\mathrm{out}\,h_t=\sum_{n=1}^{N_f}v_n k_n^\top, \quad 
y_t=H_t\,q_t=\sum_{n=1}^{N_f}v_nk_n^\top q_t \,
\end{equation}
This is a $V \times V$ table, quite similar to the array presented above, though with values represented as one-hot vectors. A visualization %
can be found in Fig.~\ref{fig:reversing-hidden} (left).
$ H_t $ is %
 fairly interpretable, as %
entries are simply $1$ where facts exist and $0$ elsewhere. Can we observe such interpretable patterns within trained models? Let us consider the compressive SSM update from Eq.~\ref{eq:compressive-recall-out}. It can be rewritten as 
\begin{equation}\label{eq:h-prime}
H'_t = \sum_{n=1}^{N_f}v'_n {k'_n}^\top, \quad y_t= E^\top \sum_{n=1}^{N_f} v'_n {k'_n}^\top q'_t \,,
\end{equation}
with compressed tokens $v'_n=E\,v_n$, $k'_n=\tilde{E}\,k_n$, $q'_t=\tilde{E}\,q_t$. 

\(H'_t\) can be viewed as a \textbf{\textit{hash table}}: (i) the model uses JL linear projections, which are by definition similarity-preserving hash functions; (ii) the internal representation is not directly interpretable (Fig.~\ref{fig:reversing-hidden}, middle), while the final output is (right); (iii) efficiency is achieved through approximation, with error bounded with high probability. The interface - hash, store, match, retrieve - closely mirrors a classical hash table (see App.~\ref{app-subsec:hidden-as-hash}). %

\paragraph{Inspecting the hash table.} We notice that $H'_t$ stores the full $V \times V$ information in a compressed $D \times N$ table, in a lossy yet decompressable way. To verify this mechanism in practice, we attempt to \textit{decompress} the whole table, thus project it back onto vocabulary space (see \citep{dar2023analyzing,geva2023dissecting}). Let us define the decompressed tokens, $v''_n=E^\top E\,v_n \approx v_n$ and $k''_n=\tilde{E}^\top \tilde{E}\,k_n\approx k_n$. Using them, we can re-formulate the model operation as (see detailed definition in App.~\ref{app-subsec:ssm-key-value-tables}): %
\begin{equation}\label{eq:h-double-prime}
H''_t = \sum_{n=1}^{N_f}{v_n''} {k_n''}^\top, \quad 
y_t= H''_t\,q_t=\sum_{n=1}^{N_f} v''_n {k''_n}^\top q_t
\end{equation}
By this interpretation, the model is thought as using the full, sparse $V \times V$ table, which is an approximation of the non-compressive one: $H''_t=E^\top H'_t \,\tilde{E} \approx H_t$. Unfortunately, due to the non-uniqueness of solutions, as discussed in Sec.~\ref{par:equiv-recall}, these actual operators do not necessarily exist as-is (see App.~\ref{app-subsec:inspect-hash} for details). Instead, one has to project the hidden state 
through a more general transformation, as described in Alg.~\ref{alg:hidden_state_inversion}. This way, we can verify that indeed, as seen in Fig.~\ref{fig:reversing-hidden}, the learned solution does leverage a hash-like mechanism. The actual SSM hidden state $H'_t$ is a random-looking efficiently compressed storage, which, when projected back to vocabulary space, seem to preserve the information stored in the context, as $H''_t\approx H_t$.
\paragraph{Conv1D as copy and shift.}
An important claim in Thm.~\ref{thm:compressive-recall} is that the $\text{Conv1D}$ applies copy and shift operations on the input sequence. In practice, the Conv1D weights cannot be observed directly, as they are mixed with other linear transformations. However, we can find evidence to its operation within $G_{vv}$ and $G_{kq}$. If the Conv1D performs copy and shift, we expect $G_{vv} \,\xi_\tau \approx x_\tau$ and $G^\top_{kq} \,\xi_\tau \approx \xi_{\tau-1}$. This can be verified given an MQAR sequence ${x_t}$, as visualized in Fig.~\ref{fig:reversing-conv} (Appendix). As expected,  $x_t \,G_{vv} \,\xi_\tau \approx \delta_{t,\tau}$ and $\xi_t \,G^\top_{kq} \,\xi_\tau \approx \delta_{t,\tau-1}$. %
Lastly, in Tab.~\ref{fig:mqar_ablation} we validate the importance of the shifting operation through an ablation test. We find that a simplified Mamba block containing only a SSM and a Conv1D of length 2 achieves nearly perfect recall, but fails completely once the Conv1D is removed.

\section{Mamba Recall Scaling Laws}\label{sec:recall-scaling-laws}
The analysis in the Sec.~\ref{sec:reverse} identifies the recall circuit in Mamba and validates it through mechanistic experiments. However, it leaves us with an open question: how do recall capabilities vary as we change the model scale? In this section, we attempt to analytically answer this core question.
\begin{figure*}[t]
    \centering
    \begin{minipage}[c]{0.020\textwidth}
        \centering
        \rotatebox{90}{\small\textbf{MQAR}}
    \end{minipage}%
    \begin{minipage}[c]{0.955\textwidth}
        \centering
        \includegraphics[width=0.56\textwidth]{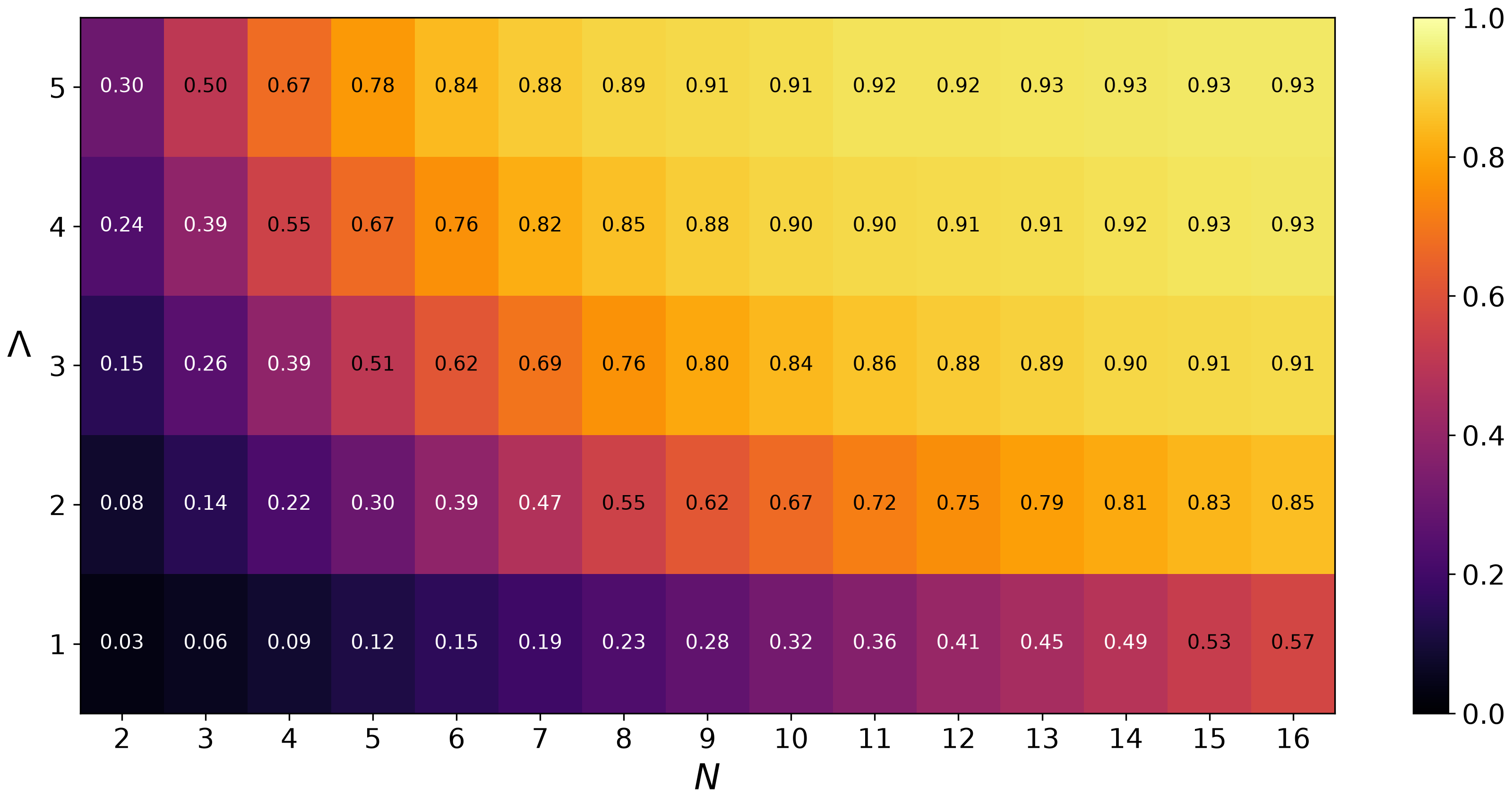}
        \hfill
        \includegraphics[width=0.38\textwidth]%
        {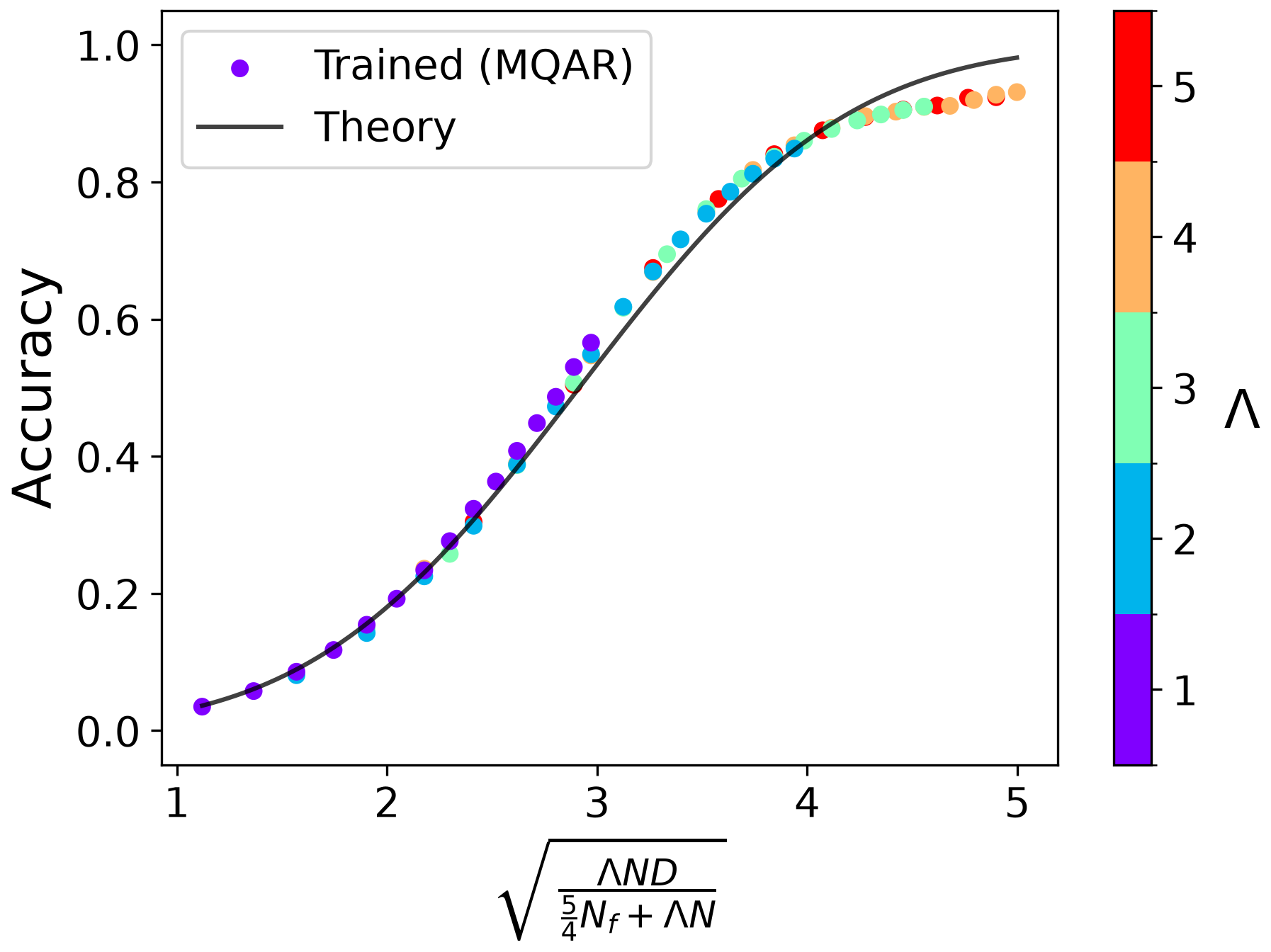}
    \end{minipage}
    \caption{
    \textbf{Multi-layer recall scaling laws.}
    Recall accuracy depends only on effective state size $N_\mathrm{eff}=\Lambda N$. \textbf{Left}: $N$-$\Lambda$ accuracy grid demonstrating the $N$-$\Lambda$ inverse relation. \textbf{Right}: In match with Thm.~\ref{thm:multi-layer-recall-scaling-laws}, accuracy collapses onto a curve of the form $p_\mathrm{recall}(x)\approx\Phi(x-b)$. (See full figure in App.~\ref{app:experiments-supp}.)
    }
    \label{fig:multi-layer-scaling-laws}
\end{figure*}
\subsection{JL-Based Recall Scaling Laws}
In Thm.~\ref{thm:compressive-recall} we suggest a mechanism that enables Mamba to approximately perform associative recall. The conclusion though, lacks a quantitative aspect. The following Lemma~\ref{lem:jl-scaling-law} addresses the question what model dimensions $D,\,N$ are required to solve an MQAR task of parameters $V,\, N_f,\, L$.
Intuitively, the proof interprets the hidden
state as information derived from a collection of approximately orthogonal hashed representations.
These representations are constructed through two successive applications of the \textit{JL lemma} (Lem.~\ref{app-lem:jl}; also see Lem.~\ref{app-lem:jl-consecutive}). First, keys, queries and values in vocabulary space $\mathbb{R}^V$ are embedded into $\mathbb{R}^D$. Then, keys and queries are further compressed into the state space $\mathbb{R}^N$. Finally, for sufficiently large dimensions $D,N$, successful retrieval with high probability is guaranteed through a statistical analysis that bounds the error introduced at each compression step. %

\smallskip
\begin{lemma}[\textbf{JL-based recall scaling laws}]\label{lem:jl-scaling-law}
Given a vocabulary size $V$, a \hyperref[alg:mamba]{\textit{single-layer simplified Mamba}} with dimensions $D=\frac{4\log{V}}{\varepsilon_v^2}$, $N=\frac{4\log{V}}{\varepsilon_k^2}$,  $\mathrm{expand}=2$, $D_\mathrm{conv}=2$ can perfectly solve an AR task if model dimensions satisfy $0 < \varepsilon_v < 1$, $0 < \varepsilon_k < 1$ and $\varepsilon_v + \varepsilon_k +2N_f\,\varepsilon_v\,\varepsilon_k< \frac{1}{2}$. The model can perfectly solve an MQAR task if model dimensions further satisfy $\varepsilon_v + \varepsilon_k +L\,\varepsilon_v\,\varepsilon_k< \frac{1}{2}$. (See proof in App.~\ref{app-lem:jl-scaling-law}.)
\end{lemma}
\subsection{Probabilistic Recall Scaling Laws}
The JL-based scaling laws developed in Lemma~\ref{lem:jl-scaling-law} provide a useful bound, quantifying model scales $D,N$ required for perfect recall. Yet, a simple summation of the $\varepsilon_k \, \varepsilon_v$ noise terms might be too strict. In practice, some of these noise terms are of opposite signs, thus cancel each other. This leads us to the following more precise probabilistic scaling laws. %
\paragraph{Designed model baseline.}Thm.~\ref{thm:compressive-recall} describes a recall circuit based on general JL matrices $E$ and $\tilde{E}=FE$. In App.~\ref{app-subsec:designed-weights}, we suggest a specific weight construction for $E$, $F$. The resulting model serves as a sub-optimal baseline solution for associative recall tasks, which is simpler for mathematical analysis. In Thm.~\ref{app-thm-designed-ar} (Appendix), we analyze this baseline model performance on AR and MQAR. Building upon this result, we proceed to analyze a general linear model.
\begin{theorem}[\textbf{Trained model recall scaling laws}]\label{thm:prob-scaling-law}
A \hyperref[alg:mamba]{\textit{single-layer simplified Mamba}} can solve AR and MQAR tasks with probabilities %
{\small 
\begin{equation} \label{eq:success-prob-ar-large-Nf}
    p_\mathrm{AR} \approx \Phi \left( \sqrt{\frac{ND}{N_f}} - \sqrt{2 \log{V}} \right), \quad
    p_\mathrm{MQAR} \approx 
    \frac{1}{L - 2N_f} \sum _{t=2N_f} ^{L} 
    \Phi \left( \sqrt{\frac{ND}{\frac{1}{2}N_f+\frac{1}{4}t}} - \sqrt{2 \log{V}} \right),
\end{equation}%
}
in the limits of $N_f \gg N,D$ and $V \gg 1$, where $\Phi$ is the normal CDF. (See proof in App.~\ref{app-thm:prob-scaling-law}.)
\end{theorem}

Lastly, we note the above result can be further accurated and brought to a unified form.

\begin{remark}[\textbf{Unified recall scaling laws}]
\label{rem:practical-scaling-law}
A \hyperref[alg:mamba]{\textit{single-layer simplified Mamba}} solves associative recall tasks with probability $p_\mathrm{recall} \approx \Phi\!\left(\sqrt{\tfrac{ND}{aN_f+N}} - \sqrt{2\log V}\right)$, where $a=1$ for AR, and $a\approx\tfrac{5}{4}$ for MQAR. (See exact assumptions and proof in App.~\ref{app-rem:practical-scaling-law}.)
\end{remark}
\begin{remark}[\textbf{Scaling laws for model dimensions}]
\label{rem:dim-scaling-law}
To solve AR with probability $\geq 1-\delta$, model dimensions must satisfy $ND = \Omega\!\left(N_f\log\!\left(\tfrac{V}{\delta}\right)\right)$. (See derivation in App.~\ref{app-rem:dim-scaling-law}.)
\end{remark}

Importantly, Rem.~\ref{rem:dim-scaling-law} is a special case of a broader information-theoretic constraint, which holds for any fixed-state recurrent model, as derived in the following Lemma.

\begin{lemma}[\textbf{Information-theoretic lower bound on state size}]\label{lem:info-lower-bound}
Consider a recurrent model with a finite state size, where the context is summarized in a hidden state of $S$ scalar entries at $b$ bits of precision each. If such a model solves an AR task with $N_f$ facts (or an MQAR task with $N_f$ facts and context length $L=4N_f$) over a vocabulary of size $V$ with success probability $>1-\delta$, for a constant $0<\delta<1$, then $b\,S=\Omega(N_f\log V)$. (See proof in App.~\ref{app-lem:info-lower-bound}.)
\end{lemma}

Specifically for Mamba, at fixed precision $b$, this reads $ND=\Omega(N_f\log V)$ (see Rem.~\ref{app-rem:info-lower-bound-mamba}, Appendix), in agreement with Rem.~\ref{rem:dim-scaling-law}.

\subsection{Multi-Layer and Multi-Head Scaling Laws}
\paragraph{Multi-layer models.}
Given the single-layer circuit and scaling laws, we ask whether these findings extend to deeper architectures.
\begin{theorem}[\textbf{Multi-layer recall scaling laws}]\label{thm:multi-layer-recall-scaling-laws}
A \hyperref[app-alg:mamba-layers]{\textit{simplified multi-layer Mamba}} with $\Lambda$ layers can solve associative recall tasks with a probability
$p_\mathrm{recall} \approx 
\Phi \left( \sqrt{\frac{\Lambda N D}{a N_f + \Lambda N}} - \sqrt{2 \log{V}} \right)$, where $a=1$ for AR, and $a\approx\tfrac{5}{4}$ for MQAR. (See proof in App.~\ref{app-thm:multi-layer-recall}.)
\end{theorem}

\paragraph{Multi-head models.}
One of the key architectural innovations introduced in Mamba-2~\citep{dao2024transformers} is the ability to process the input sequence through multiple SSM heads in parallel. Analogously to attention head patterns in Transformers, the authors present several \textit{head patterns} for SSMs (see Tab.~\ref{tab:head-patterns}, Appendix); here we ask which are most beneficial for recall.

\begin{theorem}[\textbf{Multi-head recall scaling laws}] \label{thm:multi-head-recall}
Given fixed model dimensions $D$, $N$, \hyperref[alg:mamba-2]{single-layer simplified Mamba-2} models with \(M>1\) heads and multi-head attention patterns MQA, MKA, MVA, or MHA solve recall tasks (AR or MQAR) with probabilities satisfying $p_\mathrm{MQA} = p_\mathrm{MKA} \leq p_\mathrm{single} = p_\mathrm{MVA} \leq p_\mathrm{MHA}$, where $p_\mathrm{single}$ is the recall probability for a single-head model ($M=1$) of similar dimensions $D$, $N$, as detailed in Thm.~\ref{thm:prob-scaling-law}. (See proof in App.~\ref{app-thm:multi-head-recall}.)
\end{theorem}

\begin{figure*}[t]
    \centering

    \begin{subfigure}{0.245\linewidth}
        \centering
        \includegraphics[width=\linewidth]{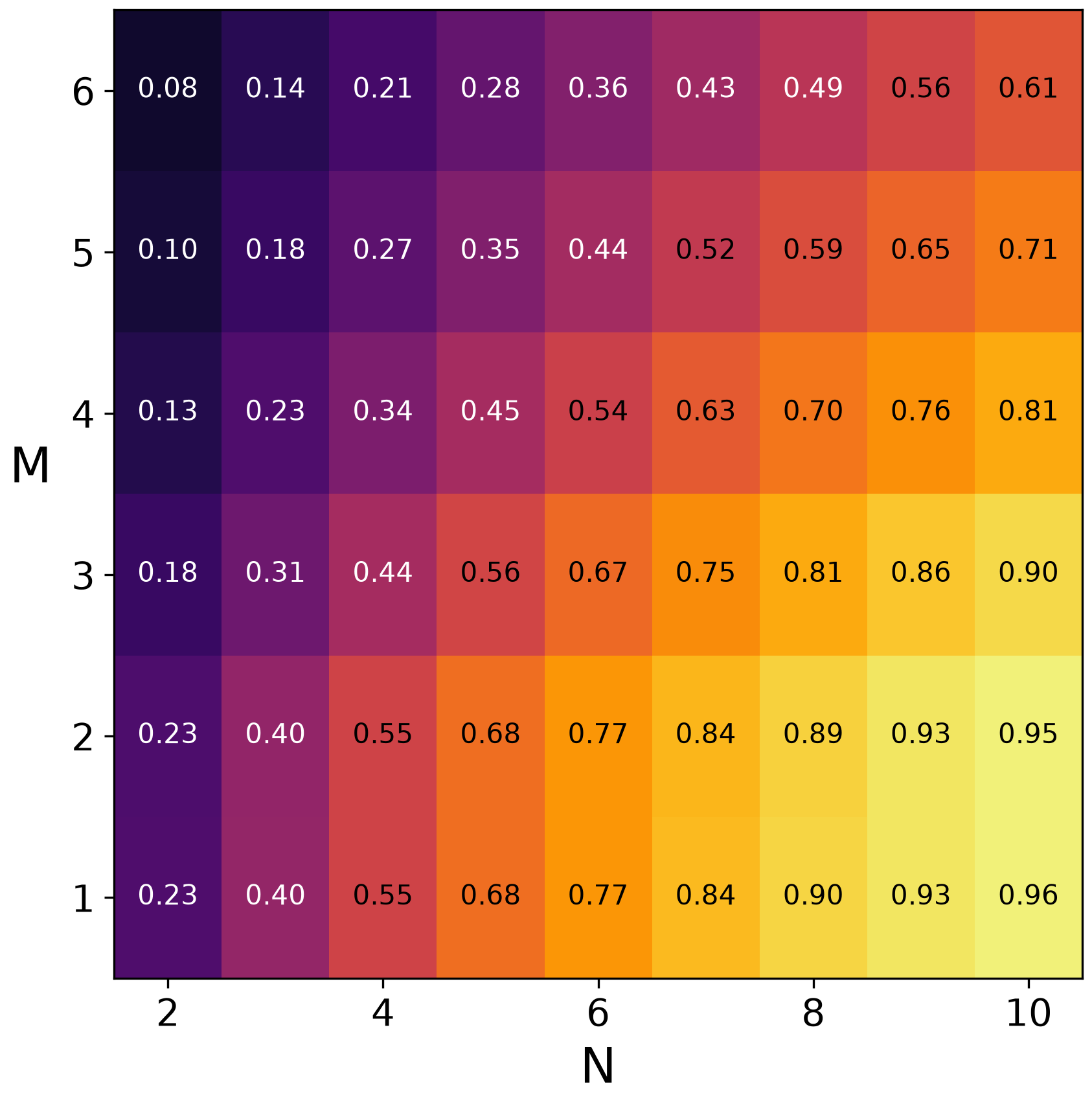}
        \caption{MQA}
        \label{fig:m2_}
    \end{subfigure}
    \hfill
    \begin{subfigure}{0.245\linewidth}
        \centering
        \includegraphics[width=\linewidth]{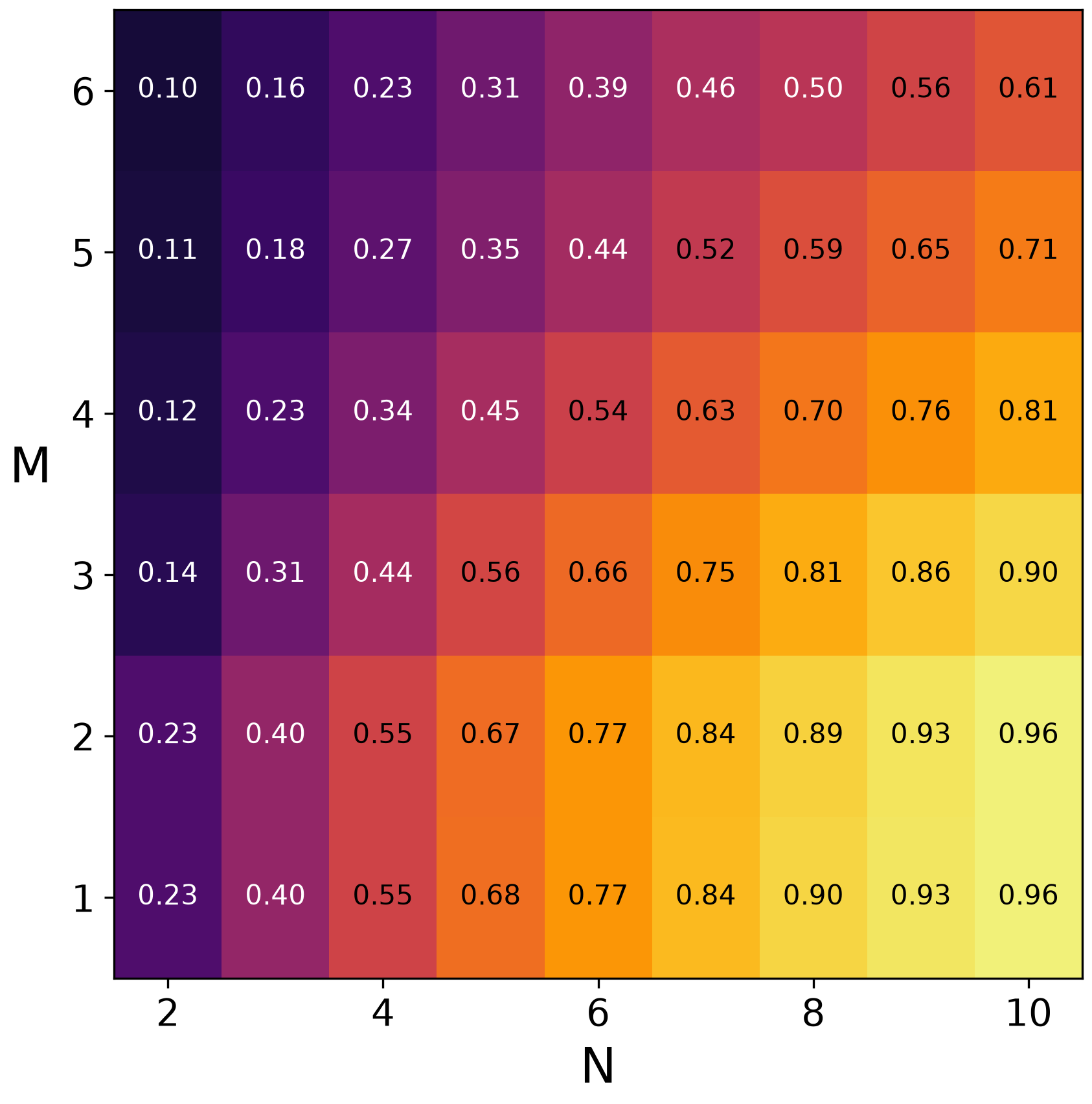}
        \caption{MKA}
        \label{fig:m3_}
    \end{subfigure}
    \hfill
    \begin{subfigure}{0.245\linewidth}
        \centering
        \includegraphics[width=\linewidth]{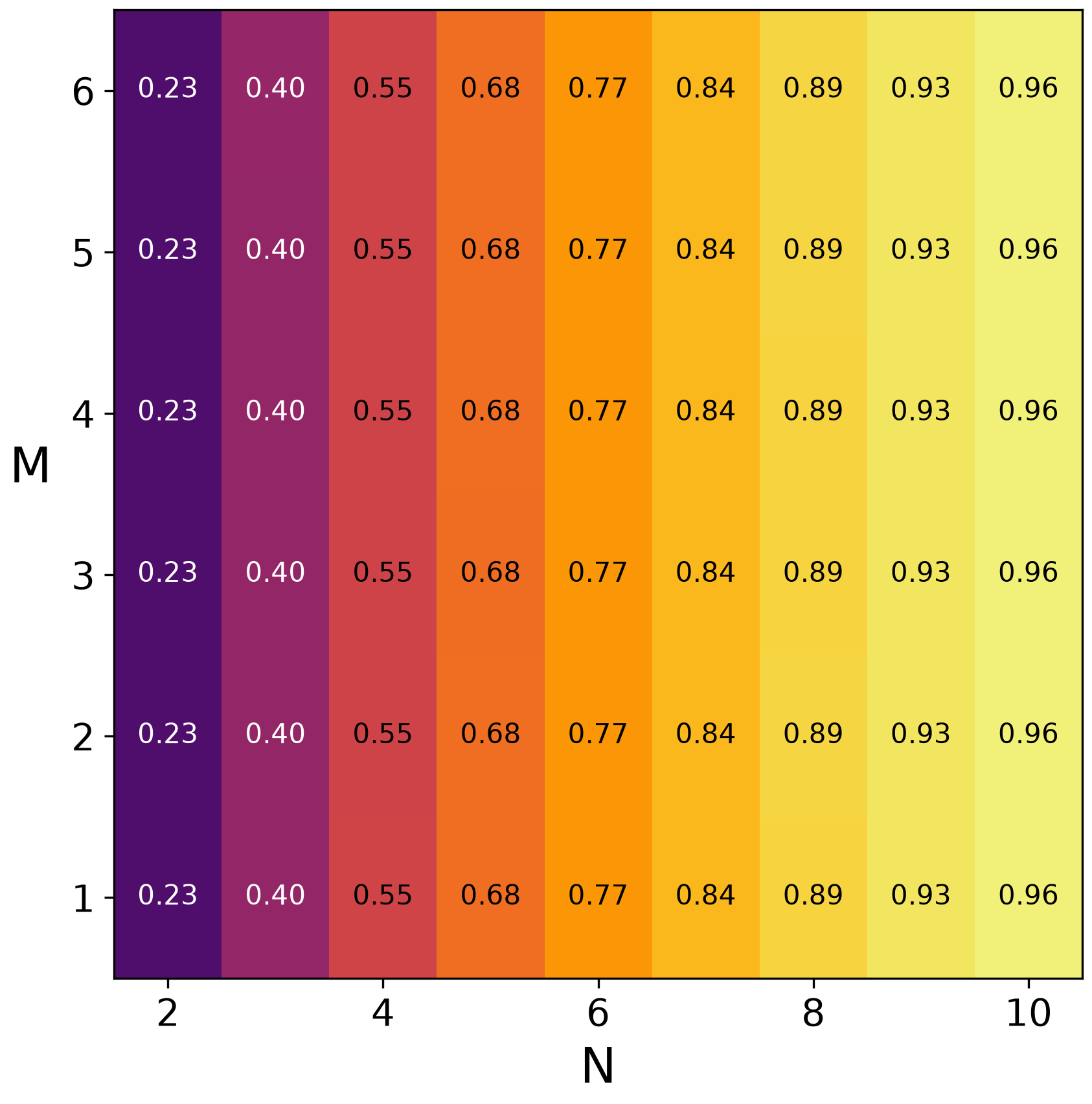}
        \caption{MVA}
        \label{fig:m1_}
    \end{subfigure}
    \hfill
    \begin{subfigure}{0.245\linewidth}
        \centering
        \includegraphics[width=\linewidth]{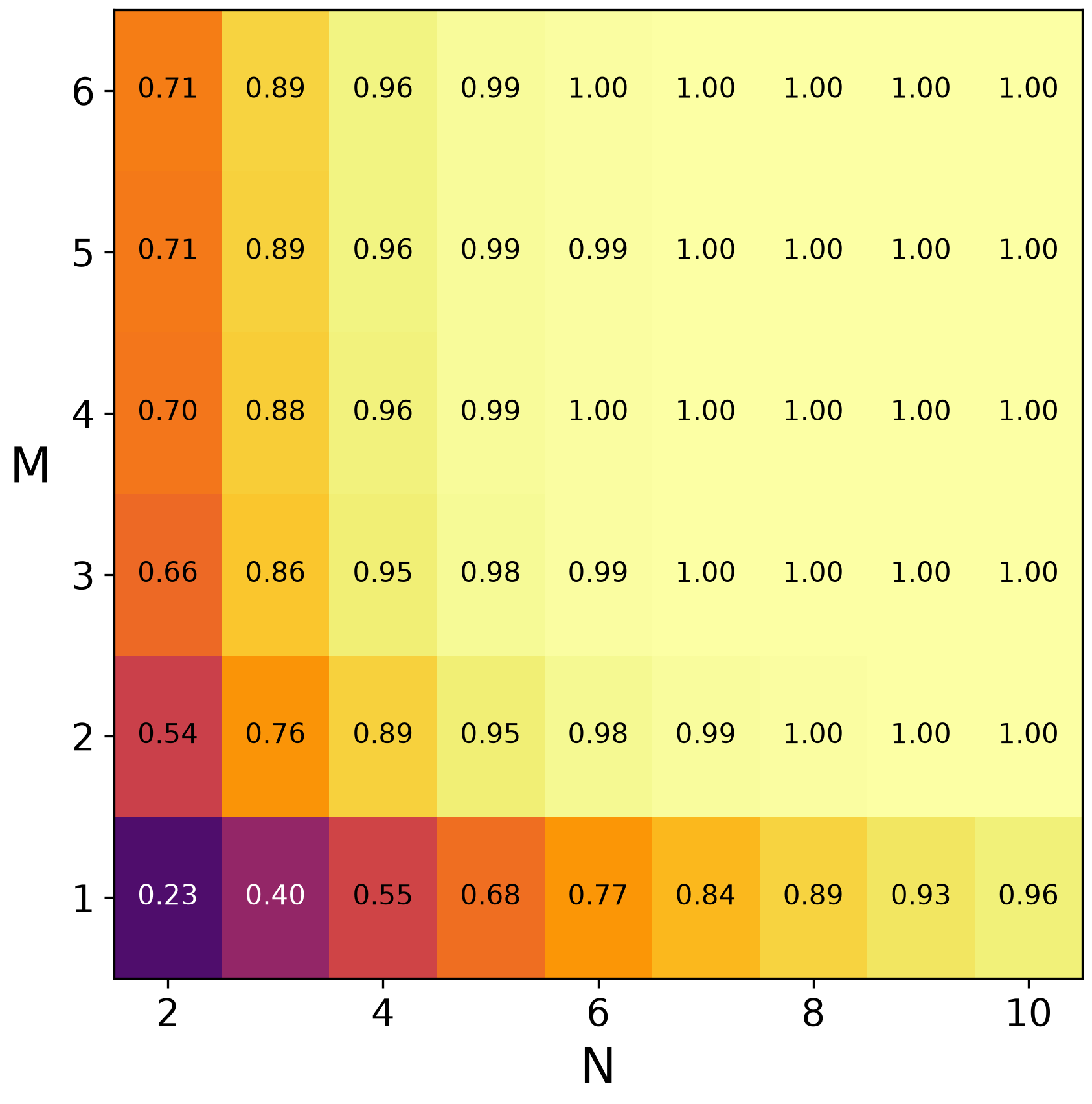}
        \caption{MHA}
        \label{fig:m4_}
    \end{subfigure}
    \caption{\textbf{Multi-head recall scaling laws.} Recall accuracy on MQAR for varying per-head state size $N$, number of heads $M$, and multi-head SSM pattern. Total channel dimension $D$ is held fixed. Consistent with Thm.~\ref{thm:multi-head-recall}, (c) MVA matches the single-head baseline, while (a) MQA and (b) MKA are weaker due to shared-value compression across heads; only (d) MHA improves recall accuracy by increasing the effective state size.%
    }
    \label{fig:multi-head-scaling-laws_}
\end{figure*}
\section{Scaling Laws Experiments}\label{sec:scaling-laws-exp}

\subsection{Single-Layer Model: Accuracy Grids}
We train simplified single-layer Mamba (linear and full) with varying $D$, $N$ across multiple AR and MQAR settings. See further details in App.~\ref{app:impl-details}.
Each grid point is trained with multiple random seeds; we report the \textit{best-of-seeds} accuracy, as our goal is to measure model \textit{capacity} - the existence of a recall solution - rather than the probability of finding it via optimization (see App.~\ref{app-subsec:seeds}).

\paragraph{Comparison with theory.} Our raw results are summarized in Fig.~\ref{fig:scaling-laws-2d-dims} (top). Trained linear (a, b) and full (c) models all exhibit the inverse $D$-$N$ tradeoff predicted by Rem.~\ref{rem:dim-scaling-law}, and the full model closely matches the linear one, confirming that our simplification preserves the core recall behavior. Fig.~\ref{app-fig:scaling-laws-2d-dims} (Appendix) adds designed models with fixed weights (Thm.~\ref{app-thm:designed-scaling-law}) and the theoretical grid of Thm.~\ref{thm:prob-scaling-law}, which closely align with the trained linear model.

\subsection{Single-Layer Model: Scaling Curves}
\label{subsec:single-layer-curves}
\paragraph{Predicting recall accuracy.}
Fig.~\ref{fig:scaling-laws-2d-dims} (bottom) shows that in both AR and MQAR settings, the $D$-$N$ accuracy grids exhibit a gradual transition from failure to near-perfect recall. This transition is well captured by a one-dimensional scaling variable $x = \tfrac{1}{\sigma}=\sqrt{\tfrac{ND}{a N_f+N}}$, where $\sigma^2$ is the theoretical variance of the model output $y_t$, and results closely follow $p_{\mathrm{recall}}(x)\approx \Phi(x-\sqrt{2\log V})$. Comparing columns (a) and (b), MQAR shows a larger variance (larger $a\approx\tfrac{5}{4}$), due to the increased number of queries, making high recall harder to achieve at the same $D,N$. Both AR and MQAR follow the same underlying recall scaling law, apparent after appropriate rescaling. We further validate the $V$ scaling law under more realistic vocabulary sizes (up to $V=50$K) and context lengths (up to $L=4$K). Our results align with the theoretical scaling laws; see details in Figs.~\ref{app-fig:large-vocab} and \ref{app-fig:large-vocab-curves} (Appendix).

\paragraph{Full model scaling behavior.}
In Fig.~\ref{fig:scaling-laws-2d-dims} (bottom), column (c) shows that a full nonlinear Mamba model also fits Rem.~\ref{rem:practical-scaling-law}, with $a_\mathrm{full} \approx \tfrac{1}{2} a_\mathrm{linear}$ (determined heuristically), indicating improved performance under a similar scaling law. This suggests both simplified and full models share a hash-like key-value compression bottleneck, with $N,D$ as the primary capacity constraints. We leave a detailed investigation of this hypothesis for future work.

\subsection{Multi-Layer Model Experiments}
\label{subsec:multi-layer-experiments}

\paragraph{The effect of additional layers.} Experiment results are summarized in Fig.~\ref{fig:multi-layer-scaling-laws}. Consistent with Thm.~\ref{thm:multi-layer-recall-scaling-laws}, measured recall accuracy is determined not by $N$ and $\Lambda$ separately, but by their product $N_{\mathrm{eff}} \equiv \Lambda N$ (an \textit{effective} state size), quantifying the contribution of additional layers.
\paragraph{Scaling curves.} Similarly to the single-layer case, accuracy collapses onto a single curve independently of depth $\Lambda$, closely following $p_{\mathrm{recall}}(x)\approx\Phi(x-\sqrt{2\log V})$; note that $x = \sqrt{\tfrac{\Lambda N D}{a N_f + \Lambda N}}$, provided by Thm.~\ref{thm:multi-layer-recall-scaling-laws}, exactly explains the observed results.

\subsection{Multi-Head Model Experiments}
\label{subsec:multi-head-experiments}

\paragraph{The effect of multi-head SSM patterns.}
We train \hyperref[alg:mamba-2]{\textit{simplified multi-head Mamba-2}} models on MQAR, varying $N$ and $M$ while keeping $D$ fixed (Fig.~\ref{fig:multi-head-scaling-laws_}). Consistent with Thm.~\ref{thm:multi-head-recall}, multi-head structure improves recall only when it enlarges both key and query dimensions, which occurs only in MHA - at the cost of more trainable parameters. MQA and MKA achieve worse accuracy due to shared-value compression across heads.
\paragraph{Connection to language-modeling performance.}
Interestingly, the trends in Thm.~\ref{thm:multi-head-recall} are consistent with the perplexity ablations for full Mamba-2~\citep[Sec.~9.4, Table~5]{dao2024transformers}: MQA and MKA both yield worse perplexity than MVA, matching our finding that shared-value compression degrades recall. When controlling for parameter count, MVA is the stronger design, thus chosen in practice. This suggests our analysis may capture the same tradeoffs observed in full Mamba models in real-world NLP tasks. We further evaluate zero-shot AR performance on pretrained Mamba-2 models of 370M to 2.7B parameters, in the spirit of \citet{ben2025overflow}, and find similar trends: accuracy decays with the number of facts at every size, and the largest state retains the most (Fig.~\ref{app-fig:zero-shot-ar}, Appendix).

\section{Conclusions} %
In this paper, we identify the neural algorithm underlying associative recall in Mamba models. We build on this understanding to develop a theoretical framework that predicts, given model dimensions, the maximum vocabulary size and number of facts over which accurate recall can be achieved in the MQAR task. Our theoretical predictions are shown to closely align with empirical results.
\paragraph{Limitations.} Our theory holds for a simplified, linear variant of Mamba; the agreement of full Mamba models with it is empirical evidence rather than a formal equivalence. We treat gating, nonlinearities, and discretization only through the constants in our bounds; a full theoretical account of their role is a direction for future research. Likewise, direct transfer of our scaling laws from synthetic tasks to language modeling remains to be established \citep{ben2025overflow}, though zero-shot recall in pretrained Mamba-2 models follows the same qualitative trends (Sec.~\ref{subsec:multi-head-experiments}, App.~\ref{app:zero_shot_ar}). Future work includes extending the analysis to other architectures, such as xLSTM~\citep{beck2024xlstm}, RWKV~\citep{peng2023rwkv} and DeltaNet~\citep{yang2024parallelizing}.

\renewcommand{\acksection}{\section*{Acknowledgments and Disclosure of Funding}}  %
\begin{ack}
This work was supported by a grant from the Tel Aviv University Center for AI and Data Science (TAD). Additionally, this research was also supported by the Ministry of Innovation, Science \& Technology, Israel (1001576154) and the Michael J. Fox Foundation (MJFF-022407). The authors declare no competing interests.
\end{ack}

\bibliographystyle{unsrtnat}  %
\bibliography{main}

@inproceedings{okpekpe2025revisiting,
  title={Revisiting associative recall in modern recurrent models},
  author={Okpekpe, Destiny and Orvieto, Antonio},
  booktitle={First Workshop on Scalable Optimization for Efficient and Adaptive Foundation Models},
  year={2025}
}

@article{gu2023mamba,
  title={{Mamba}: Linear-time sequence modeling with selective state spaces},
  author={Gu, Albert and Dao, Tri},
  journal={arXiv preprint arXiv:2312.00752},
  year={2023}
}

@article{peng2023rwkv,
  title={{RWKV}: Reinventing {RNNs} for the {Transformer} era},
  author={Peng, Bo and Alcaide, Eric and Anthony, Quentin and Albalak, Alon and Arcadinho, Samuel and Biderman, Stella and Cao, Huanqi and Cheng, Xin and Chung, Michael and Grella, Matteo and others},
  journal={arXiv preprint arXiv:2305.13048},
  year={2023}
}

@article{de2024griffin,
  title={Griffin: Mixing gated linear recurrences with local attention for efficient language models},
  author={De, Soham and Smith, Samuel L and Fernando, Anushan and Botev, Aleksandar and Cristian-Muraru, George and Gu, Albert and Haroun, Ruba and Berrada, Leonard and Chen, Yutian and Srinivasan, Srivatsan and others},
  journal={arXiv preprint arXiv:2402.19427},
  year={2024}
}

@article{ben2025overflow,
  title={Overflow prevention enhances long-context recurrent {LLMs}},
  author={Ben-Kish, Assaf and Zimerman, Itamar and Mirza, M Jehanzeb and Wolf, Lior and Glass, James and Karlinsky, Leonid and Giryes, Raja},
  journal={arXiv preprint arXiv:2505.07793},
  year={2025}
}

@article{parnichkun2025quantifying,
  title={Quantifying Memory Utilization with Effective State-Size},
  author={Parnichkun, Rom N and Tumma, Neehal and Thomas, Armin W and Moro, Alessandro and An, Qi and Suzuki, Taiji and Yamashita, Atsushi and Poli, Michael and Massaroli, Stefano},
  journal={arXiv preprint arXiv:2504.19561},
  year={2025}
}

@inproceedings{arora2023zoology,
  title={Zoology: Measuring and improving recall in efficient language models},
  author={Arora, Simran and Eyuboglu, Sabri and Timalsina, Aman and Johnson, Isys and Poli, Michael and Zou, James and Rudra, Atri and R{\'e}, Christopher},
  booktitle={The Twelfth International Conference on Learning Representations},
  year={2024}
}

@article{johnson1984extensions,
  title={Extensions of {Lipschitz} mappings into a {Hilbert} space},
  author={Johnson, William B and Lindenstrauss, Joram},
  journal={Contemporary mathematics},
  volume={26},
  pages={189--206},
  year={1984}
}

@article{eisenstadt2025overclocking,
  title={Overclocking {LLM} reasoning: Monitoring and controlling thinking path lengths in {LLMs}},
  author={Eisenstadt, Roy and Zimerman, Itamar and Wolf, Lior},
  journal={arXiv preprint arXiv:2506.07240},
  year={2025}
}

@article{elhage2021mathematical,
   title={A mathematical framework for {Transformer} circuits},
   author={Elhage, Nelson and Nanda, Neel and Olsson, Catherine and Henighan, Tom and Joseph, Nicholas and Mann, Ben and Askell, Amanda and Bai, Yuntao and Chen, Anna and Conerly, Tom and DasSarma, Nova and Drain, Dawn and Ganguli, Deep and Hatfield-Dodds, Zac and Hernandez, Danny and Jones, Andy and Kernion, Jackson and Lovitt, Liane and Ndousse, Kamal and Amodei, Dario and Brown, Tom and Clark, Jack and Kaplan, Jared and McCandlish, Sam and Olah, Chris},
   year={2021},
   journal={Transformer Circuits Thread},
   note={https://transformer-circuits.pub/2021/framework/index.html}
}

@article{olsson2022context,
  title={In-context learning and induction heads},
  author={Olsson, Catherine and Elhage, Nelson and Nanda, Neel and Joseph, Nicholas and DasSarma, Nova and Henighan, Tom and Mann, Ben and Askell, Amanda and Bai, Yuntao and Chen, Anna and others},
  journal={arXiv preprint arXiv:2209.11895},
  year={2022}
}

@article{olah2020zoom,
  author = {Olah, Chris and Cammarata, Nick and Schubert, Ludwig and Goh, Gabriel and Petrov, Michael and Carter, Shan},
  title = {Zoom In: An Introduction to Circuits},
  journal = {Distill},
  year = {2020},
  note = {https://distill.pub/2020/circuits/zoom-in},
  doi = {10.23915/distill.00024.001}
}

@inproceedings{
kantamneni2025language,
title={Language Models Use Trigonometry to Do Addition},
author={Subhash Kantamneni and Max Tegmark},
booktitle={ICLR 2025 Workshop on Building Trust in Language Models and Applications},
year={2025},
url={https://openreview.net/forum?id=CqViN4dQJk}
}

@inproceedings{
zhong2023the,
title={The Clock and the Pizza: Two Stories in Mechanistic Explanation of Neural Networks},
author={Ziqian Zhong and Ziming Liu and Max Tegmark and Jacob Andreas},
booktitle={Thirty-seventh Conference on Neural Information Processing Systems},
year={2023},
url={https://openreview.net/forum?id=S5wmbQc1We}
}

@inproceedings{chughtai2023toy,
  title={A toy model of universality: Reverse engineering how networks learn group operations},
  author={Chughtai, Bilal and Chan, Lawrence and Nanda, Neel},
  booktitle={International Conference on Machine Learning},
  pages={6243--6267},
  year={2023},
  organization={PMLR}
}

@inproceedings{hendel2023context,
  title={In-Context Learning Creates Task Vectors},
  author={Hendel, Roee and Geva, Mor and Globerson, Amir},
  booktitle={Findings of the Association for Computational Linguistics: EMNLP 2023},
pages={9318--9333},
  year={2023}
}

@article{arora2025mechanistic,
  title={Mechanistic evaluation of {Transformers} and state space models},
  author={Arora, Aryaman and Rathi, Neil and Selvam, Nikil Roashan and Csord{\'a}s, R{\'o}bert and Jurafsky, Dan and Potts, Christopher},
  journal={arXiv preprint arXiv:2505.15105},
  year={2025}
}

@article{mambamoe1,
  title={{BlackMamba}: Mixture of experts for state-space models},
  author={Anthony, Quentin and Tokpanov, Yury and Glorioso, Paolo and Millidge, Beren},
  journal={arXiv preprint arXiv:2402.01771},
  year={2024}
}

@article{wang2024mambabyte,
  title={{MambaByte}: Token-free selective state space model},
  author={Wang, Junxiong and Gangavarapu, Tushaar and Yan, Jing Nathan and Rush, Alexander M},
  journal={arXiv preprint arXiv:2401.13660},
  year={2024}
}

@article{mambaViT1,
  title={{VMamba}: Visual state space model},
  author={Liu, Yue and Tian, Yunjie and Zhao, Yuzhong and Yu, Hongtian and Xie, Lingxi and Wang, Yaowei and Ye, Qixiang and Jiao, Jianbin and Liu, Yunfan},
  journal={arXiv preprint arXiv:2401.10166},
  year={2024}
}

@article{mambaViT2,
  title={Vision {Mamba}: Efficient visual representation learning with bidirectional state space model},
  author={Zhu, Lianghui and Liao, Bencheng and Zhang, Qian and Wang, Xinlong and Liu, Wenyu and Wang, Xinggang},
  journal={arXiv preprint arXiv:2401.09417},
  year={2024}
}

@techreport{blelloch1990prefix,
  title={Prefix sums and their applications},
  author={Blelloch, Guy E},
  year={1990},
institution={School of Computer Science, Carnegie Mellon University},
  number={CMU-CS-90-190},
  month=nov
}

@inproceedings{lutati2023focus,
  title={Focus your attention (with adaptive {IIR} filters)},
  author={Lutati, Shahar and Zimerman, Itamar and Wolf, Lior},
  booktitle={Proceedings of the 2023 Conference on Empirical Methods in Natural Language Processing},
  pages={12538--12549},
  year={2023}
}

@article{mambaGraph1,
  title={{Graph-Mamba}: Towards long-range graph sequence modeling with selective state spaces},
  author={Wang, Chloe and Tsepa, Oleksii and Ma, Jun and Wang, Bo},
  journal={arXiv preprint arXiv:2402.00789},
  year={2024}
}

@article{smith2022simplified,
  title={Simplified state space layers for sequence modeling},
  author={Smith, Jimmy TH and Warrington, Andrew and Linderman, Scott W},
  journal={arXiv preprint arXiv:2208.04933},
  year={2022}
}

@article{cohen2025expressivity,
  title={On the Expressivity of Selective State-Space Layers: A Multivariate Polynomial Approach},
  author={Cohen-Karlik, Edo and Zimerman, Itamar and Galanti, Liane and Atad, Ido and Globerson, Amir and Wolf, Lior},
  journal={arXiv preprint arXiv:2502.02209},
  year={2025}
}

@article{salakhutdinov2009semantic,
  title={Semantic hashing},
  author={Salakhutdinov, Ruslan and Hinton, Geoffrey},
  journal={International Journal of Approximate Reasoning},
  volume={50},
  number={7},
  pages={969--978},
  year={2009},
  publisher={Elsevier}
}

@article{andoni2008near,
  title={Near-optimal hashing algorithms for approximate nearest neighbor in high dimensions},
  author={Andoni, Alexandr and Indyk, Piotr},
  journal={Communications of the ACM},
  volume={51},
  number={1},
  pages={117--122},
  year={2008},
  publisher={ACM New York, NY, USA}
}

@inproceedings{datar2004locality,
  title={Locality-sensitive hashing scheme based on p-stable distributions},
  author={Datar, Mayur and Immorlica, Nicole and Indyk, Piotr and Mirrokni, Vahab S},
  booktitle={Proceedings of the twentieth annual symposium on Computational geometry},
  pages={253--262},
  year={2004}
}

@inproceedings{wen2024rnns,
  title={{RNNs} are not {Transformers} (yet): The key bottleneck on in-context retrieval},
  author={Wen, Kaiyue and Dang, Xingyu and Lyu, Kaifeng},
  booktitle={The Thirteenth International Conference on Learning Representations},
  year={2025}
}

@article{bick2025understanding,
  title={Understanding the skill gap in recurrent language models: The role of the gather-and-aggregate mechanism},
  author={Bick, Aviv and Xing, Eric and Gu, Albert},
  journal={arXiv preprint arXiv:2504.18574},
  year={2025}
}

@article{arora2024simple,
  title={Simple linear attention language models balance the recall-throughput tradeoff},
  author={Arora, Simran and Eyuboglu, Sabri and Zhang, Michael and Timalsina, Aman and Alberti, Silas and Zinsley, Dylan and Zou, James and Rudra, Atri and R{\'e}, Christopher},
  journal={arXiv preprint arXiv:2402.18668},
  year={2024}
}

@article{fu2022hungry,
  title={{Hungry Hungry Hippos}: Towards language modeling with state space models},
  author={Fu, Daniel Y and Dao, Tri and Saab, Khaled K and Thomas, Armin W and Rudra, Atri and R{\'e}, Christopher},
  journal={arXiv preprint arXiv:2212.14052},
  year={2022}
}

@inproceedings{poli2023hyena,
  title={Hyena hierarchy: Towards larger convolutional language models},
  author={Poli, Michael and Massaroli, Stefano and Nguyen, Eric and Fu, Daniel Y and Dao, Tri and Baccus, Stephen and Bengio, Yoshua and Ermon, Stefano and R{\'e}, Christopher},
  booktitle={International Conference on Machine Learning},
  pages={28043--28078},
  year={2023},
  organization={PMLR}
}

@inproceedings{benkish2024decimamba,
      title={{DeciMamba}: Exploring the length extrapolation potential of {Mamba}}, 
      author={Assaf Ben-Kish and Itamar Zimerman and Shady Abu-Hussein and Nadav Cohen and Amir Globerson and Lior Wolf and Raja Giryes},
      booktitle={The Thirteenth International Conference on Learning Representations},
      year={2025},
      url={https://arxiv.org/abs/2406.14528}, 
}

@misc{ye2025longmamba,
      title={{LongMamba}: Enhancing {Mamba}'s long context capabilities via training-free receptive field enlargement}, 
      author={Zhifan Ye and Kejing Xia and Yonggan Fu and Xin Dong and Jihoon Hong and Xiangchi Yuan and Shizhe Diao and Jan Kautz and Pavlo Molchanov and Yingyan Celine Lin},
      year={2025},
      eprint={2504.16053},
      archivePrefix={arXiv},
      primaryClass={cs.CL},
      url={https://arxiv.org/abs/2504.16053}, 
}

@article{trockman2024mimetic,
  title={Mimetic initialization helps state space models learn to recall},
  author={Trockman, Asher and Harutyunyan, Hrayr and Kolter, J Zico and Kumar, Sanjiv and Bhojanapalli, Srinadh},
  journal={arXiv preprint arXiv:2410.11135},
  year={2024}
}

@article{wang2025test,
  title={Test-time regression: a unifying framework for designing sequence models with associative memory},
  author={Wang, Ke Alexander and Shi, Jiaxin and Fox, Emily B},
  journal={arXiv preprint arXiv:2501.12352},
  year={2025}
}

@article{ba2016using,
  title={Using fast weights to attend to the recent past},
  author={Ba, Jimmy and Hinton, Geoffrey E and Mnih, Volodymyr and Leibo, Joel Z and Ionescu, Catalin},
  journal={Advances in neural information processing systems},
  volume={29},
  year={2016}
}

@article{yang2024parallelizing,
  title={Parallelizing linear {Transformers} with the delta rule over sequence length},
  author={Yang, Songlin and Wang, Bailin and Zhang, Yu and Shen, Yikang and Kim, Yoon},
  journal={Advances in neural information processing systems},
  volume={37},
  pages={115491--115522},
  year={2024}
}

@article{dao2024transformers,
  title={{Transformers} are {SSMs}: Generalized models and efficient algorithms through structured state space duality},
  author={Dao, Tri and Gu, Albert},
  journal={arXiv preprint arXiv:2405.21060},
  year={2024}
}

@article{jelassi2024repeat,
  title={Repeat after me: {Transformers} are better than state space models at copying},
  author={Jelassi, Samy and Brandfonbrener, David and Kakade, Sham M and Malach, Eran},
  journal={arXiv preprint arXiv:2402.01032},
  year={2024}
}

@inproceedings{ali2024hidden,
  title={The hidden attention of {Mamba} models},
  author={Ali, Ameen and Zimerman, Itamar and Wolf, Lior},
  booktitle={Proceedings of the 63rd Annual Meeting of the Association for Computational Linguistics (Volume 1: Long Papers)},
  pages={1516--1534},
  year={2025}
}

@article{huang2025understanding,
  title={Understanding input selectivity in {Mamba}: Impact on approximation power, memorization, and associative recall capacity},
  author={Huang, Ningyuan and Sarabia, Miguel and Moudgil, Abhinav and Rodriguez, Pau and Zappella, Luca and Danieli, Federico},
  journal={arXiv preprint arXiv:2506.11891},
  year={2025}
}

@article{paszke2019pytorch,
  title={{PyTorch}: An imperative style, high-performance deep learning library},
  author={Paszke, Adam and Gross, Sam and Massa, Francisco and Lerer, Adam and Bradbury, James and Chanan, Gregory and Killeen, Trevor and Lin, Zeming and Gimelshein, Natalia and Antiga, Luca and others},
  journal={Advances in neural information processing systems},
  volume={32},
  year={2019}
}

@inproceedings{you2024revealing,
  title={Revealing and mitigating the local pattern shortcuts of {Mamba}},
  author={You, Wangjie and Tang, Zecheng and Li, Juntao and Yao, Lili and Zhang, Min},
  booktitle={Findings of the Association for Computational Linguistics: ACL 2025},
  pages={12156--12178},
  year={2025}
}

@article{beck2024xlstm,
  title={{xLSTM}: Extended long short-term memory},
  author={Beck, Maximilian and P{\"o}ppel, Korbinian and Spanring, Markus and Auer, Andreas and Prudnikova, Oleksandra and Kopp, Michael and Klambauer, G{\"u}nter and Brandstetter, Johannes and Hochreiter, Sepp},
  journal={Advances in Neural Information Processing Systems},
  volume={37},
  pages={107547--107603},
  year={2024}
}

@misc{chatgpt2025,
  author       = {OpenAI},
  title        = {{ChatGPT} ({GPT-5})},
  year         = {2025},
  url          = {https://chat.openai.com/},
  note         = {Large language model developed by OpenAI}
}

@misc{claude2025,
  author       = {Anthropic},
  title        = {Claude},
  year         = {2025},
  url          = {https://www.anthropic.com/claude},
  note         = {Large language model developed by Anthropic}
}

@misc{mambatiny2024,
  author       = {PeaBrane},
  title        = {mamba-tiny},
  year         = {2024},
  url          = {https://github.com/PeaBrane/mamba-tiny},
  note         = {Minimal Mamba implementation}
}

@inproceedings{dar2023analyzing,
  title={Analyzing {Transformers} in embedding space},
  author={Dar, Guy and Geva, Mor and Gupta, Ankit and Berant, Jonathan},
  booktitle={Proceedings of the 61st Annual Meeting of the Association for Computational Linguistics (Volume 1: Long Papers)},
  pages={16124--16170},
  year={2023}
}

@inproceedings{geva2023dissecting,
  title={Dissecting Recall of Factual Associations in Auto-Regressive Language Models},
  author={Geva, Mor and Bastings, Jasmijn and Filippova, Katja and Globerson, Amir},
  booktitle={Proceedings of the 2023 Conference on Empirical Methods in Natural Language Processing},
  pages={12216--12235},
  year={2023}
}

@inproceedings{lecorre2025how,
  title={How does {M}amba Perform Associative Recall? {A} Mechanistic Study},
  author={Le Corre, Gr{\'e}goire and Huang, Ningyuan and Bietti, Alberto},
  booktitle={NeurIPS 2025 Workshop on Mechanistic Interpretability},
  year={2025},
  url={https://openreview.net/forum?id=ur17K8zXcB}
}

@inproceedings{behrouz2026atlas,
  title={{ATLAS}: Learning to Optimally Memorize the Context at Test Time},
  author={Behrouz, Ali and Li, Zeman and Kacham, Praneeth and Daliri, Majid and Deng, Yuan and Zhong, Peilin and Razaviyayn, Meisam and Mirrokni, Vahab},
  booktitle={Proceedings of the 43rd International Conference on Machine Learning},
  volume={306},
  series={Proceedings of Machine Learning Research},
  pages={7361--7390},
  year={2026}
}

@inproceedings{nichani2025understanding,
  title={Understanding factual recall in {Transformers} via associative memories},
  author={Nichani, Eshaan and Lee, Jason D. and Bietti, Alberto},
  booktitle={The Thirteenth International Conference on Learning Representations},
  year={2025}
}

\clearpage  %
\onecolumn
\appendix

\startappendixtoc

\clearpage
\phantomsection
\section*{Appendix Contents}

\printappendixtoc
\clearpage

\setcounter{theorem}{0}

\section{Additional Experiments and Figures}
\label{app:experiments-supp}

\subsection{\texorpdfstring{Scaling Laws Experiments}{Scaling Laws Experiments}}
\label{app-subsec:scaling-laws-experiments}

\begin{figure*}[!ht]
\small
\centering
\begin{tabular}{@{\hskip 0.1in}c  @{\hskip 0.05in}c@{\hskip 0.05in}c@{\hskip 0.05in}c@{\hskip 0.05in}c@{\hskip 0.05in}c@{\hskip 0.05in}}
    \toprule
  \multirow{2}{*}{} &
  \multicolumn{3}{c}{Simplified Linear Model} &
  \multicolumn{1}{c}{Full Model} & \\
  \cmidrule(lr){2-4}
  \cmidrule(lr){5-5}
   & (a) Designed & (b) Theory & (c) Trained & (d) Trained & \\
  \cmidrule(lr){1-5}
    \makecell{\textbf{(i)} \\ $V=1024$ \\ $L=64$ \\ $N_f=16$} &

    \raisebox{-0.5\height}{
    \includegraphics[width=0.20\textwidth]{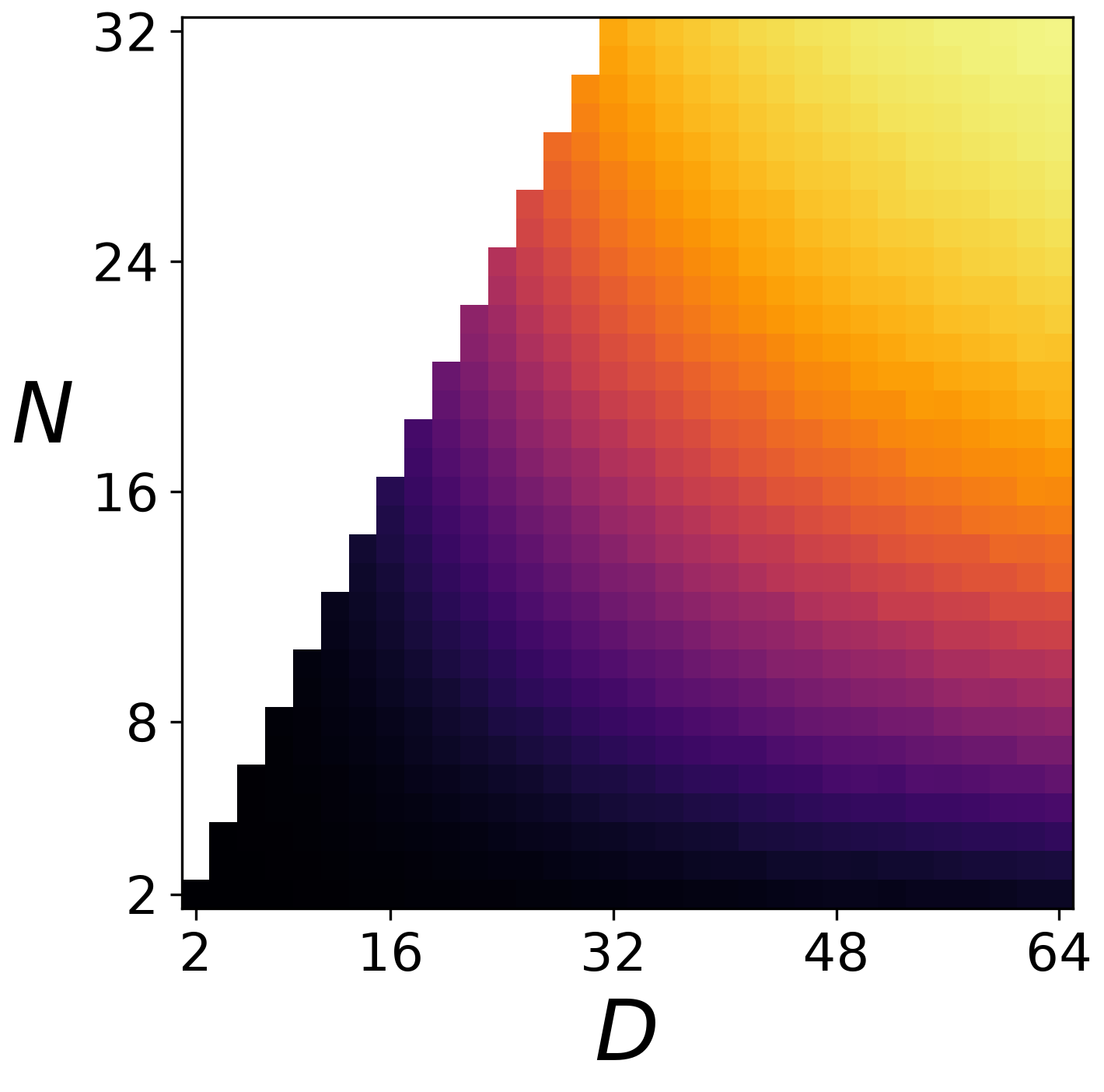}} &
    \raisebox{-0.5\height}{
    \includegraphics[width=0.20\textwidth]{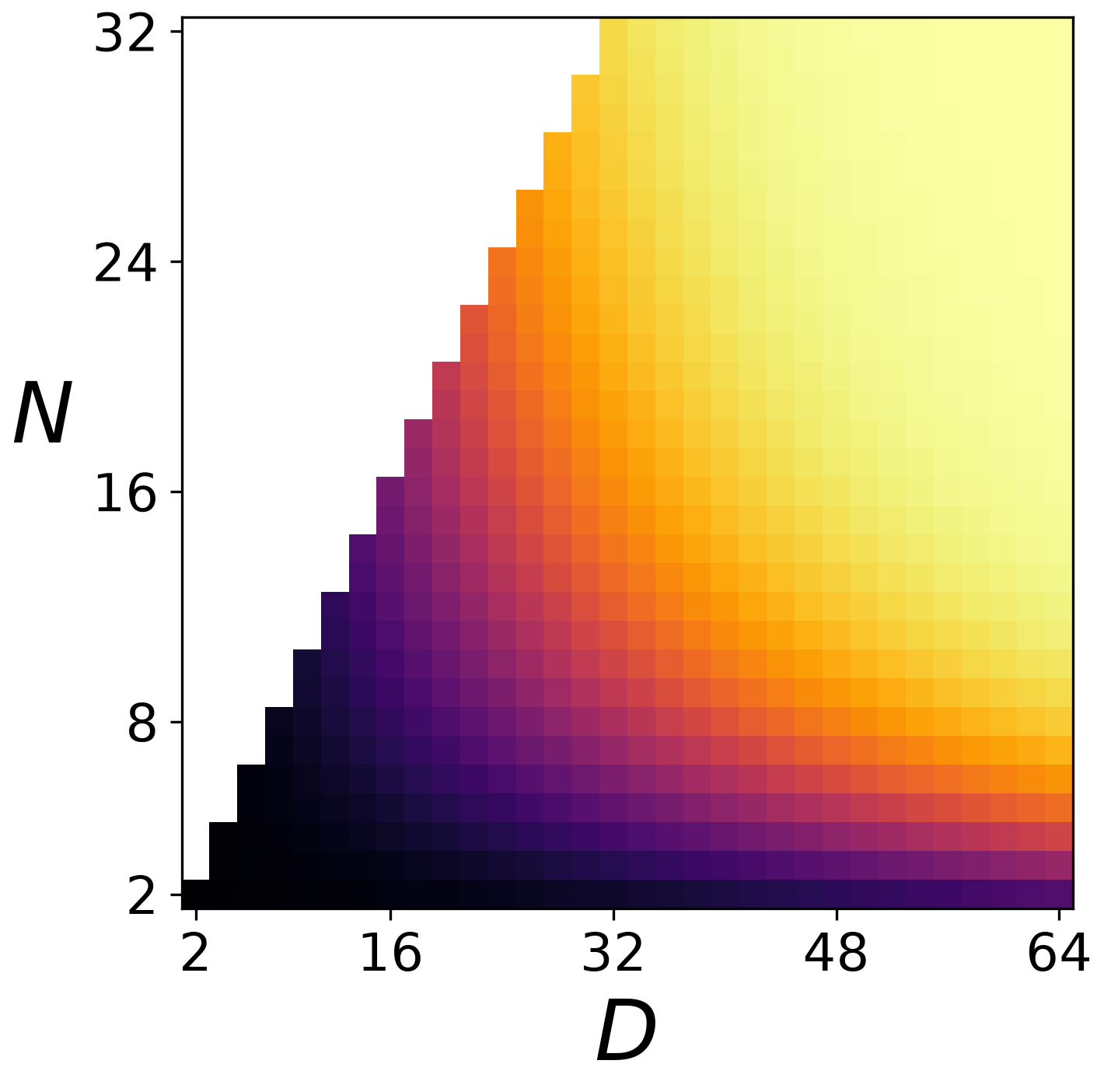}} &
    \raisebox{-0.5\height}{
    \includegraphics[width=0.20\textwidth]{figures/MQAR__D_N__linear_trained__regime_2.png}} &
    \raisebox{-0.5\height}{
    \includegraphics[width=0.20\textwidth]{figures/MQAR__D_N__full_trained__regime_2.png}} &
    \raisebox{-0.5\height}{
    \includegraphics[width=0.032\textwidth]{figures/external/colorbar.png}}
 \\
    \cmidrule(lr){1-5}
    \makecell{\textbf{(ii)} \\ $V=512$ \\ $L=128$ \\ $N_f=32$} &

    \raisebox{-0.5\height}{
    \includegraphics[width=0.20\textwidth]
    {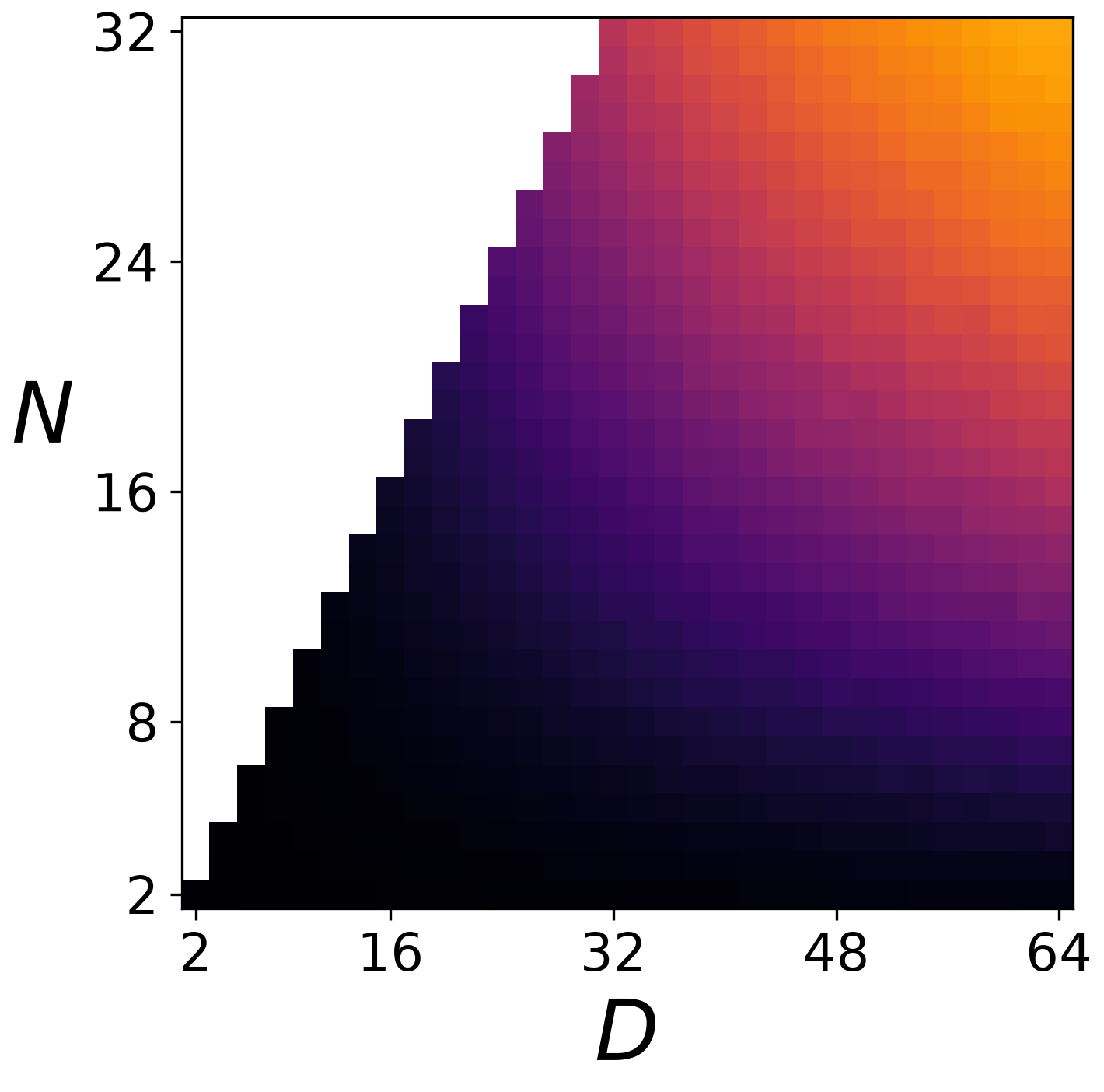}} &
    \raisebox{-0.5\height}{
    \includegraphics[width=0.20\textwidth]
    {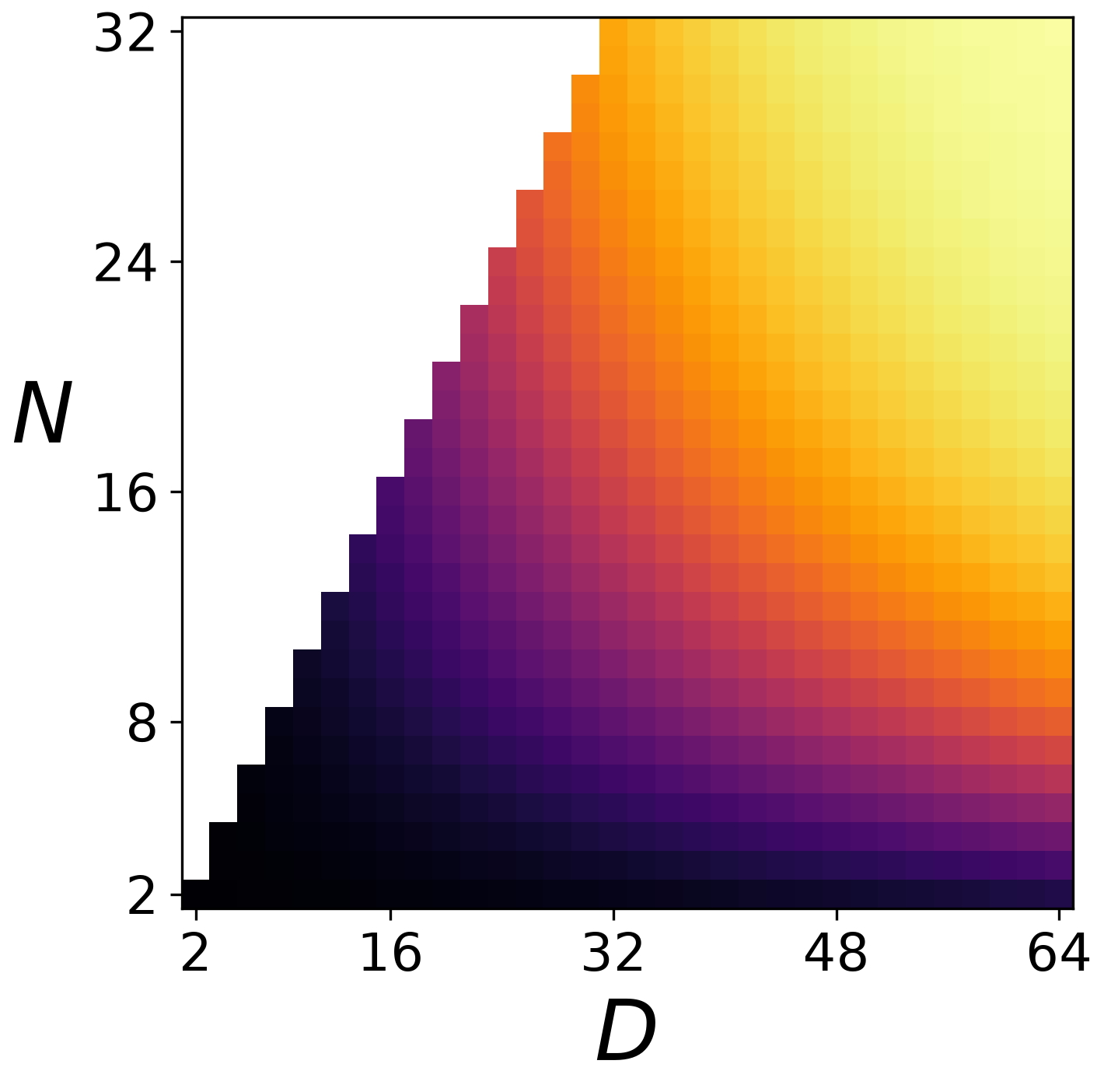}} &
    \raisebox{-0.5\height}{
    \includegraphics[width=0.20\textwidth]
    {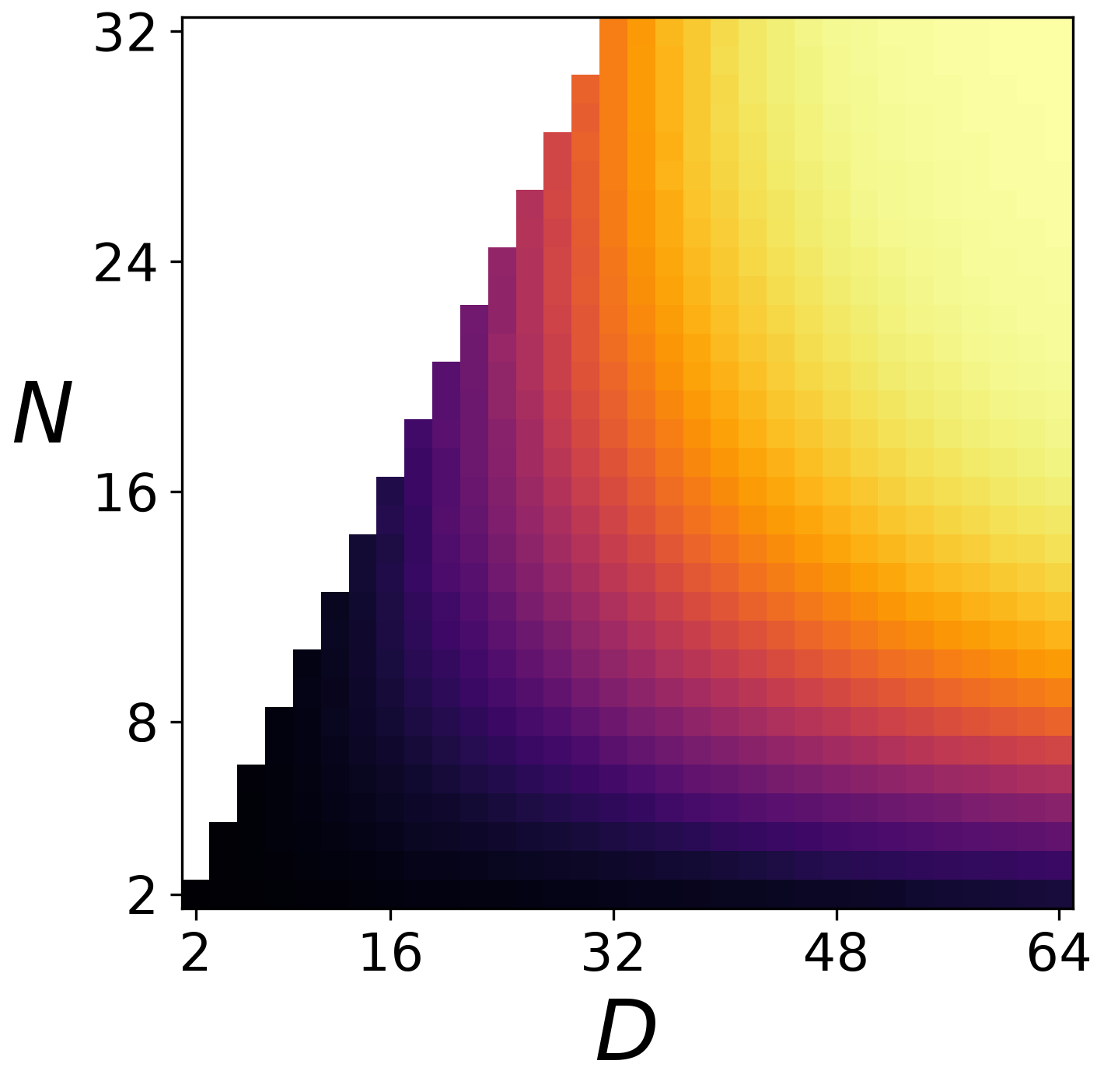}} &
    \raisebox{-0.5\height}{
    \includegraphics[width=0.20\textwidth]
    {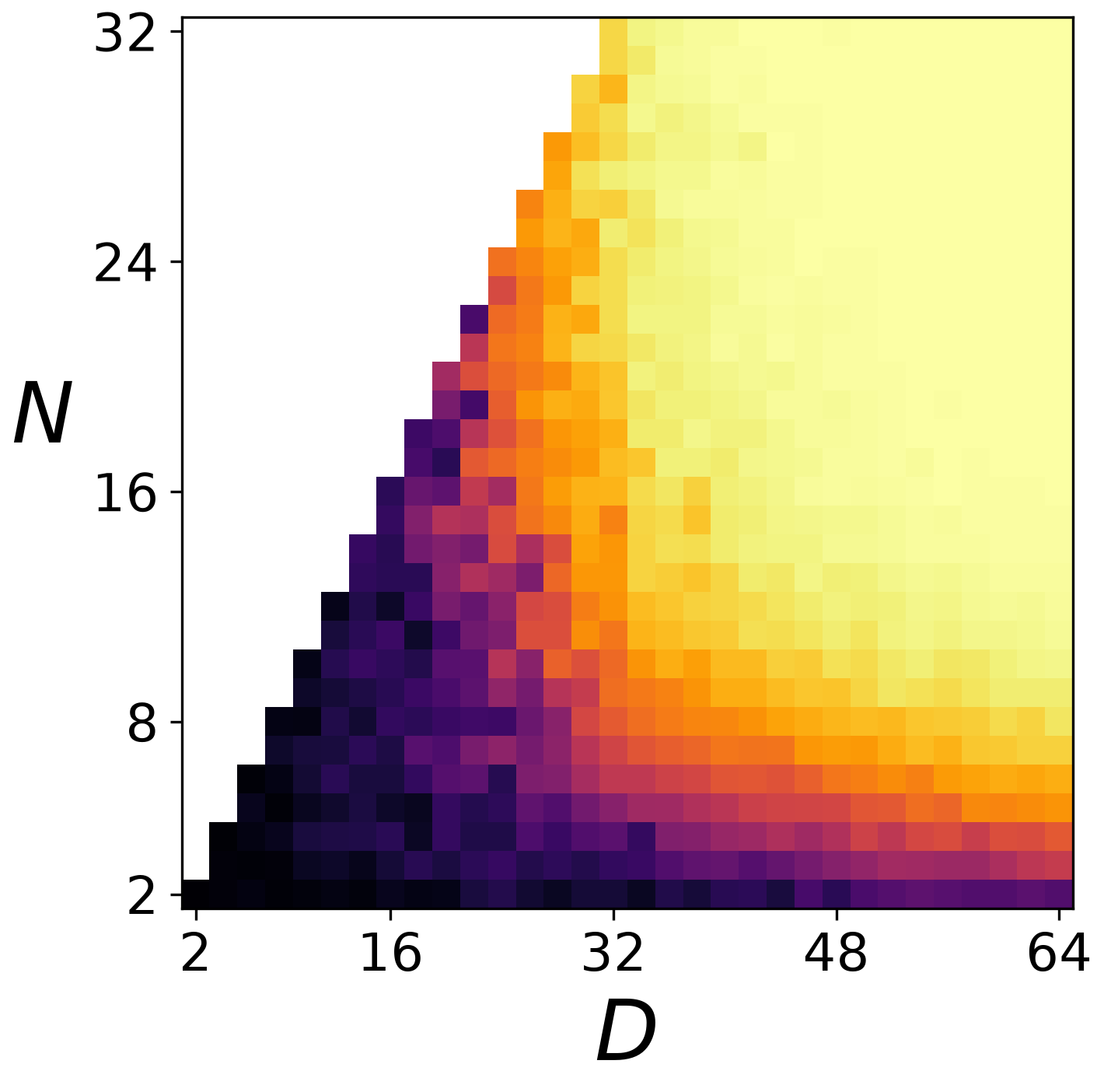}} &
    \raisebox{-0.5\height}{
    \includegraphics[width=0.032\textwidth]{figures/external/colorbar.png}}
 \\
    \cmidrule(lr){1-5}
    \makecell{\textbf{(iii)} \\ $V=512$ \\ $L=64$ \\ $N_f=16$} &

    \raisebox{-0.5\height}{
    \includegraphics[width=0.20\textwidth]
    {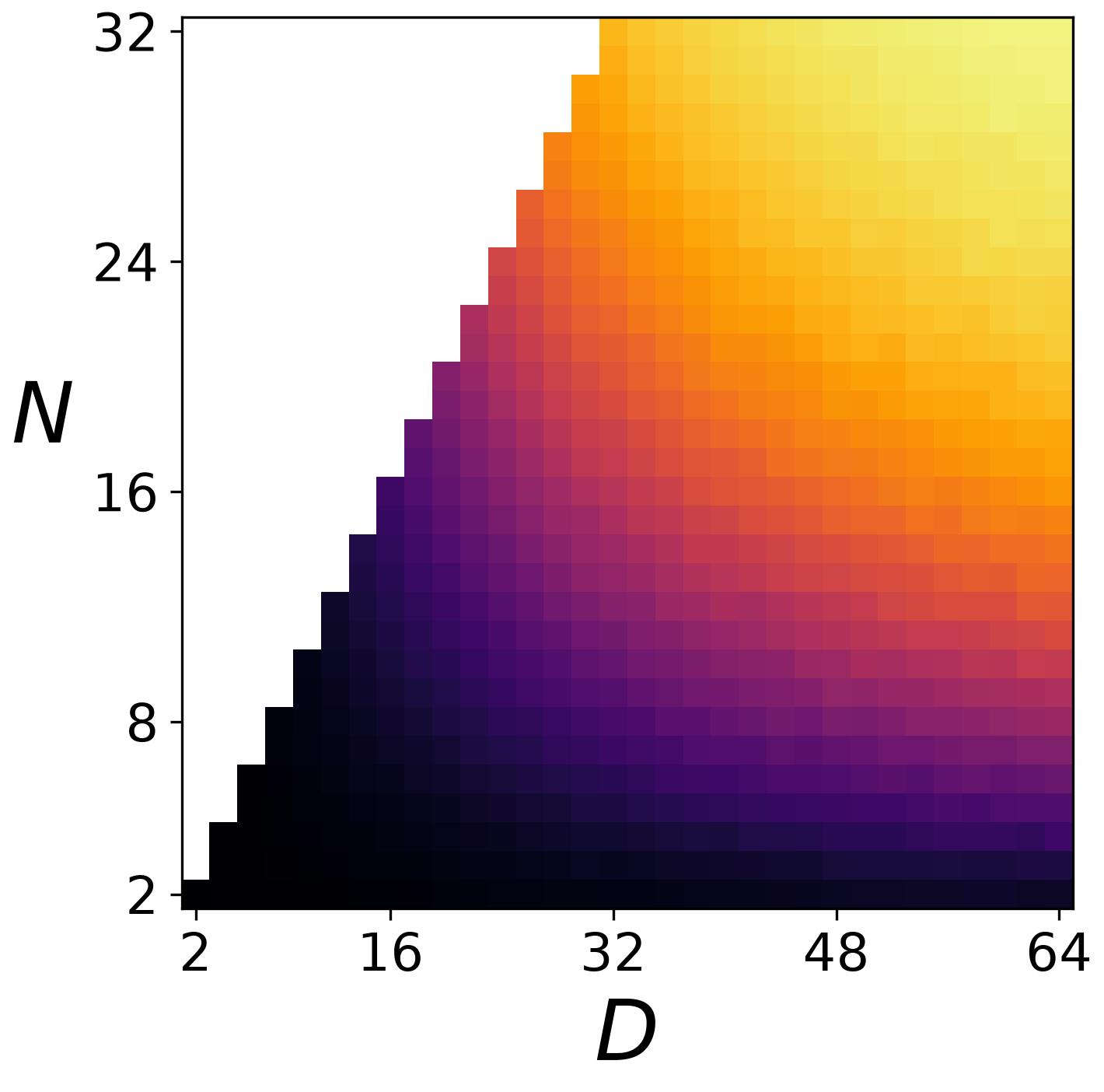}} &
    \raisebox{-0.5\height}{
    \includegraphics[width=0.20\textwidth]
    {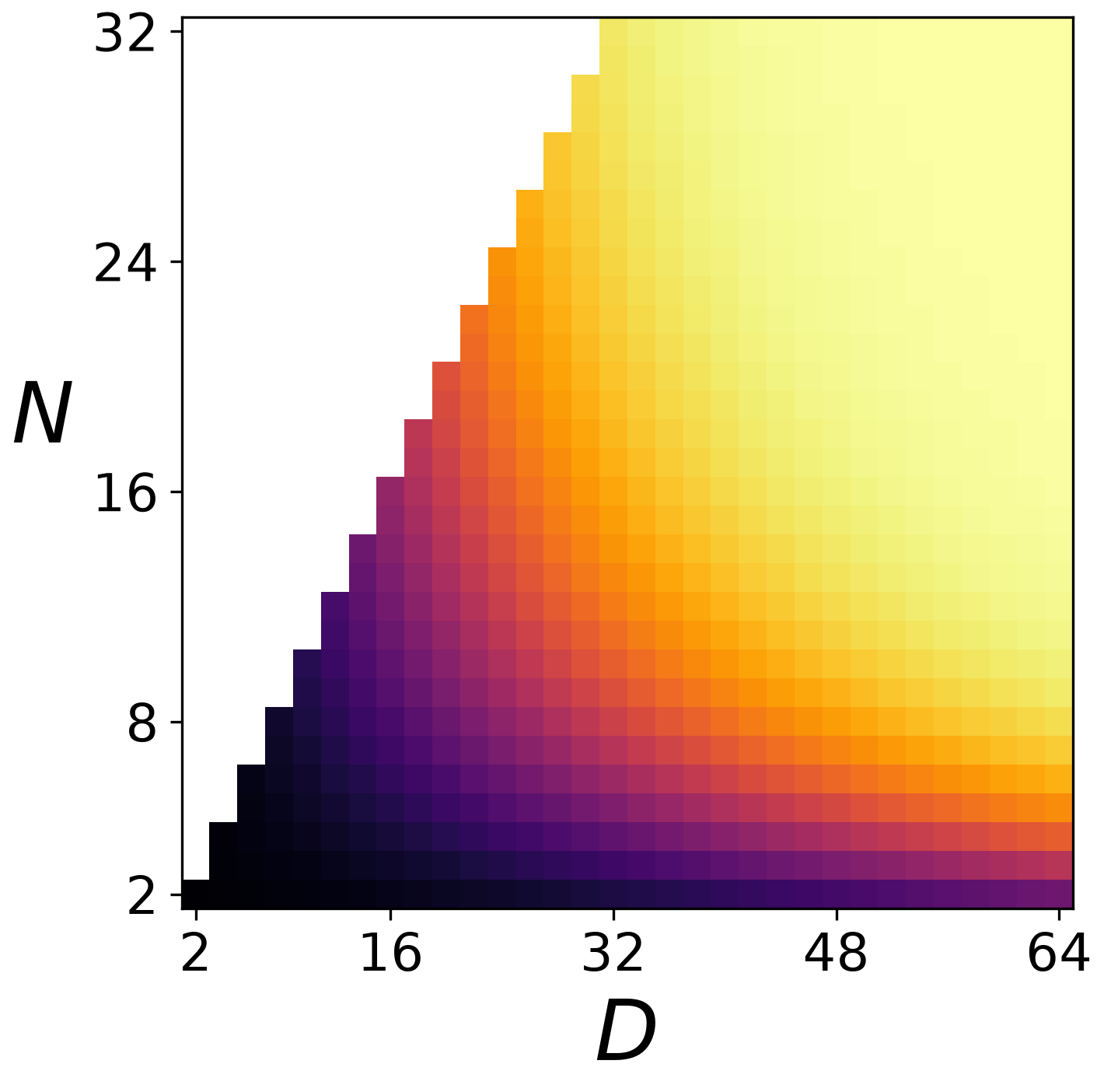}} &
    \raisebox{-0.5\height}{
    \includegraphics[width=0.20\textwidth]
    {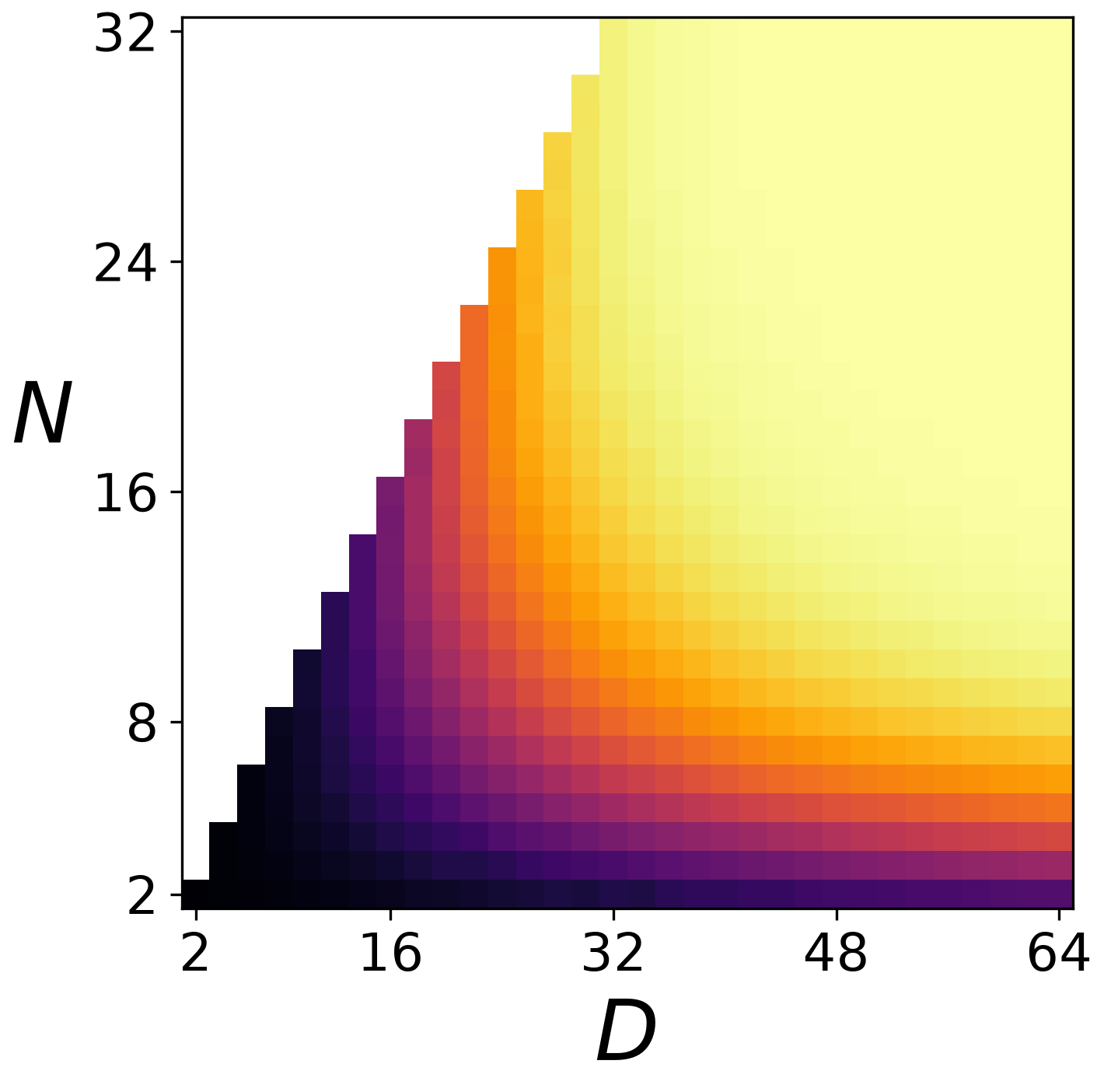}} &
    \raisebox{-0.5\height}{
    \includegraphics[width=0.20\textwidth]
    {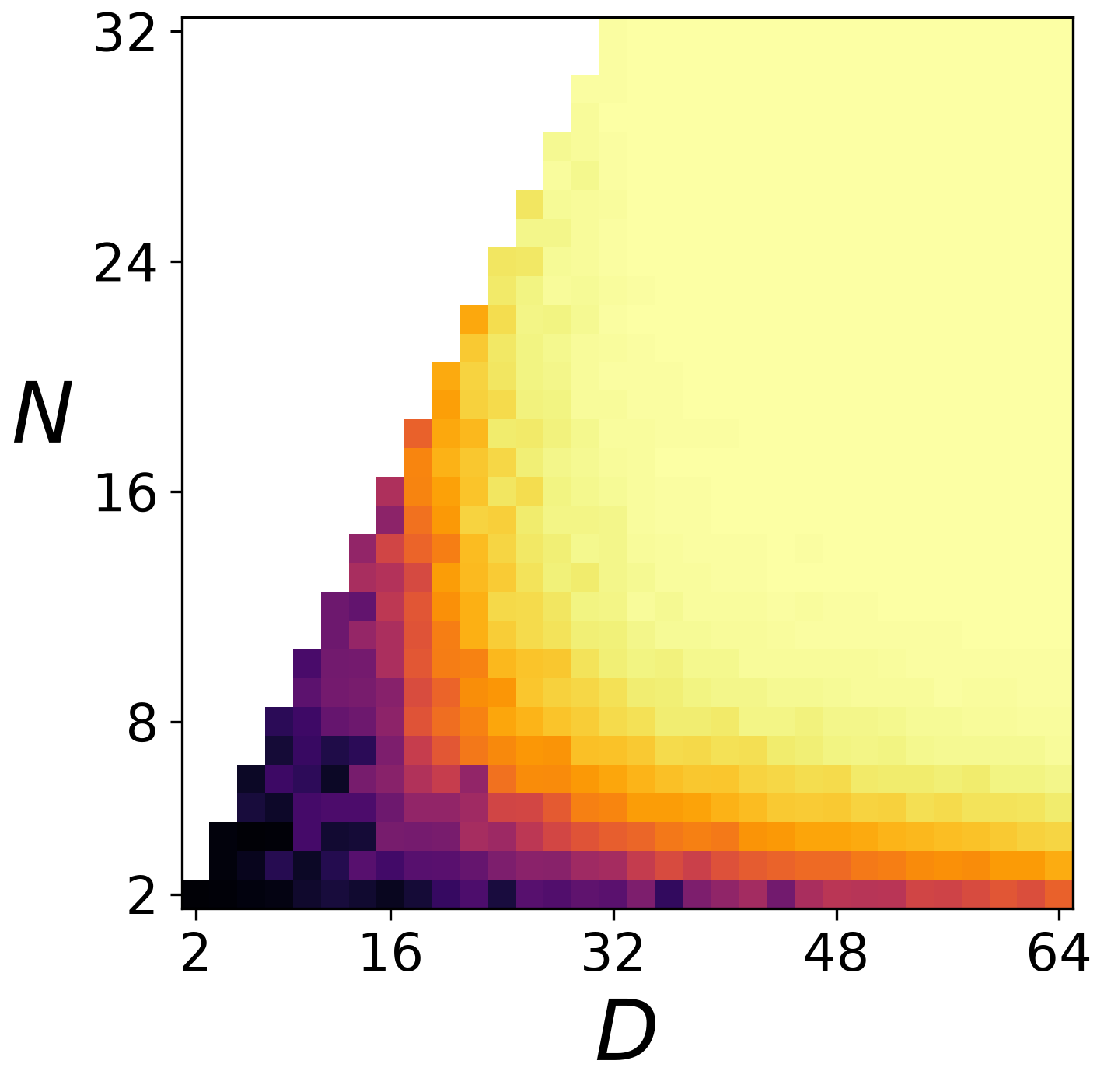}} &
    \raisebox{-0.5\height}{
    \includegraphics[width=0.032\textwidth]{figures/external/colorbar.png}}
 \\
 \bottomrule
 \end{tabular}
 \caption{%
 \textbf{Mamba recall scaling laws: theory versus empirical results.} Recall accuracy, represented by pixel color, against model embedding dimension $D$ and state size $N$. Left to right: (a) Simplified linear Mamba model with fixed weights, designed as presented in Thm.~\ref{app-thm:designed-scaling-law}. (b) Theoretical scaling law as derived in Thm.~\ref{thm:prob-scaling-law} for a trained linear model. (c) Trained linear model. (d) Trained full model.
Each row represents a different regime of MQAR parameters.\label{app-fig:scaling-laws-2d-dims}
 }
\end{figure*}

\begin{figure*}[!ht]
\small
\centering
\begin{tabular}{@{\hskip 0.06in}c@{\hskip 0.06in}c@{\hskip 0.06in}c@{\hskip 0.05in}c}
   & AR & MQAR & MQAR \\
   & (a) Linear Model & (b) Linear Model & (c) Full Model \\
  \makecell{\textbf{(i)} \\ $V=1024$ \\ $L=64$ \\ $N_f=16$} &
  \raisebox{-0.5\height}{\includegraphics[width=0.30\textwidth]{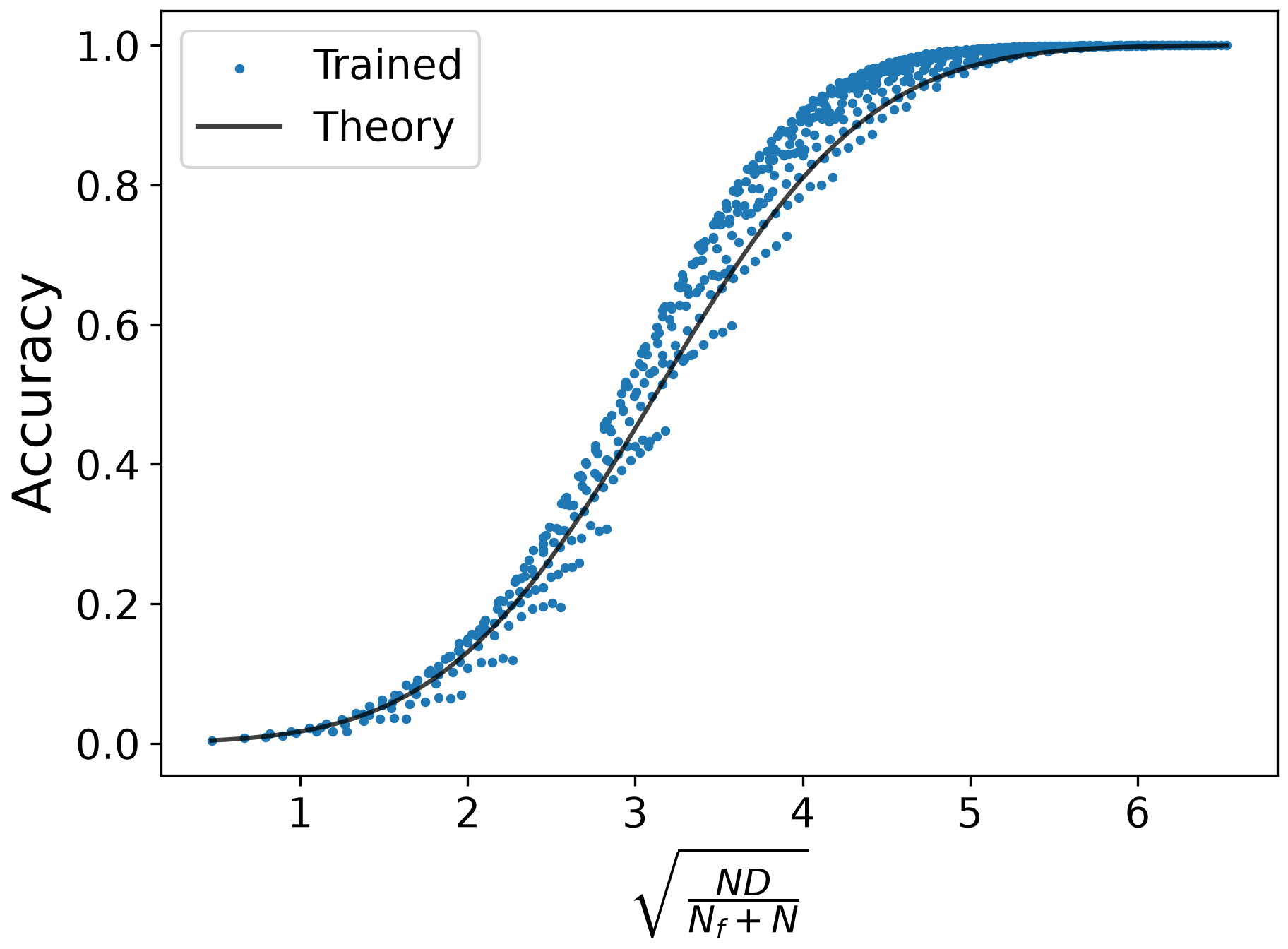}} &
  \raisebox{-0.5\height}{\includegraphics[width=0.30\textwidth]{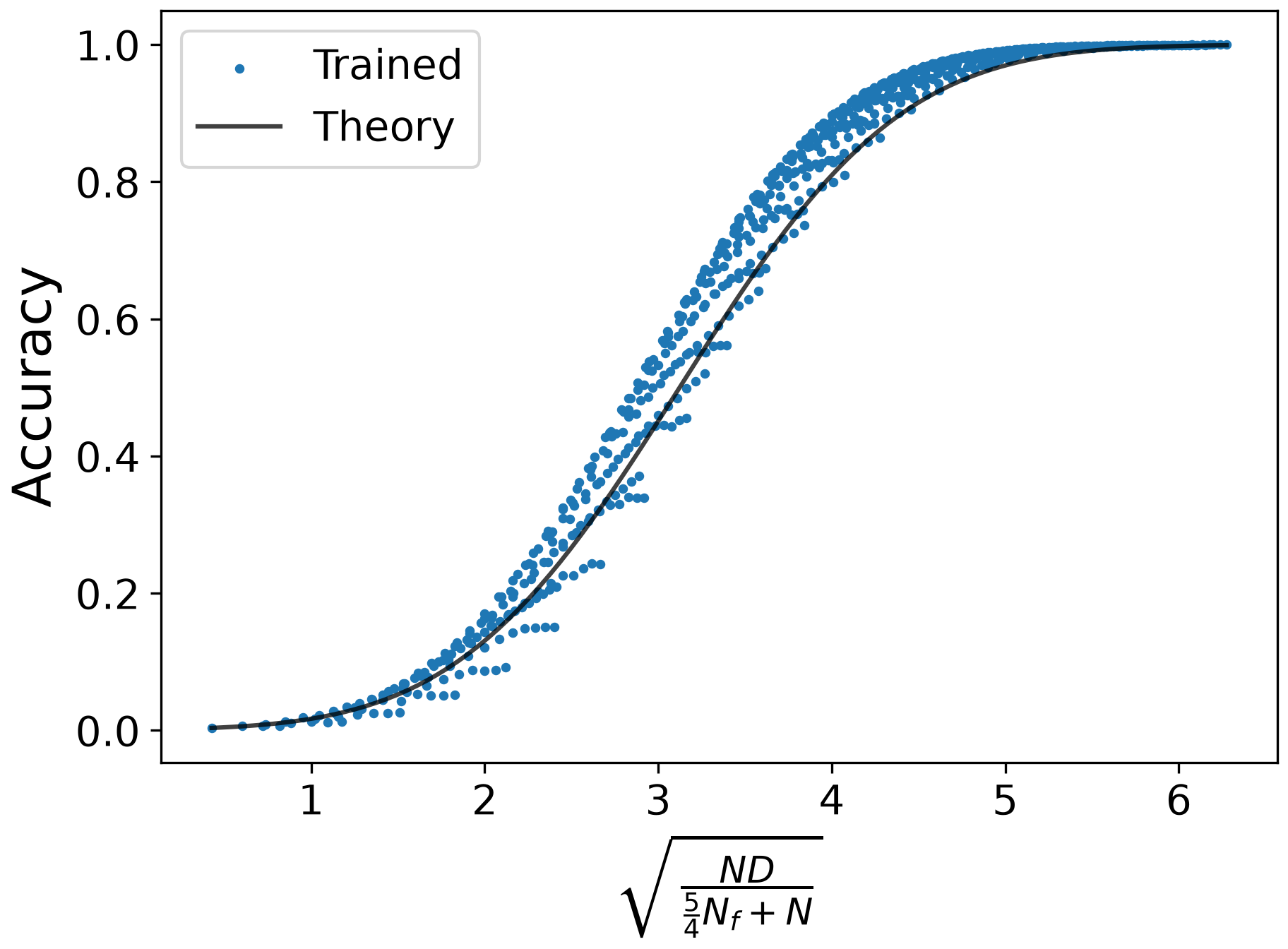}} &
  \raisebox{-0.5\height}{\includegraphics[width=0.30\textwidth]{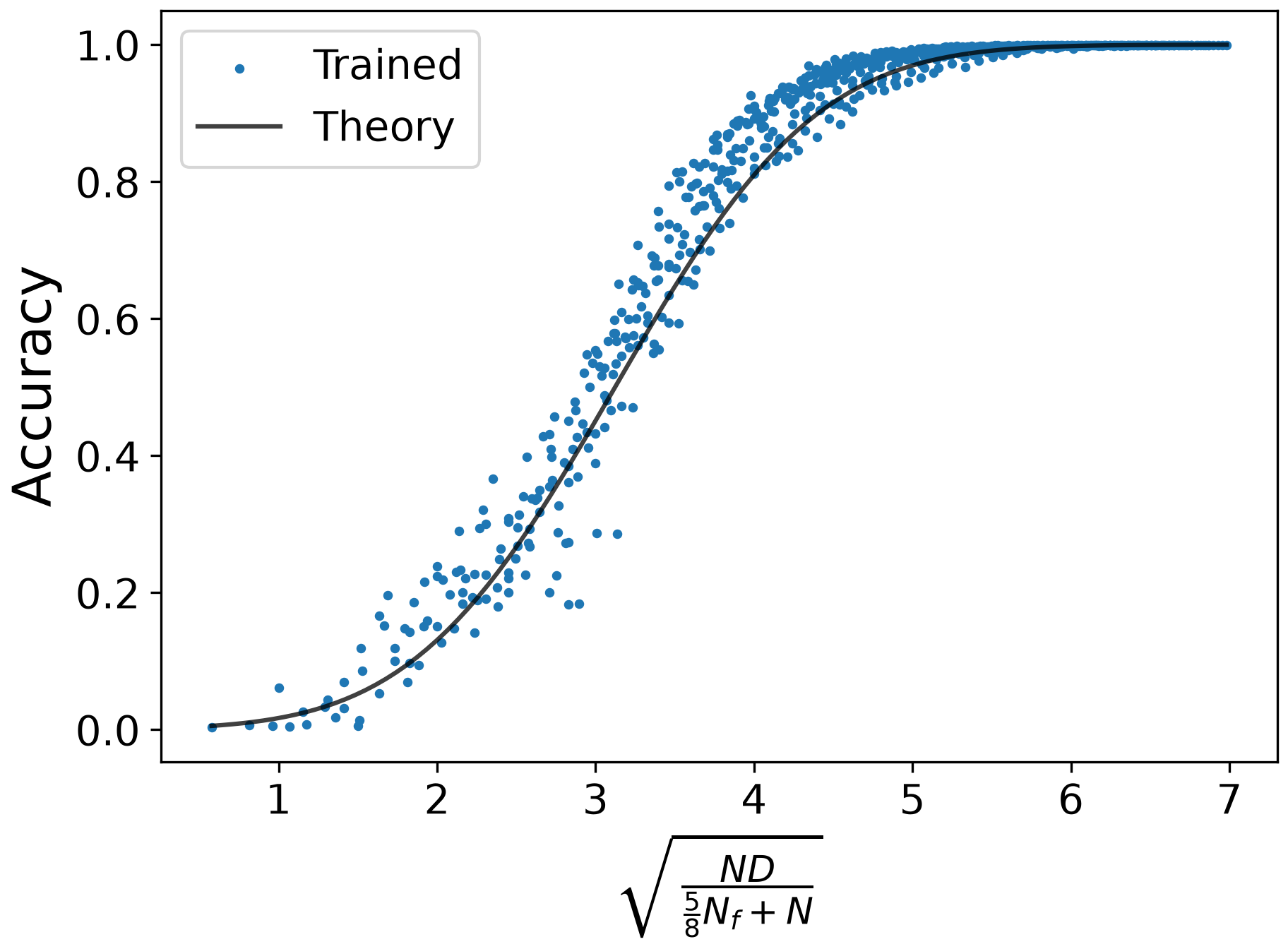}} \\
  \makecell{\textbf{(ii)} \\ $V=512$ \\ $L=128$ \\ $N_f=32$} &
  \raisebox{-0.5\height}{\includegraphics[width=0.30\textwidth]{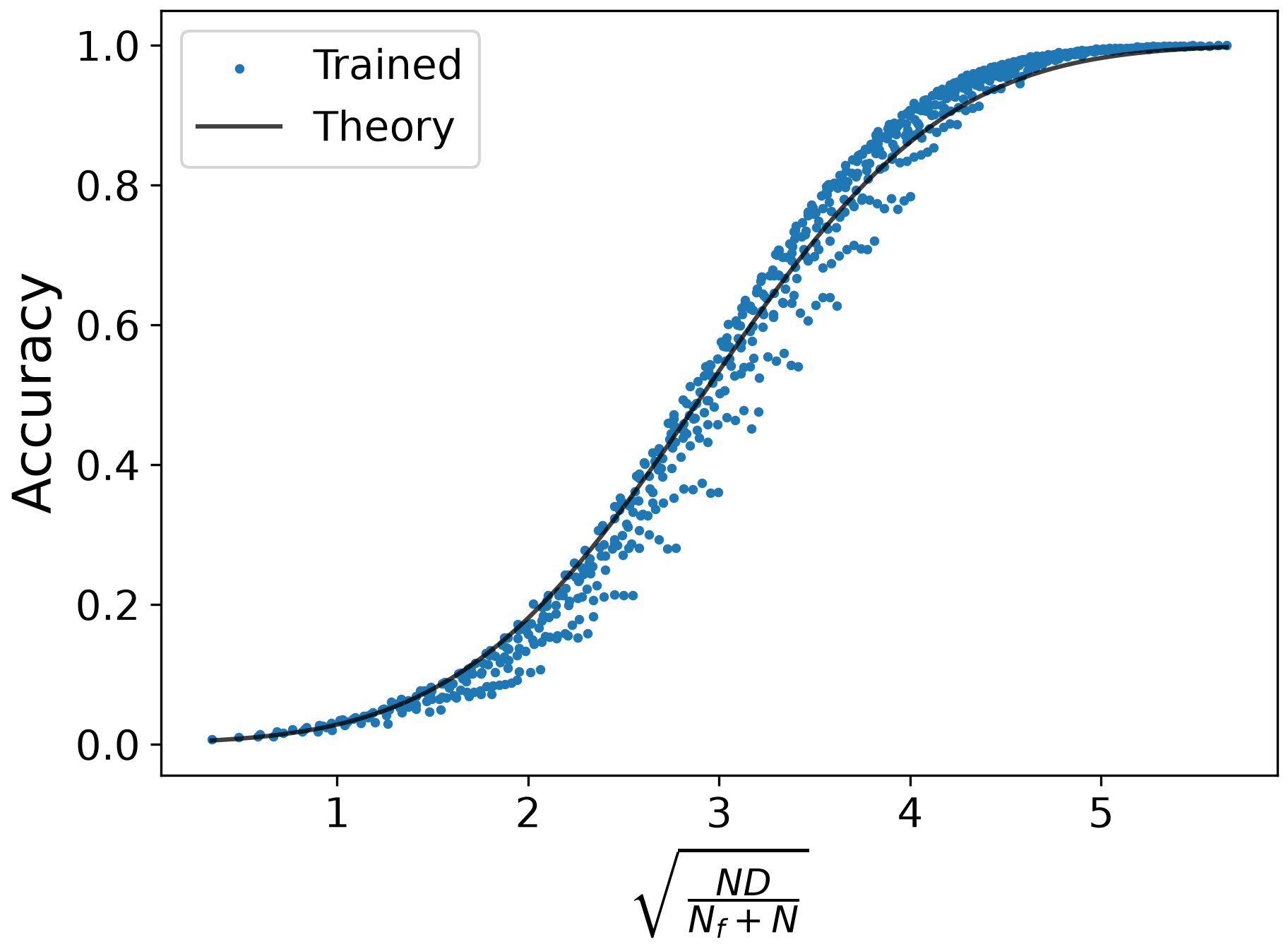}} &
  \raisebox{-0.5\height}{\includegraphics[width=0.30\textwidth]{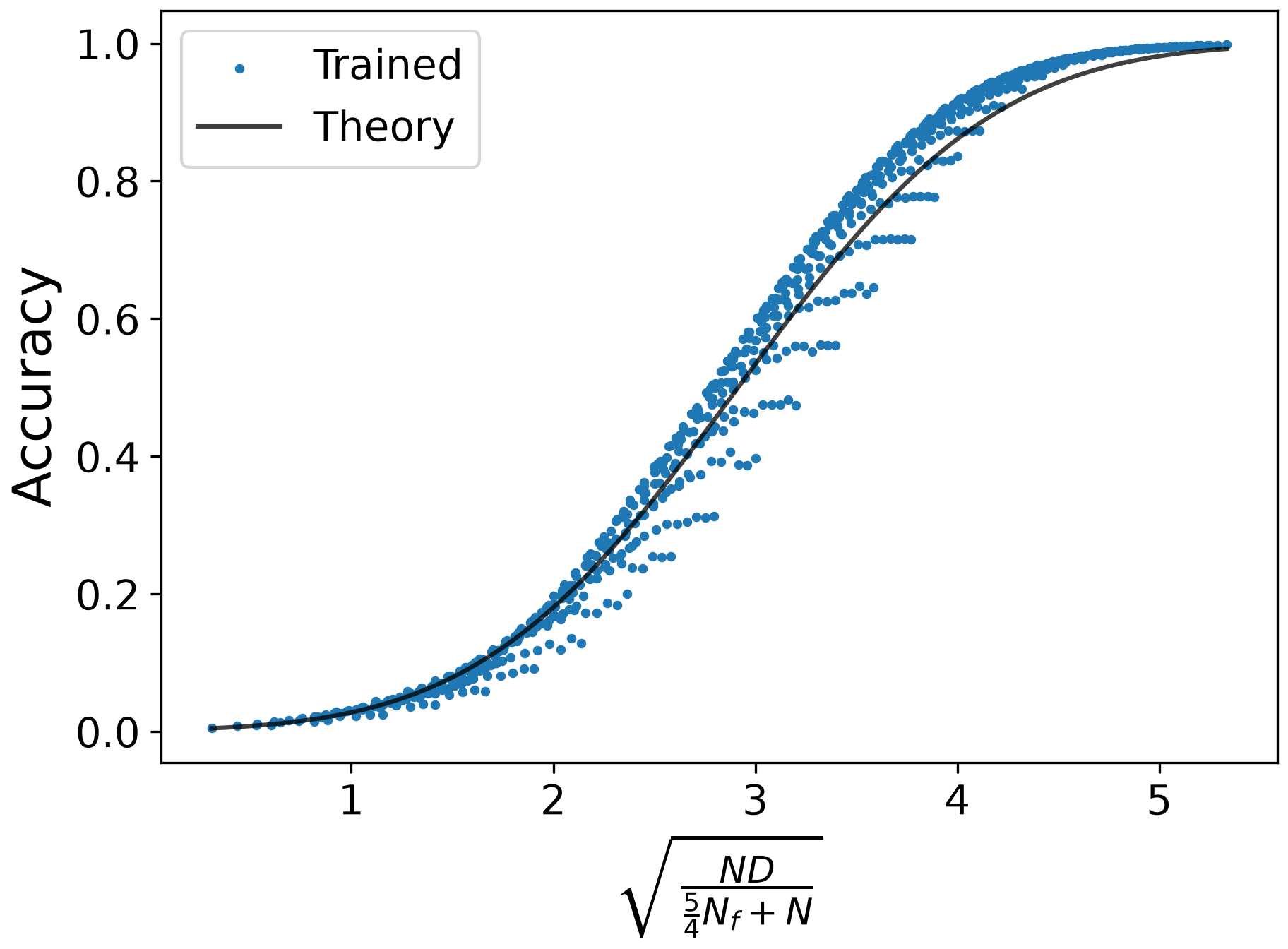}} &
  \raisebox{-0.5\height}{\includegraphics[width=0.30\textwidth]{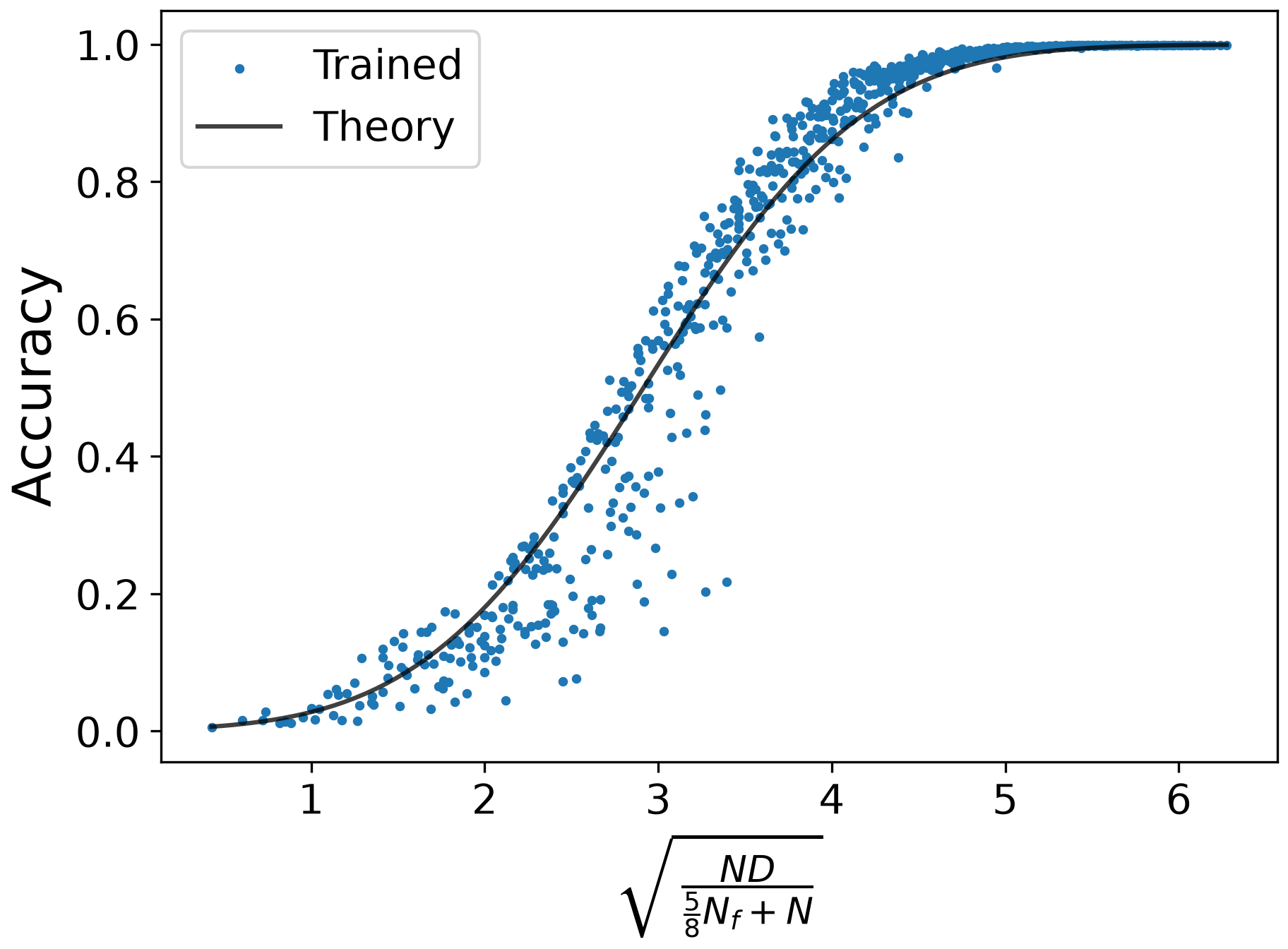}} \\
  \makecell{\textbf{(iii)} \\ $V=512$ \\ $L=64$ \\ $N_f=16$} &
  \raisebox{-0.5\height}{\includegraphics[width=0.30\textwidth]{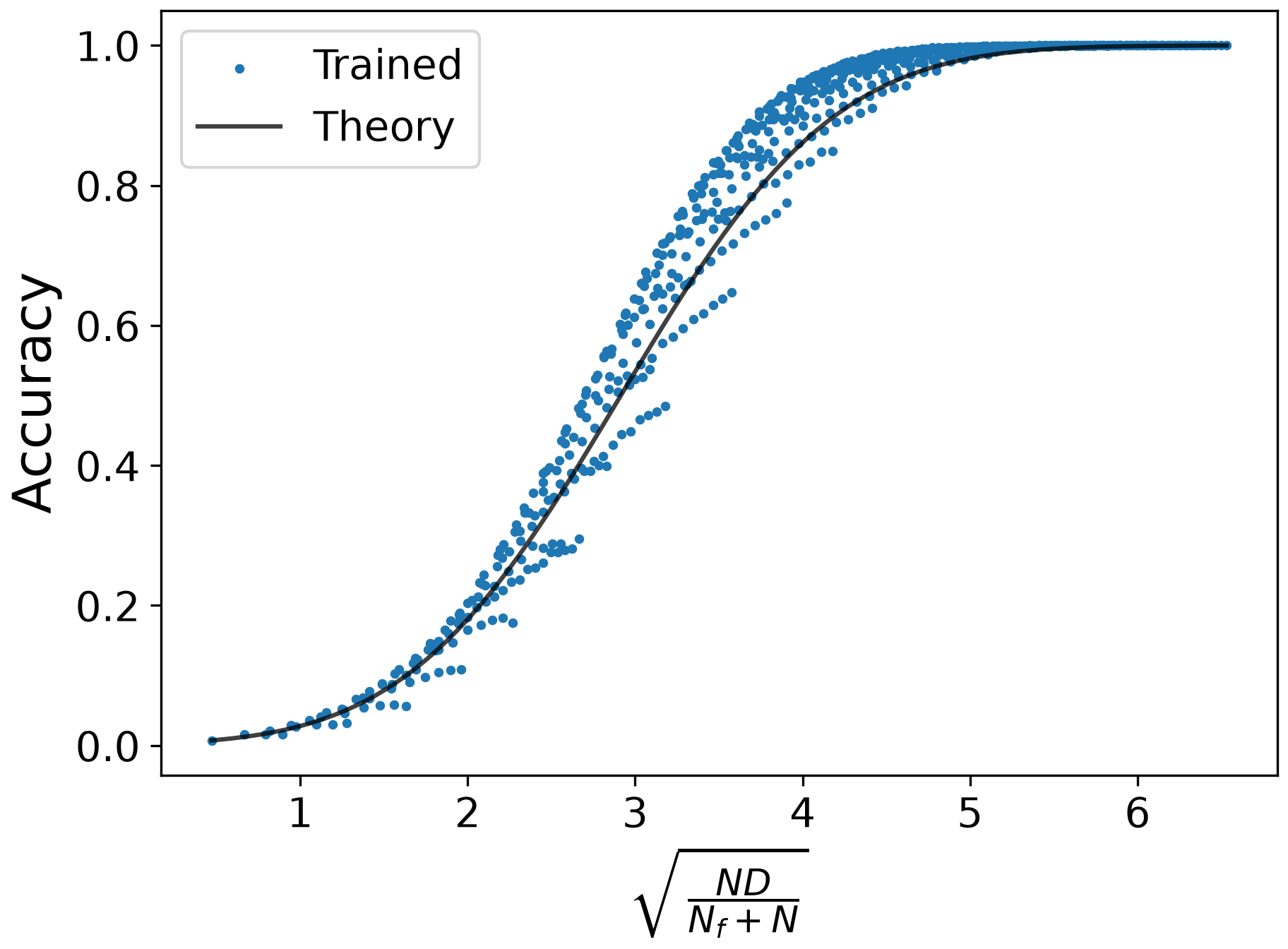}} &
  \raisebox{-0.5\height}{\includegraphics[width=0.30\textwidth]{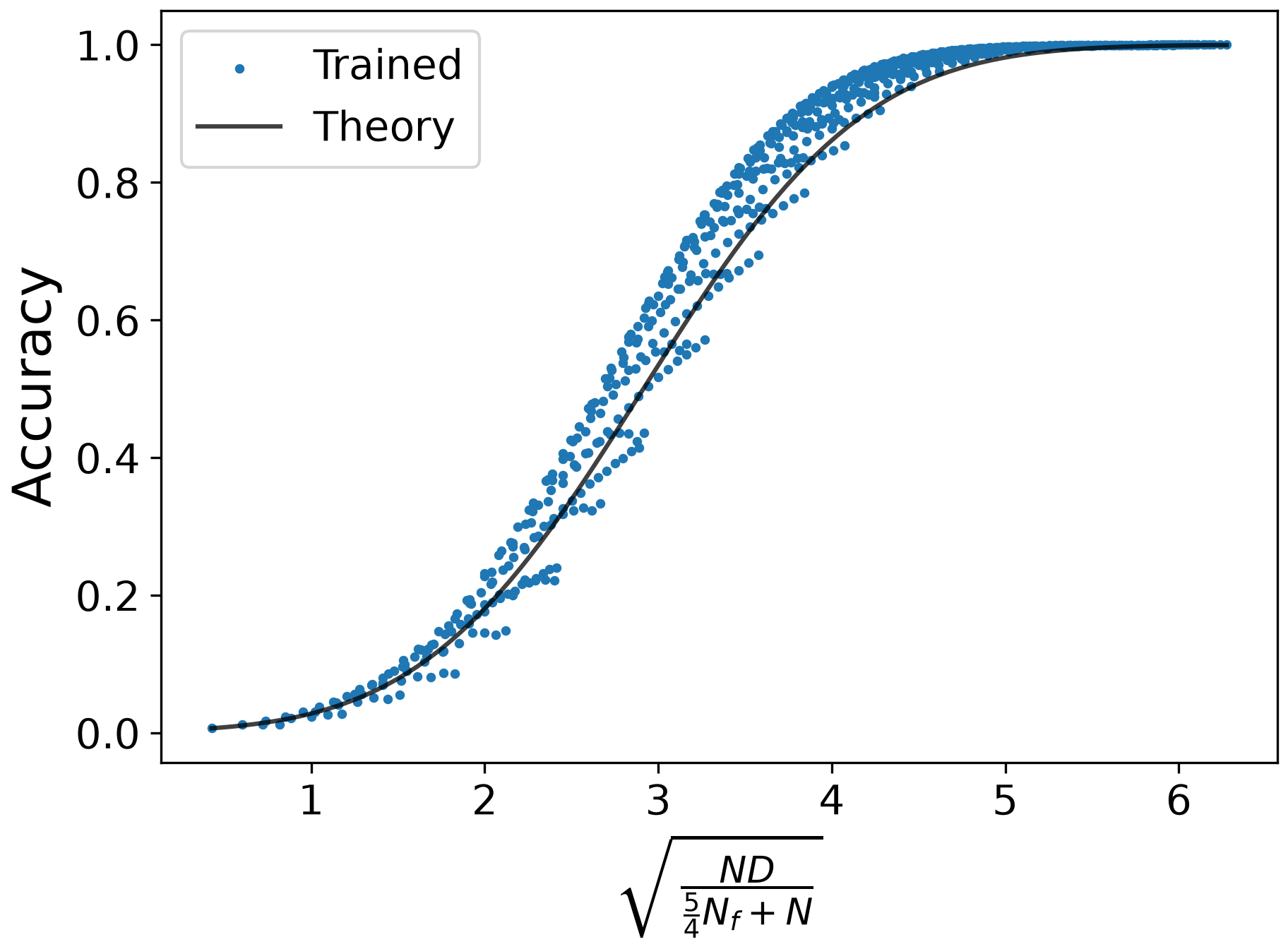}} &
  \raisebox{-0.5\height}{\includegraphics[width=0.30\textwidth]{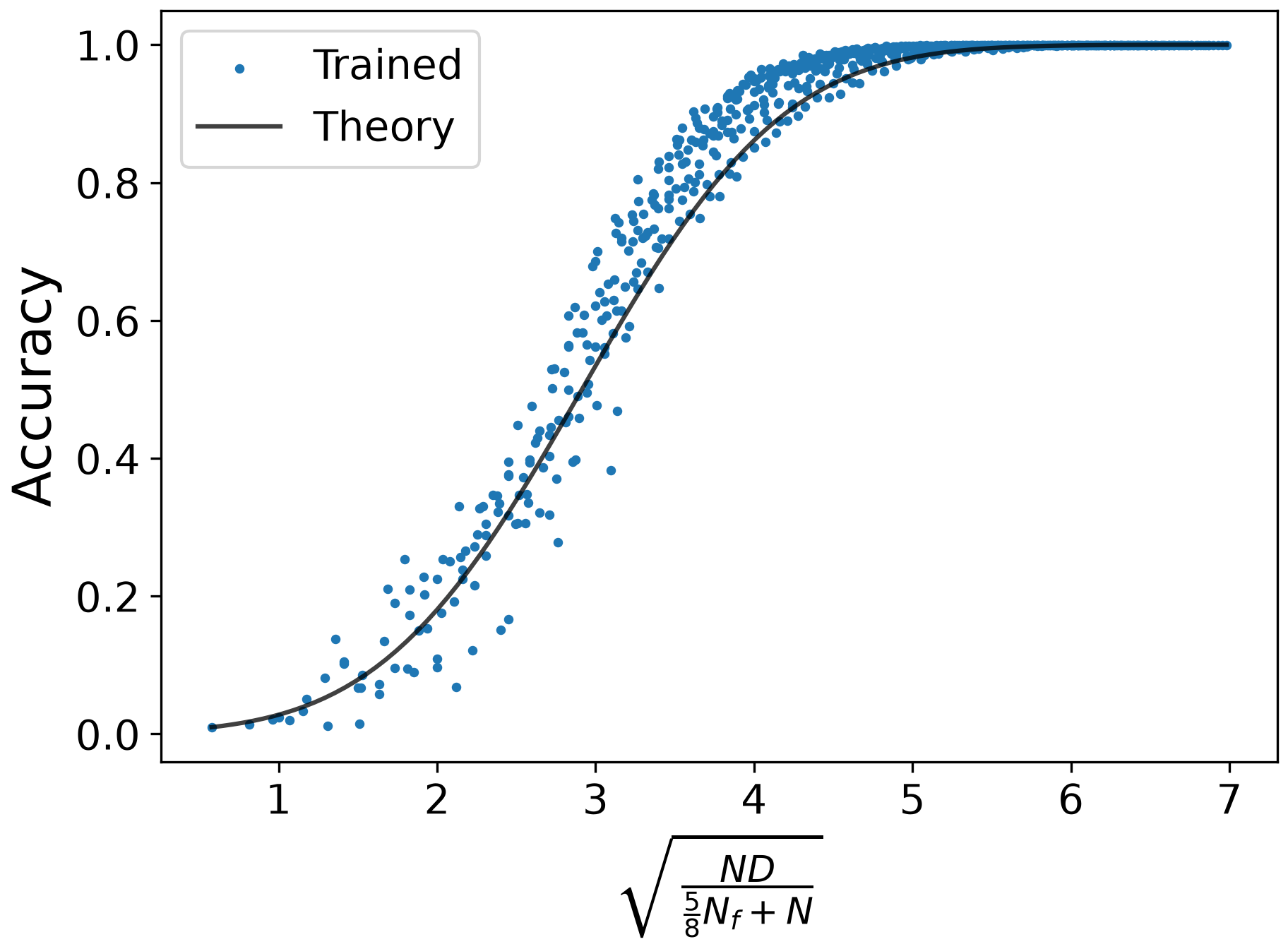}} \\
\end{tabular}
\caption{\textbf{Single-layer recall scaling curves across regimes.} The scaling curves of Fig.~\ref{fig:scaling-laws-2d-dims} (bottom), for the regimes of Fig.~\ref{app-fig:scaling-laws-2d-dims} (rows); row (i) repeats Fig.~\ref{fig:scaling-laws-2d-dims} (bottom). The scaling $p_\mathrm{recall}(x) \approx \Phi(x-b)$ with $b \approx \sqrt{2\log V}$ (Rem.~\ref{rem:practical-scaling-law}) holds across the different regimes, with $a = 1$, $a = \tfrac{5}{4}$, and $a = \tfrac{5}{8}$ (columns, left to right) consistent across regimes.}
\label{app-fig:scaling-curves-regimes}
\end{figure*}

\begin{figure*}[!ht]
    \centering
    \begin{minipage}[c]{0.035\textwidth}
        \centering
        \rotatebox{90}{\textbf{AR}}
    \end{minipage}%
    \begin{minipage}[c]{0.945\textwidth}
        \centering
        \includegraphics[width=0.55\textwidth]{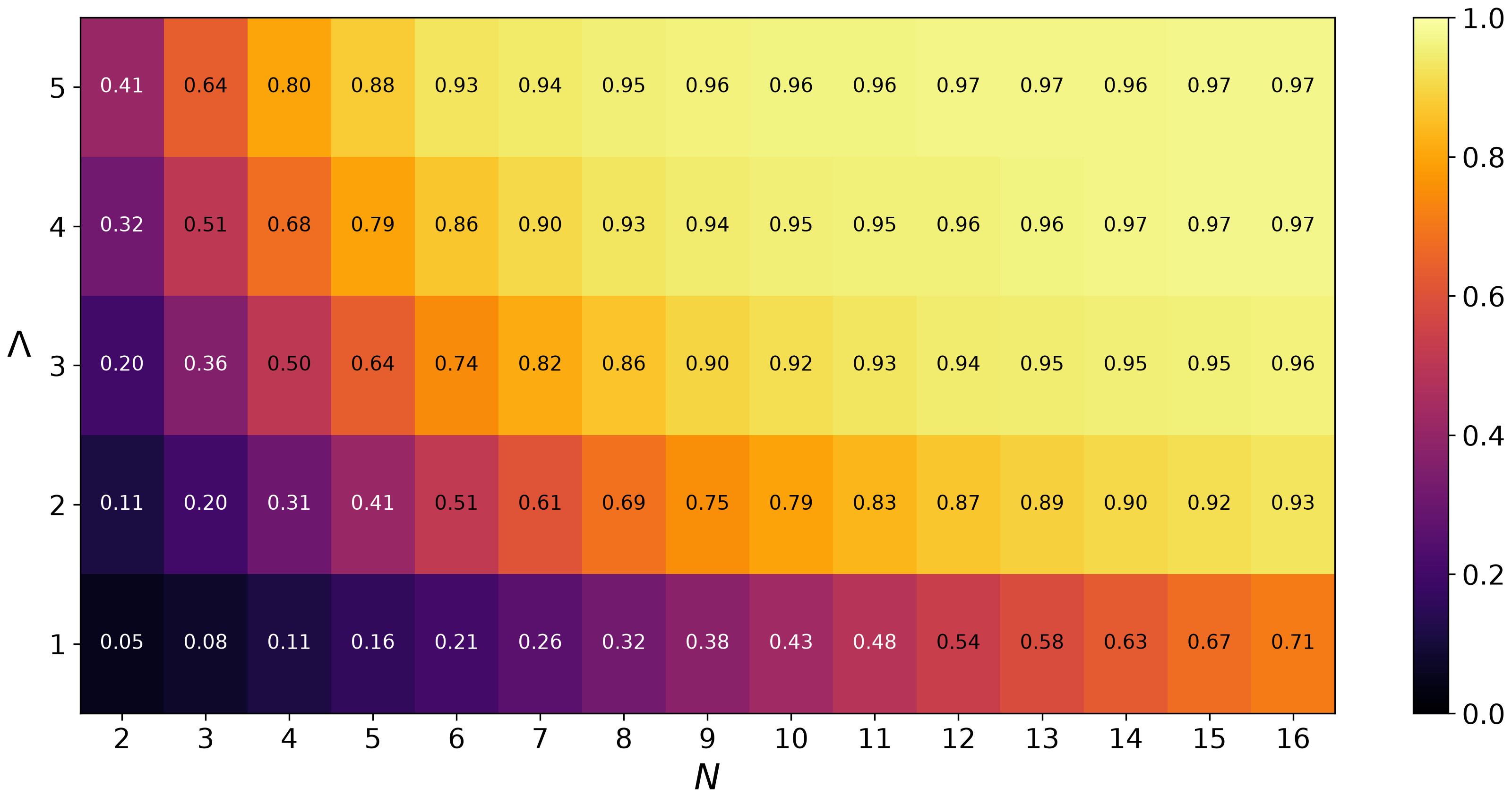}
        \hfill
        \includegraphics[width=0.38\textwidth]%
        {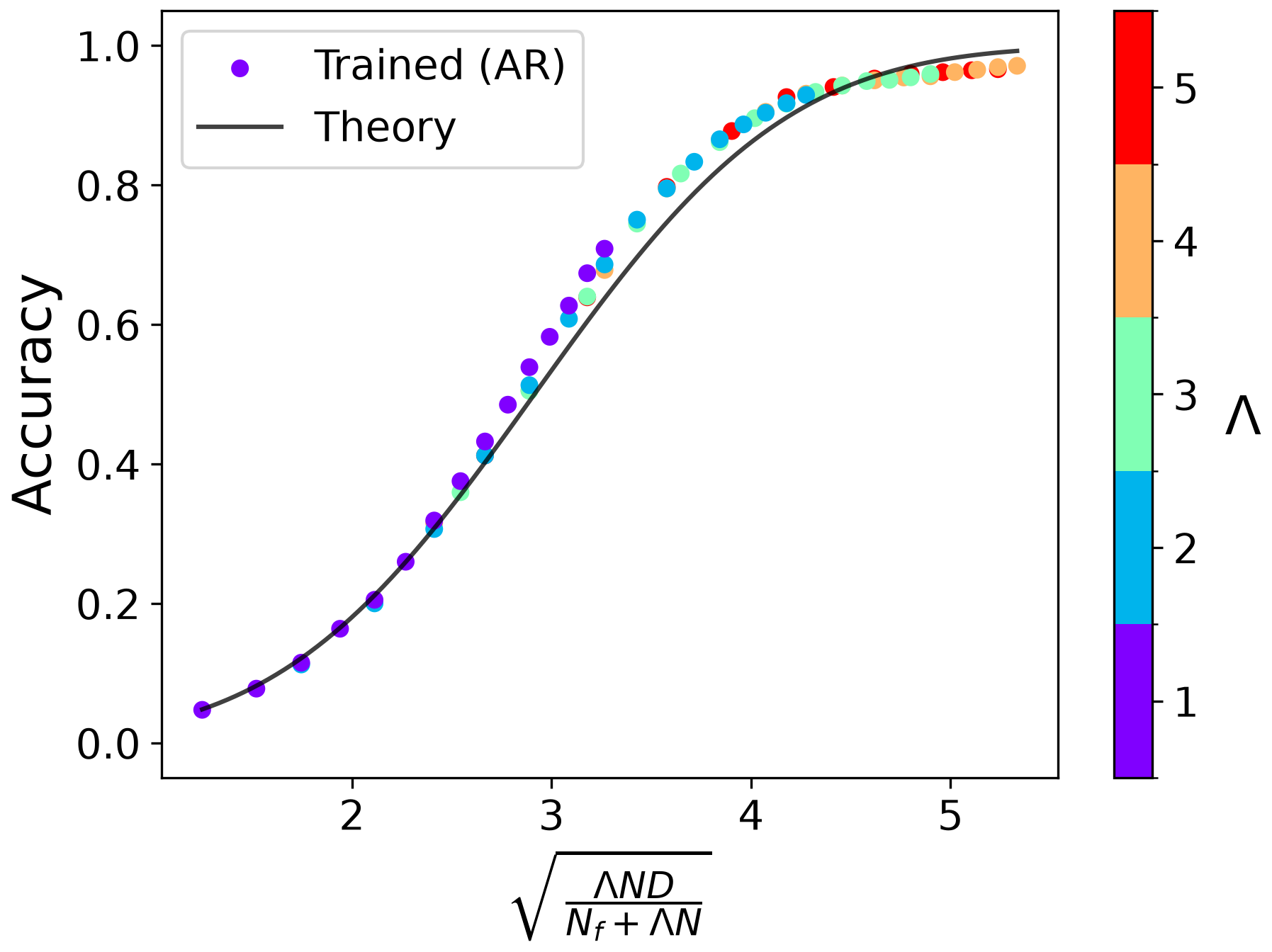}
    \end{minipage}

    \vspace{0.4em}

    \begin{minipage}[c]{0.035\textwidth}
        \centering
        \rotatebox{90}{\textbf{MQAR}}
    \end{minipage}%
    \begin{minipage}[c]{0.945\textwidth}
        \centering
        \includegraphics[width=0.55\textwidth]{figures/MQAR__N_Lambda__linear_trained.png}
        \hfill
        \includegraphics[width=0.38\textwidth]%
        {figures/appendix/MQAR__N_Lambda__linear_trained__acc_curve__sigma_inv.png}
    \end{minipage}

    \caption{
    \textbf{Multi-layer recall scaling laws.}
    The recall accuracy of a simplified multi-layer Mamba (up: AR, down: MQAR) is uniquely determined by its effective state size $N_\mathrm{eff}=\Lambda N$, rather than by each factor separately. \textbf{Left}: $N$-$\Lambda$ accuracy grids. Notice that performance is the same for $(N,\Lambda)$ pairs with the same $N_\mathrm{eff}$: for example, $p_\mathrm{MQAR} \approx 0.38$ for grid points $(12, 1),\, (6, 2),\, (4, 3)$ and $(3, 4)$. \textbf{Right}: For both AR and MQAR, we verify that empirical recall accuracy fits our theoretical prediction $p_\mathrm{recall}(x)\approx\Phi(x-b)$ from Thm.~\ref{thm:multi-layer-recall-scaling-laws}, where $x=\tfrac{1}{\sigma}$ and $b \approx \sqrt{2 \log{V}}$ as detailed in Thm.~\ref{thm:multi-layer-recall-scaling-laws}.
    }
    \label{app-fig:multi-layer-scaling-laws}
\end{figure*}

\FloatBarrier
\subsection{\texorpdfstring{Long Context and Large Vocabulary}{Long Context and Large Vocabulary}}
\label{app-subsec:large-vocab}

\begin{figure*}[!ht]
\centering
\includegraphics[width=0.70\textwidth]{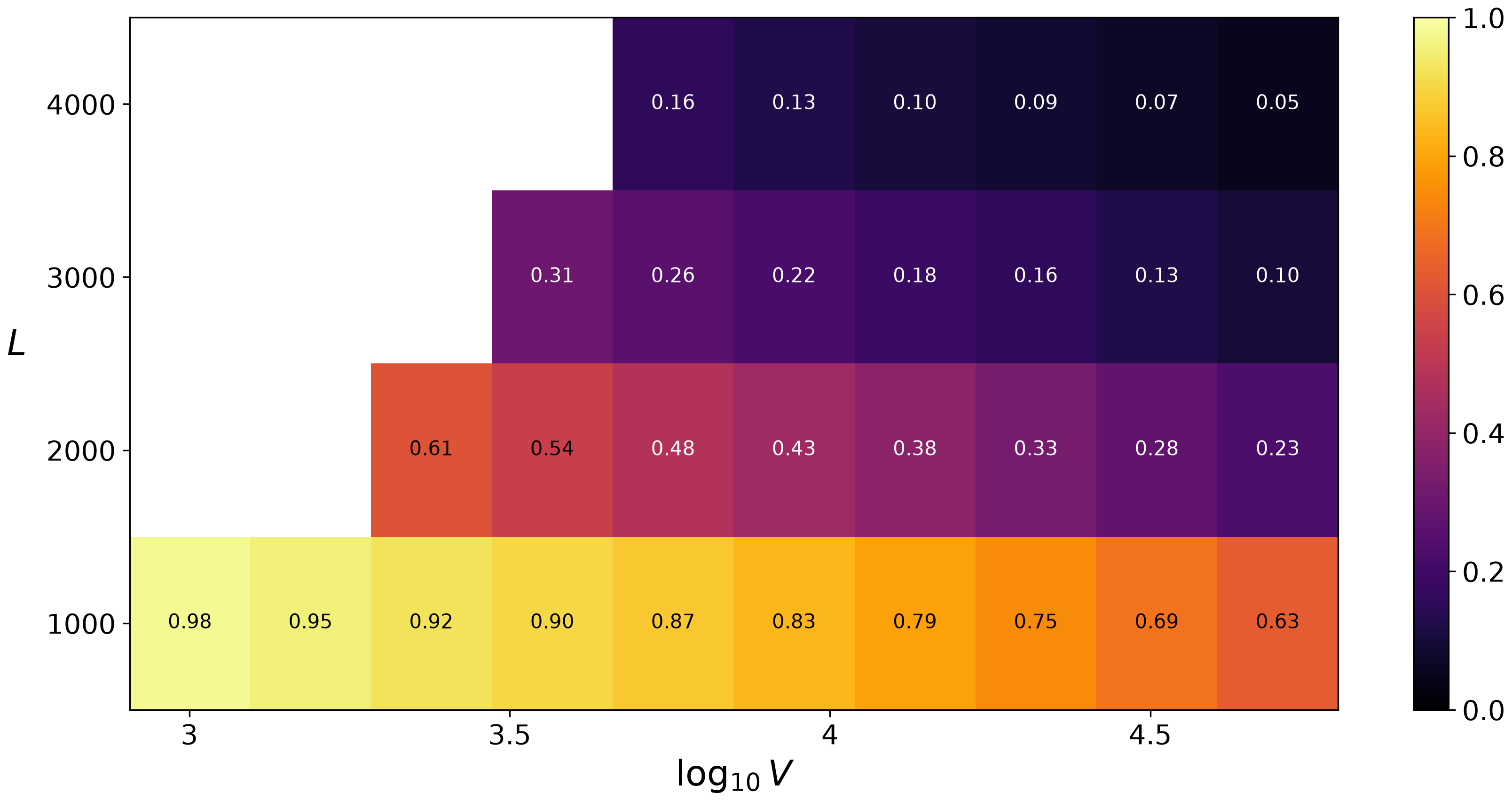}
\caption{\textbf{Larger scale experiments.} MQAR recall accuracy of a trained linear model ($D = 150$, $N = 75$) for vocabulary sizes $V = 1000$ to $V = 50000$ (logarithmically spaced) and context lengths $L = 1000$ to $L = 4000$, with $N_f = L/4$ key-value pairs.}
\label{app-fig:large-vocab}
\end{figure*}

\begin{figure*}[!ht]
\centering
\includegraphics[width=0.42\textwidth]{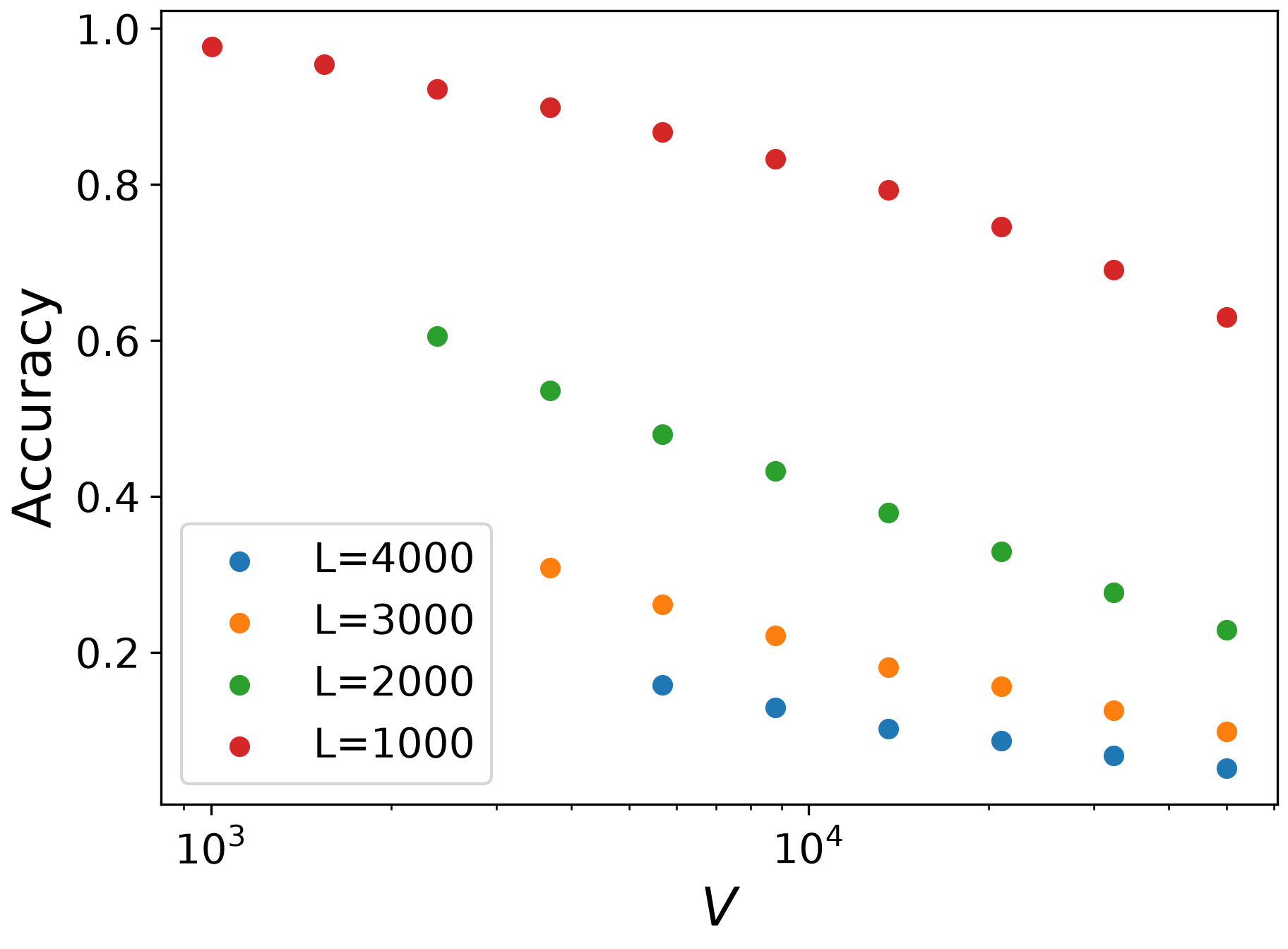}
\hspace{0.02\textwidth}
\includegraphics[width=0.42\textwidth]{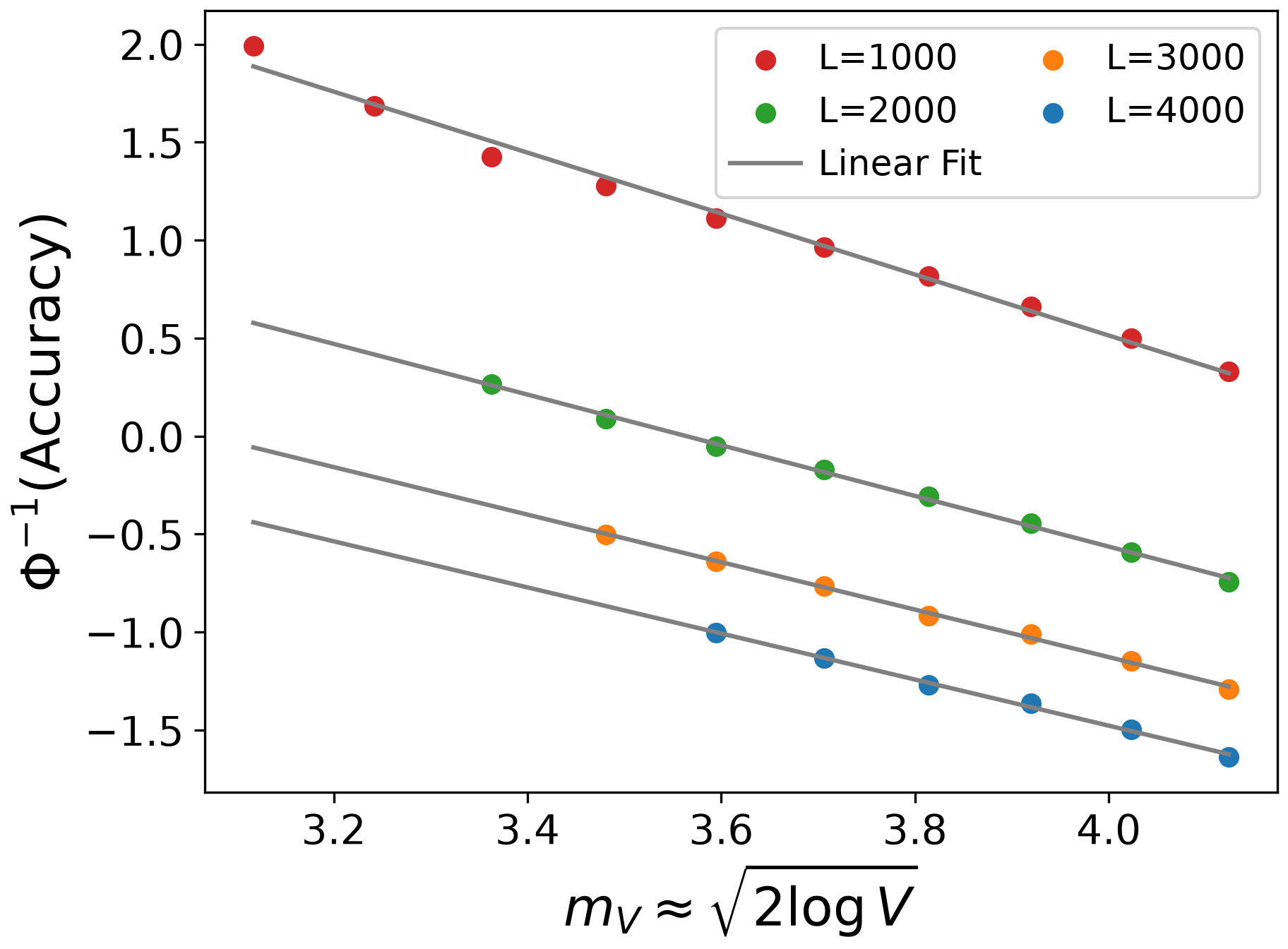}
\caption{\textbf{Larger scale experiments: scaling curves.} The accuracy of Fig.~\ref{app-fig:large-vocab} against vocabulary size (left), and its normal-inverse collapse against $m_V \approx \sqrt{2\log V}$ (see Eq.~\ref{app-eq:m-V}) (right), per context length $L$. The per-$L$ linear fits (right, gray lines) have slopes $1.18$-$1.56$, against a theoretical slope of $1$ (Rem.~\ref{rem:practical-scaling-law}). This makes sense: the theoretical scaling is an approximated relation.}
\label{app-fig:large-vocab-curves}
\end{figure*}

\FloatBarrier
\subsection{\texorpdfstring{Zero-Shot AR in Natural Language Pretrained Mamba}{Zero-Shot AR in Natural Language Pretrained Mamba}}
\label{app:zero_shot_ar}
We evaluate zero-shot AR in natural language on the pretrained Mamba-2 family~\citep{dao2024transformers} (370M, 780M, 1.3B and 2.7B parameters), using a simplified version of a protocol from~\citep{ben2025overflow}. Each prompt lists $N_f$ random key-value pairs (8-letter keys, 5-digit values) and ends with one of the keys; the greedy continuation is correct if it starts with the matching value. An example with $N_f=10$, with highlighted \f{key}, \w{value} and \q{query} tokens:
\smallskip

\begin{center}
\small
\parbox{0.95\linewidth}{\raggedright\texttt{Complete the following sequence: \f{bxwgemhg}=\w{64304}, \f{qcfhfxbi}=\w{12277}, \f{ykgtjfjo}=\w{64826}, \f{gmcamfhj}=\w{91471}, \f{vgpdjtrp}=\w{14082}, \f{fvsvvjyl}=\w{99729}, \f{ejhmfjrp}=\w{16832}, \f{whscazat}=\w{02116}, \f{iqkgfmpc}=\w{87547}, \f{qqcbobqf}=\w{67650}, \q{gmcamfhj}=}}
\end{center}
with expected output \texttt{\w{91471}}. Accuracy is averaged over 32 queries per context and 30 random contexts per $N_f$. Fig.~\ref{app-fig:zero-shot-ar} shows the results.

\begin{figure*}[!ht]
\centering
\includegraphics[width=0.5\textwidth]{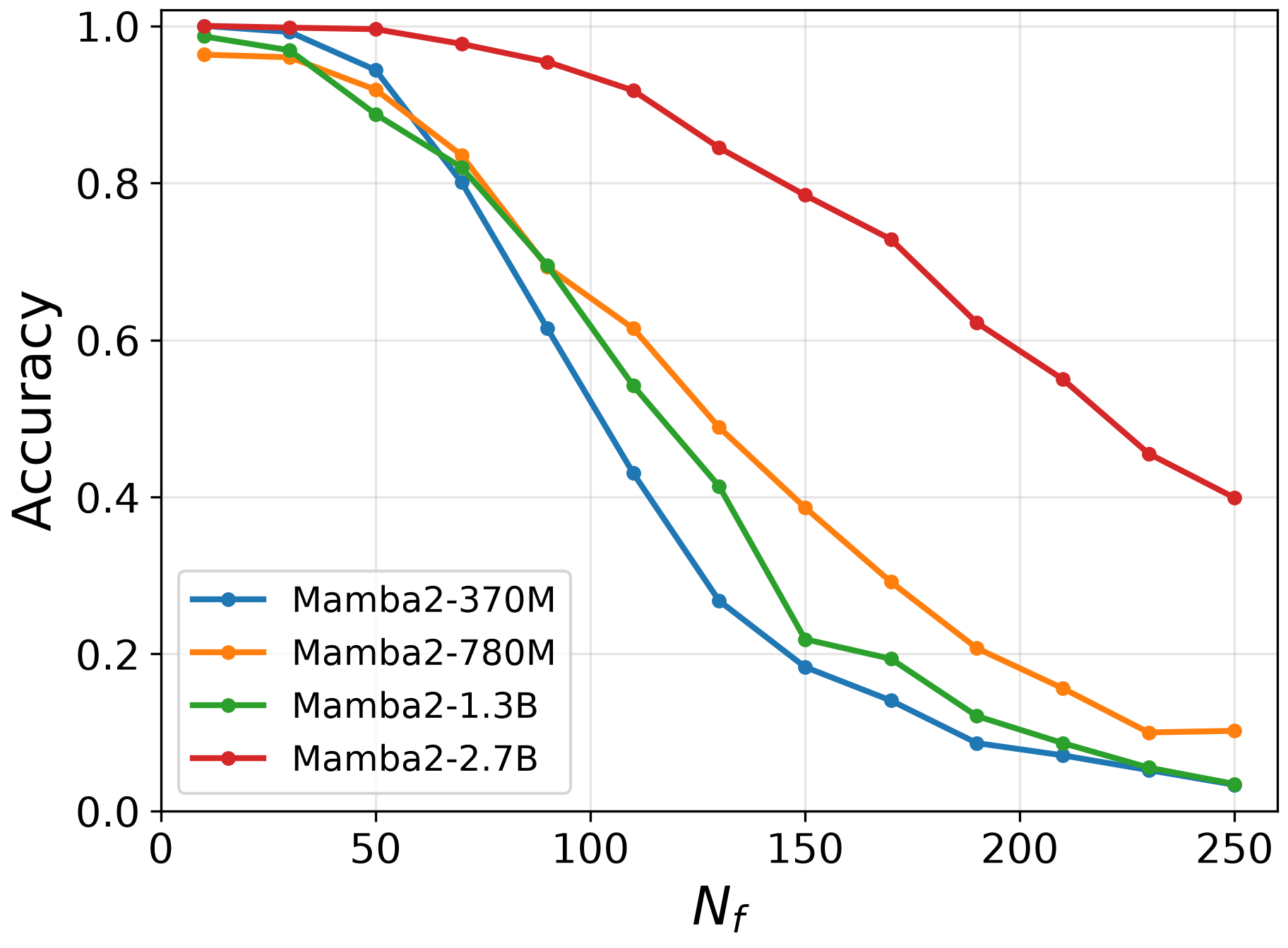}
\caption{\textbf{Zero-shot associative recall in pretrained Mamba-2 models.} Recall accuracy against the number of facts $N_f$ in the prompt. Each point is the mean success over 30 random contexts. The trend, recall degrades with the number of facts, is qualitatively similar to what we prove and measure for the synthetic AR tasks and AR-trained Mamba models. The largest model has the highest accuracy, as expected (while the smaller ones are not consistently ordered by size).}
\label{app-fig:zero-shot-ar}
\end{figure*}

\FloatBarrier
\subsection{\texorpdfstring{Reverse Engineering Mamba on MQAR}{Reverse Engineering Mamba on MQAR}}
\label{app-subsec:reverse-engineering}

\begin{figure*}[!ht]
\centering
\small
\begin{tabular}{c@{\hskip 0.04in}c@{\hskip 0.04in}c}
  $H_t$ & $H'_t$ & $H''_t$ \\
  \makebox[0pt][r]{\makecell{$D{=}64$ \\ $N{=}32$}\hspace{0.08in}}\raisebox{-0.5\height}{\includegraphics[width=0.30\textwidth]{figures/mech_interp__D64_N32__hidden_state__H_ideal.png}} &
  \raisebox{-0.5\height}{\includegraphics[width=0.17\textwidth]{figures/mech_interp__D64_N32__hidden_state__H_noisy.png}} &
  \raisebox{-0.5\height}{\includegraphics[width=0.30\textwidth]{figures/mech_interp__D64_N32__hidden_state__H_projected.png}} \\
  \makebox[0pt][r]{\makecell{$D{=}64$ \\ $N{=}24$}\hspace{0.08in}}\raisebox{-0.5\height}{\includegraphics[width=0.30\textwidth]{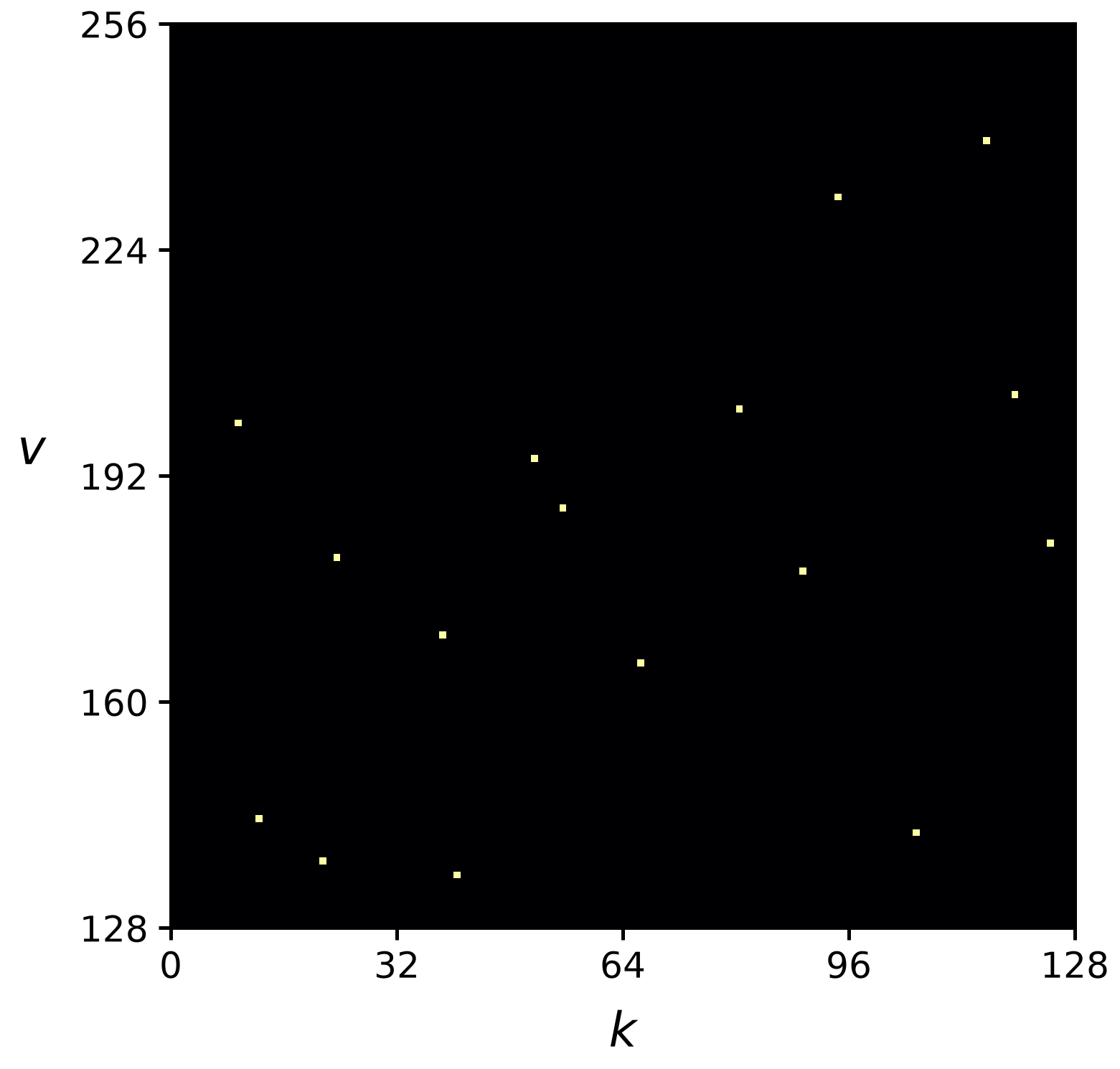}} &
  \raisebox{-0.5\height}{\includegraphics[width=0.138\textwidth]{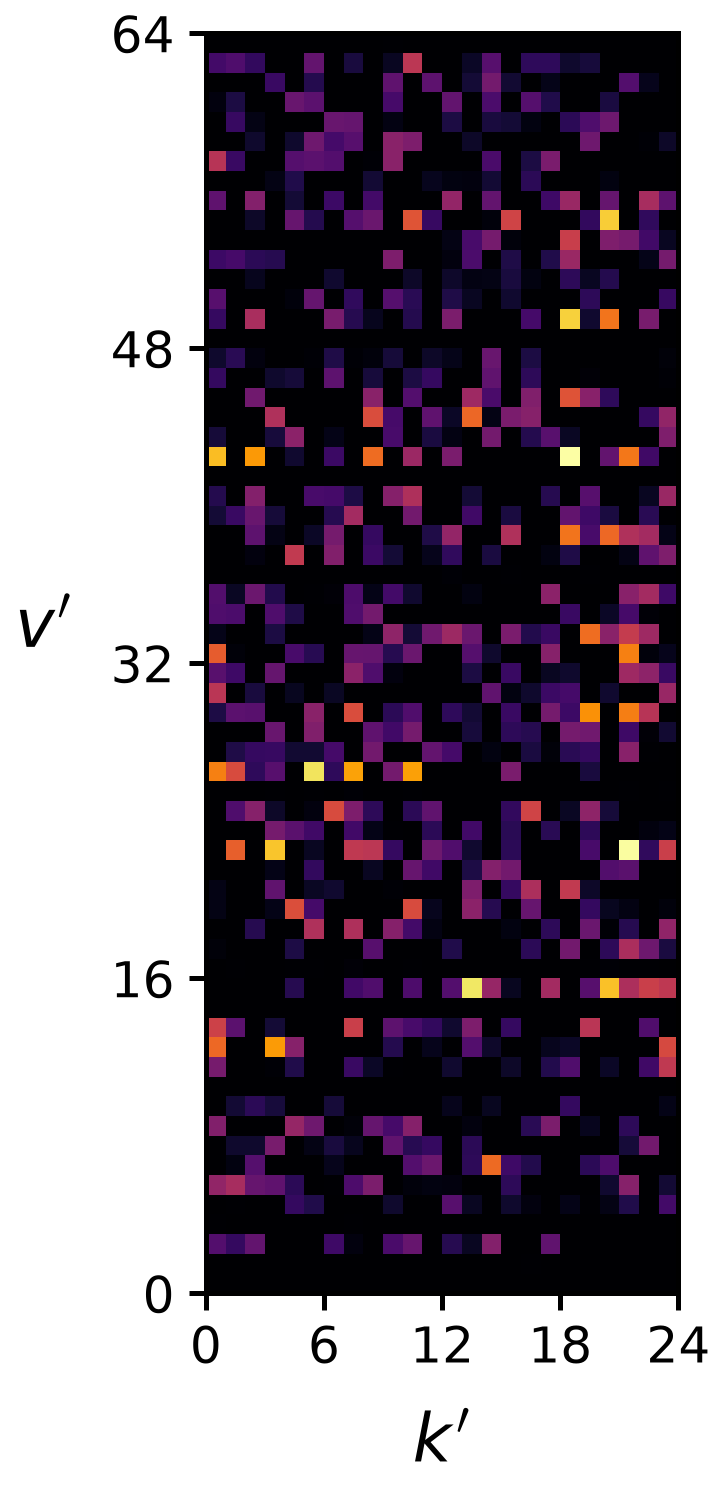}} &
  \raisebox{-0.5\height}{\includegraphics[width=0.30\textwidth]{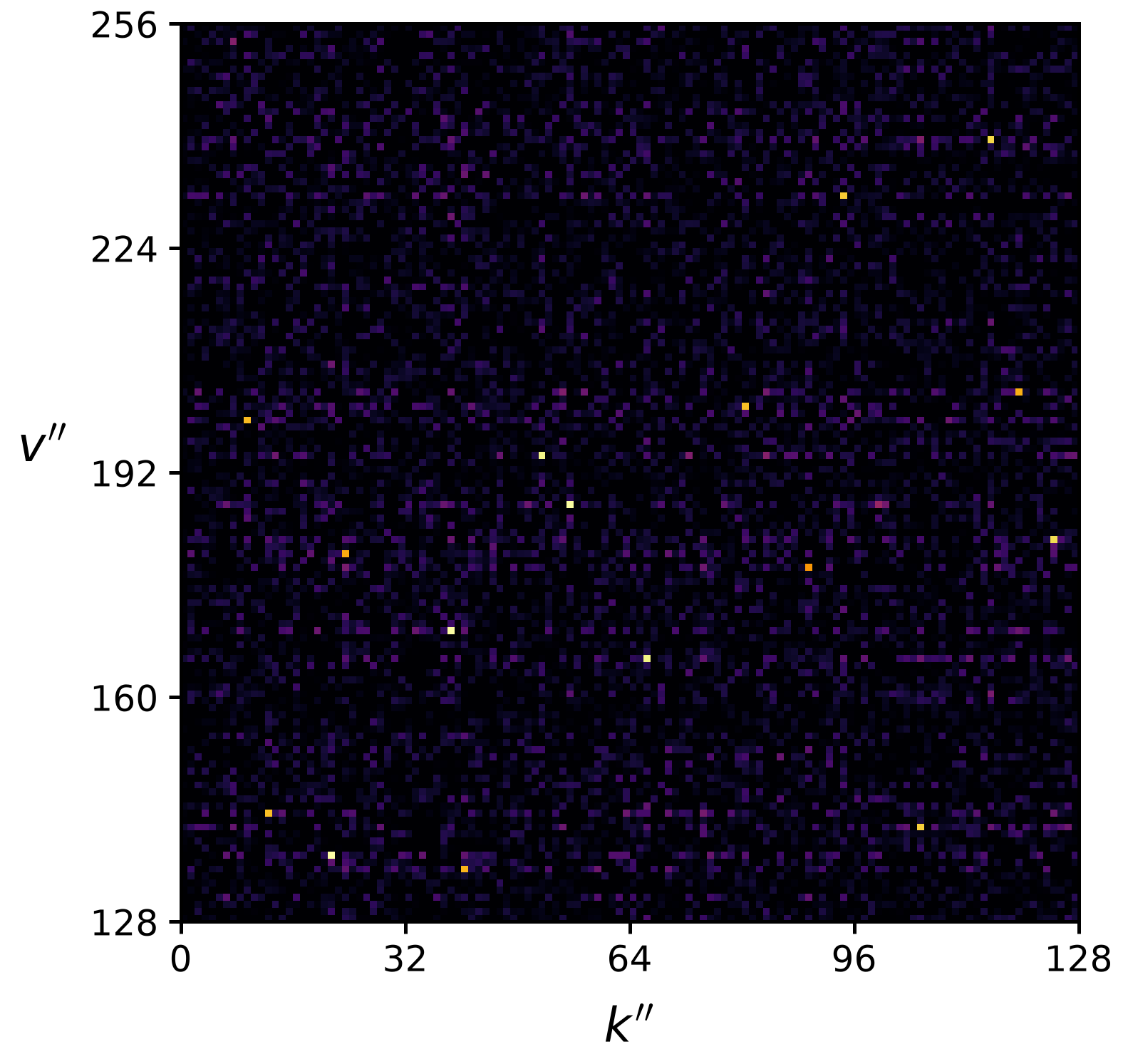}} \\
  \makebox[0pt][r]{\makecell{$D{=}48$ \\ $N{=}32$}\hspace{0.08in}}\raisebox{-0.5\height}{\includegraphics[width=0.30\textwidth]{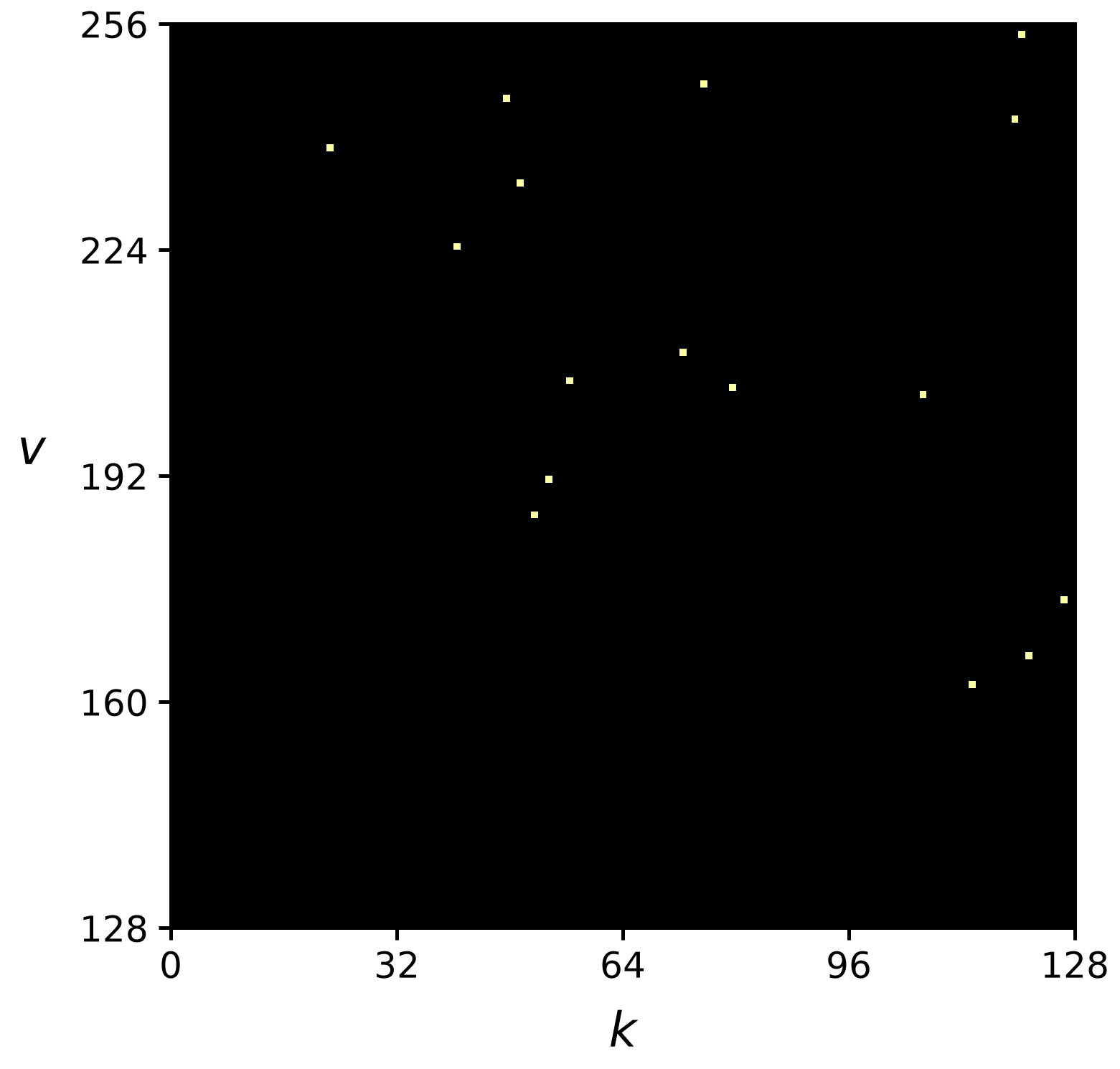}} &
  \raisebox{-0.5\height}{\includegraphics[width=0.208\textwidth]{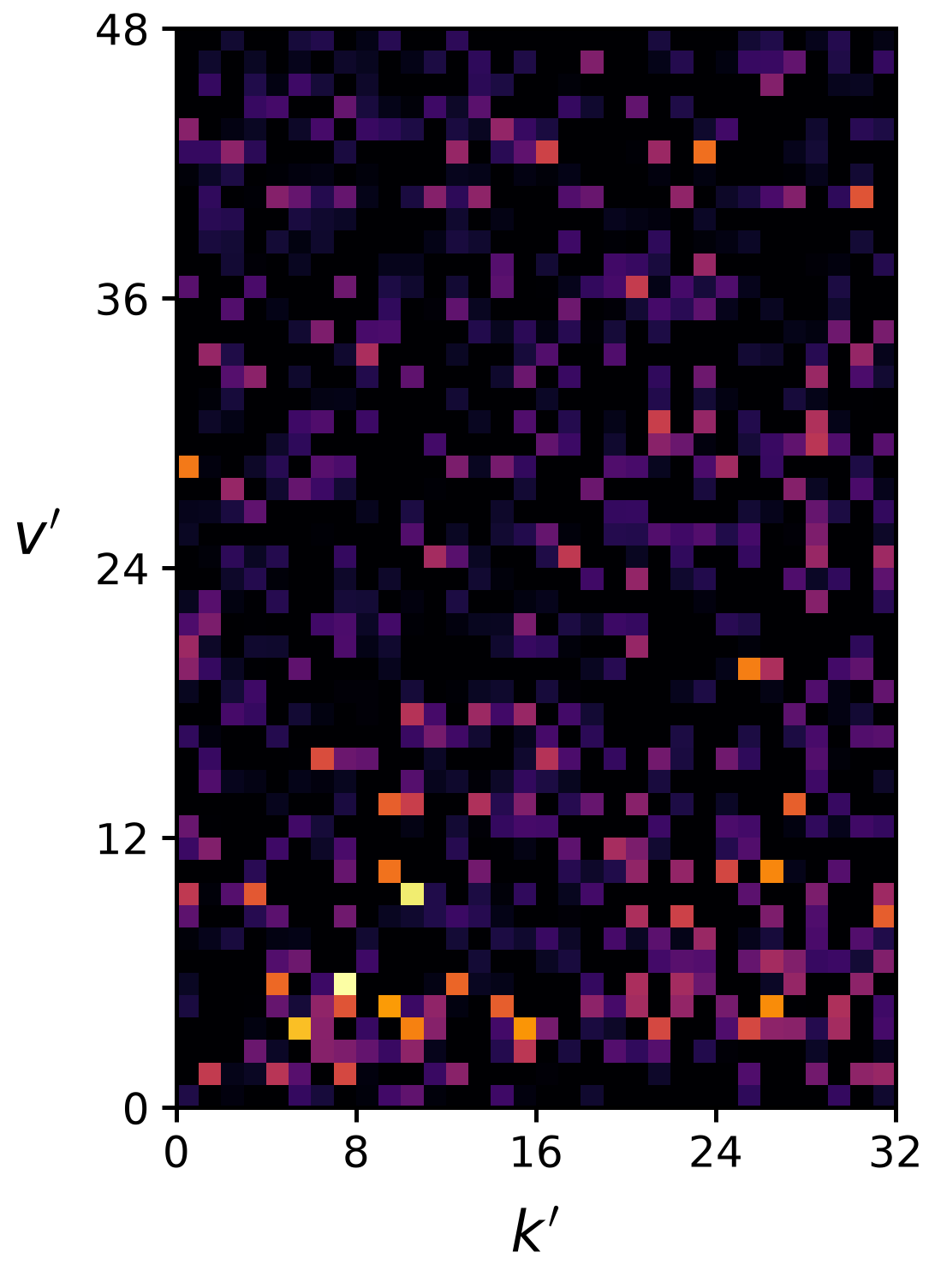}} &
  \raisebox{-0.5\height}{\includegraphics[width=0.30\textwidth]{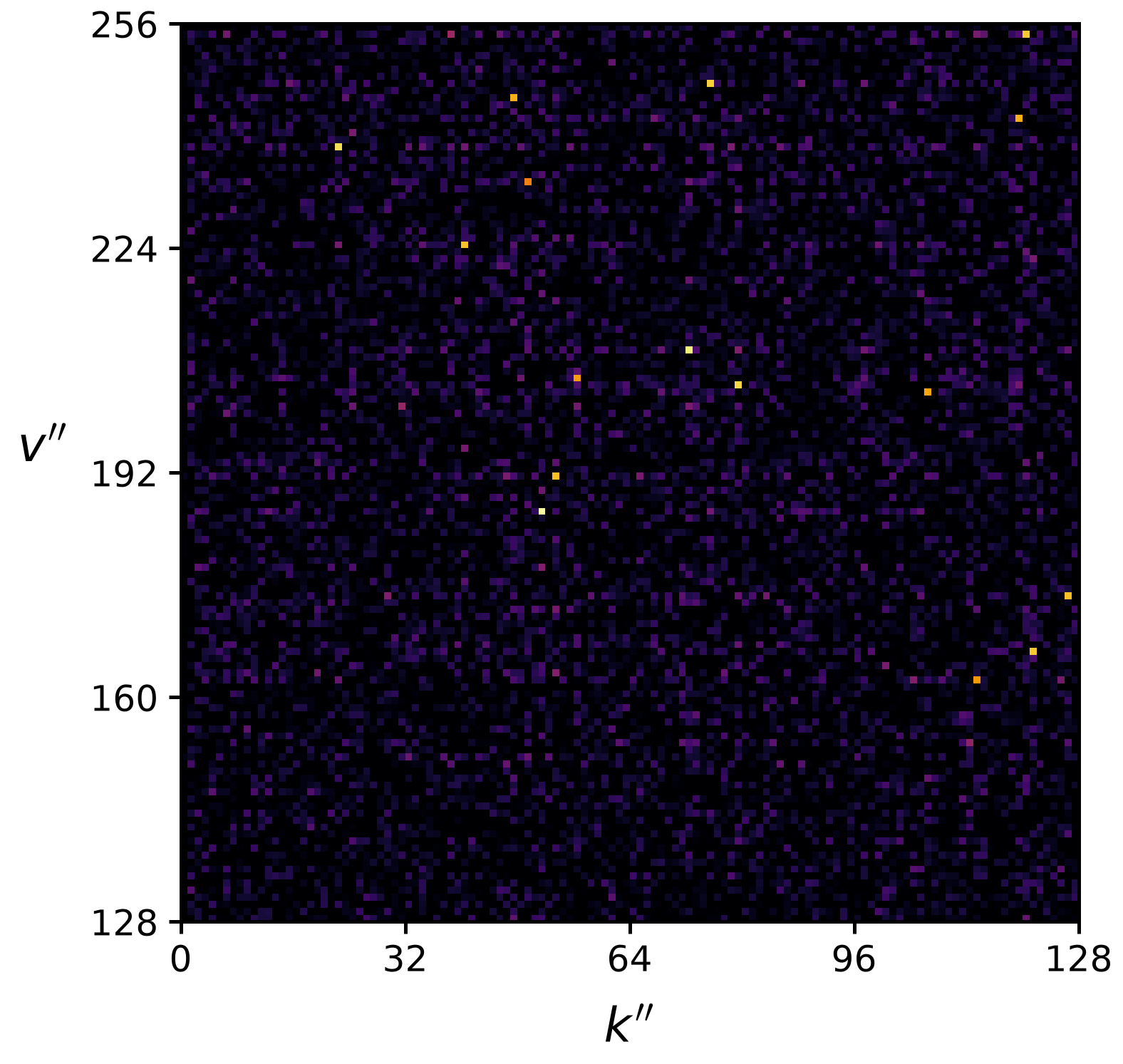}} \\
\end{tabular}
\caption{\textbf{Interpreting the hidden state.} The context key-value information $H_t$ (left) is compressed by the model into a hidden state $H'_t$ (middle), which lacks clear interpretability. As described in Alg.~\ref{alg:hidden_state_inversion}, projection onto vocabulary space (right) reveals an interpretable pattern: $H''_t \approx H_t$, indicating that a compression-decompression scheme is learned. Rows: same analysis repeated for trained models with varying $D$, $N$. Note that $H$ and $H''$ are $V{\times}V$, while $H'$ is $D{\times}N$.}
\label{app-fig:reversing-hidden}
\end{figure*}

\begin{figure*}[!ht]
\centering
\small
\begin{tabular}{c@{\hskip 0.05in}c}
  $\alpha_{\tau t}$ & $\alpha'_{\tau t}$ \\
  \makebox[0pt][r]{\makecell{$D{=}64$ \\ $N{=}32$}\hspace{0.08in}}\raisebox{-0.5\height}{\includegraphics[width=0.28\textwidth]{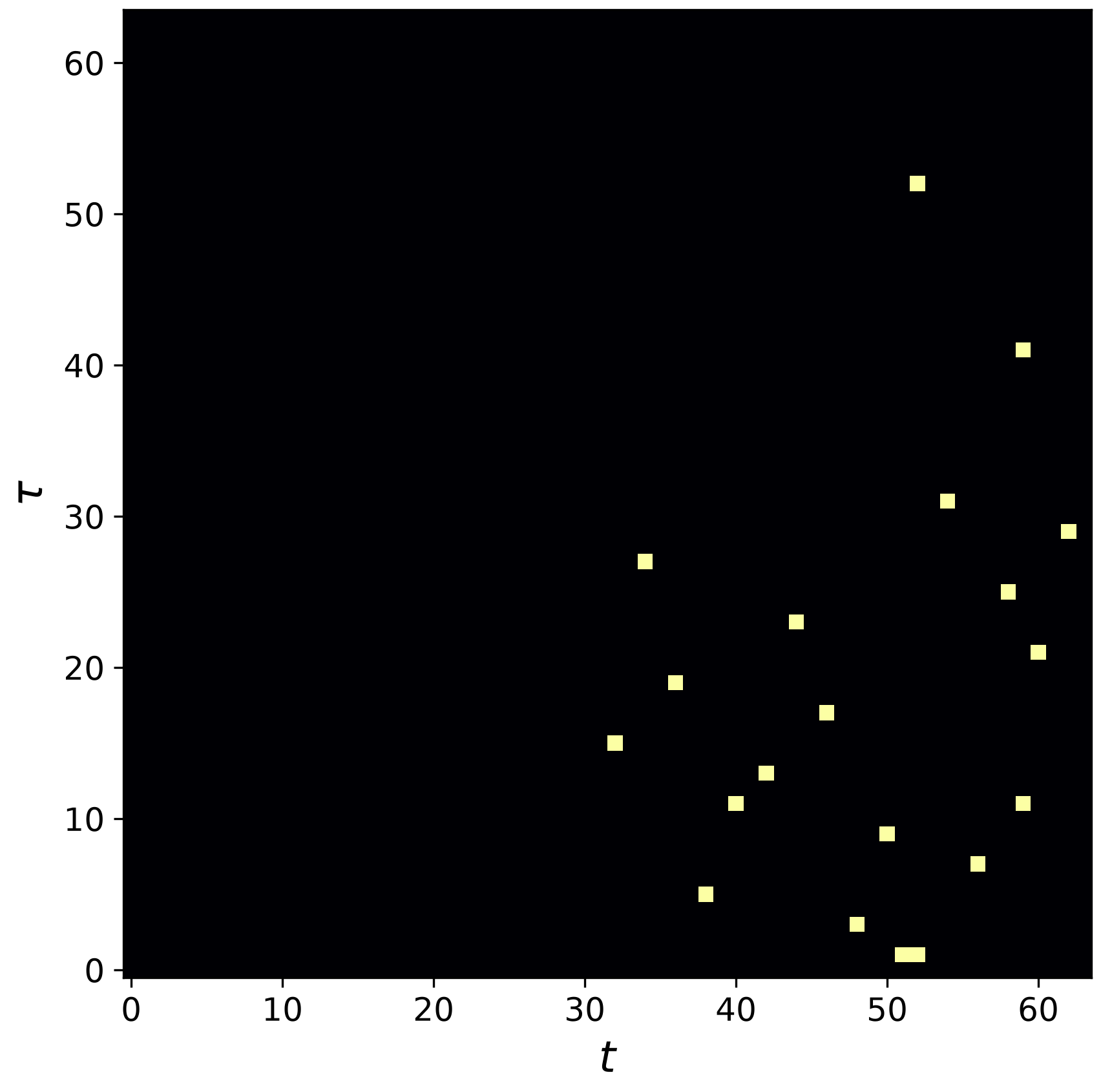}} &
  \raisebox{-0.5\height}{\includegraphics[width=0.28\textwidth]{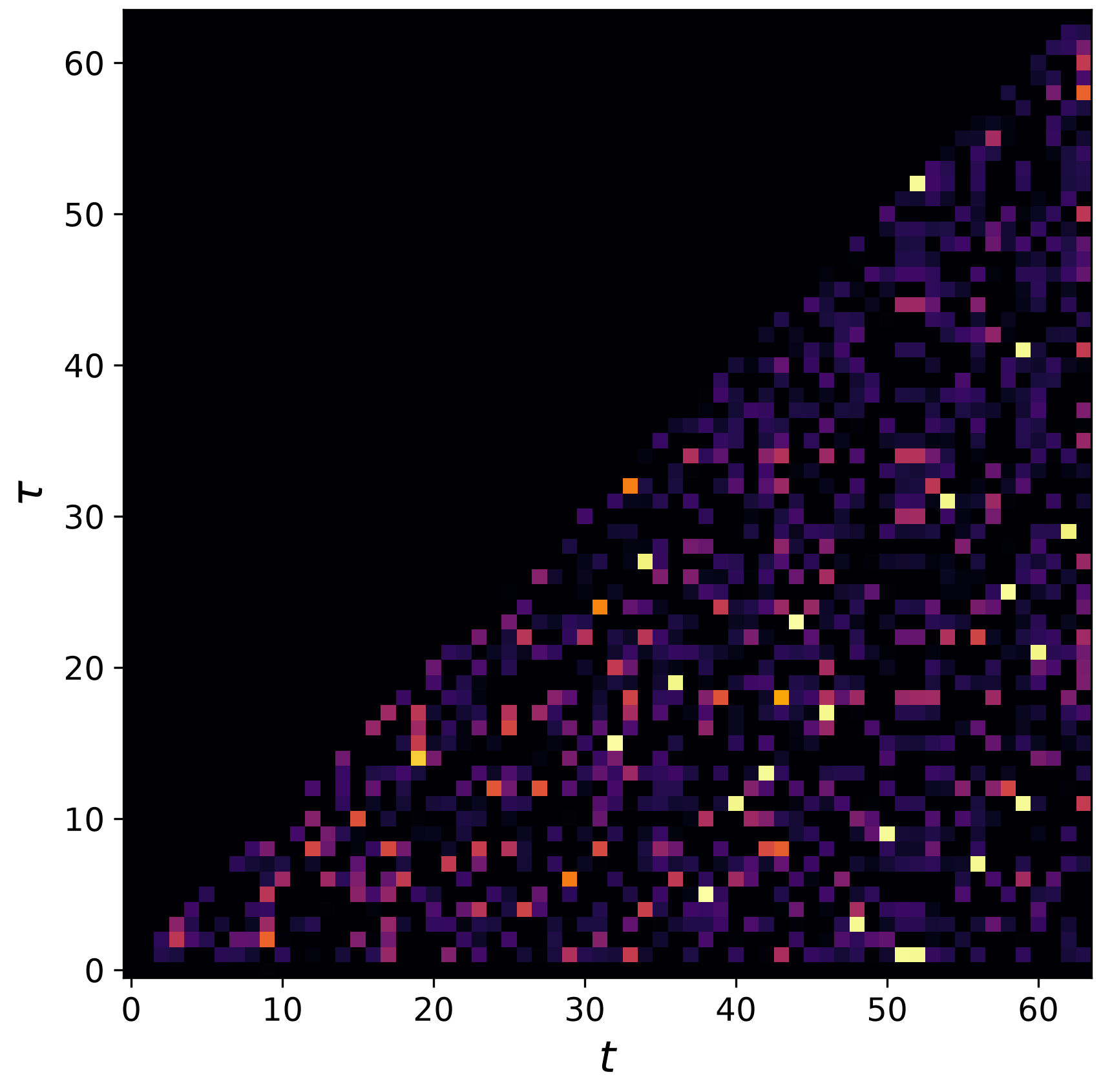}} \\
  \makebox[0pt][r]{\makecell{$D{=}64$ \\ $N{=}24$}\hspace{0.08in}}\raisebox{-0.5\height}{\includegraphics[width=0.28\textwidth]{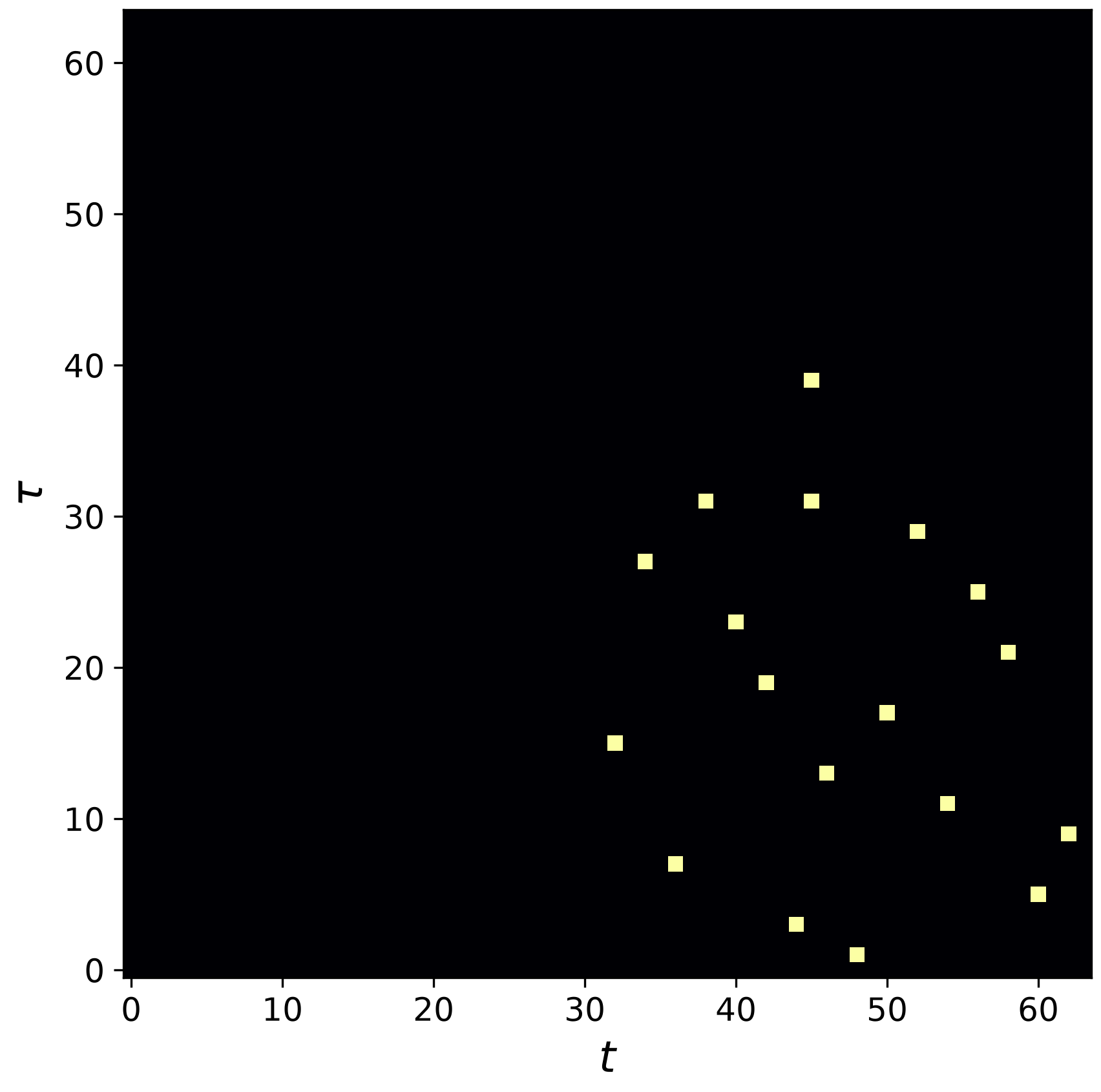}} &
  \raisebox{-0.5\height}{\includegraphics[width=0.28\textwidth]{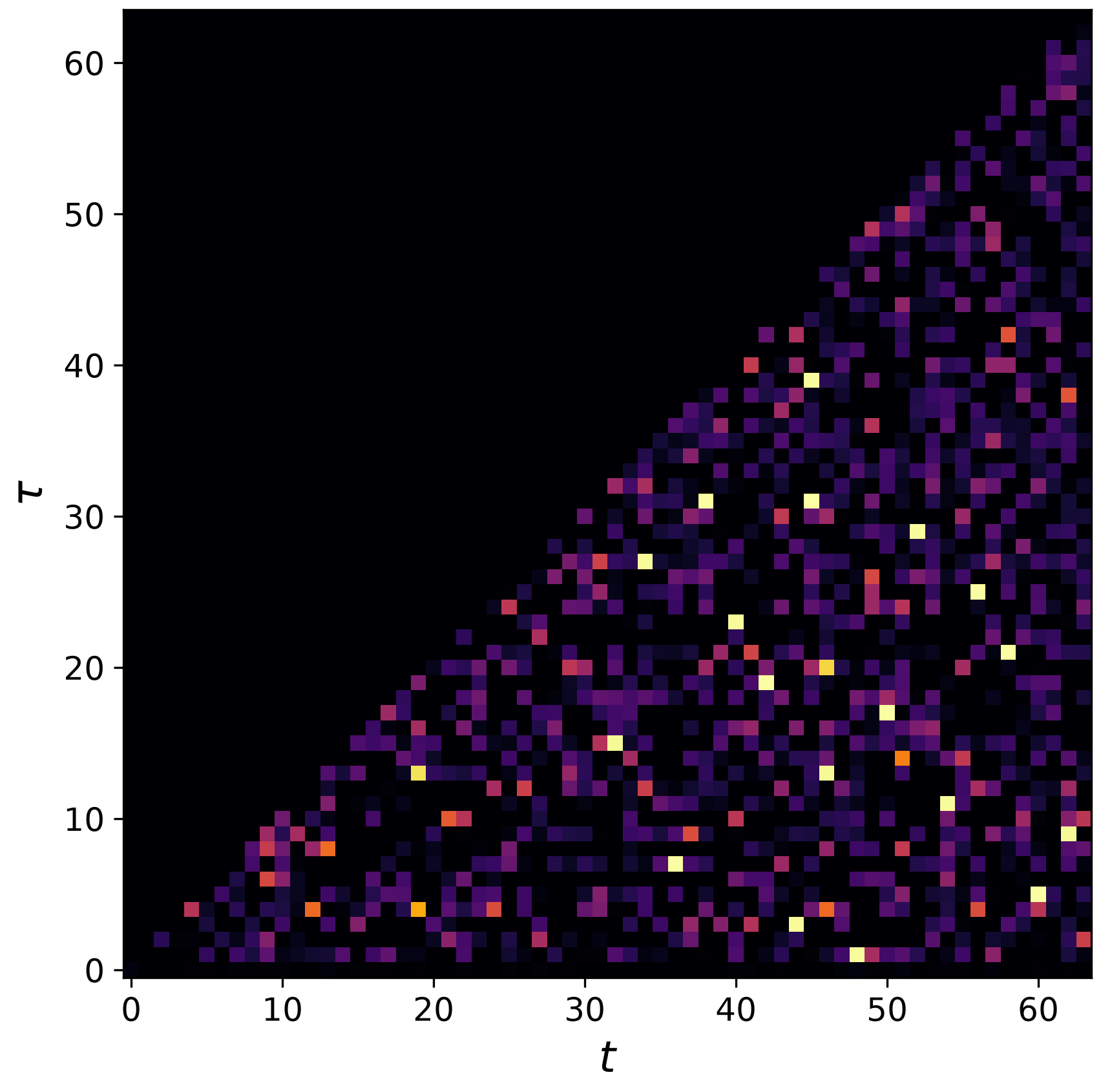}} \\
  \makebox[0pt][r]{\makecell{$D{=}48$ \\ $N{=}32$}\hspace{0.08in}}\raisebox{-0.5\height}{\includegraphics[width=0.28\textwidth]{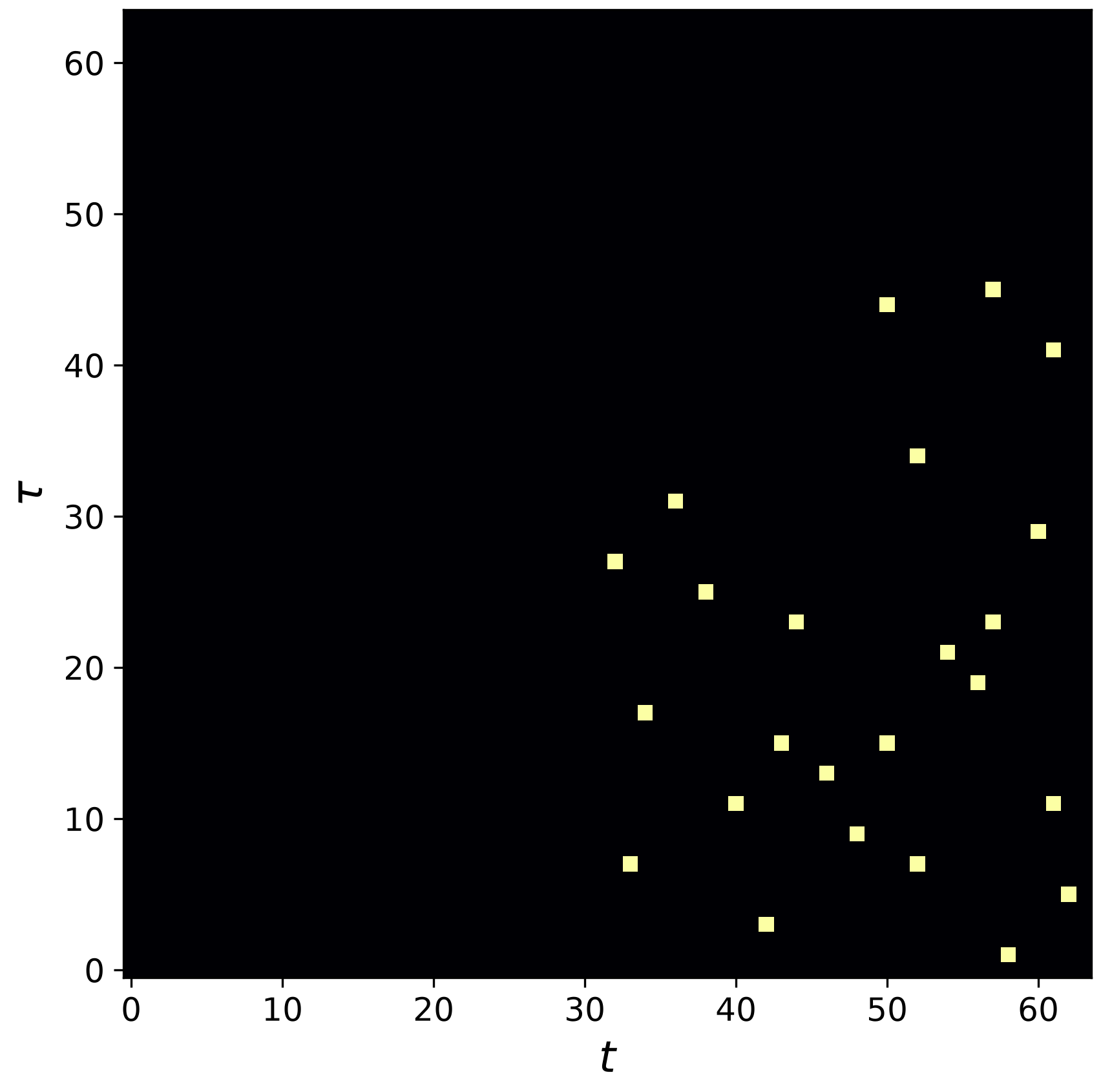}} &
  \raisebox{-0.5\height}{\includegraphics[width=0.28\textwidth]{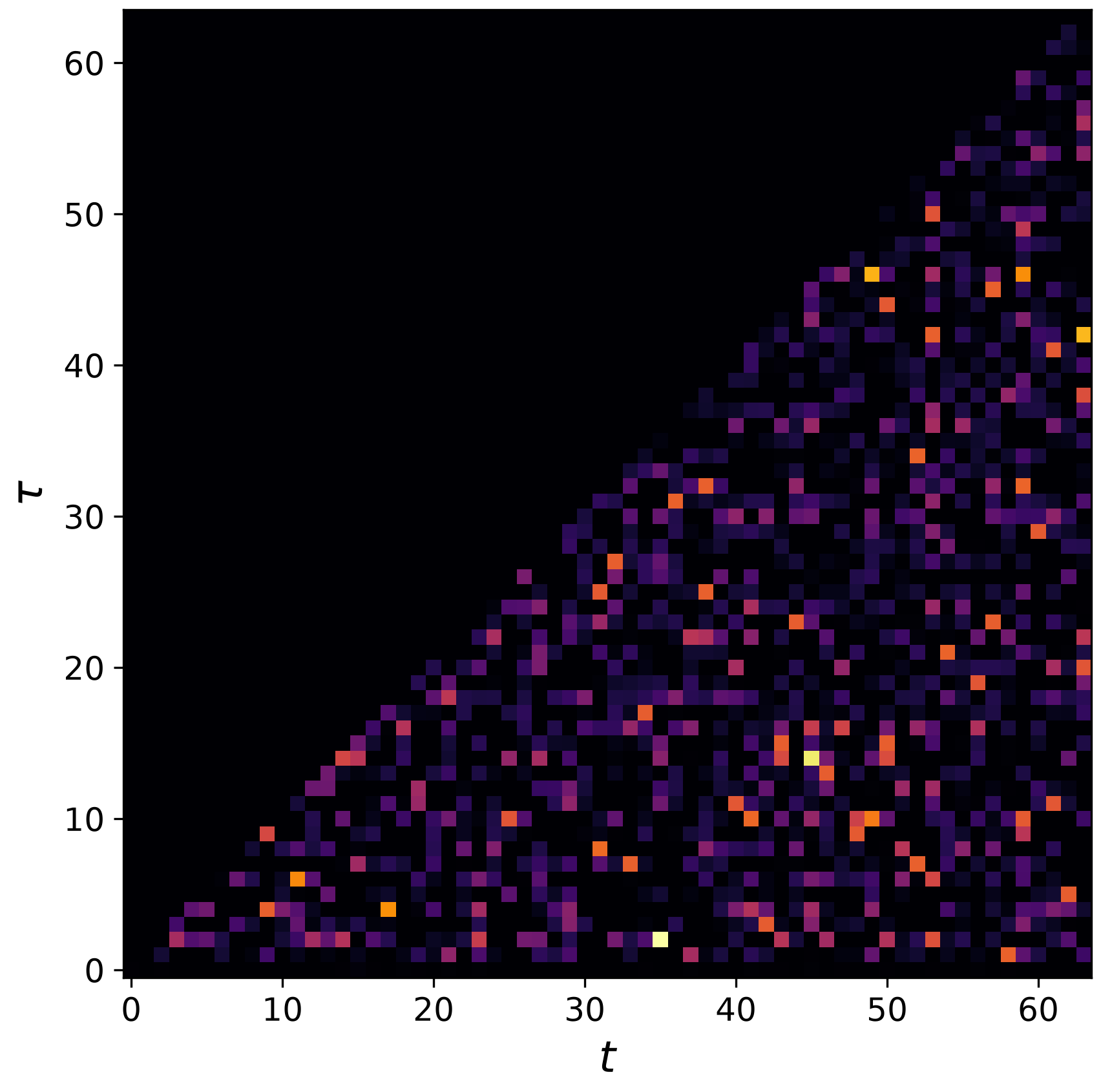}} \\
\end{tabular}
\caption{\small \textbf{Implicit attention maps.}
The ideal attention map of a theoretical non-compressive model (left) is well approximated by a trained real-world model (right). Rows: same analysis repeated for trained models with varying $D$, $N$.}
\label{fig:reversing-attn}
\end{figure*}

\begin{figure*}[!ht]
\centering
\small
\begin{tabular}{c@{\hskip 0.05in}c}
  $x_t G_{vv} \xi_\tau$ & $\xi_t G^\top_{kq} \xi_\tau$ \\
  \makebox[0pt][r]{\makecell{$D{=}64$ \\ $N{=}32$}\hspace{0.08in}}\raisebox{-0.5\height}{\includegraphics[width=0.28\textwidth]{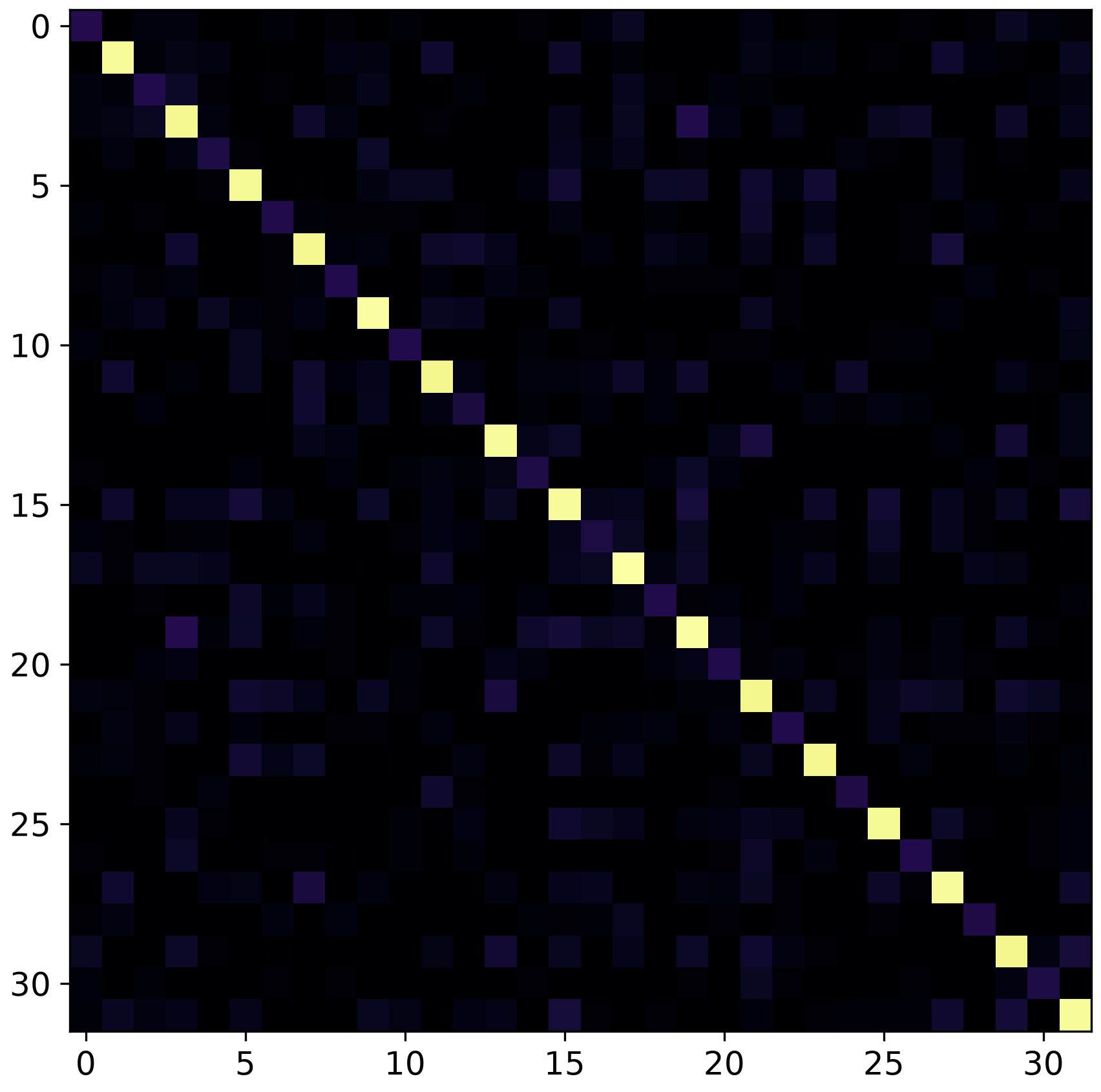}} &
  \raisebox{-0.5\height}{\includegraphics[width=0.28\textwidth]{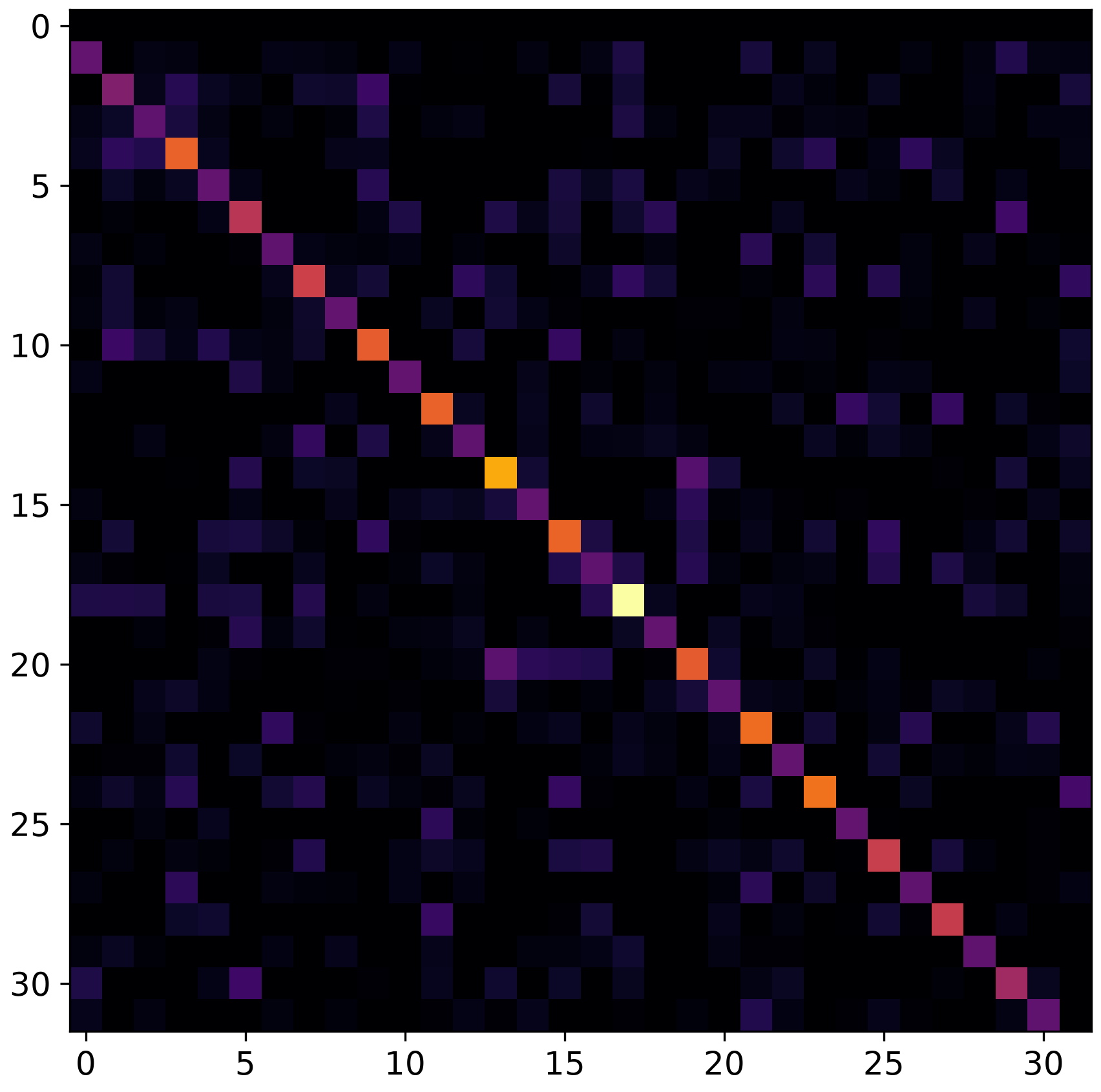}} \\
  \makebox[0pt][r]{\makecell{$D{=}64$ \\ $N{=}24$}\hspace{0.08in}}\raisebox{-0.5\height}{\includegraphics[width=0.28\textwidth]{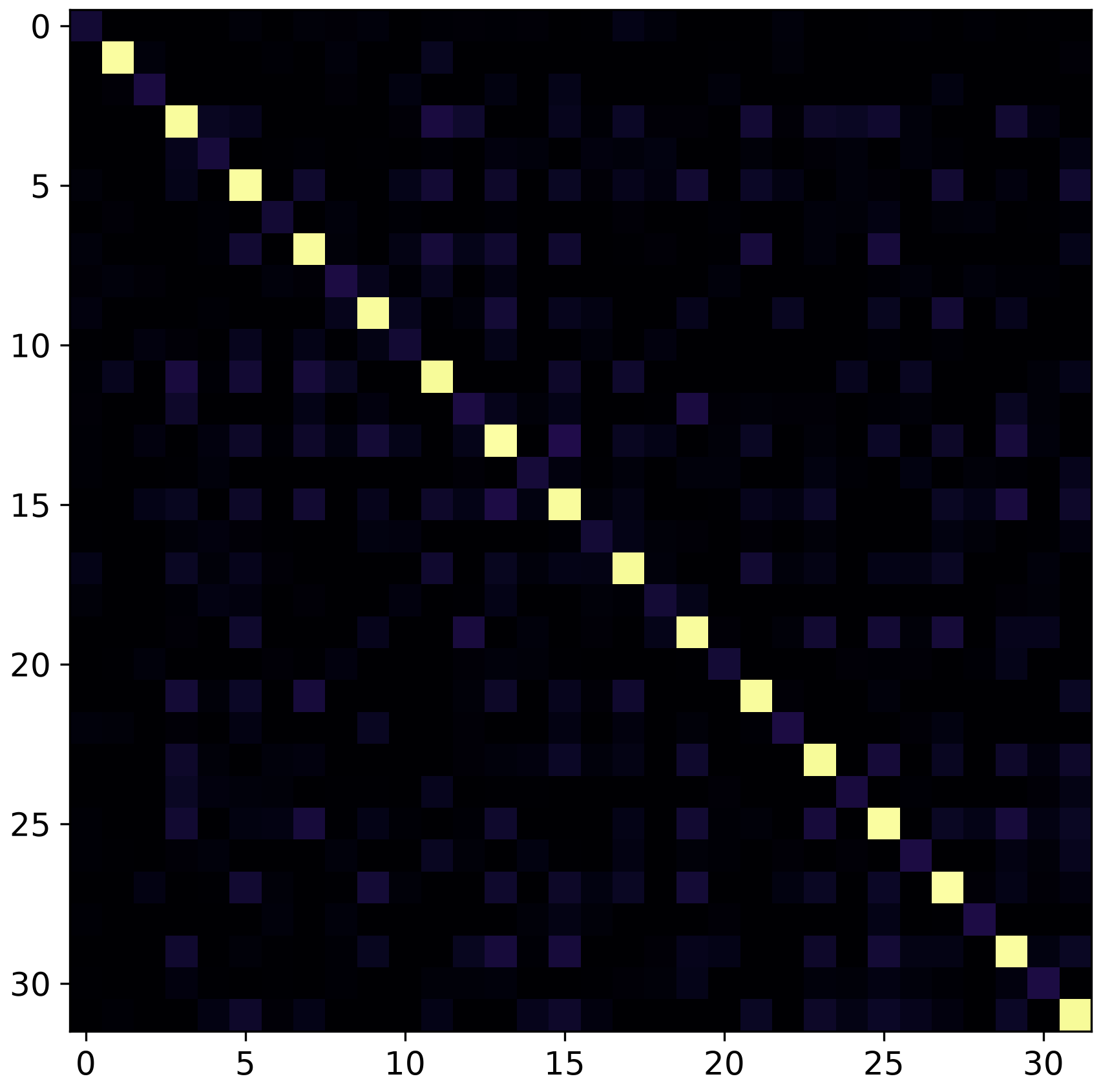}} &
  \raisebox{-0.5\height}{\includegraphics[width=0.28\textwidth]{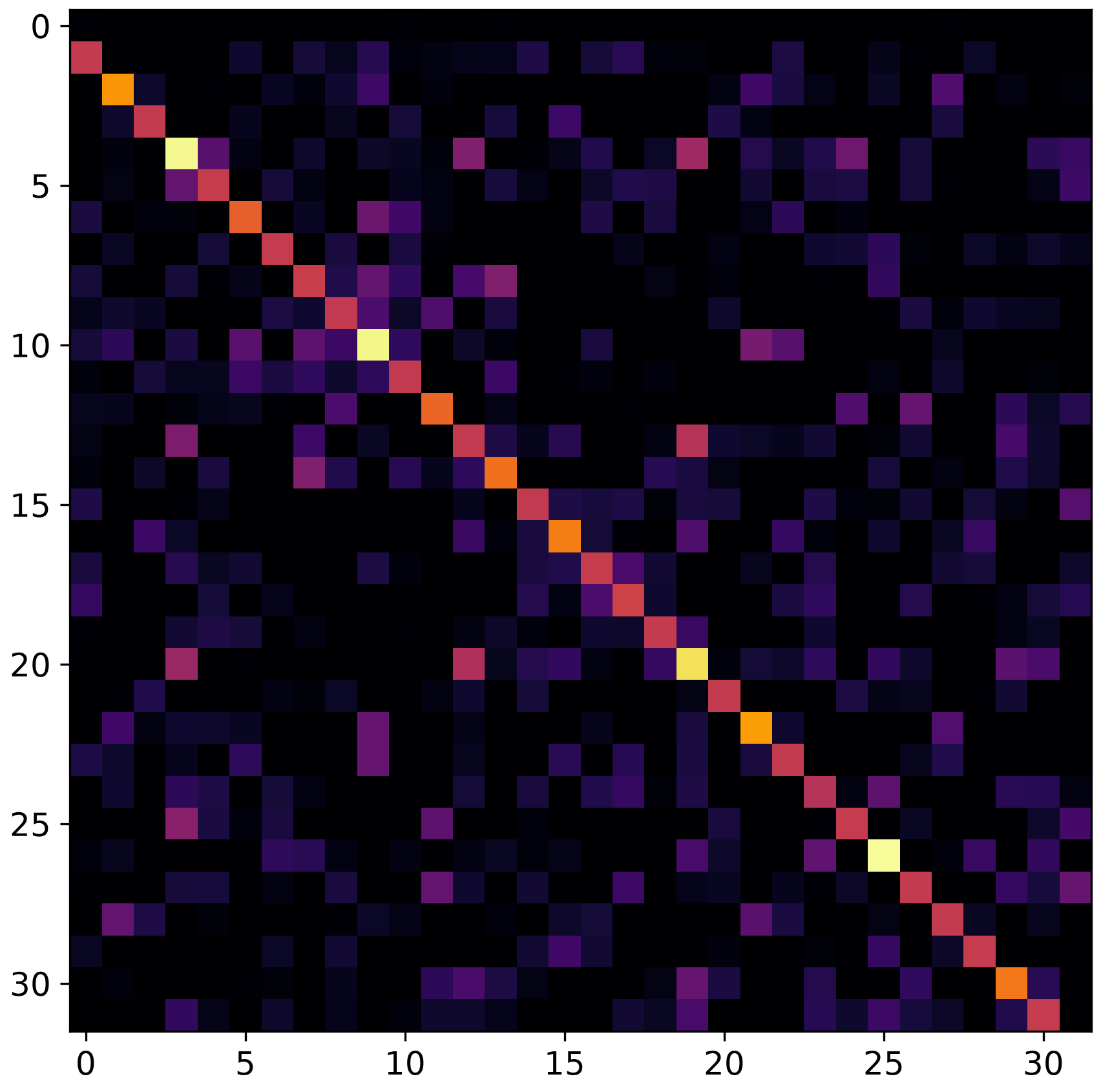}} \\
  \makebox[0pt][r]{\makecell{$D{=}48$ \\ $N{=}32$}\hspace{0.08in}}\raisebox{-0.5\height}{\includegraphics[width=0.28\textwidth]{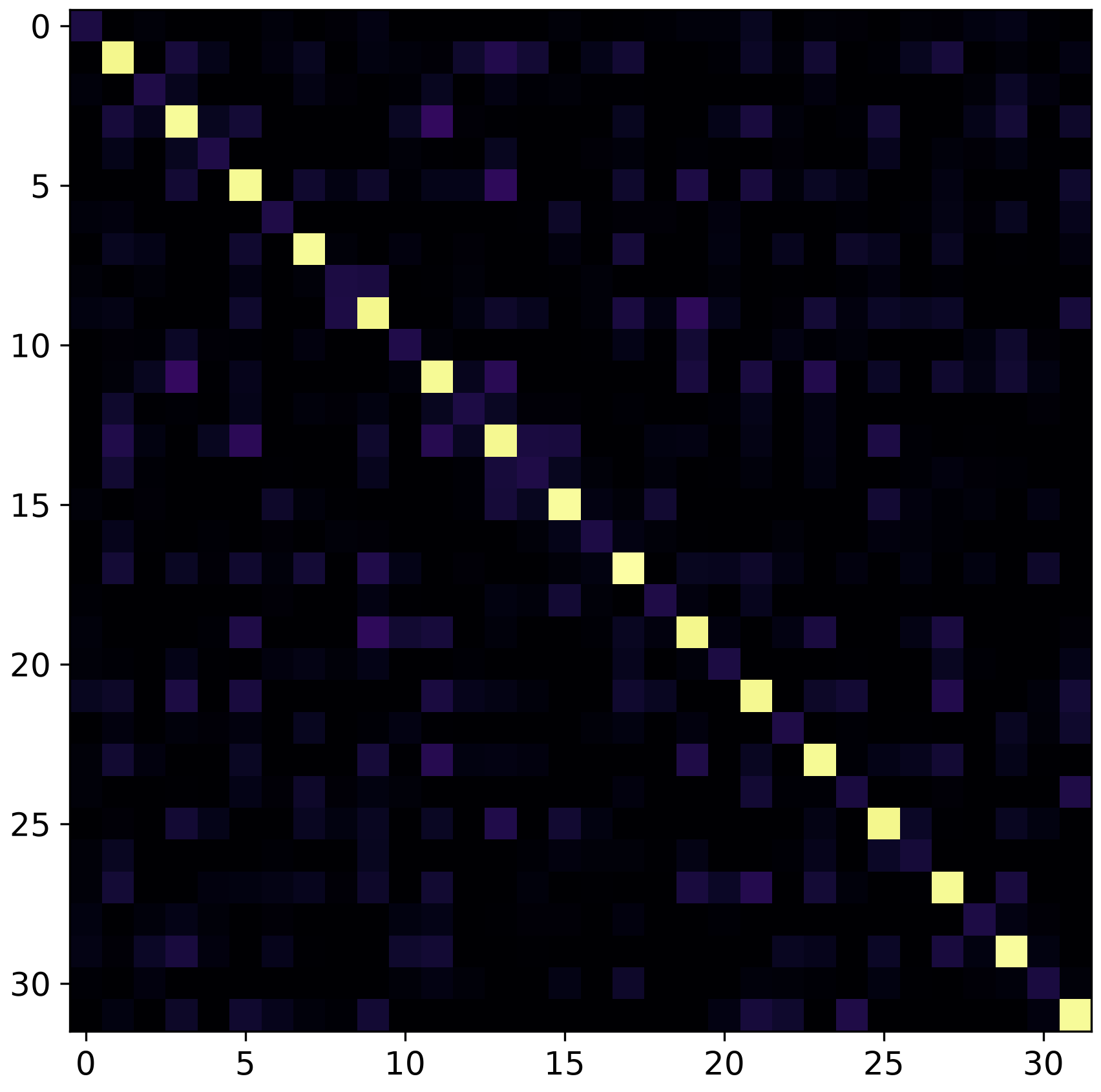}} &
  \raisebox{-0.5\height}{\includegraphics[width=0.28\textwidth]{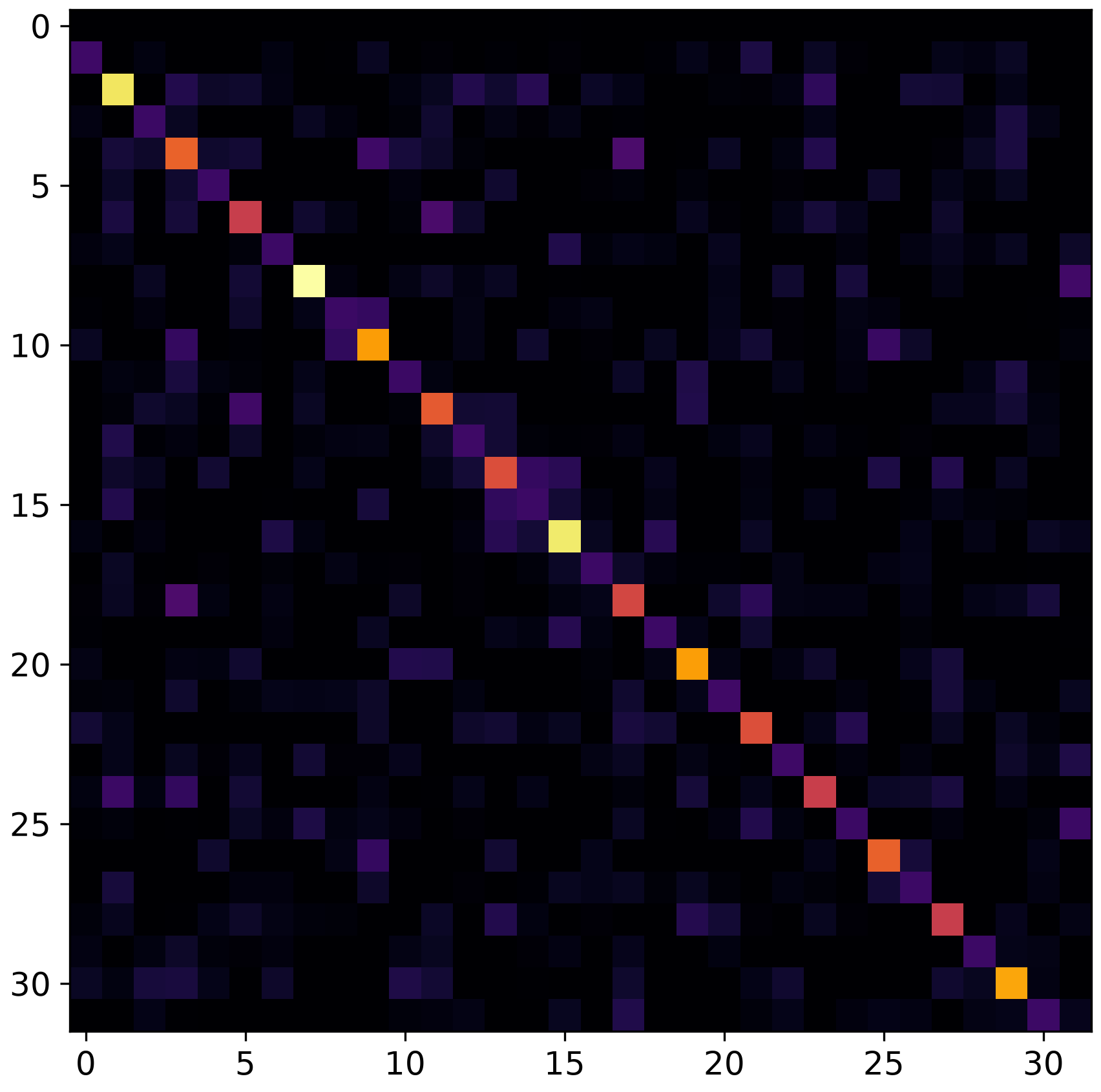}} \\
\end{tabular}
\caption{\small \textbf{Conv1D copy \& shift.}
Values are copied (left), as indicated by the dominant main diagonal. Keys are shifted (right), as indicated by the dominant second diagonal. Rows: same analysis repeated for trained models with varying $D$, $N$.}
\label{fig:reversing-conv}
\end{figure*}

\begin{figure*}[!ht]
\centering
\small
\begin{tabular}{c@{\hskip 0.05in}c}
  $G_{vv}$ & $G_{kq}$ \\
  \makebox[0pt][r]{\makecell{$D{=}64$ \\ $N{=}32$}\hspace{0.08in}}\raisebox{-0.5\height}{\includegraphics[width=0.25\textwidth]{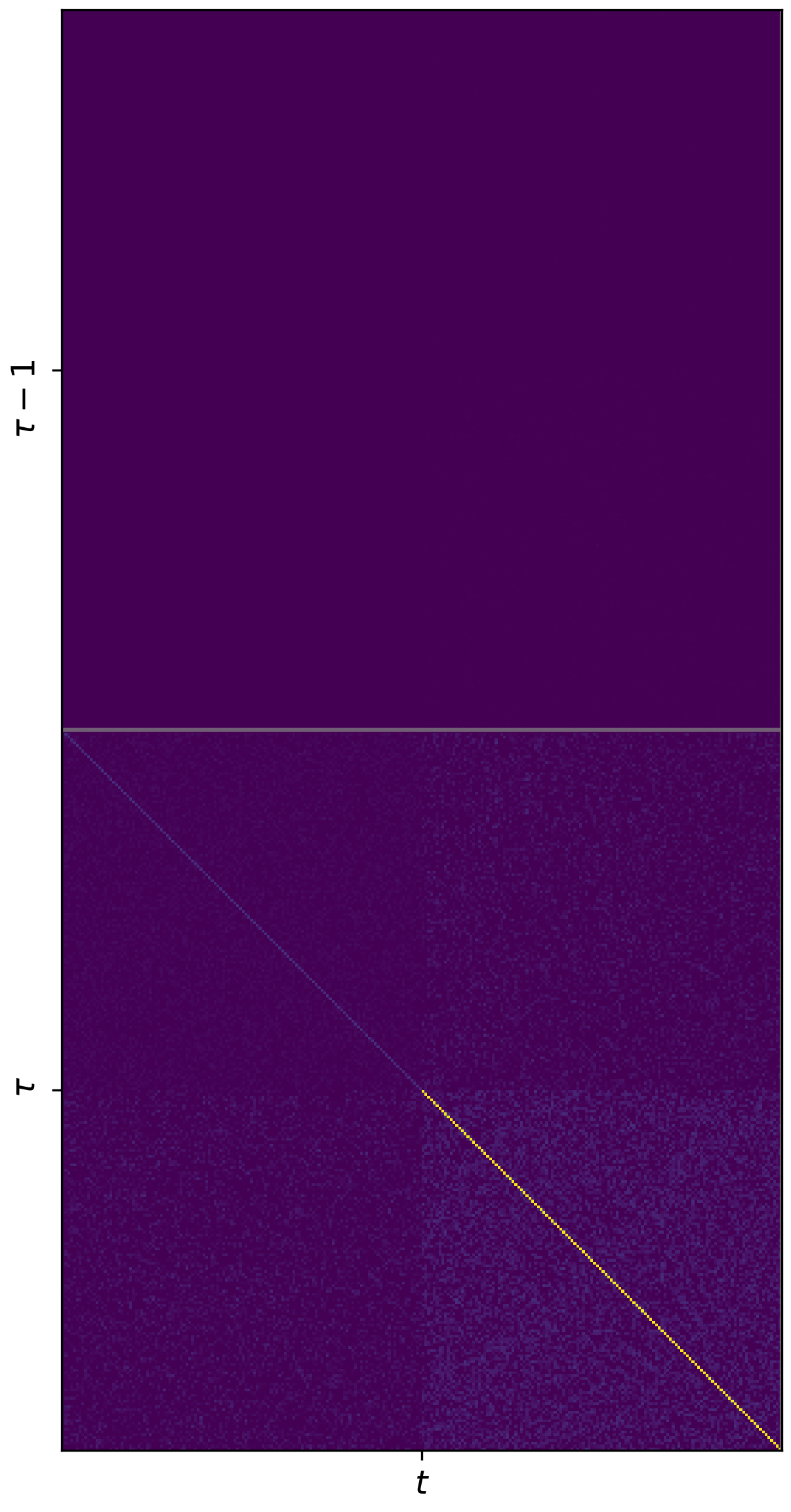}} &
  \raisebox{-0.5\height}{\includegraphics[width=0.48\textwidth]{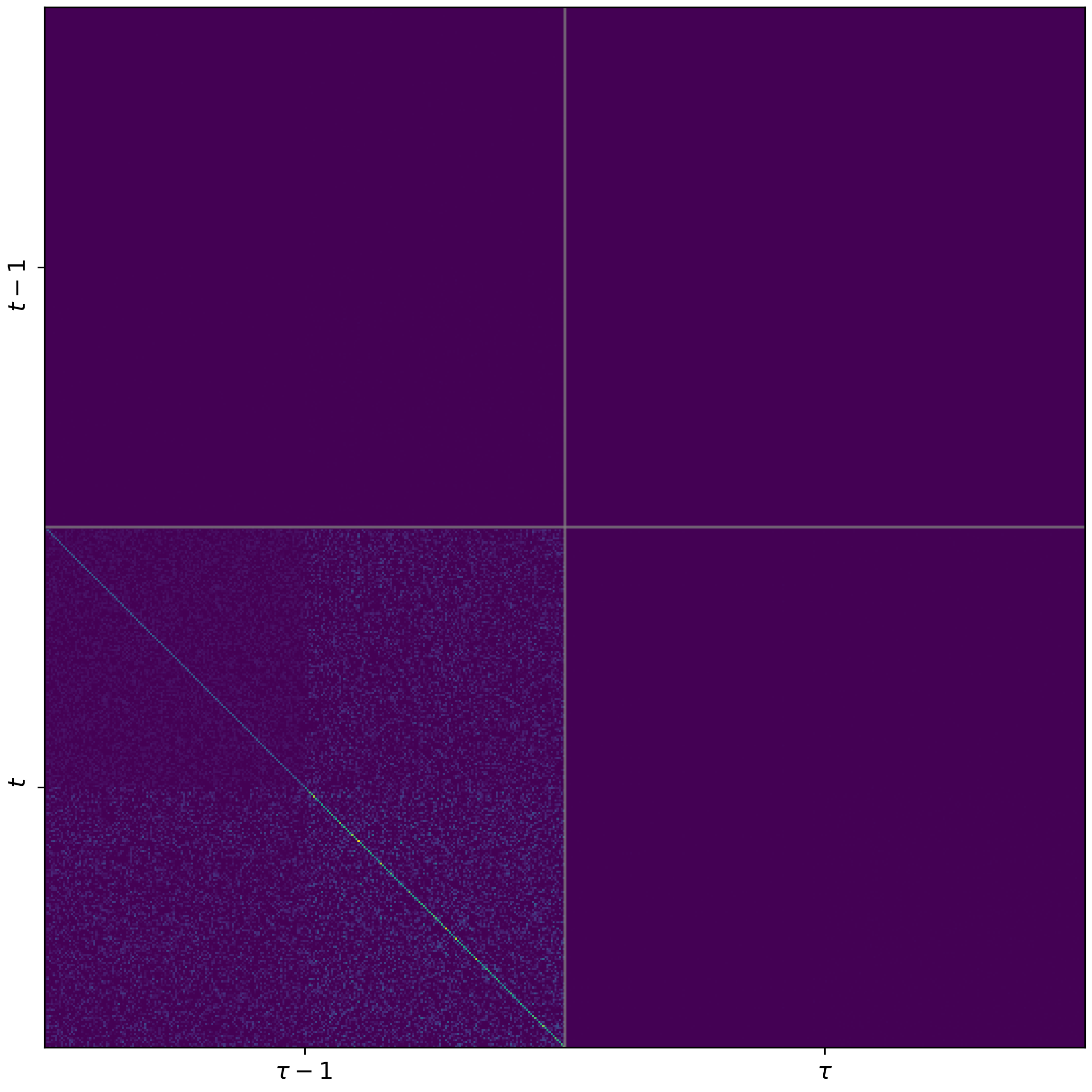}} \\
  \makebox[0pt][r]{\makecell{$D{=}64$ \\ $N{=}24$}\hspace{0.08in}}\raisebox{-0.5\height}{\includegraphics[width=0.25\textwidth]{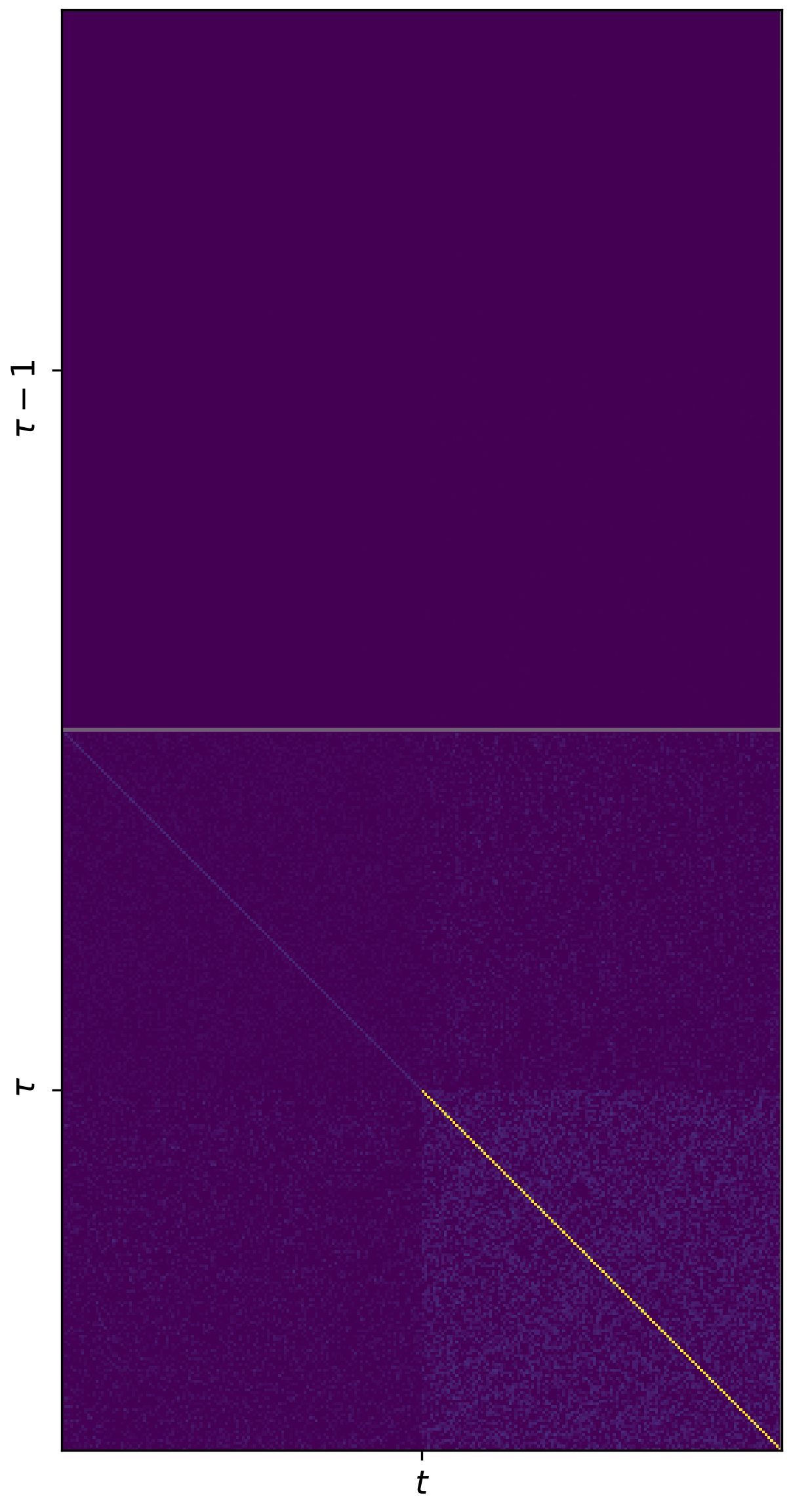}} &
  \raisebox{-0.5\height}{\includegraphics[width=0.48\textwidth]{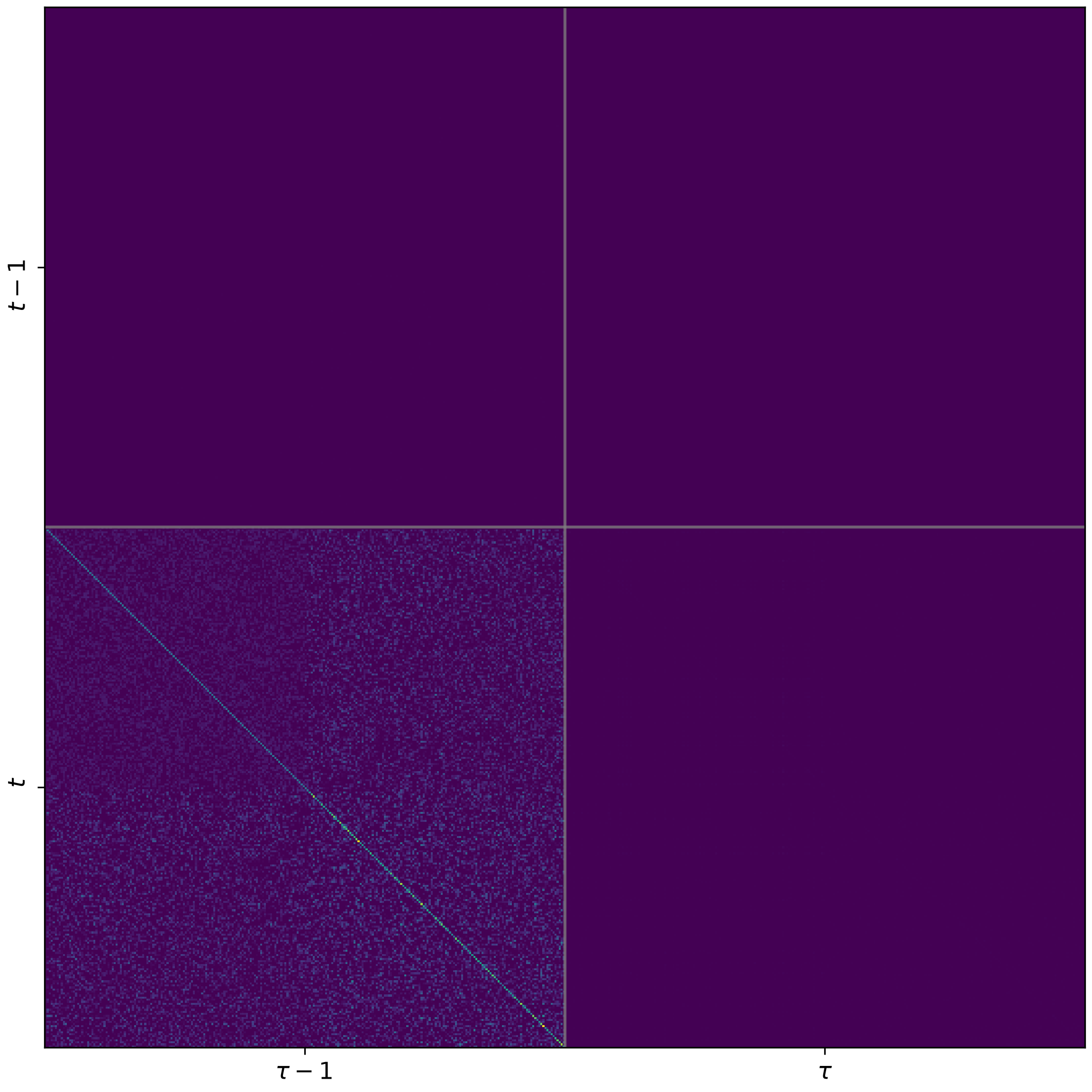}} \\
  \makebox[0pt][r]{\makecell{$D{=}48$ \\ $N{=}32$}\hspace{0.08in}}\raisebox{-0.5\height}{\includegraphics[width=0.25\textwidth]{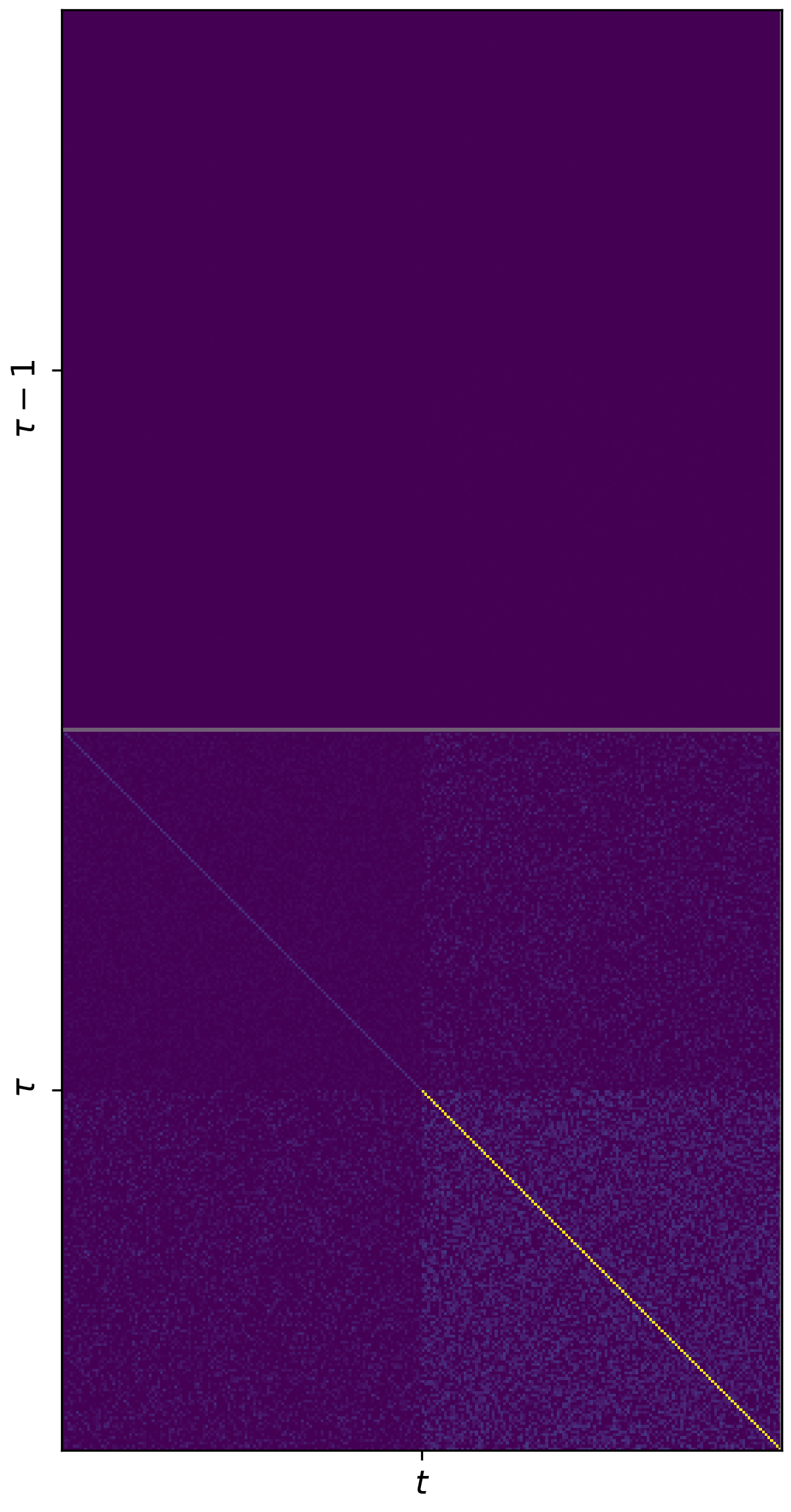}} &
  \raisebox{-0.5\height}{\includegraphics[width=0.48\textwidth]{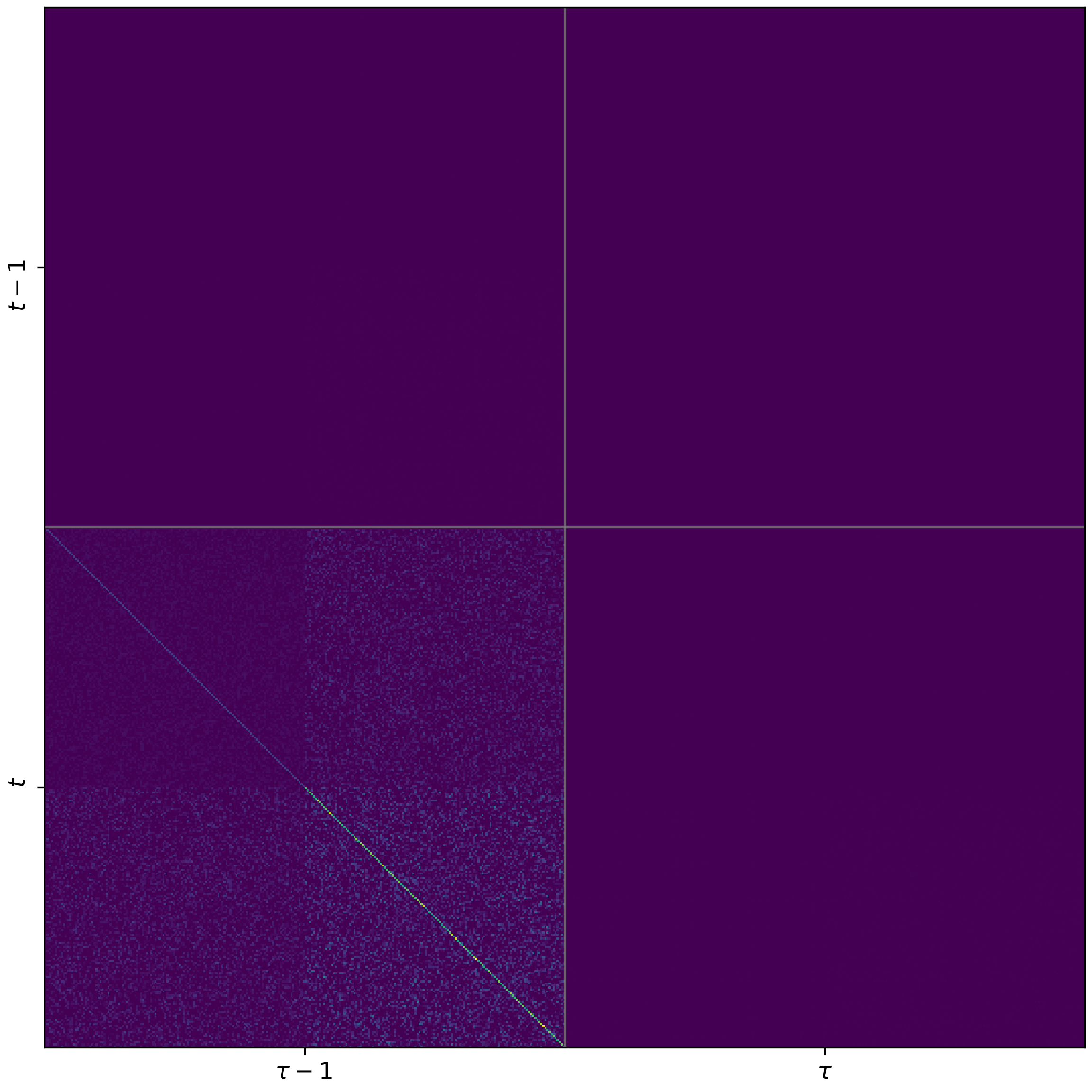}} \\
\end{tabular}
\caption{\textbf{Invariant operators.}
The model's internal mechanism is revealed when applying the appropriate transformations.
As predicted by Eq.~\ref{eq:Gvv-Gkq}, the core operation compares a query $q_t$ with a key $k_{\tau-1}$, then outputs the corresponding value $v_\tau$. Rows: same analysis repeated for trained models with varying $D$, $N$.}
\label{app-fig:reversing-operators}
\end{figure*}

\FloatBarrier

\section{Notation and Dimensions}\label{app:notation-and-dims}

This appendix provides reference tables for the notation appearing in the paper, as well as tables for tensor dimensions, as used in our algorithms and analysis.

\subsection{Multi-Query Associative Recall (MQAR)}

Table~\ref{tab:mqar_notation} includes notation for the MQAR benchmark.

\renewcommand{\arraystretch}{1.15}

\begin{table}[h]
    \caption{Notation for MQAR}
    \label{tab:mqar_notation}
    \centering
    \begin{tabular}{llll}
        \hline
        \textbf{Symbol} & \textbf{Meaning} & \textbf{Value} & \textbf{Notes} \\
        \hline
        $\mathcal{V}$ & Vocabulary & - \\
        $\mathcal{V}_k$ & Key vocabulary & - \\
        $\mathcal{V}_v$ & Value vocabulary & - \\
        $V$ & Vocabulary size & $\lvert \mathcal{V} \rvert$ \\
        $V_k$ & Key vocabulary size & $\lvert \mathcal{V}_k \rvert$ & $V_k=\frac{V}{2}$, unless stated otherwise \\
        $V_v$ & Value vocabulary size & $\lvert \mathcal{V}_v \rvert$ & $V_v=\frac{V}{2}$, unless stated otherwise \\
        $N_f$ & Number of facts & - \\
        $L$ & Total context length & - & $L \geq 4N_f $, following implementation~\citep{arora2023zoology} \\
        - & Facts section length & $2N_f$ & $N_f$ key-value token pairs \\
        - & Queries section length & $L-2N_f$ & $N_f$ query tokens, $L-3N_f$ random padding tokens \\
        
        \hline
    \end{tabular}
\end{table}

\subsection{Mamba Model}

Tables~\ref{tab:mamba_notation}, ~\ref{tab:mamba_dims} include detailed notation and dimensions for Mamba model and tensors.

For detailed dimensions of hidden state and key-value matrices $H_t$, $H'_t$, $H''_t$ used in Sec.~\ref{subsec:mech-interp} for mechanistic interpretability, see Appendix~\ref{app-subsec:ssm-key-value-tables}.

For Mamba-2 dimensions, see Appendix~\ref{app:mamba-2-multi-head}.

\begin{table}[h]
    \caption{Notation for Mamba}
    \label{tab:mamba_notation}
    \centering
    \begin{tabular}{llll}
        \hline
        \textbf{Symbol} & \textbf{Meaning} & \textbf{Value} & \textbf{Notes} \\
        \hline
        $D$ & Embedding size & - \\
        $N$ & SSM State size & - \\
        $\mathrm{expand}$ & SSM expansion factor & $\frac{D_\mathrm{in}}{D}$ \\
        $D_\mathrm{in}$ & SSM number of channels & $\mathrm{expand}\cdot D$ \\
        $D_\mathrm{conv}$ & Convolution kernel size & - \\

        $\Lambda$ & Number of layers & - \\
        
        \hline
    \end{tabular}
\end{table}

\begin{table}[h]
    \caption{Dimensions for Mamba}
    \label{tab:mamba_dims}
    \centering
    \begin{tabular}{llll}
        \hline
        \textbf{Symbol} & \textbf{Meaning} & \textbf{Dimension} & \textbf{Notes} \\
        \hline
        $x$         & Model input sequence (one-hot)        & $\mathbb{R}^{V \times L}$        & $x_t \in \mathbb{R}^V$ \\
        $y$         & Model output sequence (one-hot)       & $\mathbb{R}^{V \times L}$        & $y_t \in \mathbb{R}^V$ \\
        
        $x^e$         & Model input sequence (embedded)        & $\mathbb{R}^{D \times L}$        & $x^e_t \in \mathbb{R}^D$ \\
        
        $\hat{x}$   & SSM input sequence                    & $\mathbb{R}^{D_\mathrm{in} \times L}$ & $\hat{x}_t \in \mathbb{R}^{D_\mathrm{in}}$ \\
        $\hat{y}$   & SSM output sequence                   & $\mathbb{R}^{D_\mathrm{in} \times L}$ & $\hat{y}_t \in \mathbb{R}^{D_\mathrm{in}}$ \\
        $\hat{z}$   & Gate sequence                   & $\mathbb{R}^{D_\mathrm{in} \times L}$ & $\hat{z}_t \in \mathbb{R}^{D_\mathrm{in}}$ \\
        
        \hline

        $E$         & Model embedding        & $\mathbb{R}^{D \times V}$ \\
        
        $W_\mathrm{conv}$         & Convolution kernel        & $\mathbb{R}^{D_\mathrm{in} \times D_\mathrm{conv}}$  \\ 
        
        $P_\mathrm{in}$         & Input projection        & $\mathbb{R}^{D_\mathrm{in} \times D}$ \\
        $P_\mathrm{out}$         & Output projection        & $\mathbb{R}^{D \times D_\mathrm{in}}$  \\ 

        $S_B$         & $B$ projection        & $\mathbb{R}^{N \times D_\mathrm{in}}$  \\ 

        $S_C$         & $C$ projection        & $\mathbb{R}^{N \times D_\mathrm{in}}$  \\ 
        
        \hline
        
        $A$         & SSM transition coefficient  & $\mathbb{R}^{N}$ \\
        $B$         & SSM input coefficient        & $\mathbb{R}^{N \times L}$        & $B_t \in \mathbb{R}^{N}$ \\
        $C$         & SSM readout coefficient      & $\mathbb{R}^{N \times L}$        & $C_t \in \mathbb{R}^{N}$ \\
    
        $\Delta$   & SSM discretization parameter   & $\mathbb{R}^{D_\mathrm{in} \times L}$        
        & $\Delta_t \in \mathbb{R}^{D_\mathrm{in}}$ \\
        
        $\bar{A}$   & SSM Discrete transition coefficient   & $\mathbb{R}^{D_\mathrm{in} \times N \times L}$        
        & $\bar{A}_t \in \mathbb{R}^{D_\mathrm{in} \times N}$ \\
        $\bar{B}$   & SSM Discrete input coefficient   & $\mathbb{R}^{D_\mathrm{in} \times N \times L}$        & $\bar{B}_t \in \mathbb{R}^{D_\mathrm{in} \times N}$ \\
        $h$         & SSM hidden state             & $\mathbb{R}^{D_\mathrm{in} \times N \times L}$        & $h_t \in \mathbb{R}^{D_\mathrm{in} \times N}$ \\

        $\alpha$         & SSM equivalent attention matrix             & $\mathbb{R}^{L \times L}$ \\
        
        \hline
    
    \end{tabular}
\end{table}

\subsection{SSM State Update Rule}\label{app-subsec:ssm-dims}

In the following Appendix, we clarify shapes and dimensions for SSM state update.

\paragraph{Full Mamba model.}
For each timestep $t$, the SSM vectors $B_t, \, C_t \in \mathbb{R}^N$ and discrete matrices $\bar{A_t}, \bar{B_t} \in \mathbb{R}^{D_\mathrm{in} \times N}$ are computed as:
\begin{equation} \label{app-eq:TimeVariantMatrices1}
    B_t = S_B \, \hat{x}_t\,, \quad C_t = S_C \, \hat{x}_t \,, \quad
    \Delta_t = \text{SoftPlus}(S_{\Delta} \hat{x}_t)\,,
\quad
     \bar{A}_t = \exp (\Delta_t A)\,, \quad\bar{B}_t = \Delta_t B_t %
\end{equation}
where $S_B, S_C, S_{\Delta}$ are linear projection matrices, SoftPlus is a
smooth approximation of
ReLU, and biases are omitted for simplicity.

Then, for any timestep $t$, each SSM channel $d=1,\dots,D_\mathrm{in}$ has a scalar input $\hat{x}^d_t \in \mathbb{R}$ and a scalar output $\hat{y}^d_t \in \mathbb{R}$. The output is computed through a recurrent state-update step:
\begin{equation}
\left\{
\begin{aligned}
    {h}^d_t &= {\bar{A}}^d_t \odot {h}^d_{t-1} + \hat{x}^d_t \,{\bar{B}}^d_t \\
    \hat{y}^d_t &= {h}^d_t \cdot {C}_t
\end{aligned}
\right.
\end{equation}
where ${h}_t^d, {\bar{A}}_t^d, {\bar{B}}_t^d, {C}_t \in \mathbb{R}^N$ are all $N$-vectors, $\odot$ denotes an element-wise product, and $\cdot$ a vector dot-product. (For $\hat{x}^d_t \,{\bar{B}}^d_t$ we have a scalar multiplying a vector.)

Per-channel, we have a state vector ${h}^d_t \in \mathbb{R}^N$. In total, all channels together form a state matrix $h_t\in \mathbb{R}^{D_\mathrm{in} \times N }$.

\paragraph{Simplified Mamba model.}

In order to ease Mamba model analysis, we preset a simplified SSM, where discretization is removed, and $\bar{A}$ is ignored (by setting its values to $1$). The SSM operation can be now written as:
\begin{equation}\label{eq:simple-ssm-channel}
\left\{
\begin{aligned}
    {h}^d_t &= {h}^d_{t-1} + \hat{x}^d_t \,{B}_t \\
    \hat{y}^d_t &= {h}^d_t \cdot {C}_t
\end{aligned}
\right.
\end{equation}
Note that ${\bar{B}}^d_t={{B}}_t\in\mathbb{R}^N$ is now shared between channels.

If we consider the channel dimension, we have a vector input and output $\hat{x}_t,\hat{y}_t\in \mathbb{R}^{D_\mathrm{in}}$, and a state matrix $h_t\in \mathbb{R}^{D_\mathrm{in} \times N }$. Equation~\ref{eq:simple-ssm-channel} can be written in a matrix form as:
\begin{equation}\label{eq:simple-ssm-matrix}
\left\{
\begin{aligned}
    h_t &= h_{t-1} + \hat{x}_t\,B_t^\top \\
  \hat{y}_t &= h_t\, C_t
\end{aligned}
\right.
\end{equation}
where $\hat{x}_t\,B_t^\top \in \mathbb{R}^{D_\mathrm{in} \times N }$ is a vector-vector outer product. 

\renewcommand{\arraystretch}{1.0}

\color{black}

\color{black}
\section{Mamba Recall Circuits: Proofs}\label{app:recall-circuits-proofs}

\subsection{\texorpdfstring{Johnson-Lindenstrauss (JL) Lemma}{Johnson-Lindenstrauss (JL) Lemma}}\label{app-subsec:jl-lemma}

\begin{lemma}[\textbf{Johnson-Lindenstrauss lemma, inner-product form}~\citep{johnson1984extensions}]\label{app-lem:jl}
Let $x_1,\dots,x_n\in\mathbb{R}^m$ be nonzero vectors and $0<\varepsilon<1$. If $d \geq \frac{c\,\log n}{\varepsilon^2}$ for an absolute constant $c$, then there exists a linear map $E\in\mathbb{R}^{d\times m}$ such that for all $i,j$:
\begin{equation}\label{app-eq:jl-inner-product}
\bigl|\langle E\,x_i,E\,x_j\rangle - \langle x_i,x_j\rangle\bigr| \,\leq\, \varepsilon\,\|x_i\|\,\|x_j\|
\end{equation}
\end{lemma}
\begin{proof}[\textbf{Proof}]
{}We use the probabilistic method: draw $E$ with i.i.d. $\mathcal{N}(0,\frac{1}{d})$ entries, and show the desired map exists since random $E$ satisfies Eq.~\ref{app-eq:jl-inner-product} with nonzero probability.

\paragraph{Norm concentration.}
For any fixed vector $u$, the entries of $E\,u$ are i.i.d. $\mathcal{N}(0,\frac{\|u\|^2}{d})$, so $\|E\,u\|^2\sim\|u\|^2\cdot\chi^2_d/d$, which concentrates around its mean $\|u\|^2$: a Chernoff bound on the $\chi^2_d$ distribution gives
\begin{equation}\label{app-eq:jl-norm-concentration}
\Pr\Bigl[\,\bigl|\,\|E\,u\|^2-\|u\|^2\bigr|\geq\varepsilon\,\|u\|^2\Bigr]\,\leq\, 2\,e^{-\left(\varepsilon^2/2-\varepsilon^3/3\right)d/2}
\end{equation}

\paragraph{Polarization.}
Let $u_i=x_i/\|x_i\|$. Suppose the norms of $u_i$, $u_j$, $u_i{+}u_j$ and $u_i{-}u_j$ are all preserved within $(1\pm\varepsilon)$, as in Eq.~\ref{app-eq:jl-norm-concentration}. Then, since $\langle E\,u_i,E\,u_j\rangle=\frac{1}{4}\bigl(\|E(u_i{+}u_j)\|^2-\|E(u_i{-}u_j)\|^2\bigr)$, the inner product deviates by at most $\bigl|\langle E\,u_i,E\,u_j\rangle - \langle u_i,u_j\rangle\bigr|\leq\varepsilon\cdot\frac{\|u_i\|^2+\|u_j\|^2}{2}=\varepsilon$. Scaling both sides back by $\|x_i\|\,\|x_j\|$ yields Eq.~\ref{app-eq:jl-inner-product}.

\paragraph{Union bound.}
The above requires at most $2n^2$ norm events to hold simultaneously; by Eq.~\ref{app-eq:jl-norm-concentration}, the probability that any of them fails is at most $4n^2\, e^{-\left(\varepsilon^2/2-\varepsilon^3/3\right)d/2}$, which drops below $1$ once $d\geq \frac{c\,\log n}{\varepsilon^2}$ for a suitable absolute constant $c$. A map satisfying Eq.~\ref{app-eq:jl-inner-product} for all $i,j$ therefore exists.
\end{proof}

\paragraph{Corollary for one-hot vectors.}
In this paper the lemma is applied to sets of one-hot vectors (and their near-orthonormal embeddings); for the one-hot case, bounding the Gram entries of $E$ directly sharpens the constant to $c=4$. (See for instance our usage in Lemma~\ref{lem:jl-scaling-law}.)

\subsection{Non-Compressive Recall Circuit}

\label{app-thm-non-compressive-recall}
\begin{theorem}[\textbf{Perfect non-compressive recall circuit}]
Given a vocabulary size $V$, a \hyperref[alg:mamba]{\textit{single-layer simplified Mamba}} with dimensions $D=V$, $N=V$, $\mathrm{expand}=2$, $D_\mathrm{conv}=2$ can \textbf{perfectly solve} an MQAR task, i.e. with recall probability = 1.
\end{theorem}

\begin{proof}[\textbf{Proof}]

We construct model weights as follows:

\paragraph{Embeddings.\quad}
Let us choose the trivial identity embedding $\in \mathbb{R}^{V\times V}$:

\begin{equation}
\begin{aligned}
E  &= I_V \\
\end{aligned}
\end{equation}

The embedded input vectors are then trivially $x^e_t=E\,x_t=x_t$.

\paragraph{Projections.\quad} We choose:
\begin{equation}
\begin{aligned}
P_\mathrm{in} &= {\bigl(I_V \mid I_V)}^\top \\
P_\mathrm{out} &= \bigl(0\mid I_v)
\end{aligned}
\end{equation}
The input projection simply duplicates the state:
\begin{equation}
    x^p_t=P_\mathrm{in}\,x^e_t=\begin{pmatrix}x_{t}\\x_{t}\end{pmatrix} 
\end{equation}
The output projection extracts the "current" half of duplicated current-previous vectors, then projects it onto the value vocabulary space $\mathcal{V}_v$.

\paragraph{Conv1D kernel.\quad}
We choose
\begin{equation}
  W_\mathrm{conv}=\bigl(W_p\mid W_c)=\bigl(\begin{smallmatrix}1_V\\0_V\end{smallmatrix}\mid \;\begin{smallmatrix}0_V\\1_V\end{smallmatrix}\bigr)  
\end{equation}
with $1_V$ and $0_V$ all-ones and all-zeros column vectors in $\mathbb{R}^V$, respectively.
This implies:
\begin{equation}
\begin{aligned}
\mathrm{diag}\bigl(W_c)&=\bigl(0\mid I_V) \\
\mathrm{diag}\bigl(W_p)&=\bigl(I_V\mid 0) \\
\end{aligned}
\end{equation}
This kernel applies identity and shift operations, to form a vector:
\begin{equation}
\label{eq:perfect-ssm-in}
\hat{x}_t =\mathrm{Conv1D}\bigl(x^p_t,\,W_\mathrm{conv}\bigr)= \begin{pmatrix}x^e_{t-1}\\x^e_{t}\end{pmatrix} = \begin{pmatrix}x_{t-1}\\x_{t}\end{pmatrix}    
\end{equation}
This stacked vector is the input to the SSM.

\paragraph{SSM projections and embeddings.\quad}
For the SSM selection matrices, let us choose:
 \begin{equation}
\begin{aligned}
S_B & = \bigl(I_V\mid 0) \\
S_C &= \bigl(0\mid I_V) \\
\end{aligned}
\end{equation}
That is to say, $S_B$ outputs the previous token, 
\begin{equation}
    B_t = S_B \, \,\hat{x}_t=  x_{t-1}
\end{equation}
while $S_C$ outputs the current token:
\begin{equation}
    C_t = 
    S_C \,\hat{x}_t = 
  x_t
    \end{equation}
As we will soon see, this allows $S_B$ and $S_C$ to function as key and query extractors, respectively.

\paragraph{Overall mechanism.\quad}
Let us now look at consecutive context token pairs $(x_{\tau-1},x_\tau)$, not necessarily a key-value pair,  and a current token $x_t$, not necessarily a (query) key token. 
Let us also denote the stacked token-pair vectors as $\xi_t \equiv \begin{pmatrix}x_{t-1}\\x_{t}\end{pmatrix}$. Note that in this non-compressive scheme described, the SSM input is exactly $\hat{x}_t=\xi_t$ (see~\eqref{eq:perfect-ssm-in}).

With the matrix construction defined above, the SSM recurrent update rule from Alg. \ref{alg:mamba} now becomes:
\begin{equation}
     h_t -h_{t-1}= \hat{x}_t\,B_t^\top = 
    \xi_t \,{x_{t-1}}^\top
\end{equation}
Hence, assuming $h_0=0$, we have
\begin{equation}
\label{eq:perfect-ssm-update}
\begin{aligned}
    h_t 
    &=\sum_{\tau=0}^t \hat{x}_t\,B_t^\top
    =\sum_{\tau=1} ^t \xi_\tau \,{x_{\tau-1}}^\top  \\
    \hat{y}_t 
    & =h_t \, C_t = \sum_{\tau=1} ^t \xi_\tau \,{x_{\tau-1}}^\top x_t
\end{aligned}
\end{equation}
The overall model output can be now expressed as:
\begin{equation}
y_t = E^\top P_\mathrm{out}\,\hat{y}_t = \sum_{\tau=1} ^t (P_\mathrm{out}\,\xi_\tau) \,{x_{\tau-1}}^\top x_t 
\end{equation}
However, we have constructed $P_\mathrm{out}$ to extract the current token:
\begin{equation}
    P_\mathrm{out}\,\xi_\tau=\bigl(0\mid I_V) \, \xi_\tau=x_\tau
\end{equation}
We then have:
\begin{equation}
y_t = 
\sum_{\tau=1} ^t x_\tau \,{x_{\tau-1}}^\top x_t
\end{equation}
Hence, if we now have a \textit{query} $x_t = q_m \in \mathcal{V}_k$, then
\begin{equation}
    y_t = 
\sum_{n=1} ^{N_f} x_\tau \,{x_{\tau-1}}^\top q_m
\end{equation}
If we know the query corresponds to a \textit{key} fact in the context $q_m=k
_{n^*}$ which is a part of the a fact pair $(k_{n^*}, v_{n^*})$, then finally, since one-how keys and values are orthonormal:
\begin{equation}
    y_t = 
\sum_{n=1} ^{N_f} v_n \,{k_n}^\top k
_{n^*}=
\sum_{n=1} ^{N_f} v_n \,\delta_{n,n^*}=
v_{n^*}
\end{equation}
with $\delta_{i,j}$ the Kronecker delta. 

This means the model achieves perfect recall: for any query, it exactly retrieves the correct value from the context.
\end{proof}

\subsection{Compressive Recall Circuit}
\label{app-thm-compressive-recall}

\begin{lemma}[\textbf{Two consecutive JL transformations}]\label{app-lem:jl-consecutive}
Let $x_1,\dots,x_V\in\mathbb{R}^V$ be nonzero vectors and $0<\varepsilon_v<\varepsilon_k<1$. If $D\geq\frac{c\,\log V}{\varepsilon_v^2}$ and $N\geq\frac{c\,\log V}{\varepsilon_k^2}$ for an absolute constant $c$, then there exist $E\in\mathbb{R}^{D\times V}$ and $F\in\mathbb{R}^{N\times D}$ such that the JL lemma (Eq.~\ref{app-eq:jl-inner-product}) holds \textbf{simultaneously} for $E$ w.r.t. $\varepsilon_v$ and for $\tilde{E}\equiv FE$ w.r.t. $\varepsilon_k$.
\end{lemma}
\begin{proof}[\textbf{Proof}]
Set the auxiliary distortions $\delta_v=\min\bigl(\varepsilon_v,\frac{\varepsilon_k}{3}\bigr)\geq\frac{\varepsilon_v}{3}$ and $\delta_k=\frac{\varepsilon_k}{3}$; the assumed $D$, $N$ then satisfy $D\geq\frac{c'\,\log V}{\delta_v^2}$, $N\geq\frac{c'\,\log V}{\delta_k^2}$. Apply the JL lemma (Lem.~\ref{app-lem:jl}) to $\{x_i\}_{i=1}^V$ to obtain $E$ with distortion $\delta_v\leq\varepsilon_v$. Fix this realization of $E$, and apply the JL lemma (Lem.~\ref{app-lem:jl}) again to the finite set $\{E\,x_i\}_{i=1}^V\subset\mathbb{R}^D$ to obtain $F$ with distortion $\delta_k$. For any $i,j$, the triangle inequality gives
\begin{equation}
\begin{split}
\bigl|\langle \tilde{E}x_i,\tilde{E}x_j\rangle-\langle x_i,x_j\rangle\bigr| \,&\leq\, \delta_v\,\|x_i\|\|x_j\| + \delta_k\,\|E\,x_i\|\|E\,x_j\| \,\leq \\
&\leq\, \bigl(\delta_v+\delta_k+\delta_v\delta_k\bigr)\|x_i\|\|x_j\| \,\leq \\
&\leq\, \varepsilon_k\,\|x_i\|\|x_j\|
\end{split}
\end{equation}
using $\|E\,x_i\|\leq\sqrt{1+\delta_v}\,\|x_i\|$, where the last inequality holds since $\delta_v,\delta_k\leq\frac{\varepsilon_k}{3}$ and $\varepsilon_k<1$.
\end{proof}

\begin{theorem}[\textbf{Efficient compressive recall circuit}]
Given an MQAR task with vocabulary size $V$, $N_f$ facts and context length $L=4N_f$, a \hyperref[alg:mamba]{\textit{single-layer simplified Mamba}} with dimensions satisfying $ND=O(N_f\log{V})$,  $\mathrm{expand}=2$, $D_\mathrm{conv}=2$ can solve the task \textbf{with high probability}. (See Lem.~\ref{lem:jl-scaling-law} and Thm.~\ref{thm:prob-scaling-law} for the exact quantitative laws.)
\end{theorem}
\begin{proof}[\textbf{Proof}]

We apply quite the same mechanism described in Theorem~\ref{thm:non-compressive-recall} , with modifications to enable compression and decompression.

\paragraph{Projectors.\quad}
As in the non-compressive circuit, let us define similar \textit{previous and current} projectors, now in $\mathbb{R}^{D \times 2D}$ compressed space:
\begin{equation}
\begin{aligned}
M_p &=   \begin{pmatrix}I_D & 0 \end{pmatrix} \\
M_c &=   \begin{pmatrix}0 & I_D \end{pmatrix}
\end{aligned}
\end{equation}
\paragraph{Embeddings.\quad}
Given a general embedding $E\in \mathbb{R}^{D\times V}$,  the embedded input vectors are then $x^e_t=E\,x_t \in \mathbb{R}^D$.

\textbf{Projections.\quad}
We choose:\begin{equation}
\begin{aligned}
P_\mathrm{in} &= M_p^\top+M_c^\top= \bigl(I_D \mid I_D)^\top \\
P_\mathrm{out} &= M_c= \bigl(0\mid I_D)
\end{aligned}
\end{equation}
The input projection simply duplicates the embedded state:
\begin{equation}
    x^p_t=P_\mathrm{in}\,x^e_t=\begin{pmatrix}E\,x_{t}\\E\,x_{t}\end{pmatrix} 
\end{equation}
The output projection extracts the "current" half of duplicated current-previous vectors.

\paragraph{Conv1D kernel.\quad}
We choose
\begin{equation}
  W_\mathrm{conv}=\bigl(W_p\mid W_c)=\bigl(\begin{smallmatrix}1_D\\0_D\end{smallmatrix}\mid \;\begin{smallmatrix}0_D\\1_D\end{smallmatrix}\bigr)  
\end{equation}
such that

\begin{equation}
\begin{aligned}
\mathrm{diag}\bigl(W_c)&=\bigl(0\mid I_D) \\
\mathrm{diag}\bigl(W_p)&=\bigl(I_D\mid 0)
\end{aligned}
\end{equation}

This kernel applies identity and shift operations, to form

\begin{equation}
\hat{x}_t =\mathrm{Conv1D}\bigl(x^p_t,\,W_\mathrm{conv}\bigr)= \begin{pmatrix}x^e_{t-1}\\x^e_{t}\end{pmatrix} = \begin{pmatrix}E\,x_{t-1}\\E\,x_{t}\end{pmatrix}    
\end{equation}

\paragraph{SSM projections and embeddings.\quad}
Let $F\in \mathbb{R}^{N \times D}$ be an embedding matrix. We choose:
  \begin{equation}
  \begin{aligned}
      S_B=F\,M_p=\bigl(F\mid 0) \\
      S_C=F\,M_c=\bigl(0\mid F)
  \end{aligned}
  \end{equation}  
That is to say, $S_B$ outputs a compressed previous token, 
\begin{equation}
    B_t=S_B\,\hat{x}_t=F\,M_p\,\hat{x}_t=F\,x^e_{t-1}=F\,E\,x_{t-1} = \tilde{E} x_{t-1} \in \mathbb{R}^N
\end{equation}
while $S_C$ outputs a compressed current token, 
\begin{equation}
    C_t=S_C\,\hat{x}=F\,M_c\,\hat{x}_t=F\,x^e_t=F\,E\,x_t = \tilde{E} x_t \in \mathbb{R}^N
\end{equation}
where we have introduced $\tilde{E}\equiv F E$.

\paragraph{Overall mechanism.\quad}
Given the context consecutive key-value token pairs $(x_{\tau-1},x_\tau)\equiv (k_\tau,v_\tau)$, and a query token $x_t\equiv q_t$, we notice the matrix setup defined above describes a simple compression-decompression scheme. 

The SSM hidden state and output from~\ref{eq:perfect-ssm-update} are now expressed as:

\begin{equation}
\label{eq:noisy-ssm-update}
\begin{aligned}
    h_t 
    &=\sum_{\tau=0}^t \hat{x}_t\,B_t^\top
    =\sum_{(\tau-1) \in \{\tau_k\}} ^t \begin{pmatrix}E\,x_{\tau-1}\\E\,x_{\tau}\end{pmatrix} \,(\tilde{E}\,x_{\tau-1})^\top  \\
    \hat{y}_t 
    & =h_t \, C_t = \sum_{(\tau-1) \in \{\tau_k\}} ^t \begin{pmatrix}E\,x_{\tau-1}\\E\,x_{\tau}\end{pmatrix} \,(\tilde{E}\,x_{\tau-1})^\top (\tilde{E}\,x_t)
\end{aligned}
\end{equation}
If we recall that $P_\mathrm{out} = M_c= \bigl(0\mid I_D)$, the overall layer output is therefore:
\begin{equation}
y_t = 
E^\top P_\mathrm{out}\,\hat{y}_t
\sum_{(\tau-1) \in \{\tau_k\}} ^t E^\top (E\,x_{\tau-1}) \,(\tilde{E}\,x_{\tau-1})^\top (\tilde{E}\,x_t)
\end{equation}
This expression can be seen as simple bilinear function of the original one-hot (non-compressed) tokens:
\begin{equation}
y_t = \sum_{\tau=0}^t {(E^\top E)\,x_\tau\,{x_{\tau-1}}^\top ({\tilde{E}}^\top \tilde{E})\,\,x_t}
\end{equation}
If we further look at the Gram matrices, defined as
\begin{equation}
    \begin{aligned}
        G_E &\equiv E^\top E \in \mathbb{R}^{V\times V} \\
        G_{\tilde{E}} &\equiv {\tilde{E}}^\top \tilde{E} \in \mathbb{R}^{V\times V}\\
    \end{aligned}
\end{equation}
(which, in our case, corresponds to applying successive compression and decompression), the output vector can now be expressed as:
\begin{equation}
y_t = \sum_{\tau=0}^t G_E\,x_\tau \, {x_{\tau-1}}^\top \,\, G_{\tilde{E}} \,\, x_t
\end{equation}
This would have reduced to a perfect retrieval, as described in the perfect case, only if we had perfect embeddings, satisfying  $G_E=I_V$ and $G_{\tilde{E}} = I_V$, which is, of course, impossible for $N<D<V$.

However, for proper, large enough dimensions, we can approximately achieve perfect recall, as we will now state.

\paragraph{Applying JL Lemma.\quad}
Using \textit{Johnson–Lindenstrauss lemma}~\citep{johnson1984extensions} (Lem.~\ref{app-lem:jl}) twice, for arbitrarily small $0 < \varepsilon_k,\, \varepsilon_v < 1$, if model dimensions are large enough, $D>\frac{c\,\log{V}}{\varepsilon_v^2}$, $N > \frac{c\,\log V}{\varepsilon_k^2}$  for some constant $c$, we can construct embedding matrices $E$, $\tilde{E}$ for which
\begin{equation}
    \bigl|\langle Ex_i, Ex_j \rangle - \langle x_i, x_j \rangle \bigr|
    \,=\bigl|\, x_i^\top  (E^\top  E - I)\, x_j
    \,\bigl|\,\le\, \varepsilon_v \,\|x_i\|\,\|x_j\|
\end{equation}
and similarly
\begin{equation}
    \bigl|\langle \tilde{E}x_i, \tilde{E}x_j \rangle - \langle x_i, x_j \rangle \bigr|
    \,=\bigl|\, x_i^\top  (\tilde{E}^\top  \tilde{E} - I)\, x_j
    \,\bigl|\,\le\, \varepsilon_k \,\|x_i\|\,\|x_j\|
\end{equation}

Note that the two applications are not independent, since $\tilde{E}$ must factor as $\tilde{E}=FE$. Using Lem.~\ref{app-lem:jl-consecutive}, there exist matrices $E$ and $\tilde{E}=FE$ satisfying both JL properties simultaneously.

Since our input vectors are all one-hot (satisfying $\|x_i\|=1$):
\begin{equation}
    \bigl|\langle Ex_i, Ex_j \rangle - \langle x_i, x_j \rangle \bigr|
    \,\le\, \varepsilon_v 
\end{equation}
For such embeddings, we can further bound the inner products for one-hot vectors, by:
\begin{equation}
x_i^\top  G_E \, x_j =(E\, x_i)^\top  (E\, x_j) \in
\begin{cases}
(1-\varepsilon_v,\,1+\varepsilon_v) & \text{if}\;\; i = j \\
(-\varepsilon_v,\,\varepsilon_v) & \text{if}\;\; i \ne j~
\end{cases}
\end{equation}
and similarly:
\begin{equation}
x_i^\top  G_{\tilde{E}} \, x_j =({\tilde{E}}\, x_i)^\top  ({\tilde{E}}\, x_j) \in
\begin{cases}
(1-\varepsilon_k,\,1+\varepsilon_k) & \text{if}\;\; i = j \\
(-\varepsilon_k,\,\varepsilon_k) & \text{if}\;\; i \ne j~
\end{cases}
\end{equation}
Even without an exact analysis (which is further developed later in Lemma~\ref{lem:jl-scaling-law} and Theorem~\ref{thm:prob-scaling-law}), we claim that for small enough $\varepsilon_v$, $\varepsilon_k$, almost-perfect recall is achievable. To see why, let us examine the output vector $y_t$.

\paragraph{Entry-wise analysis.\quad}
As before, let us assume the context contains a single correct fact pair with a key $m$ and a value $i$, both one-hot encoded by token vectors $(k_m,v_i)$.

\label{out-vec-decomposition}
The vector $y_t$ has $V$ entries in total, divided into:
\begin{itemize}
    \item $1$ \textit{correct} value entry $i$, corresponds to the correct fact value $v_i$. 
    \item $N_f-1$ \textit{wrong} value entries, matching $N_f-1$ wrong fact values.
    \item $V_v-N_f$ \textit{empty} value entries, matching values that do not appear in the context. 
    \item $V_k$ \textit{key} entries.
\end{itemize}
For any vector entry $p$, using the one-hot basis vector $e_p$, we can express the entry scalar value using the following bilinear forms:
\begin{equation}
y_t^p = e_p^\top  y_t =\sum_{\tau=0}^t \bigl(e_p^\top  G_E\,x_\tau \bigr)\,\bigl({x_{\tau-1}}^\top G_{\tilde{E}}\;k_m\bigr)
\end{equation}
However, since $E,\tilde{E}$ are JL embeddings, we can state (not quantitatively for now; see Lemma~\ref{lem:jl-scaling-law} and Theorem~\ref{thm:prob-scaling-law}):
\begin{equation}
e_p^\top  G_E \, x_\tau \;\approx\;
\begin{cases}
1 & x_\tau = e_p \\
0 & x_\tau \ne e_p ~ \,,\quad
\end{cases}
x_{\tau-1} ^\top  G_{\tilde{E}}\,k_m \;\approx\;
\begin{cases}
1 & x_{\tau-1} = k_m \\
0 & x_{\tau-1} \ne k_m ~
\end{cases}
\end{equation}
Note that this approximation becomes more exact as the model dimensions grow. 
Since in MQAR the tokens do not repeat, we have both terms $\approx1$ only if $(x_{\tau-1}, \,x_\tau)=(e_p, \,k_m)$. However, this can only be found in the single fact where $k_m$ appears; namely, $(x_{\tau-1}, \,x_\tau)=(v_i, \,k_m)$. This occurs only for this single specific entry $i$, while for all other entries, we have at least one of the terms $\approx0$. Therefore:
\begin{equation}
y_t^p  \;\approx\;
\begin{cases}
1 & p=i \\
0 & p \ne i ~
\end{cases}
\end{equation}
or, more simply, $y_t \approx v_i$.

That is to say, for small enough $\varepsilon_v$, $\varepsilon_k$, or equivalently for large enough $D$, $N$, the simplified linear Mamba model with weights as constructed above can achieve almost-perfect recall.

Quantitatively, the matching term contributes $v_i+O(\varepsilon_v+\varepsilon_k)$, while each of the $N_f-1$ non-matching terms leaves an $O(\varepsilon_v\varepsilon_k)$ residue per entry. For large $N_f$, these $O(\varepsilon_v\varepsilon_k)$ terms dominate; in the worst case, the residues add up to $O(N_f\,\varepsilon_v\varepsilon_k)$, while in the mean case they cancel down to $O\big(\sqrt{N_f}\,\varepsilon_v\varepsilon_k\big)$. Thus, to solve the task, $\sqrt{ND}=O(N_f\log V)$ is required in the worst case, while $ND=O(N_f\log V)$ suffices with high probability. (See Lem.~\ref{lem:jl-scaling-law} for the exact worst-case conditions, and Thm.~\ref{thm:prob-scaling-law} for the exact high-probability law.)

\end{proof}

\section{Mechanistic Interpretability of the Recall Circuit: Supplementary Material}

\subsection{Invariant Operators}
\label{app-subsec:invariant-ops}

$G_{kq}$ is a bilinear operator, operates on token pairs: $\xi_t = \bigl(\begin{smallmatrix} x_{t-1} \\ x_{t} \end{smallmatrix} \bigr)$ and $\xi_{\tau} = \bigl(\begin{smallmatrix} x_{\tau-1} \\ x_{\tau} \end{smallmatrix} \bigr)$. Therefore, it consists of 4 blocks, corresponding to the 4 combinations of $x_t, x_{t-1}$ parts and $x_{\tau}, x_{\tau-1}$ parts. 

As predicted by theory and observed in Fig.~\ref{app-fig:reversing-operators}, $\xi_\tau^\top \,G_{kq} \,\xi_t$ is nonzero only for the $(\tau-1, t)$ block. This corresponds to the mechanism of matching between a query $q_t=x_{t}$ and its corresponding key $k_t=x_{\tau-1}$. Note that $x_\tau$ and $x_{t-1}$ are completely ignored by $G_{kq}$, which makes sense: they are unrelated to the query-key matching procedure.

Similarly, $G_{vv}$ is thought as a matching between token pairs $\xi_{\tau} = \bigl(\begin{smallmatrix} x_{\tau-1} \\ x_{\tau} \end{smallmatrix} \bigr)$ and output tokens $y_t$. Hence, it consists of 2 blocks, corresponding the $x_{\tau}, x_{\tau-1}$ parts and the $y_t$ part. 

With a fit to our theory, Fig.~\ref{app-fig:reversing-operators}, shows that $y_t^\top \,G_{vv} \,\xi_t$ is nonzero only for the $(t,\tau)$ block. This corresponds to the mechanism of retrieving $y_t=x_{\tau}$, which is the correct value given a query-key match $x_t=x_{\tau-1}$. Also note that $G_{vv}$ ignores $x_{\tau-1}$ (the key token): only the value $x_\tau$ matters.

\subsection{SSM Hidden State and Key-Value Tables}
\label{app-subsec:ssm-key-value-tables}

We shortly present and discuss the definitions for $H_t$, $H'_t$ and $H''_t$.

\renewcommand{\arraystretch}{1.25}

\begin{table}[h]
    \caption{Dimensions for hidden state and key-value tables}
    \label{tab:hidden_state_kv_dims}
    \centering
    \begin{tabular}{llll}
        \hline
        \textbf{Symbol} & \textbf{Meaning} & \textbf{Dimension} & \textbf{Notes} \\
        \hline
        
        $h_t^\mathrm{ideal}$         & Hidden state, non-compressive circuit        & $\mathbb{R}^{2V \times V}$  &  As constructed in Thm.~\ref{thm:non-compressive-recall}  \\

        $h_t^\mathrm{noisy}$         & Hidden state, compressive circuit        & $\mathbb{R}^{2D \times N}$  &  As constructed in Thm.~\ref{thm:compressive-recall}  \\

        $H_t$         & Projected hidden state, non-compressive circuit        & $\mathbb{R}^{V \times V}$  &  $=(0 \mid I_V)\,h_t^\mathrm{ideal}$  \\
        
        $H'_t$         & Projected hidden state, compressive circuit        & $\mathbb{R}^{D \times N}$  &  $=(0 \mid I_D)\,h_t^\mathrm{noisy}$  \\
        
        $H''_t$         & Decompressed hidden state, compressive circuit        & $\mathbb{R}^{V \times V}$  &  $=E^\top H'_t \,\tilde{E}$  \\
        
        \hline
    
    \end{tabular}
\end{table}

\renewcommand{\arraystretch}{1.00}

For clarity, let us denote the variables from Thm.~\ref{thm:non-compressive-recall} as \textit{ideal}, and the variables from Thm.~\ref{thm:compressive-recall} as \textit{noisy}.

For the two circuits, ideal and compressive, we now define:
\begin{equation}
H_t\equiv P_\mathrm{out}\,h_t^\mathrm{ideal}=\sum_{n=1}^{N_f}v_n k_n^\top
\end{equation}

\begin{equation}
H'_t\equiv P_\mathrm{out}\,h_t^\mathrm{noisy}=\sum_{n=1}^{N_f}v'_n {k'_n}^\top
\end{equation}

\begin{equation}
H''_t \equiv E^\top H'_t \,\tilde{E} = \sum_{n=1}^{N_f}{v_n''} {k_n''}^\top
\end{equation}

where the \textit{compressed} tokens are $v'_n=E\,v_n$, $k'_n=\tilde{E}\,k_n$, $q'_t=\tilde{E}\,q_t$, and the \textit{decompressed} tokens are $v''_n=E^\top E\,v_n \approx v_n$ and $k''_n=\tilde{E}^\top \tilde{E}\,k_n\approx k_n$. (Note that $E\in \mathbb{R}^{D \times V}$ and $\tilde{E} \in \mathbb{R}^{N \times V}$ are JL matrices, thus achieving approximately perfect compression and decompression, given sufficiently large $D,N$.)

The model output for each of the circuits can be now written as:

\begin{equation}
    y^\mathrm{ideal}_t = H_t\,q_t=\sum_{n=1}^{N_f}v_nk_n^\top q_t \,
\end{equation}

\begin{equation}
y_t^\mathrm{noisy} = E^\top H'_t\,q'_t= E^\top \sum_{n=1}^{N_f} v'_n {k'_n}^\top q'_t \,
\end{equation}

\begin{equation}
y_t^\mathrm{noisy} = H''_t\,q_t=\sum_{n=1}^{N_f} v''_n {k''_n}^\top q_t
\end{equation}

where the latter two equations are equivalent.

Throughout the manuscript, the \textit{projected hidden states} $H_t$, $H'_t$, $H''_t$ are sometimes referred to as \textit{hidden states} as well. The reason is that for MQAR, the \textit{informative} parts of $h_t^\mathrm{ideal}$ and $h_t^\mathrm{noisy}$ (which contain all key-value information) entirely lie in the $H_t$, $H'_t$ halves, respectively.

\subsection{Hidden State as a Hash Table}
\label{app-subsec:hidden-as-hash}
Intuitively, for a perfect recall, the SSM hidden state must encode the full context \textit{information}. This fact-pairs information can be thought of as a simple key-value array, where each key entry contains $1$ in the corresponding fact value $v_\tau$ index, if exists, else $0$ anywhere.
The ideal, non-realistic, circuit of Thm.~\ref{thm:non-compressive-recall} suggests such storage via an outer product key-value memory (see definitions and dimensions in App.~\ref{app-subsec:ssm-key-value-tables}):
\begin{equation}
H_t\equiv P_\mathrm{out}\,h_t=\sum_{n=1}^{N_f}v_n k_n^\top, \quad
y_t=H_t\,q_t=\sum_{n=1}^{N_f}v_nk_n^\top q_t \,
\end{equation}
This is a $V \times V$ table, quite similar to the array presented above, though with values represented as one-hot vectors. A visualization can be found in Fig.~\ref{fig:reversing-hidden} (left).
$H_t$ is fairly interpretable, as entries are simply $1$ where facts exist and $0$ elsewhere. Can we observe such interpretable patterns within trained models? Let us consider the compressive SSM update from Eq.~\ref{eq:compressive-recall-out}. It can be rewritten as
\begin{equation}
H'_t = \sum_{n=1}^{N_f}v'_n {k'_n}^\top, \quad y_t= E^\top \sum_{n=1}^{N_f} v'_n {k'_n}^\top q'_t \,,
\end{equation}
with compressed tokens $v'_n=E\,v_n$, $k'_n=\tilde{E}\,k_n$, $q'_t=\tilde{E}\,q_t$.

\(H'_t\) can be viewed as a \textbf{\textit{hash table}}, for several reasons: (i) both mechanisms rely on a hash function. In our case, the model uses JL linear projections, which is by definition a similarity-preserving hash function; (ii) in both cases, the internal representation is not directly interpretable (Fig.~\ref{fig:reversing-hidden}, middle), while the final output is fairly interpretable (Fig.~\ref{fig:reversing-hidden}, right); (iii) efficiency is achieved by design through approximation, where the error is bounded with high probability due to the properties of the hash function; (iv) the interface closely follows that of a hash table: First, keys and queries are first hashed via \(\mathrm{hash}(x) = \tilde{E}x\). Next, the model \textit{stores} a fact into the table via $H_t=H_{t-1}+\text{encode}(v_n) \, \text{hash}(k_n)^\top$. Finally, when handling a query, the model \textit{matches} \(\mathrm{hash}(q_t)\) with \(\mathrm{hash}(k_i)\) for all entries \(i\), using inner products. If a key $k_n$ exists such that \(\mathrm{hash}(k_n) \approx \mathrm{hash}(q_t)\), the corresponding value is approximately \textit{retrieved}.

\subsection{Inspecting the Hash Table}
\label{app-subsec:inspect-hash}

The starting point of our analysis is the \textit{non-rotated} recall circuit, as constructed in Theorem~\ref{thm:compressive-recall}, where $P_\mathrm{in}$, $P_\mathrm{out}$, and $W_\mathrm{conv}$ are carefully designed with identity and zero matrices and vectors. In this setup, $S_B$ and $S_C$ must have a structure $S_B=(F\mid 0)$ and $S_C=(0\mid F)$ in order to solve MQAR. If we fix $P_\mathrm{in}$, $P_\mathrm{out}$, and $W_\mathrm{conv}$ as in the construction, then indeed $F$ can be extracted from learned model weights, allows to compute $\tilde{E}$, then to compute Equation~\ref{eq:h-double-prime} and $H''=E^\top H' \,\tilde{E}$. 

However, in the general \textit{rotated} case, where $P_\mathrm{in}$, $P_\mathrm{out}$, and $W_\mathrm{conv}$ are not constrained, the operator $F$ does not exist as-is in any of the model weights: $S_B$ and $S_C$ do not have the simple structure described in Theorem~\ref{thm:compressive-recall}, since the key and query compression operations $F$ are distributed across a mixture of linear operators $S_B$, $S_C$, $P_\mathrm{in}$, and $W_\mathrm{conv}$. Therefore, we must use the net operations $\Pi_{k,\mathrm{in}}$ and $\Pi_{q,\mathrm{in}}$ (see Equation~\ref{eq:proj-opertators}), which correctly combine all these linear transformations. Decompression via $k''_n=\tilde{E}^\top \tilde{E}\,k_n$ is a special case, which only applies for the \textit{non-rotated} circuit; its general \textit{rotation-invariant} form is $k''_n=G_{kq}^\top \,k_n$. Similarly, $H''_t=E^\top H'_t \,\tilde{E}$ is a special \textit{non-rotated} case; its general invariant form is $H''_t=\Pi_{v,\mathrm{out}}\,H'_t\; \Pi_{q,\mathrm{in}}$, which is derived by replacing $E^\top$ and $\tilde{E}$ with their general equivalent operators from Equation~\ref{eq:proj-opertators}.

\subsection{Noisy Attention Maps}
\label{noisy-attn}
In addition to the hash-like recurrent view of the recall mechanism, we can alternatively view it through the lens of implicit attention maps. Following \citet{ali2024hidden,dao2024transformers}, for our simplified model, the attention matrix of the ideal circuit becomes $\alpha_{\tau t}=B_\tau^\top\,C_t=x_{\tau-1}^\top x_t$, such that $y_t=\sum_{\tau=0}^t x_\tau \,\alpha_{\tau t}$, equivalently to Eq.~\ref{eq:perfect-recall-out}. Simply put, this formulation implies that given a query $x_t=q_t$, the model attends to previous tokens that satisfy $x_{\tau-1}=q_t$. For the realistic, compressive circuit, we have, based on Eq.~\ref{eq:compressive-recall-out}, $\alpha'_{\tau t}=B_\tau^\top\,C_t=(\tilde{E}\,x_{\tau-1})^\top (\tilde{E}\,x_t)=x_{\tau-1}^\top\tilde{E}^\top \tilde{E}\,x_t \approx x_{\tau-1}^\top x_t$, hence $\alpha'_{\tau t} \approx \alpha_{\tau t}$. (In equivalent circuits, the map is more complicated, $\alpha'_{\tau t}=\xi_\tau^\top G_{kq} \,\xi_t$, yet still $\approx\alpha_{\tau t}$). Indeed, computing these maps for both ideal and trained compressive models, we verify in Fig.~\ref{fig:reversing-attn} that the learned map $\alpha'_{\tau t}$ functions as an approximation to the ideal one $\alpha_{\tau t}$, revealing a circuit-level resemblance between regimes.

\section{Designed Model Baseline}
\label{app:designed-model}

In this appendix, we describe an example designed model with a specific weight construction. This model serves as a sub-optimal baseline solution for associative recall tasks, which is simpler for mathematical analysis. Later, in App.~\ref{app-thm:designed-scaling-law}, we theoretically analyze the model performance on (single-query) AR and MQAR tasks. The designed model and the theoretical results are used in Section~\ref{sec:scaling-laws-exp} and Figure~\ref{app-fig:scaling-laws-2d-dims}, compared to the trained models.

\subsection{Designed Model Weights}
\label{app-subsec:designed-weights}

\paragraph{Design of $\mathbf{E}$ and $\mathbf{F}$.}
Following Theorem~\ref{thm:compressive-recall}, given general matrices $E\in \mathbb{R}^{D\times V}$ and $F\in \mathbb{R}^{N\times D}$, we choose:
\begin{equation}
    \label{app-eq:compressive-recall-weights}
    \begin{gathered}
        P_\mathrm{in} =  \bigl(\begin{smallmatrix} I_D \\ I_D \end{smallmatrix} \bigr)
        \,, \quad 
        P_\mathrm{out} = \bigl(0\mid I_D)\,, \quad%
    W_\mathrm{conv}=\bigl(\begin{smallmatrix}1_D\\0_D\end{smallmatrix}\mid \;\begin{smallmatrix}0_D\\1_D\end{smallmatrix}\bigr)\,, \quad
    S_B=\bigl(F\mid 0)\,, \quad
      S_C=\bigl(0\mid F)
    \end{gathered}
\end{equation}

We would now like to choose specific $E$ and $F$ matrices.

The embedding matrix $E$ is constructed as follows: 
\begin{itemize}
    \item Sample a random matrix $G \in \mathbb{R}^{D \times V}$ with i.i.d.\ standard normal entries.
    \item Rescale each column $G_i$ by its norm $\lVert G_i \rVert$, to form a matrix $\hat{G} \in \mathbb{R}^{D \times V}$ with columns of unit norm, $\lVert \hat{G}_i \rVert=1$. 
    \item Set the embedding matrix $E = \hat{G} \in \mathbb{R}^{D \times V}$.
\end{itemize}

The matrix $F$ is constructed as follows: 
\begin{itemize}
    \item Sample a random matrix $P \in \mathbb{R}^{D \times D}$ with i.i.d.\ standard normal entries.
    \item Extract an orthonormal basis $Q \in \mathbb{R}^{D \times D}$ via a QR decomposition of $P$.
    \item Select its first $N$ columns, $U=Q[:N]\in \mathbb{R}^{N \times D}$.
    \item Scale the matrix by a scalar factor $a\in\mathbb{R}$: $F=a\,U \in \mathbb{R}^{N \times D}$.
    \item We later choose $a=\sqrt{\frac{D}{N}}$.
\end{itemize}

As in Theorem~\ref{thm:compressive-recall}, we define $\tilde{E} \equiv F E \in \mathbb{R}^{N \times V}$. We further define the Gram matrices:

\begin{equation}
    \begin{aligned}
        G_E &\equiv E^\top E \in \mathbb{R}^{V\times V} \\
        G_{\tilde{E}} &\equiv {\tilde{E}}^\top \tilde{E} \in \mathbb{R}^{V\times V}\\
    \end{aligned}
\end{equation}

We are interested in the statistics of $G_E$ and $G_{\tilde{E}}$, which we will later use to evaluate the accuracy of the designed model on AR and MQAR. Note that for large $V$, the entries of $G_E$ and $G_{\tilde{E}}$ obtained from a single realization of $E$ and $F$ exhibit empirical statistics (e.g. mean and variance) that closely match the corresponding ensemble statistics, averaged over different realizations of $E$ and $F$.

\paragraph{Statistics of $\mathbf{G_E}$.}
Each column $e_i$ of the embedding matrix $E\in\mathbb{R}^{D\times V}$ is obtained by normalizing an i.i.d.\ Gaussian vector and is therefore uniformly distributed on the unit sphere $S^{D-1}$. As a consequence, the Gram matrix $G_E=E^\top E$ has diagonal entries $(G_E)_{ii}=\|e_i\|^2=1$ deterministically. For $i\neq j$, the off-diagonal entries $(G_E)_{ij}=\langle e_i,e_j\rangle$ are independent samples from the inner-product distribution of two random unit vectors in $\mathbb{R}^D$, which is known to have a probability density:
\begin{equation}
f_D(t)= \begin{cases}
     \frac{\Gamma(\frac{D}{2})}{\sqrt{\pi}\,\Gamma(\frac{D-1}{2})}{(1-t^2)}^{\frac{D-3}{2}} & t\in[-1,1] \\
     0 & \text{otherwise}
\end{cases}
\end{equation}
Alternatively, $(G_E)_{ij}=2Z-1$ where $Z \sim \mathrm{Beta}(\frac{D-1}{2}, \frac{D-1}{2})$.
This distribution is centered around $0$, with:

\begin{equation}\label{eq:G-E-stats-ij}
\mathbb{E}(G_E)_{ij} = 0,\qquad 
\mathrm{Var}(G_{E})_{ij} = \frac{1}{D}
\end{equation}

a variance of $\frac{1}{D}$. For large $D$, the distribution of $(G_E)_{ij}$ is well approximated by $\mathcal{N}(0,\frac{1}{D})$.

\paragraph{Statistics of $\mathbf{G_{\tilde{E}}}$.}

The projected embedding matrix is defined as $\tilde E=F E=a\,U E$, where the rows of $U\in\mathbb{R}^{N\times D}$ form an orthonormal basis of a random $N$-dimensional subspace. Writing $\Pi=U^\top U$ for the associated rank-$N$ orthogonal projector, the Gram matrix of $\tilde E$ can be written as $G_{\tilde E}=a^2 E^\top \Pi E$. Its diagonal entries satisfy $(G_{\tilde E})_{ii}=a^2\,Z$, where $Z=e_i^\top \Pi e_i  \sim \mathrm{Beta}(\frac{N}{2},\frac{D-N}{2})$ distribution. Consequently,
\begin{equation}
\mathbb{E}(G_{\tilde E})_{ii}=a^2\frac{N}{D},\qquad 
\mathrm{Var}(G_{\tilde E})_{ii}=a^4\frac{\frac{N}{D}(1-\frac{N}{D})}{D+2} \approx a^4 \frac{N}{D}\left(1-\frac{N}{D}\right)\frac{1}{D}
\end{equation}
For $i\neq j$, the off-diagonal entries are given by $(G_{\tilde E})_{ij}=a^2 e_i^\top \Pi e_j$, which are centered around $0$ with variance $a^4 \frac{N}{D^2}$, and are well approximated by $\mathcal{N}(0,a^4 \frac{N}{D^2})$ in the high-dimensional regime.
If we now choose $a=\sqrt{\frac{D}{N}}$, we have:

\begin{equation}\label{eq:G-E-tilde-stats-ii}
\mathbb{E}(G_{\tilde E})_{ii} = 1,\qquad 
\mathrm{Var}(G_{\tilde E})_{ii} \approx \frac{1}{N}-\frac{1}{D}
\end{equation}

\begin{equation}\label{eq:G-E-tilde-stats-ij}
\mathbb{E}(G_{\tilde E})_{ij} = 0,\qquad 
\mathrm{Var}(G_{\tilde E})_{ij} = \frac{1}{N}
\end{equation}

\section{Mamba Recall Scaling Laws: Proofs}\label{app:scaling-laws-proofs}

\subsection{JL-Based Scaling Laws}

\begin{lemma}[\textbf{JL-based recall scaling laws}]\label{app-lem:jl-scaling-law}
Given a vocabulary size $V$, a \hyperref[alg:mamba]{\textit{single-layer simplified Mamba}} with dimensions $D=\frac{4\log{V}}{\varepsilon_v^2}$, $N=\frac{4\log{V}}{\varepsilon_k^2}$, $\mathrm{expand}=2$, $D_\mathrm{conv}=2$ can perfectly solve an AR task if model dimensions satisfy $0 < \varepsilon_v < 1$, $0 < \varepsilon_k < 1$ and
\begin{equation}
\varepsilon_v + \varepsilon_k +2N_f\,\varepsilon_v\,\varepsilon_k< \frac{1}{2}.
\end{equation} 
The model can perfectly solve an MQAR task if model dimensions further satisfy 
\begin{equation}
\varepsilon_v + \varepsilon_k +L\,\varepsilon_v\,\varepsilon_k< \frac{1}{2}.
\end{equation}
\end{lemma}
\begin{proof}[\textbf{Proof}]

We use the lower and upper bounds derived from the \textit{JL lemma} (Lem.~\ref{app-lem:jl}) to separate the correct value output entry from all other entries.

\paragraph{JL bounds.}
Assume the model weights are constructed as described in Theorem~\ref{thm:compressive-recall}.
Given a query $x_t=q_t=k_m$, the model output vector $y_t$ has value entries given by: (since $v_p$ are one-hot vectors)
\begin{equation}\label{eq:y-p-sum}
y_t^p = v_p^\top  y_t =\sum_{\tau=0}^t \bigl(v_p^\top  G_E\,x_\tau \bigr)\,\bigl({x_{\tau-1}}^\top G_{\tilde{E}}\;k_m\bigr)
\end{equation}
Since $E,\tilde{E}$ are both JL embeddings, we have:
\begin{equation}
v_p^\top  G_E\,x_\tau\in
\begin{cases}
(1-\varepsilon_v,\,1+\varepsilon_v) & x_\tau = v_p \\
(-\varepsilon_v,\,\varepsilon_v) & x_\tau \ne v_p ~
\end{cases}
\end{equation}
\begin{equation}
x_{\tau-1}^\top \,  G_{\tilde{E}}\,k_m \;\in\;
\begin{cases}
(1-\varepsilon_k,\,1+\varepsilon_k) & x_{\tau-1} = k_m \\
(-\varepsilon_k,\,\varepsilon_k) & x_{\tau-1} \ne k_m ~
\end{cases}
\end{equation}
For simplicity, let us first examine a single-query retrieval, where a query $q_t$ has a sequence index $t$. We thus assume the context contains $t-1$ tokens previous to the query.%
\paragraph{Entry-wise analysis.}
We start with the output vector $y_t$ decomposition. For each entry $p$, we analyze the contribution of each token pair in the context to the sum in Eq.~\ref{eq:y-p-sum}.
For brevity, we denote $g_\tau=v_p^\top  G_E\,x_\tau$ and $\tilde{g}_\tau = x_{\tau-1}^\top \,  G_{\tilde{E}}\,k_m$. Given a value entry $p$, for each token pair $(x_{\tau-1}, x_\tau)$, we say we have a \textit{value match} if $x_\tau=v_p$ (thus $g_\tau \approx 1$) and a \textit{key match} if $x_{\tau-1}=k_m$ (such that $\tilde{g}_\tau \approx1$).  As discussed in~\ref{app-rem:key-value-selectivity}, we assume key and value embeddings are orthogonal.

For the \textit{\textbf{correct}} entry $i$, the total $t$ context token pairs are divided into: %

\begin{itemize}
\item $t$ token pairs:
\begin{itemize}
    \item $1$ \textit{correct} fact pair: for this pair, both key and value match, $(x_{\tau-1},x_\tau)=(k_m,v_i)$. Hence, applying JL lower bound, $g_\tau\tilde{g}_\tau>(1-\varepsilon_v)(1-\varepsilon_k)$.
\end{itemize}
    \begin{itemize}
        \item $t-1$ other token pairs: since key and values are unique (and $k_m,v_i$ have already appeared in the correct fact), for each of these pairs we have $x_{\tau-1}\neq k_m$ and $x_\tau \neq v_i$. Therefore, using JL bound, $g_\tau\tilde{g}_\tau>-\varepsilon_v \varepsilon_k$.
    \end{itemize}
\end{itemize}

Applying JL lemma lower bound, the contributions add up to:
\begin{align}\label{eq:yi}
y_t^i \;&>\; (1-\varepsilon_v)(1-\varepsilon_k) 
   \,- (t-1)\varepsilon_v\varepsilon_k 
   \approx\; 1 - \varepsilon_v - \varepsilon_k 
   - t\, \varepsilon_v \varepsilon_k
\end{align}
For each of the $N_f-1$ \textit{\textbf{wrong}} entries $j$, we have:
\begin{itemize}
    \item $t$ token pairs:
    \begin{itemize}
        \item $1$ \textit{correct} fact: for this fact, we have a key match and a value mismatch, since $x_{\tau-1}= k_m$ and $x_\tau = v_i\neq v_j$. Using JL upper bound, this fact contributes $g_\tau \tilde{g}_\tau < \varepsilon_v (1+\varepsilon_k)$.
    \end{itemize}
    \begin{itemize}
        \item $1$ \textit{wrong} fact with $x_\tau=v_j$: we have a key mismatch, since $x_{\tau-1}= k_n\neq k_m$, and a value match, since $x_\tau = v_j$ match entry index $j$. Hence, applying JL bound, this fact contributes  $g_\tau \tilde{g}_\tau < (1+\varepsilon_v) \varepsilon_k$.
        \item $t-2$ \textit{mismatching} token pairs with $x_\tau\neq v_j$: as before, we have key mismatch, and now also a value mismatch, since $x_\tau\neq v_j$. These facts contribute $g_\tau \tilde{g}_\tau < \varepsilon_v \varepsilon_k$.
    \end{itemize}
\end{itemize}
Applying JL lemma upper bound, we have:
\begin{align}\label{eq:yj}
y_t^j \;&<\; \varepsilon_v (1+\varepsilon_k) +(1+\varepsilon_v)\varepsilon_k 
   \,+ (t-2)\varepsilon_v\varepsilon_k
= \varepsilon_v +\varepsilon_k 
   \,+t\,\varepsilon_v\varepsilon_k
\end{align}
For each of the rest $V-N_f$ \textit{\textbf{empty}} entries $l$, corresponding to tokens $y_l$ which do no appear in any context fact, we simply have:

\begin{itemize}
    \item $t$ token pairs: 
    \begin{itemize}
        \item 1 \textit{correct} fact: as above, a key match and a value mismatch.
        \item  $t-1$ \textit{mismatching} pairs: a key mismatch and a value mismatch.
    \end{itemize}
\end{itemize}

Applying JL lemma upper bound, we have:
\begin{align}\label{eq:yl}
y_t^l \;&<\; \varepsilon_v (1+\varepsilon_k) + (t-1)\,\varepsilon_v\varepsilon_k =
\varepsilon_v +t\,\varepsilon_v\varepsilon_k 
\end{align}

\paragraph{Perfect recall condition.}
For perfect retrieval, we want the correct-value entry $i$  to be maximal:
\begin{equation}
  y_t^i=\max_p \; y_t^p
\end{equation}
Using our JL-based bounds, $y_t^i>y_t^{p \neq i}$  is achieved when
\begin{equation}
    1 - \varepsilon_v - \varepsilon_k 
   - t \varepsilon_v \varepsilon_k 
   > 
   \varepsilon_v +\varepsilon_k 
   \,+t\,\varepsilon_v\varepsilon_k
\end{equation}
or, simpler, if
\begin{equation}\label{app-eq:jl-scaling-law-eps}
     \varepsilon_v + \varepsilon_k 
   + t \varepsilon_k \varepsilon_v < \frac{1}{2}
\end{equation}
\paragraph{From JL bound to a scaling law.}
We can now think of $\varepsilon_v$, $\varepsilon_k$ as of new, "natural" dimension variables, replacing the original dimensions as in a coordinate transformation:
\begin{equation}
\begin{aligned}
     \varepsilon_v &\equiv \sqrt{\frac{c\,\log{V}}{D}}  \\ 
     \varepsilon_k &\equiv \sqrt{\frac{c\,\log{V}}{N}}  \\ 
\end{aligned}
\end{equation}
JL lemma generally requires $\varepsilon_v, \varepsilon_k <1$, or equivalently, $D,N>c\,\log{V}$. To these requirements we add the specific MQAR condition given in equation~\ref{app-eq:jl-scaling-law-eps}, which can now be seen as a \textit{scaling law}:
\begin{equation}
     \sqrt{\frac{c\,\log{V}}{D}} +\,\sqrt{\frac{c\,\log{V}}{N}}
   + \,t {\frac{c\,\log{V}}{\sqrt{ND}}} < \frac{1}{2}
\end{equation}

Choosing $t=2N_f$ yields the condition $\varepsilon_v < 1$, $\varepsilon_k < 1$ and $\varepsilon_v + \varepsilon_k +2N_f\,\varepsilon_v\,\varepsilon_k< \frac{1}{2}$ to guarantee perfect recall on AR. For MQAR, if dimensions satisfy $\varepsilon_v + \varepsilon_k +L\,\varepsilon_v\,\varepsilon_k< \frac{1}{2}$, we are guaranteed to perfectly handle all queries, at any $2N_f<t \leq L$.

\end{proof}

\subsection{Probabilistic Scaling Laws: Designed Model}

In the following Lemma, we analyze the predicted success probability of the designed model (see App.~\ref{app-subsec:designed-weights}) solving AR and MQAR tasks.

\begin{theorem}[\textbf{Designed model scaling laws}]\label{app-thm:designed-scaling-law}\label{app-thm-designed-ar}
A \hyperref[alg:mamba]{\textit{single-layer simplified Mamba}} with weights as designed in~\ref{app-subsec:designed-weights} solves an AR task with a probability approximated by %
{\small 
\begin{equation} \label{app-eq:success-prob-designed-AR-large-Nf}
    p_\mathrm{AR} \approx \Phi \left( \sqrt{\frac{ND}{2N_f}} - \sqrt{2\log{V}} \right),
\end{equation}%
}
and an MQAR task of $L=4N_f$ with a probability
{\small \begin{equation} \label{app-eq:success-prob-designed-MQAR}
    p_\mathrm{MQAR} \approx 
    \frac{1}{L - 2N_f} \sum _{t=2N_f} ^{L} 
    \Phi \left( \sqrt{\frac{ND}{t}} - \sqrt{2\log{V}} \right),
\end{equation}}
in the limits of $N_f \gg N,D$ and $V \gg 1$, where $\Phi$ is the standard normal CDF.
\end{theorem}

\begin{proof}

Let us denote the input sequence by $x=x_1,x_2,\dots,x_t=k_1,v_1,k_2,v_2\dots,k_{N_f},v_{N_f},q$, where the last token is the query $x_t=q=k_m\in \mathcal{V}_k$. Given the model weights as designed in~\ref{app-subsec:designed-weights}, the model output $y_t$ is:

\begin{equation}
    y_t 
    = E^\top \sum_\tau^t (E x_\tau) {(\tilde{E} x_{\tau-1})}^\top (\tilde{E} q)
\end{equation}

Thus, the $i^\text{th}$ coordinate of $y_t$ can be written using the one-hot vector $v_i=e_i$:

\begin{equation}
\begin{aligned}
    y_t^i 
    &= v_i^\top y_t = \\
    &= v_i^\top E^\top \sum_\tau^t (E x_\tau) {(\tilde{E} x_{\tau-1})}^\top (\tilde{E} q) = \\
    &= \sum_\tau^t (Ev_i)^\top (E x_\tau) {(\tilde{E} x_{\tau-1})}^\top (\tilde{E} k_m) = \\
    &= \sum_\tau^t (v_i^\top G_E \,x_\tau) (x_{\tau-1}^\top G_{\tilde{E}}\, k_m) = \\
    &= \sum_{n=1}^{N_f} (v_i^\top G_E \,v_n) (k_n^\top G_{\tilde{E}}\, k_m) + 
    \sum_{n=1}^{N_f} (v_i^\top G_E \,k_{n+1}) (v_n^\top G_{\tilde{E}}\, k_m) \\
\end{aligned}
\end{equation}

where in the last equality we split the total $2N_f$ pairs $(x_{\tau-1}, x_\tau)$ into $N_f$ key-value fact pairs $(k_n,v_n)$ and $N_f$ value-key "off-fact" pairs $(v_n,k_{n+1})$. %

Defining random variables $S_{ii}$ and $Z_{ij}$ ("signal" and "zero" terms) for the diagonal and off-diagonal entries of $G_E$, we have:
\begin{equation}
x_i^\top G_E\, x_j \;=\;
\begin{cases}
S_{ii} & i = j \\
Z_{ij} & i \ne j 
\end{cases}
\end{equation}

Similarly, for $G_{\tilde{E}}$, we define $\tilde{S}_{ii}$ and $\tilde{Z}_{ij}$:

\begin{equation}
x_i^\top G_{\tilde{E}}\, x_j \;=\;
\begin{cases}
\tilde{S}_{ii} & i = j \\
\tilde{Z}_{ij} & i \ne j
\end{cases}
\end{equation}

\paragraph{Entry-wise analysis.} 
Now let us assume the retrieved fact is $(k_m,v_i)$. We use $i$ for the \textit{correct} value entry, $j$ for \textit{wrong} value entries (associated with "wrong facts" $(k_n,v_j)$ in the context) and $l$ for \textit{empty} entries (all other entries; also see Equations~\ref{eq:yi}, ~\ref{eq:yj}, ~\ref{eq:yl}). Then, the output entries become:

\begin{align}
y_t^i &= S_{ii} \, \tilde{S}_{mm} + 
\sum^{2N_f-1}_p Z_p \tilde{Z}_p  \\
y_t^j &= S_{jj} \, \tilde{Z}_{nm} + Z_{ji} \, \tilde{S}_{mm} +
\sum^{2N_f-2}_p Z_p \tilde{Z}_p \\
y_t^l &= Z_{li} \, \tilde{S}_{mm} + 
\sum^{2N_f-1}_p Z_p \tilde{Z}_p
\end{align}

However, for the specific weights of the designed model, $S_{ii}^d=S_{jj}^d=1$ deterministically, and $\tilde{S}_{mm}^d$ is distributed around $1$, such that overall, the mean and variance of the \textit{correct}, \textit{wrong} and \textit{empty} coordinates are, respectively:

\begin{equation} \label{eqs-variance}
\begin{aligned} 
\mu_c & =\mathbb{E}(y_t^i) = 1,\qquad
\sigma_c^2=\mathrm{Var}(y_t^i) = \tilde{\sigma}_S^2 + 
(2N_f-1) {\sigma}_Z^2 \, \tilde{\sigma}_Z^2  \\
\mu_w & =\mathbb{E}(y_t^j) = 0,\qquad
\sigma_w^2=\mathrm{Var}(y_t^j) = \tilde{\sigma}_Z^2 + (1+\tilde{\sigma}_S^2) \, \sigma_Z^2 +
(2N_f-2) {\sigma}_Z^2 \, \tilde{\sigma}_Z^2  \\
\mu_e & =\mathbb{E}(y_t^l) = 0,\qquad
\sigma_e^2=\mathrm{Var}(y_t^l) = (1+\tilde{\sigma}_S^2) \, \sigma_Z^2 + 
(2N_f-1) {\sigma}_Z^2 \, \tilde{\sigma}_Z^2 
\end{aligned}
\end{equation}

where $\sigma_Z^2$, $\tilde{\sigma}_Z^2$, $\tilde{\sigma}_S^2$ are the variances of $Z$, $\tilde{Z}$ and $\tilde{S}$ respectively. Using Equations~\ref{eq:G-E-stats-ij}, ~\ref{eq:G-E-tilde-stats-ii}, ~\ref{eq:G-E-tilde-stats-ij}, we can use $\sigma_Z^2=\frac{1}{D}$, $\tilde{\sigma}_Z^2=\frac{1}{N}$, $\tilde{\sigma}_S^2=\frac{1}{N}-\frac{1}{D}$ to compute entry variances $\sigma_c^2$, $\sigma_w^2$, $\sigma_e^2$ for the designed model output $y_t$.

\paragraph{Normal approximation.} 
For large enough $N_f$, using Central Limit Theorem, we have at each entry a summation of many independent random variables. We can thus approximate:

\begin{equation} \label{eqs-normal}
    y_t^i \sim \mathcal{N}(1,\sigma_c^2), \qquad
    y_t^j \sim \mathcal{N}(0,\sigma_w^2), \qquad
    y_t^l \sim \mathcal{N}(0,\sigma_e^2)
\end{equation}

Note that we have a single correct $i$ entry, $n_w \equiv N_f-1$ wrong $j$ entries, and $n_e \equiv V-N_f$ empty $l$ entries. The model outputs the correct entry if $i = \arg\max_p y^p_t$. In a probabilistic manner:

\begin{equation}
    p_\mathrm{success}=\text{Pr}(y^i_t>y^p_t \quad \forall p\neq i)
\end{equation}

Now let us condition on the realized value of the correct component $y_t^i=x$.
Given this value, all other entries are independent. The probability that a single wrong entry satisfies $y_t^j<x$ is $\Phi(\frac{x}{\sigma_w})$, while the probability that a single empty entry satisfies $y_t^\ell<x$ is $\Phi(\frac{x}{\sigma_e})$, where $\Phi$ is the standard normal CDF. Therefore, the conditional probability that $y_t^i$ exceeds all other entries is:
\begin{equation}
\text{Pr}(y^i_t>y^p_t\ \forall p\neq i \mid y_t^i=x)
=
\left[\Phi\!\left(\frac{x}{\sigma_w}\right)\right]^{n_w}
\left[\Phi\!\left(\frac{x}{\sigma_e}\right)\right]^{n_e}
\end{equation}

The overall success probability is thus given by:

\begin{equation} \label{eq:p-succ-exact}
    p_\mathrm{success} 
= \int_{-\infty}^{\infty}
\frac{1}{\sigma_c}
\phi\left(\frac{x-1}{\sigma_c}\right)
\left[\Phi\!\left(\frac{x}{\sigma_w}\right)\right]^{n_w}
\left[\Phi\!\left(\frac{x}{\sigma_e}\right)\right]^{n_e}
\, dx
\end{equation}

where $\phi$ is the standard normal PDF.

\paragraph{Large-$N_f$ approximation.} 
In the limit of $N_f \gg N,D$, we claim the variance terms in Eq.~\ref{eqs-variance} are all dominated by the "noise" component $2N_f \, {\sigma}_Z^2 \, \tilde{\sigma}_Z^2$, and are thus approximately equal:

\begin{equation}
    \sigma_c^2 \approx \sigma_w^2 \approx \sigma_e^2 \approx 
    2N_f {\sigma}_Z^2 \, \tilde{\sigma}_Z^2=\frac{2N_f}{N D} \equiv \sigma^2
\end{equation}

Using this approximation, the success probability in Eq.~\ref{eq:p-succ-exact} becomes:

\begin{equation} \label{eq:p-succ-appx-large-Nf}
\begin{aligned}
    p_\mathrm{success} 
&= \int_{-\infty}^{\infty}
\frac{1}{\sigma}
\phi\left(\frac{x-1}{\sigma}\right)
\left[\Phi\!\left(\frac{x}{\sigma}\right)\right]^{n_w}
\left[\Phi\!\left(\frac{x}{\sigma}\right)\right]^{n_e}
\, dx = \\
&= \int_{-\infty}^{\infty}
\frac{1}{\sigma}
\phi\left(\frac{x-1}{\sigma}\right)
\left[\Phi\!\left(\frac{x}{\sigma}\right)\right]^{V-1}
\, dx = \\
&= \int_{-\infty}^{\infty}
\phi\left(z-\frac{1}{\sigma}\right)
\left[\Phi\!\left(z\right)\right]^{V-1}
\, dz
\end{aligned}
\end{equation}

where we use $n_w+n_e=V-1$ and $z=\frac{x}{\sigma}$.
\paragraph{Large-$V$ approximation.} If we further assume $V \gg 1$, which is typical for the AR task, the expression for the success probability can be rather simplified. 
Define $a \equiv \frac{1}{\sigma}$ and let $Z\sim\mathcal{N}(0,1)$. Then Eq.~\eqref{eq:p-succ-appx-large-Nf} can be rewritten as an expectation
\begin{equation}
    p_\mathrm{success}
    =\mathbb{E}\Big[\Phi(Z+a)^{V-1}\Big]
\end{equation}
Equivalently, let $U_1,\dots,U_{V-1}\stackrel{\text{i.i.d.}}{\sim}\mathcal{N}(0,1)$ be independent of $Z$, and denote their maximum by $M_{V-1} = \max_{p} U_p$. Conditioning on $Z$ gives
\begin{equation}
    \Pr\!\big(M_{V-1}\le Z+a \,\big|\, Z\big)=\Phi(Z+a)^{V-1}
\end{equation}
and therefore
\begin{equation}
    p_\mathrm{success}
    =\Pr\!\big(M_{V-1}\le Z+a\big)
    =\Pr\!\big(Z+a\ge M_{V-1}\big)
\end{equation}

The maximum of $V-1$ Gaussians concentrates near a typical level of
\begin{equation}\label{app-eq:m-V}
m_V = \sqrt{2\log V}\left(1 - \frac{\log\log V + \log 4\pi}{4\log V} + O\!\left(\left(\frac{\log\log V}{\log V}\right)^{\!2}\right)\right) \approx \sqrt{2\log V}
\end{equation}
Approximating $M_{V-1}$ by its typical value $m_V$ yields:
\begin{equation}
    p_\mathrm{success}
    =\Pr\!\big(Z \ge M_{V-1}-a\big)
    \approx \Pr\!\big(Z \ge m_V-a\big)
    =\Phi(a-m_V)
\end{equation}

Thus finally:
\begin{equation}
    p_\mathrm{AR} =
    p_\mathrm{success} \approx 
    \Phi \left( \frac{1}{\sigma} - m_V \right) =
    \Phi \left( \sqrt{\frac{ND}{2N_f}} - \sqrt{2\log{V}} \right)
\end{equation}

\paragraph{MQAR.}

Given a query $q_t=k_m$ at $2N_f<t \leq L$, associated with a value $v_i$, and assuming large $N_f$, we now have $y_t^p\approx \delta^{ip} + \sum_{(a,b)}^{t-1}\,\epsilon_v^{pa}\,\epsilon_k^{bm}$, since we sum over all the token pairs up to index $t$. We thus have $y_t^p \sim \mathcal{N}(\delta^{ip},\sigma_t^2)$, with $\sigma_t^2 \approx \frac{t}{ND}$. Repeating the $V \gg 1$ approximation, we have a success probability for a single query given by $p(t)=\Pr\!\big(y_t^i > \sigma_t \,m_V \big) \approx \Phi(1/\sigma_t-m_V)$. The overall success probability is the average over queries, $p_{MQAR}=\frac{1}{L - 2N_f} \sum _{t=2N_f} ^{L} p(t)$, which yields Eq.~\ref{app-eq:success-prob-designed-MQAR}.

\end{proof}

\color{black}

\subsection{Probabilistic Scaling Laws: Trained Model}

\begin{remark}[\textbf{Key-value selection mechanism.}]
\label{app-rem:key-value-selectivity} 
Let us consider the following insight. If a realistic, compressive model is able to perfectly ignore non-key-value $(x_{\tau-1}, x_\tau)$ pairs, then its recall abilities may improve. Such non-key-value pairs include, i.e. value-key pairs between context facts, and key-key or value-value \textit{random non-query} token pairs between queries. 

Such selectivity is achievable if the learned embeddings $E_k$ and $E_v$ span orthogonal subspaces (which we denote shortly as $E_k \perp  E_v$), and similarly for $\tilde{E}_k$ and $\tilde{E}_v$. Let us analyze it case by case, considering a query $x_t=k_m\in\mathcal{V}_k$. 
\begin{itemize}
    \item If $\tilde{E}_k \perp  \tilde{E}_v$, then for all $(v_i, x_\tau)$ pairs in the context, we have $v_i^\top G_{\tilde{E}}\,k_m=0$, which means these pairs are not written into the hidden state $h_t$. This implies all value-key and value-value pairs are ignored. 
    \item Key-key pairs are written into $h_t$. However, if $E_k \perp E_v$, then for a key-key pair $(k_l, k_n)$, given a value entry $p$ where $e_p\in \mathcal{V}_v$, we have $e_p^\top G_E\,k_n=0$. This implies that even though written into $h_t$, this pair has zero contribution when retrieved, thus the "value" $k_n$ would not affect the output $y_t$.
\end{itemize} The overall consequence of this selection mechanism is that only key-value pairs contribute to $y_t$.%

From a mechanistic interpretability perspective, Fig.~\ref{app-fig:reversing-operators} tells us that, at least empirically, the model indeed learns such orthogonal embeddings,
which justifies our conclusion.

\end{remark}

\begin{theorem}[\textbf{Probabilistic recall scaling laws}]\label{app-thm:prob-scaling-law}
A \hyperref[alg:mamba]{\textit{single-layer simplified Mamba}} can solve an AR task with a probability %
{\small 
\begin{equation} \label{app-eq:success-prob-ar-large-Nf}
    p_\mathrm{AR} \approx \Phi \left( \sqrt{\frac{ND}{N_f}} - \sqrt{2 \log{V}} \right),
\end{equation}%
}
and an MQAR task of $L=4N_f$ with a probability%
{\small 
\begin{equation} \label{app-eq:success-prob-mqar-large-Nf}
    p_\mathrm{MQAR} \approx 
    \frac{1}{L - 2N_f} \sum _{t=2N_f} ^{L} 
    \Phi \left( \sqrt{\frac{ND}{\frac{1}{2}N_f+\frac{1}{4}t}} - \sqrt{2 \log{V}} \right).
\end{equation}%
}
in the limits of $N_f \gg N,D$ and $V \gg 1$, where $\Phi$ is the standard normal CDF.
\end{theorem} %
\begin{proof}[\textbf{Proof}]
(Also see proof for Thm.~\ref{app-thm:designed-scaling-law}.) As discussed in Rem.~\ref{app-rem:key-value-selectivity}, a simplified linear Mamba model can have a mechanism to ignore non-key-value pairs. More precisely, given a query $q_t=k_m$ 
we actually have $y_t^p\approx \delta^{ip} + \sum_{{(a,b)}_{kv}}^{t-1}\,\epsilon_v^{pa}\,\epsilon_k^{bm}$ where summation include only $kv$ token pairs $(x_a,x_b)$.
As in Thm.~\ref{app-thm:designed-scaling-law}, we have $y_t^p \sim \mathcal{N}(\delta^{ip},\sigma_t^2)$. However, now only $kv$ pairs contribute to the accumulated variance. Therefore, $\sigma_t^2=\frac{N_{kv}(t)}{ND}$, where $N_{kv}(t)$ is the expected number of $kv$ pairs until $t$. For the AR task, we have a single query at $t=2N_f+1$, thus $N_{kv}(t)=N_f$: we have $N_f$ facts ($kv$ pairs) before $t$ and $N_f$ non-fact $vk$ pairs, which have negligible variance. This implies $\sigma_t^2=\frac{N_f}{ND}$, thus yields Eq.~\ref{app-eq:success-prob-ar-large-Nf}. For MQAR with a minimal $L=4N_f$, the \textit{queries} section is a sequence $q_1,r_1,\dots,q_{N_f},r_{N_f}$ where queries $q_t\in \mathcal{V}_k$ are padded with random tokens $r_t\in \mathcal{V}$, including both keys and values. This means that the $2N_f$ token pairs in the queries section are, on average, $\frac{1}{2}$ $kk$ pairs, $\frac{1}{4}$ $kv$ pairs, and $\frac{1}{4}$ $vk$ pairs. Therefore, $N_{kv}(t)=N_f+(t-2N_f)/4$, from which we obtain Eq.~\ref{app-eq:success-prob-mqar-large-Nf}.
\end{proof}

\subsection{Unified Recall Scaling Laws}

\begin{remark}[\textbf{Unified recall scaling laws}]
\label{app-rem:practical-scaling-law}

We further consider a higher-order correction and an additional simplification of Thm.~\ref{app-thm:designed-scaling-law} and Thm.~\ref{app-thm:prob-scaling-law}. We derive a final simplified but useful form of the recall scaling laws. First, we relax the $N_f \gg N,D$ approximation; second, we use a linear approximation for $\Phi$ (which is reasonable for moderate $\sqrt{{ND}/{N_f}}$). Under these assumptions, a \hyperref[alg:mamba]{\textit{single-layer simplified Mamba}} solves associative recall tasks with a probability 
{\small 
\begin{equation} \label{app-eq:practical-success-prob}
    p \approx 
    \Phi \left( \sqrt{\frac{N D}{a N_f + N}} - \sqrt{2 \log{V}} \right),
\end{equation}}
where
\begin{equation}
    a \approx \;\;\;
    \begin{array}{c|cc}
         & \text{AR} & \text{MQAR} \\
        \hline
        \text{Designed} & 2 & 3 \\
        \text{Trained} & 1 & \tfrac{5}{4}
    \end{array}
\end{equation}
\end{remark}
\begin{proof}[\textbf{Proof}] Let us prove the statement case by case.

\paragraph{Designed model, moderate $N_f$.}

Recall Thm.~\ref{app-thm:designed-scaling-law}. Starting with Eq.~\ref{eqs-variance}, the mean and variance of the \textit{correct}, \textit{wrong} and \textit{empty} coordinates of $y_t$ can be approximated as:

\begin{equation} \label{eqs-variance-explicit}
\begin{aligned} 
\mu_c & =\mathbb{E}(y_t^i) = 1,\qquad
\sigma_c^2=\mathrm{Var}(y_t^i) \approx 
\tilde{\sigma}_S^2 + 
2N_f\, {\sigma}_Z^2 \, \tilde{\sigma}_Z^2 \\
\mu_w & =\mathbb{E}(y_t^j) = 0,\qquad
\sigma_w^2=\mathrm{Var}(y_t^j) \approx 
\tilde{\sigma}_Z^2 + \, \sigma_Z^2 +
2N_f\, {\sigma}_Z^2 \, \tilde{\sigma}_Z^2 \\
\mu_e & =\mathbb{E}(y_t^l) = 0,\qquad
\sigma_e^2=\mathrm{Var}(y_t^l) \approx 
\sigma_Z^2 + 
2N_f\, {\sigma}_Z^2 \, \tilde{\sigma}_Z^2 
\end{aligned}
\end{equation}

where $\sigma_Z^2=\frac{1}{D}$, $\tilde{\sigma}_Z^2=\frac{1}{N}$, $\tilde{\sigma}_S^2=\frac{1}{N}-\frac{1}{D}$. Note that if $N \sim D$ (same order of magnitude), $\tilde{\sigma}_S^2$ becomes a second order term, thus relatively negligible. In this case,

\begin{equation} \label{eqs-variance-appx}
\begin{aligned} 
\sigma_c^2&=\mathrm{Var}(y_t^i) \approx 
\frac{2N_f}{ND} \\
\sigma_w^2&=\mathrm{Var}(y_t^j) \approx 
\frac{2N_f}{ND} + \frac{1}{D}  + \frac{1}{N}\\
\sigma_e^2&=\mathrm{Var}(y_t^l) \approx 
\frac{2N_f}{ND}  + \frac{1}{D}
\end{aligned}
\end{equation}

Let us assume $1 < N_f \ll V$. Since there is $1 \ll V$ \textit{correct} coordinate, $N_f-1 \ll V$ \textit{wrong} coordinates and $V-N_f$ \textit{empty} coordinates, it is reasonable to approximate all of them to have a similar variance:

\begin{equation}
    \sigma_c^2 \approx \sigma_w^2 \approx \sigma_e^2 \approx 
    \frac{2 N_f}{N D}  + \frac{1}{D} = 
    \frac{2 N_f + N}{N D}
\end{equation}

The above equation holds for an AR task, where the query is at $t=2N_f$. More generally, for MQAR, given any single query at $2N_f<t \leq L$, we obtain the variance:

\begin{equation}
    \sigma^2(t) \approx 
    \frac{t}{N D}  + \frac{1}{D} = 
    \frac{t + N}{N D}
\end{equation}

Then, we can apply the same derivation from Thm.~\ref{app-thm:designed-scaling-law}, which yields the following recall probabilities:
\begin{equation}
    p_\mathrm{AR} \approx \Phi \left( \sqrt{\frac{ND}{2N_f + N}} - \sqrt{2\log{V}} \right),
\end{equation}%
\begin{equation}
    p_\mathrm{MQAR} \approx 
    \frac{1}{L - 2N_f} \sum _{t=2N_f} ^{L} 
    \Phi \left( \sqrt{\frac{ND}{t + N}} - \sqrt{2\log{V}} \right),
\end{equation}

\paragraph{Linearization of $\Phi$.}
To obtain a simple closed-form approximation for MQAR, we linearize around $\Phi=\tfrac{1}{2}$, i.e. around zero argument: $\Phi(x) \approx \frac{1}{2} + \frac{x}{\sqrt{2\pi}}$. Applying this approximation and then re-expressing the result in terms of $\Phi$ gives
\begin{equation}
    p_\mathrm{MQAR} \approx 
    \Phi\!\left(
    \frac{1}{L-2N_f}\sum_{t=2N_f}^{L} \sqrt{\frac{ND}{t+N}}
    - \sqrt{2\log V}
    \right).
\end{equation}
A coarse linearization would replace the mean by its value at the mean index, $\langle t \rangle = 3N_f$, yielding
\begin{equation}
    p_\mathrm{MQAR} \approx 
    \Phi\!\left(
    \sqrt{\frac{ND}{3N_f+N}} - \sqrt{2\log V}
    \right).
\end{equation}
A slightly more accurate approximation is obtained by replacing the sum with an integral:
\begin{equation}
    \frac{1}{L-2N_f}\sum_{t=2N_f}^{L} \frac{1}{\sqrt{t+N}}
    \approx
    \frac{1}{L-2N_f}\int_{2N_f}^{L} (t+N)^{-1/2}\,dt
    =
    \frac{2(\sqrt{L+N}-\sqrt{2N_f+N})}{L-2N_f}
\end{equation}
For the setting $L=4N_f$, this becomes
\begin{equation}
    p_\mathrm{MQAR} \approx
    \Phi\!\left(
    \frac{\sqrt{ND}}{N_f}\left(\sqrt{4N_f+N}-\sqrt{2N_f+N}\right)
    - \sqrt{2\log V}
    \right).
\end{equation}
We can rewrite this in the form
\begin{equation}
    p_\mathrm{MQAR} \approx
    \Phi\!\left(
    \sqrt{\frac{ND}{aN_f+N}}
    - \sqrt{2\log V}
    \right),
\end{equation}
where
\begin{equation}
    a =
    \frac{3-\frac{N}{N_f}+\sqrt{\left(4+\frac{N}{N_f}\right)\left(2+\frac{N}{N_f}\right)}}{2}.
\end{equation}
This implies $2.91 < a < 3$, so $a \approx 3$ is an accurate approximation.

\paragraph{Trained models.} We follow the proof of Thm.~\ref{app-thm:prob-scaling-law}, and apply the same moderate-$N_f$ correction as above. We have $\sigma^2(t)=\frac{N_{kv}(t)}{ND}+\frac{1}{D}$, where $N_{kv}(t)=N_f+(t-2N_f)/4$ is the expected number of $kv$ pairs until $t$. Thus,

\begin{equation}
    \sigma^2(t)=
    \frac{N_f+\frac{1}{4}(t-2N_f)}{ND}+\frac{1}{D}=
    \frac{\frac{1}{2}N_f+\frac{1}{4}t+N}{ND}
\end{equation}

This allows us to express the success probability for associative recall tasks as:
\begin{equation}
    p_\mathrm{AR} \approx \Phi \left( \sqrt{\frac{ND}{N_f + N}} - \sqrt{2 \log{V}} \right),
\end{equation}
\begin{equation}
    p_\mathrm{MQAR} \approx 
    \frac{1}{L - 2N_f} \sum _{t=2N_f} ^{L} 
    \Phi \left( \sqrt{\frac{ND}{\frac{1}{2}N_f+\frac{1}{4}t + N}} - \sqrt{2 \log{V}} \right).
\end{equation}
For MQAR, linearizing around $\Phi=\tfrac{1}{2}$ and re-expressing the result gives:
\begin{equation}
    p_\mathrm{MQAR} \approx
    \Phi\!\left(
    \frac{1}{L-2N_f}\sum_{t=2N_f}^{L}\sqrt{\frac{ND}{\frac{1}{2}N_f+\frac{1}{4}t+N}}
    - \sqrt{2\log V}
    \right)
\end{equation}
A coarse linearization replaces the mean by its value at the mean index, $\langle t\rangle=3N_f$, yielding
\begin{equation}
    p_\mathrm{MQAR} \approx
    \Phi\!\left(
    \sqrt{\frac{ND}{\frac{5}{4}N_f+N}}
    - \sqrt{2\log V}
    \right).
\end{equation}
A slightly more accurate approximation is obtained by replacing the sum with an integral:
\begin{equation}
    \frac{1}{L-2N_f}\sum_{t=2N_f}^{L}\frac{1}{\sqrt{2N_f+t+4N}}
    \approx
    \frac{2\left(\sqrt{L+2N_f+4N}-\sqrt{4N_f+4N}\right)}{L-2N_f}.
\end{equation}
For the setting $L=4N_f$, this gives
\begin{equation}
    p_\mathrm{MQAR} \approx
    \Phi\!\left(
    \frac{2\sqrt{ND}}{N_f}\left(\sqrt{6N_f+4N}-\sqrt{4N_f+4N}\right)
    - \sqrt{2\log V}
    \right)
    =
    \Phi\!\left(
    \sqrt{\frac{ND}{aN_f+N}}
    - \sqrt{2\log V}
    \right),
\end{equation}
where
\begin{equation}
    aN_f+N
    =
    \frac{\left(\sqrt{6N_f+4N}+\sqrt{4N_f+4N}\right)^2}{16} =
    \frac{5}{4}N_f+N+\frac{1}{4}\sqrt{(3N_f+2N)(N_f+N)}.
\end{equation}
Thus, the refined denominator is close to our coarse approximation $\frac{5}{4}N_f+N$, up to an additive term which can be used to derive higher-order corrections.

\end{proof}

\begin{remark}[\textbf{Scaling laws for model dimensions}]
\label{app-rem:dim-scaling-law}
Using Eq.~\ref{eq:success-prob-ar-large-Nf}, the model can solve AR with a probability greater than $1-\delta$ provided the model dimensions satisfy
$\sqrt{ND / N_f} \;\gtrsim\; \sqrt{2\log V} + n_\sigma$,
where $n_\sigma = \Phi^{-1}\!\left(1-\delta\right)$.
For $\delta \ll 1$, we approximate $n_\sigma^2 \approx 2\log\!\left(1/\delta\right)$, and the dimensions requirement becomes, up to constant factors:
\begin{equation}
\label{app-eq:dim-scaling-law-ND}
ND = \Omega\!\left( N_f\,\log\!\left(\tfrac{V}{\delta} \right) \right)
\end{equation}
\end{remark}

\subsection{\texorpdfstring{Information-Theoretic Lower Bound for Model Dimensions}{Information-Theoretic Lower Bound for Model Dimensions}}\label{app-subsec:info-lower-bound}

\begin{lemma}[\textbf{Information-theoretic lower bound on state size}]\label{app-lem:info-lower-bound}
Consider a recurrent model with a finite state size, where the context is summarized in a hidden state of $S$ scalar entries at $b$ bits of precision each. If such a model solves an AR task with $N_f$ facts (or an MQAR task with $N_f$ facts and context length $L=4N_f$) over a vocabulary of size $V$ with success probability $>1-\delta$, for a constant $0<\delta<1$, then $b\,S=\Omega(N_f\log V)$.
\end{lemma}
\begin{proof}[\textbf{Proof}]
Following the AR task definition, we note that the $N_f$ facts consist of distinct keys from the key vocabulary $\mathcal{V}_k$ and of values drawn i.i.d. uniformly from the value vocabulary $\mathcal{V}_v$ (which size is $|\mathcal{V}_v|=V_v=V/2$), and that the query is drawn uniformly among the context's $N_f$ keys, independently of the values. The context thus carries $N_f\log_2 V_v$ bits of information about the fact values $v_{1:N_f}\equiv(v_1,\dots,v_{N_f})$. Consider the recurrent state at the end of the facts section, $h_T$, where $T=2N_f$. It is a deterministic function of the context, and from this point on the model's answers depend on the values only through $h_T$. The state consists of $S$ scalar entries at $b$ bits of precision each, so it can take at most $2^{bS}$ distinct configurations. Consequently, the state retains at most $bS$ bits of information about the values: $\mathbb{I}(v_{1:N_f};h_T)\leq bS$, where $\mathbb{I}(\cdot\,;\cdot)$ denotes mutual information.

On the other hand, since the query may be any of the context's keys, all values can be recovered from $h_T$: running the model from $h_T$ separately with any single query $q\in\{x_0,x_2,\dots,x_{T-2}\}$, and collecting the answers, forms a decoder for $v_{1:N_f}$. By Fano's inequality, any decoder that recovers $v_{1:N_f}$ with error probability smaller than $\delta$ must satisfy $\mathbb{H}(v_{1:N_f}\mid h_T)\leq 1+\delta\,N_f\log_2 V_v$, where $\mathbb{H}(\cdot)$ denotes the Shannon entropy. Since $\mathbb{H}(v_{1:N_f})=N_f\log_2 V_v$, the state must retain $\mathbb{I}(v_{1:N_f};h_T) = \mathbb{H}(v_{1:N_f})-\mathbb{H}(v_{1:N_f}\mid h_T) \geq (1-\delta)\,N_f\log_2 V_v - 1$.

Combining the two: $bS \geq (1-\delta)\,N_f\log_2 V_v - 1$, and since $V_v=V/2$, for a constant $0<\delta<1$ this gives $bS=\Omega(N_f\log V)$. The same bound holds for MQAR, since $L=4N_f$.
\end{proof}

\begin{remark}[\textbf{Information-theoretic lower bound for Mamba dimensions}]\label{app-rem:info-lower-bound-mamba}
Notably, Lemma~\ref{app-lem:info-lower-bound} assumes nothing about the recurrence itself, only that the context is summarized in a fixed-size state. It therefore holds for both the simplified linear model and the \hyperref[par:mamba-background]{\textit{full nonlinear Mamba}}, whose hidden state has the same dimensions.

\paragraph{Single-layer Mamba.}
For a \hyperref[alg:mamba]{\textit{single-layer simplified Mamba}}, the recurrent state is $h_t\in\mathbb{R}^{D_\mathrm{in}\times N}$ (where $D_\mathrm{in}=\mathrm{expand}\cdot D$; in our scope we usually pick $\mathrm{expand}=2$), i.e. $S=\mathrm{expand}\cdot ND=2ND$ entries. At fixed precision $b$, Lemma~\ref{app-lem:info-lower-bound} thus reads:
\begin{equation}\label{app-eq:info-bound-single-layer}
ND = \Omega(N_f\log V)
\end{equation}

\paragraph{Multi-layer Mamba.}
For models with $\Lambda$ layers (both linear and full), the total recurrent state across layers has $S=\mathrm{expand}\cdot\Lambda\,ND$ entries, and the bound becomes:
\begin{equation}\label{app-eq:info-bound-multi-layer}
\Lambda\,ND = \Omega(N_f\log V)
\end{equation}
\end{remark}

\FloatBarrier
\section{Multi-Layer Recall Circuits}
\label{app:multi-layer-circuits}
\subsection{Multi-Layer Mamba}
\label{app-subsec:multi-layer}

In favor of the analysis of a multi-layer Mamba model, we present the following simplified model, detailed in Alg.~\ref{app-alg:mamba-layers}. As in Alg.~\ref{alg:mamba}, we assume gating, discretization, nonlinearities, biases and normalization layers are all removed. Note, however, the insertion of residual connections, as they are important for multi-layer recall, as we soon observe.

\begin{algorithm}[ht]
\scriptsize
\caption{\small Simplified Multi-Layer Mamba}
\label{app-alg:mamba-layers}

\begin{algorithmic}[1]
\REQUIRE $x \in \mathbb{R}^{V \times L}$
\ENSURE $y \in \mathbb{R}^{V \times L}$
\STATE \textbf{Parameters:} $E,\;
\{P^l_\mathrm{in},\; P^l_\mathrm{out},\; W^l_\mathrm{conv},\; S^l_B,\; S^l_C \;|\; l=1,\dots,\Lambda\}$

\STATE $x_{\mathrm{e}} \leftarrow E\,x$
\STATE $x^1 \leftarrow x_{\mathrm{e}}$

\FOR{$l \leftarrow 1$ \textbf{to} $\Lambda$}
  \STATE \textbf{Mixer:}
  \STATE $x^l_{\mathrm{in}} \leftarrow P^l_{\mathrm{in}}\,x^l$
  \STATE $\hat{x}^l \leftarrow \mathrm{Conv1D}(x^l_{\mathrm{in}}, W^l_{\mathrm{conv}})$
  \STATE $B^l,\;C^l \leftarrow S^l_B\,\hat{x}^l,\; S^l_C\,\hat{x}^l$
  \STATE $\hat{y}^l \leftarrow \mathrm{SSM}(\hat{x}^l, B^l, C^l)$
  \STATE $y^l_{\mathrm{out}} \leftarrow P^l_{\mathrm{out}}\,\hat{y}^l$
  \STATE $y^l \leftarrow x^l + y^l_{\mathrm{out}}$
  \STATE $x^{l+1} \leftarrow y^l$
\ENDFOR

\STATE $y \leftarrow E^\top\,y^\Lambda$
\STATE \textbf{return} $y$
\end{algorithmic}
\end{algorithm}

\subsection{Multi-Layer Recall Circuit and Scaling Laws}
\label{app-subsec:multi-layer-scaling}

\paragraph{Intuition.} 
Let us briefly describe the mechanism. The residual connections between layers allow each Mamba layer $l$ to receive, as an input $x^l$, a linear combination of the original sequence $Ex$ and the current cumulative recall output $\sum _k ^{l-1} y^{k}_{AR}$ as an input. Each layer then adds its own independent associative recall result $y^l_{AR}=Ex\,\alpha^l_{AR}$, which is then similarly passed through a residual connection to the next layer, such that model output is $y =\sum _k ^{\Lambda} y^{k}_{AR}$.

\color{black}
\vspace{5pt}
\label{app-thm:multi-layer-recall}
\begin{theorem}[\textbf{Multi-layer recall scaling laws}]
A \hyperref[app-alg:mamba-layers]{\textit{simplified multi-layer Mamba}} with $\Lambda$ layers can solve recall tasks with the same probability as a single-layer model with the same dimensions and state size $N^\mathrm{eff}=\Lambda N$. In particular, the model 
can solve an AR task with a probability %
\vspace{-6pt}
{\small 
\begin{equation} \label{eq:success-prob-ar-multi-layer}
    p_\mathrm{AR} \approx \Phi \left( \sqrt{\frac{\Lambda ND}{N_f}} - \sqrt{2 \log{V}} \right),
\end{equation}%
}
and an MQAR task of $L=4N_f$ with a probability
\vspace{-4pt}
{\small 
\begin{equation} \label{eq:success-prob-mqar-multi-layer}
    p_\mathrm{MQAR} \approx 
    \frac{1}{L - 2N_f} \sum _{t=2N_f} ^{L} 
    \Phi \left( \sqrt{\frac{\Lambda ND}{\frac{1}{2}N_f+\frac{1}{4}t}} - \sqrt{2 \log{V}} \right).
\end{equation}
}
in the limits of $N_f \gg N,D$ and $V \gg 1$, where $\Phi$ is the standard normal CDF. 
Following Rem.~\ref{app-rem:practical-scaling-law}, we can refine and rewrite these results as
{\small 
\begin{equation}
    p_\mathrm{recall} \approx 
    \Phi \left( \sqrt{\frac{\Lambda N D}{a N_f + \Lambda  N}} - \sqrt{2 \log{V}} \right),
\end{equation}}
where $a=1$ for AR, and $a\approx\tfrac{5}{4}$ for MQAR.

\end{theorem}

\begin{proof}[\textbf{Proof}]

Let us construct the multi-layer Mamba solution to MQAR step by step.

\paragraph{\underline{Step 1: Single-layer Mamba}}

First, we show a single-layer Mamba with residual connection can solve MQAR. We slightly modify the weights construction of Theorem~\ref{thm:compressive-recall}, such that residual connection does not affect the recall circuit output.

\paragraph{Weights.\quad} Given embedding matrices $E\in \mathbb{R}^{D\times V}$ and $F^1\in \mathbb{R}^{N\times D}$,  and given a parameter $a^1\in \mathbb{R}$, we choose:
\begin{equation}
    \begin{gathered}
        P^1_\mathrm{in} =  \bigl(\begin{smallmatrix} I_D \\ I_D \end{smallmatrix} \bigr)
        \,, \quad 
        P^1_\mathrm{out} = \bigl(0\mid I_D)\,, \\
    W^1_\mathrm{conv}=\bigl(\begin{smallmatrix}1_D\\0_D\end{smallmatrix}\mid \;\begin{smallmatrix}0_D\\1_D\end{smallmatrix}\bigr)\,, \quad
    S^1_B=\bigl(F^1\mid a^1 F^1)\,, \quad
      S^1_C=\bigl(0\mid F^1)
    \end{gathered}
\end{equation}
As before, let us define $\tilde{E}^1\,\equiv F^1 E \in \mathbb{R}^{N \times V}$. Inside the SSM, we now have $\hat{x}^1 \in \mathbb{R}^{D_\mathrm{in} \times L}$ such that $\hat{x}^1_t=\bigl(\begin{smallmatrix}
    E\,x_{t-1}\\E\,x_{t}\end{smallmatrix} \bigr)$, $C^1_t=\tilde{E}^1\,x_t$, and most importantly $B^1_t=\tilde{E}^1\,x_{t-1}+a^1\tilde{E}^1\,x_t$.

\paragraph{Attention view.\quad} We now leverage \textit{attention} formulation to view SSM operation. Namely, the SSM attention matrix $\alpha^1\in \mathbb{R}^{L \times L}$ should satisfy $\hat{y}^1_t=\sum_{\tau=0}^t \hat{x}^1_\tau \,\alpha^1_{\tau t}$, or in matrix form $\hat{y}^1=\hat{x}^1\alpha^1\in \mathbb{R}^{D_\mathrm{in} \times L}$. (Note the right-side matrix multiplication, which operates along time axis rather than dimension axis). Following \citet{ali2024hidden,dao2024transformers}, the attention matrix for our simplified SSM is: 
\begin{equation}
    \alpha^1_{\tau t}={B^1_\tau}^\top\,C^1_t=(\tilde{E}^1\,x_{\tau-1}+a^1\tilde{E}^1\,x_\tau)^\top (\tilde{E}^1\,x_t)
\end{equation}
As in Theorem~\ref{thm:compressive-recall}, we choose $E$, $F^1$ such that $\tilde{E}^1 \in \mathbb{R}^{N\times V}$ is a JL embedding matrix w.r.t some small $\varepsilon_k>0$. Defining $G^1 \equiv (\tilde{E}^1) ^{\top} \tilde{E}^1$ for short, we have:
\begin{equation}
    \alpha^1_{\tau t}=
    x_{\tau-1}^\top G^1\,x_t +
    a^1\,x_{\tau}^\top G^1 \,x_t 
\end{equation}
Alternatively, minding that for a JL matrix, $G^1=I_V+O(\varepsilon_k)$, 
\begin{equation}
    \alpha^1_{\tau t}=
    x_{\tau-1}^\top G^1x_t + a^1\,x_\tau^\top x_t+O(\varepsilon_k)
\end{equation}
which yields the $L \times L$ matrix form 
\begin{equation}
    \alpha^1
    =
    \alpha^1_{AR}+a^1I_{L}+O(\varepsilon_k)
\end{equation}
where $(\alpha^1_{AR})_{\tau t}=x_{\tau-1}^\top G^1x_t$ is the compressive-recall attention matrix equivalent of Theorem~\ref{thm:non-compressive-recall}  (which performs recall by matching MQAR queries with keys), and $I_L\in \mathbb{R}^{L\times L}$ is the identity matrix in sequence space. We denote $\alpha^1_{AR}$ this way to avoid confusion with the overall SSM attention matrix $\alpha^1$.

\paragraph{Layer output.\quad} Given the input sequence $x^1=Ex\in \mathbb{R}^{D\times L}$, the SSM output is:
\begin{equation}
\hat{y}^1=\hat{x}^1 \alpha^1=P^1_\mathrm{in}Ex(\alpha^1_{AR}+a^1I_L)=P^1_\mathrm{in}Ex\alpha^1_{AR}+a^1P^1_\mathrm{in}Ex
\end{equation}
When projected out, since $P^1_\mathrm{in} =  \bigl(\begin{smallmatrix} I_D \\ I_D \end{smallmatrix} \bigr)
        \,, 
        P^1_\mathrm{out} = \bigl(0\mid I_D)$, we have:
\begin{equation}
y^1_\mathrm{out}=P^1_\mathrm{out}\,\hat{y}^1=\hat{x}^1 \alpha^1=Ex\alpha^1_{AR}+a^1Ex
\end{equation}
The residual connection now implies:
\begin{equation}
    y^1=x^1+y^1_\mathrm{out}=Ex\alpha^1_{AR}+(a^1+1)Ex
\end{equation}
We are now able to choose $a^1=-1$, such that
\begin{equation}
    y^1=Ex\alpha^1_{AR}
\end{equation}
and model output is
\begin{equation}
    y=E^\top y^1=E^\top E x\alpha^1_{AR}
\end{equation}
which is, by design, the same solution to MQAR as in Theorem~\ref{thm:compressive-recall}.

\paragraph{\underline{Step 2: Two-layer Mamba}}

Building upon previous step, we now describe a two-layer Mamba solution to MQAR.

Assume a \hyperref[app-alg:mamba-layers]{\textit{simplified multi-layer Mamba}} with $\Lambda=2$. We set first layer ($l=1$) weights as in Step 1, but this time without specifying $a^1=-1$. 

\paragraph{Layer input.\quad} Importantly, the second layer ($l=2$) input is:
\begin{equation}
    x^2=y^1=x^1+y^1_\mathrm{out}=Ex\alpha^1_{AR}+(a^1+1)Ex
\end{equation}
For clarity, let us denote $y^1_{AR}=Ex \alpha^1_{AR}$ and $b=a^1+1$. Layer input now reads:
\begin{equation}
    x^2=y^1_{AR}+bEx
\end{equation}
\paragraph{Weights.\quad} Given a matrix $F^2\in \mathbb{R}^{N \times D}$ and a parameter $a^2 \in \mathbb{R}$, let us set layer weights as follows:
\begin{equation}
    \begin{gathered}
        P^2_\mathrm{in} =  \bigl(\begin{smallmatrix} I_D \\ I_D \end{smallmatrix} \bigr)
        \,, \quad 
        P^2_\mathrm{out} = \bigl(0\mid I_D)\,, \\
    W^2_\mathrm{conv}=\bigl(\begin{smallmatrix}1_D\\0_D\end{smallmatrix}\mid \;\begin{smallmatrix}0_D\\1_D\end{smallmatrix}\bigr)\,, \quad
    S^2_B= \frac{1}{b} \bigl(F^2\mid a^2 F^2)\,, \quad
      S^2_C=\frac{1}{b}\bigl(0\mid F^2)
    \end{gathered}
\end{equation}

\paragraph{Key and value subspaces.\quad} Inside layer 2 SSM, after input projections, we have $
\hat{x}^2_t=
\bigl(\begin{smallmatrix}
    bE\,x_{t-1} + (y_{AR}^1)_{t-1}
    \\
    bE\,x_{t}+(y_{AR}^1)_t
    \end{smallmatrix} \bigr)$. 
Importantly, note that $\mathbb{R}^{D}$ can be decomposed into two orthogonal key and value subspaces, namely $\mathbb{R}^D = U \oplus W$, such that $E=E_k+E_v$ and $(E_k)_i \perp (E_v)_j$  for any $i,j$. Consequently, we can choose $F^2$ which acts only on embedded key vectors $Ek_i\in U$ and zeros out embedded value vectors $Ev_j \in W$: for any $j$, $F^2Ev_j=0$.

Notice that, as a result of Theorem~\ref{thm:compressive-recall}, any $(y^1_{AR})_t$ is approximately a \textit{value} vector, up to a small $O(\varepsilon_k+\varepsilon_v)$ error. (This can be further guaranteed by modifying layer 1 output projection $P^1_\mathrm{out}$ to zero out the key subspace $U$.) This implies $F(y^1_{AR})_t \approx0$, for any $t$. Then, upon SSM projections, we have, approximately, $C^2_t=\tilde{E}^2\,x_t$ and $B^2_t=\tilde{E}^2\,x_{t-1}+a^2\tilde{E}^2\,x_t$, without any $y^1_{AR}$ terms, where we have defined $\tilde{E}^2=F^2E$.

\paragraph{Attention view.} Replicating the attention matrix computation from Step 1, and similarly defining $G^2 \equiv (\tilde{E}^2) ^{\top} \tilde{E}^2$, we now have for layer 2 SSM:
\begin{equation}
    \alpha^2
    =
     \alpha^2_{AR}+a^2I_{L}+O(\varepsilon_k)
\end{equation}
where $(\alpha^2_{AR})_{\tau t}=x_{\tau-1}^\top G^2x_t$ is a compressive-recall attention matrix. Note that $G^2 \neq G^1$ in general, since generally $F^2 \neq F^1$.

\paragraph{Layer output.} Following Step 1, we compute the the SSM output:
\begin{equation}
\hat{y}^2_\mathrm{out} = P^2_\mathrm{out}\,\hat{y}^2=
P^2_\mathrm{out}\,\hat{x}^2 \alpha^2=
P^2_\mathrm{out}\,P^2_\mathrm{in}\,
x^2 (\alpha^2_{AR}+a^2 I_L)
\end{equation}
But, since projections are trivial:
\begin{equation}
y^2_\mathrm{out}=
x^2 (\alpha^2_{AR}+a^1I_L)
\end{equation}
Considering residual connection, the layer output is:
\begin{equation}
    y^2=x^2+y^2_\mathrm{out}=x^2 (\alpha^2_{AR}+b^2 I_L)
\end{equation}
where $b^2=a^2+1$. Recall that $x^2=y^1_{AR}+b^1 Ex$. Therefore,
\begin{equation}
    y^2=
    (y^1_{AR}+b^1Ex) (\alpha^2_{AR}+b^2 I_L)
\end{equation}
We denote $y^2_{AR}=Ex \alpha^2_{AR}$, since it is a compressive-recall output, similarly to these in Step 1 and in Theorem~\ref{thm:compressive-recall}. This yields:
\begin{equation}
    y^2=
    y^1_{AR} \, \alpha^2_{AR} +b^1 y^2_{AR} +b^2 y^1_{AR}+b^1 b^2 Ex
\end{equation}
Note that we have additional degrees of freedom by choosing the scales of $b^1,b^2$ and the norms of $E_k,E_v,F^1,F^2$, w.r.t to some arbitrary scale factor ${\varepsilon}$. This can be done such that:
\begin{equation}
    y^2=
    b^1 y^2_{AR} +b^2 y^1_{AR}+O(\varepsilon)
\end{equation}
This, of course, does not mean the $O(\varepsilon)$ terms are negligible, but rather they have weaker contribution to model output, hence are less probable to change the maximal entry of retrieved value.

\paragraph{Effective state size.} Importantly, both $y^1_{AR}$ and $y^2_{AR}$ are linear mixtures of the input sequence $Ex$, namely, $y^l_{AR}=Ex\alpha^l_{AR}$. Let us absorb the constants $b^1,b^2$ into $\alpha_{AR}^2, \alpha_{AR}^1$, and omit the error term: 
\begin{equation}
    y^2 \approx
    Ex(\alpha^1_{AR}+\alpha^2_{AR})
    \equiv E x \, \alpha^\mathrm{eff}
\end{equation}
Moreover, the layers attention matrices are of a similar form, $(\alpha^l_{AR})_{\tau t}=x_{\tau-1}^\top G^lx_t$. Hence:
\begin{equation}
    \alpha^\mathrm{eff}_{\tau t} =(\alpha^1_{AR}+\alpha^2_{AR})_{\tau t}=
    x_{\tau-1}^\top (G^1+G^2)\,x_t
\end{equation}
Since $G^l=\tilde{E}^{l^\top}\tilde{E}^l=E^\top F^{l \top} F^l E$, we conclude:
\begin{equation}
    \alpha^\mathrm{eff}_{\tau t} =
    x_{\tau-1}^\top E^\top (F^{1 \top} F^1+F^{2 \top} F^2)E\,x_t
\end{equation}
This motivates us to define an effective matrix,
\begin{equation}
     F^\mathrm{eff}\equiv \begin{pmatrix} F^1 \\ F^2  \end{pmatrix}
\in \mathbb{R}^{2N \times D}
\end{equation}
and effective net embedding: 
\begin{equation}
     \tilde{E}^\mathrm{eff}\equiv F^\mathrm{eff}E
\in \mathbb{R}^{2N \times V}
\end{equation}
Finally, the model output now becomes:
\begin{equation}\label{eq:multi-layer-recall-out}
y_t = E^\top y^2_t=\sum_{\tau=0} ^t E^\top (E\,x_{\tau}) \, (\tilde{E}^\mathrm{eff}\,x_{\tau-1})^\top (\tilde{E}^\mathrm{eff}\,x_t) \, 
\end{equation}
Hence, our multi-layer model with $\Lambda=2$ is mathematically equivalent to a single-layer model with larger state size, $N^\mathrm{eff}=2N=\Lambda N$.

\paragraph{\underline{Step 3: Multi-layer Mamba}}

Statement is trivially proven by induction. We use $\Lambda=1$ (Step 1) as a \textit{base step}, and a similar claim as done for $\Lambda=2$ (Step 2) as and \textit{induction step}. 

More intuitively: Step 2 shows that a Mamba layer $l$ which receives a linear combination $x^l=\,\sum_k^{l-1}a^k \,y^{k}_{AR} + b^l\,Ex$ as an input, can output a similar (independent) associative-recall result $y^l_{AR}=Ex\,\alpha^l_{AR}$, which is then similarly passed through a residual connection to the next layer, as $x^{l+1}=y^l_{AR}+\,\sum_k^{l-1}a^k \,y^{k}_{AR} + b^l\,Ex$.

Conclusion is, a simplified multi-layer Mamba model with $\Lambda$ layers can solve MQAR equivalently to a single-layer model with a larger state size $N^\mathrm{eff}=\Lambda N$. 
\end{proof}
\section{Multi-Head Recall Circuits and Mamba-2}
\label{app:mamba-2-multi-head}

\subsection{Multi-Head Patterns}
\paragraph{Multi-head patterns for SSM.}
One of the key architectural innovations introduced in Mamba-2~\citep{dao2024transformers}, compared to Mamba, is the ability to process the input sequence through multiple SSM heads in parallel. Equivalently to parameter sharing across attention heads in Transformers, the authors present several \textbf{\textit{head patterns}} for multi-head SSM (see Tab.~\ref{tab:head-patterns}). In each head pattern, a different combination of $\hat{x}$, $B$ and $C$ is shared across SSM heads, correspondingly to sharing $V$, $K$ and $Q$ between attention heads in a Transformer. Specifically, MVA shares $B,C$ between the SSM heads, MKA shares $C$, MQA shares $B$, and MHA shares none.

\renewcommand{\arraystretch}{1.25}

\begin{table}[h]
\caption{Multi-Head Patterns in Mamba-2}
\label{tab:head-patterns}
\centering
\begin{tabular}{lcccc}
\textbf{} &
\textbf{Multi-Head SSM} & 
\textbf{Multi-Contract SSM} & 
\textbf{Multi-Expand SSM} & 
\textbf{Multi-Input SSM} \\
& Multi-Head Attn. & Multi-Query Attn. & Multi-Key Attn. & Multi-Value Attn. \\
& \textit{MHA} & \textit{MQA} & \textit{MKA} & \textit{MVA} \\
\hline
$\hat{x}$ & $(L, M, P)$ & $(L, 1, P)$ & $(L, 1, P)$ & $(L, M, P)$ \\
$B$ & $(L, M, N)$ & $(L, 1, N)$ & $(L, M, N)$ & $(L, 1, N)$ \\
$C$ & $(L, M, N)$ & $(L, M, N)$ & $(L, 1, N)$ & $(L, 1, N)$ \\
\hline
\end{tabular}
\end{table}

\renewcommand{\arraystretch}{1.00}

It is assumed that the number of heads $M$ divides model dimension $D$, such that head dimension is $P=\frac{D}{M}$.
Also note that as done in Eq.~\ref{eq:simple-ssm} and Alg.~\ref{alg:mamba}, we assume SSM matrix $A$ is an identity operation, hence omit it. Importantly, the head pattern actually \textit{used} in Mamba-2 in practice is MVA. Here, however, we study all patterns, in order to better understand recall capacity in Mamba-2.

\paragraph{Simplified multi-head model.} Following these head patterns, we now generalize our single-head simplified model (Alg.~\ref{alg:mamba}) to match Mamba-2 multi-head SSM patterns. In Alg.~\ref{alg:mamba-2}, we present a simplified multi-head model, which follows the Mamba-2 design but keeps critical components only. Also notice the different structure of projections and convolutions (compared to Mamba) used to produce $\hat{x}$, $B$ and $C$ from input sequence $Ex$. See App.~\ref{app-subsec:mamba-2-notation-and-dims} (Table~\ref{tab:mamba2_notation} and Table~\ref{tab:mamba2_dims}) for notation and dimensions details. It should be noted that a \hyperref[alg:mamba]{\textit{single-layer (single-head) simplified Mamba}} is roughly a single-head redundant case of Alg.~\ref{alg:mamba-2}, with $M=1$, $P=D_\mathrm{in}$, and $p=\text{MVA}$.

\renewcommand{\arraystretch}{1.00}

\color{black}

\begin{algorithm}[h] %
\scriptsize
\caption{\small Simplified Single-Layer Mamba-2}
\label{alg:mamba-2}

\begin{algorithmic}[1]
\REQUIRE $x$ \shape{$L, V$}
\ENSURE $y$ \shape{$L, V$}

\STATE \textbf{Parameters:} $E,\;P_\mathrm{in},\; P_\mathrm{out},\; W_\mathrm{conv},\; W_x, \; W_y$
\STATE \textbf{Head Pattern:} $p\in \{ \text{MHA, MQA, MKA, MVA} \} $

\STATE $x_{\mathrm{e}} \leftarrow E\,x$ \shape{$L, D$}
\STATE \textbf{Mixer:}

\STATE $\hat{x}\leftarrow \mathrm{Conv1D}(P^{\hat{x}}_{\mathrm{in}}\,x_{\mathrm{e}}, W^{\hat{x}}_{\mathrm{conv}})$

\STATE $B,\,C \leftarrow S_{B,C}\,\mathrm{Conv1D}(P^{B,C}_{\mathrm{in}}\,x_{\mathrm{e}}, W^{B,C}_{\mathrm{conv}})$
\STATE $\hat{x}$ : \shape{$L,D$}

\STATE $B : \begin{cases}
    \textcolor{lightblue}{(L,N)} & \text{if $p \in$ \{MQA, MVA\}}  \\
    \textcolor{lightblue}{(L,N,M)} & \text{if $p \in$ \{MHA, MKA\}}
\end{cases}$

\STATE $C : \begin{cases}
    \textcolor{lightblue}{(L,N)} & \text{if $p \in$ \{MKA, MVA\}}  \\
    \textcolor{lightblue}{(L,N,M)} & \text{if $p \in$ \{MHA, MQA\}}
\end{cases}$

\STATE $\hat{x}_\mathrm{heads} \leftarrow \text{split}(\hat{x}, M)$  \shape{$L, P, M$}

\STATE $\hat{x}_\mathrm{shared} \leftarrow \text{contract}(\hat{x}_\mathrm{heads}, W_x) $  \shape{$L, P$}

\FOR{$h \leftarrow 1$ \textbf{to} $M$}
    \STATE $\hat{x}^h \leftarrow \begin{cases}
        \hat{x}_\mathrm{heads}^h & \text{if $p \in$ \{MHA, MVA\}} \\
        \hat{x}_\mathrm{shared} & \text{if $p \in$ \{MQA, MKA\}}
    \end{cases}$ \shape{$L, P$}

    \STATE $\hat{y}^h \leftarrow \mathrm{SSM}(\hat{x}^h, B^h, C^h)$ \shape{$L, P$}
\ENDFOR

\STATE $\hat{y} \leftarrow \text{aggregate}({\{\hat{y}^h\}}, W_y) $ \shape{$L, D$}

\STATE $y_\mathrm{out} \leftarrow P_\mathrm{out} \, \hat{y} $ \shape{$L, D$}

\STATE $y \leftarrow E^\top\,y_{\mathrm{out}}$ \shape{$L, V$}

\STATE \textbf{return} $y$
\end{algorithmic}
\end{algorithm}

\subsection{Multi-Head Recall Circuits} \label{app-par:multi-head-circuits}

Upon defining the multi-head SSM, one may ask whether it is beneficial in terms of associative recall capabilities, and if so, which head pattern should be chosen. Equipped with the simplified Mamba-2 model, let us now analyze the associative recall capacity of each head pattern.

\begin{theorem}[\textbf{Multi-head recall scaling laws}] \label{app-thm:multi-head-recall}
Given fixed model dimensions $D$, $N$, \hyperref[alg:mamba-2]{single-layer simplified Mamba-2} models with \(M>1\) heads and multi-head attention patterns MQA, MKA, MVA, or MHA solve recall tasks (AR or MQAR) with probabilities satisfying
\vspace{-6pt}
{\small 
\begin{equation} \label{eq:success-prob-ar-large-Nf---}
    p_\mathrm{MQA} = 
    p_\mathrm{MKA} \leq 
    p_\mathrm{single} =
    p_\mathrm{MVA} \leq 
    p_\mathrm{MHA} \,,
\end{equation} \vspace{-2pt}
}
where $p_\mathrm{single}$ is the recall probability for a single-head model ($M=1$) of similar dimensions $D$, $N$, as detailed in Thm.~\ref{thm:prob-scaling-law}.
\end{theorem} %
\begin{proof}[\textbf{Proof}]
See the following proofs for lemmas~\ref{lem:mva-recall}, ~\ref{lem:mqa-recall}, ~\ref{lem:mka-recall} and~\ref{lem:mha-recall}.
\end{proof} %

\begin{lemma}[\textbf{MVA recall circuit}]\label{lem:mva-recall}
A \hyperref[alg:mamba-2]{single-layer simplified Mamba-2} model with $M$ heads using MVA pattern and dimensions $D$, $N$ (with $\mathrm{expand}=2$, $D_\mathrm{conv}=2$, as in Thm.~\ref{thm:compressive-recall}) is \textbf{equivalent}, in term of associative recall (AR or MQAR tasks), to a similar single-head model of the same dimensions $D'=D$, $N'=N$. Specifically, the model probability of success satisfies:
{\small 
\begin{equation}
    p_\mathrm{MVA}[D,N,M] =
    p_\mathrm{single}[D,N]
\end{equation} %
}
where $p_\mathrm{single}[D',N']$ is the recall probability for a single-head model ($M=1$) of dimensions $D'$, $N'$, as detailed in Thm.~\ref{thm:prob-scaling-law}.
\end{lemma}

\begin{proof}[\textbf{Proof}]
We use the compressed recall weights as in Theorem~\ref{thm:compressive-recall}, with slight adaptations to Mamba-2. Given a matrix $F\in \mathbb{R}^{N\times D}$, we choose weights $W^{\hat{x}}_\mathrm{conv}=\bigl(0_D\mid 1_D)$, $W^B_\mathrm{conv}=\bigl(1_D\mid 0_D)$, $W^C_\mathrm{conv}=\bigl(0_D\mid 1_D)$ and $S_B=S_C=F$, such that values are $\hat{x}_t=E\,x_t$, keys are $B_t=FE\,x_{t-1}$ and queries are $C_t=FE\,x_t$. The value $\hat{x}_t$ is then split into $M$ heads $\hat{x}^h_t$:
\begin{equation}
    \hat{x}_t
     = \begin{pmatrix}
\hat{x}_t^1 \\
\vdots \\
\hat{x}_t^M
\end{pmatrix}
\end{equation}
In each SSM head, we have:
\begin{equation}
    \hat{y}^h_t=\sum_\tau^t {\hat{x}^h_\tau} \,B_\tau ^\top \, C_t
\end{equation}
Assume, for simplicity, heads aggregation has trivial weights $W_y^h=1$ (see Alg.~\ref{alg:mamba-2}); the SSM output is thus $\{\hat{y}^h_t\}$ stacked:
\begin{equation}
\hat{y}_t = 
\begin{pmatrix}
\hat{y}_t^1 \\ \vdots \\
\hat{y}_t^M
\end{pmatrix} = 
\begin{pmatrix}
\sum_\tau^t {\hat{x}^1_\tau} \,B_\tau ^\top \, C_t \\ \vdots \\
\sum_\tau^t {\hat{x}^M_\tau} \,B_\tau ^\top \, C_t
\end{pmatrix} =
\sum_\tau^t  \begin{pmatrix}
{\hat{x}^1_\tau}  \\ \vdots \\
{\hat{x}^M_\tau}
\end{pmatrix} \,B_\tau ^\top \, C_t =
\sum_\tau^t  
{\hat{x}_\tau}  \,B_\tau ^\top \, C_t
\end{equation}
which is exactly the output we have for a single-head model of the same dimensions $D$, $N$. 

\end{proof}

\begin{lemma}[\textbf{MQA recall circuit}]
\label{lem:mqa-recall}
A \hyperref[alg:mamba-2]{single-layer simplified Mamba-2} model with $M$ heads using MQA pattern and dimensions $D$, $N$ (with $\mathrm{expand}=2$, $D_\mathrm{conv}=2$, as in Thm.~\ref{thm:compressive-recall}) is \textbf{equivalent}, in term of associative recall (AR or MQAR tasks), to a similar single-head model of dimensions $D'=\frac{D}{M}$, $N'=N$. Specifically, the model probability of success satisfies:
{\small 
\begin{equation} \label{eq:mqa-prob}
    p_\mathrm{MQA}[D,N,M] =
    p_\mathrm{single}[\tfrac{D}{M},N] \leq
    p_\mathrm{single}[D,N]
\end{equation} %
}
(Namely, for MQA pattern, splitting into heads might hurt recall performance, despite using a larger query dimension $MN$.)
\end{lemma}

\begin{proof}[\textbf{Proof}]
We choose convolution kernels as done in Lem.~\ref{lem:mva-recall}, and fix $S_B=F$ for $F\in \mathbb{R}^{N\times D}$. We now have $M$ distinct projections $S^h_C=F^h\in \mathbb{R}^{N\times D}$, for generally different matrices $\{F^h\}$. As before, we have shared keys $B_t=FE\,x_{t-1}$, but now distinct queries $C^h_t=F^hE\,x_t$. 

\paragraph{Non-contracted values and attention view.} Let us first ignore values $\hat{x}$ contraction, i.e. assume the SSM input is $\hat{x}_\mathrm{heads}\in \mathbb{R}^{D \times L}$ instead of $\hat{x}_\mathrm{shared}\in \mathbb{R}^{P \times L}$. Per time $t$, the input value is thus $\hat{x}_t=Ex_t$. 
Recall that each SSM head can be equivalently thought as of an attention head (see \citet{ali2024hidden,dao2024transformers}), with attention matrix $\alpha^h_{\tau t}={B^h_\tau}^\top C^h_t$. In our case:
\begin{equation}
    \alpha^h_{\tau t}={B_\tau}^\top C^h_t=x^\top_{\tau-1}E^\top F^\top F^h E \,x_t
\end{equation}
The ideal attention matrix for associative recall (see Thm.~\ref{thm:non-compressive-recall}) is $\alpha^h_{\tau t}=x^\top_{\tau-1} x_t$, which matches queries to keys. As in Thm.~\ref{thm:compressive-recall}, this attention matrix can be approximated using JL matrices for compression. In our case, we should require $E^\top F^\top F^h E \approx I_V$. This can be achieved if $\tilde{E}\equiv FE$ is a JL matrix (Lem.~\ref{app-lem:jl-consecutive}) and also $F^h=F$ for all heads $h$. Only this way, we can achieve $\alpha^h_{\tau t} \approx x^\top_{\tau-1} x_t$.
Note, however, this implies our queries are not independent anymore, as now $C^h_t=FEx_t$ are in fact shared between all heads; the additional weights become redundant. This means the recall circuit is equivalent to the one in Lem.~\ref{lem:mva-recall}, thus performs equivalently to a single-head model with $D'=D$, $N'=N$.
\paragraph{Values contraction.} We now consider the actual SSM input, which is given by:
\begin{equation}
\hat{x}_\mathrm{shared} = 
\text{contract}(\hat{x}_\mathrm{heads}, W_x) =
\sum_{h=1}^M \,W^h_x \, \hat{x}_\mathrm{heads}^h 
\in \mathbb{R}^{P \times L}
\end{equation}
Note that in SSM terms, this transformation further compresses values $\hat{x}$ into a smaller dimension $P=\tfrac{D}{M}$, but does not affect key and query dimensions. Therefore, the multi-head model is equivalent to a single-head model with $N'
=N$ and $D'=\tfrac{D}{M}\leq D$. This consequence yields Eq.~\ref{eq:mqa-prob}. 

\end{proof}

\begin{lemma}[\textbf{MKA recall circuit}]
\label{lem:mka-recall}
A \hyperref[alg:mamba-2]{single-layer simplified Mamba-2} model with $M$ heads using MKA pattern and dimensions $D$, $N$ (with $\mathrm{expand}=2$, $D_\mathrm{conv}=2$, as in Thm.~\ref{thm:compressive-recall}) is \textbf{equivalent}, in term of associative recall (AR or MQAR tasks), to a similar single-head model of dimensions $D'=\frac{D}{M}$, $N'=N$. Specifically, the model probability of success satisfies:
{\small 
\begin{equation} \label{eq:mka-prob}
    p_\mathrm{MKA}[D,N,M] =
    p_\mathrm{single}[\tfrac{D}{M},N] \leq
    p_\mathrm{single}[D,N]
\end{equation} %
}
(Namely, for MKA pattern, splitting into heads might hurt recall performance, despite using a larger key dimension $MN$.)
\end{lemma}

\begin{proof}[\textbf{Proof}]
We repeat the proof of Lem.~\ref{lem:mqa-recall}, but now $B^h_\tau = F^h E x_{\tau-1}$ is head-specific, while $C_t = F E x_t$ is shared across heads. The same constraint (forcing $F^h=F$) applies here in order to achieve recall. Hence, the quantitative analysis is the same.
\end{proof}

\begin{lemma}[\textbf{MHA recall circuit}]
\label{lem:mha-recall}
A \hyperref[alg:mamba-2]{single-layer simplified Mamba-2} model with $M$ heads using MHA pattern and dimensions $D$, $N$ (with $\mathrm{expand}=2$, $D_\mathrm{conv}=2$, as in Thm.~\ref{thm:compressive-recall}) can solve associative recall tasks (AR or MQAR) with a probability of success that satisfies:
{\small 
\begin{equation} \label{eq:mha-prob}
    p_\mathrm{single}[D,N] \leq
    p_\mathrm{MHA}[D,N,M] \leq
    p_\mathrm{single}[D,MN]
\end{equation} %
}
(Namely, for MHA pattern, splitting into heads improves recall performance, due to the larger key and query dimensions $MN$.)
\end{lemma}

\begin{proof}[\textbf{Proof}]

We prove the lemma using two alternative explanations: effective state size, and statistical error analysis.

\textbf{Effective state size.\quad} Let us carefully examine the model output computation. For simplicity, let us rely on the weights from Lem.~\ref{lem:mva-recall}. We choose head-dependent key and query weights, $S^h_B=S^h_C=F^h$, where $F^h\in \mathbb{R}^{N\times D}$. Each head output is then
\begin{equation}
\hat{y}^h_t=\sum_{\tau=0} ^t \hat{x}_{\tau} \, (F^h\,\hat{x}_{\tau-1})^\top (F^h\,\hat{x}_t) \,,
\end{equation}
with $\hat{x}_t \equiv E\,x_t$. The overall SSM output is, hence, aggregation of $\{\hat{y}^h_t\}$:
\begin{equation}
    \hat{y}_t = \sum_{h=1}^M \,W^h_y \, \sum_{\tau=0} ^t \hat{x}_{\tau} \, (F^h\,\hat{x}_{\tau-1})^\top (F^h\,\hat{x}_t) ,
\end{equation}
which is, since $W_y^h$ and $F^h$ are all independent, equivalent to:
\begin{equation}
    \hat{y} = \sum_{\tau=0} ^t \hat{x}_{\tau} \, (F'\,\hat{x}_{\tau-1})^\top (F'\,\hat{x}_t) ,
\end{equation}
where $F'\in \mathbb{R}^{(MN)\times D} \equiv \mathbb{R}^{N'\times D}$. This implies that effectively, MHA patterns utilizes a larger state size $N'=MN$.

\textbf{Statistical error analysis.\quad} Alternatively, let us quantify how the additional heads shrink the recall error. We have shown in Theorem~\ref{thm:prob-scaling-law} that the \textit{correct}, \textit{wrong} and \textit{empty} entries of the output vector $y_t$, namely $i,j,l$ respectively, are distributed as (see Eq.~\ref{eqs-variance}, Eq.~\ref{eqs-normal} and Thm.~\ref{app-thm:prob-scaling-law}):
\begin{equation}
\begin{aligned}
y_t^i & \sim \mathcal{N}(1,\; 
\sigma_d^2+\tilde \sigma_d^2+N_f\,\sigma_o^2\,\tilde{\sigma}_o^2)  \\
y_t^j & \sim \mathcal{N}(0,\; 
\sigma_o^2 + \tilde{\sigma}_o^2+N_f\,\sigma_o^2\,\tilde{\sigma}_o^2) \\
y_t^l & \sim \mathcal{N}(0,\; 
\sigma_o^2+N_f\,\sigma_o^2\,\tilde{\sigma}_o^2)
\end{aligned}\end{equation}
The $\sigma_o^2$, $\sigma_d^2$ terms originate from embedding compression noise, via $\hat{x}_i^\top\hat{x}_j=x_i^\top E^\top E \,x_j \sim  \mathcal{N}(\delta_{ij},\frac{c}{D})$, with $c$ a constant. This variance terms $\sigma_o^2,\,\sigma_d^2 = O(\frac{1}{D})$ cannot be reduced using additional heads, since they arise from input and output embedding $E$ compression and decompression. However, the splitting into multiple SSM heads does reduce the variance. The $\tilde{\sigma}_o^2$, $\tilde{\sigma}_d^2$ terms arise from SSM compression noise, originally via $(F^h\hat{x}_i)^\top(F^h\hat{x}_j)=x_i^\top E^\top F^{h\top} F^h\, E \,x_j \sim  \mathcal{N}(\delta_{ij},\frac{c}{N})$. Now since we have $M$ independent noise terms, each with variance  $(\tilde{\sigma}^h)^2 =O(\frac{1}{N})$, their weighted average have a smaller variance, reduced by a factor of $M$, namely $\tilde{\sigma}^2 =O(\frac{1}{MN})$. Therefore, we can effectively think of this model as of a single-head model with state size $N'=MN$.
\paragraph{Probability bounds.} The analysis above shows that there exist weights for which the multi-head model solves associative recall with success probability $p_\mathrm{single}[D,MN]$. However, the model is not guaranteed to converge to this particular solution under SGD (or any other optimization method), so this quantity should be viewed only as an upper bound:
\begin{equation}
p_\mathrm{MHA}[D,N,M] \leq p_\mathrm{single}[D,MN]
\end{equation}
Conversely, if all heads share the same weights, i.e., $F^h = F \in \mathbb{R}^{N \times D}$ for all $h$, then the model redundantly implements the MVA solution from Lem.~\ref{lem:mva-recall}. This yields a lower bound:
\begin{equation}
p_\mathrm{MHA}[D,N,M] \geq p_\mathrm{single}[D,N].
\end{equation}
\end{proof}

\subsection{Notation and Dimensions}
\label{app-subsec:mamba-2-notation-and-dims}
We present here the notation and dimensions to be used throughout the analysis of associative recall in Mamba-2. Several differences from Mamba (see Table~\ref{tab:mamba_dims}) are worthy of explanation.
Firstly, note that we begin already with a simplified model, hence omit state update matrix $A$, which is assumed an identity operation. 
Secondly, mind the different pattern of projections and convolutions used in Mamba-2 to produce SSM sequences $\hat{x}$, $B$ and $C$: 
no expansion is used for input projection (thus $D_\mathrm{in}$ is not used), but rather three different projections and convolution kernels, one per $\hat{x},B,C$. (In implementation, these are usually fused into larger, unified projection and kernel; here we separate notation for clarity.) 

\renewcommand{\arraystretch}{1.25}

\begin{table}[h]
    \caption{Notation for Mamba-2}
    \label{tab:mamba2_notation}
    \centering
    \begin{tabular}{llll}
        \hline
        \textbf{Symbol} & \textbf{Meaning} & \textbf{Value} & \textbf{Notes} \\
        \hline
        $D$ & Embedding size & = Number of SSM channels\\
        $N$ & SSM State size & - \\
        $D_\mathrm{conv}$ & Convolution kernel size & - \\

        $\Lambda$ & Number of layers & - \\

        $M$ & Number of SSM heads & - \\
        $P$ & SSM head size & $\frac{D}{M}$ \\
        $p$ & Head pattern & MHA, MKA, MQA or MVA \\
        
        \hline
    \end{tabular}
\end{table}

\begin{table}[h]
    \caption{Dimensions for Mamba-2}
    \label{tab:mamba2_dims}
    \centering
    \begin{tabular}{llll}
        \hline
        \textbf{Symbol} & \textbf{Meaning} & \textbf{Dimension} & \textbf{Notes} \\
        \hline
        $x$         & Model input sequence (one-hot)        & $\mathbb{R}^{V \times L}$        & $x_t \in \mathbb{R}^V$ \\
        $y$         & Model output sequence (one-hot)       & $\mathbb{R}^{V \times L}$        & $y_t \in \mathbb{R}^V$ \\
        
        $x^e$         & Model input sequence (embedded)        & $\mathbb{R}^{D \times L}$        & $x^e_t \in \mathbb{R}^D$ \\
        
        $\hat{x}$   & SSM input sequence                    & $\mathbb{R}^{D \times L}$ & $\hat{x}_t \in \mathbb{R}^{D}$ \\
        $\hat{y}$   & SSM output sequence                   & $\mathbb{R}^{D \times L}$ & $\hat{y}_t \in \mathbb{R}^{D}$ \\
        
        \hline

        $E$         & Model embedding        & $\mathbb{R}^{D \times V}$ \\
        
        $P^{\hat{x}}_\mathrm{in}$ & Input projection for $\hat{x}$ & $\mathbb{R}^{D \times D}$ \\
        $P^{B,C}_\mathrm{in}$ & Input projection for $B,C$ & $\mathbb{R}^{D \times D}$ \\
        
        $W^{\hat{x}}_\mathrm{conv}$         & Convolution kernel for $\hat{x}$        & $\mathbb{R}^{D \times D_\mathrm{conv}}$  \\ 
        $W^{B,C}_\mathrm{conv}$         & Convolution kernel for $B,C$        & $\mathbb{R}^{D \times D_\mathrm{conv}}$  \\ 

        $S_B^h$         & $B^h$ projection & $\mathbb{R}^{N \times D}$  
        & Per head $h$; shared if $p \in \{\mathrm{MQA}, \mathrm{MVA}\}$
        \\ 
        $S_C^h$         & $C^h$ projection & $\mathbb{R}^{N \times D}$  
        & Per head $h$; shared if $p \in \{\mathrm{MKA}, \mathrm{MVA}\}$
        \\ 

        $B^h$         & SSM input coefficient        & $\mathbb{R}^{N \times L}$        & Per head $h$; shared if $p \in \{\mathrm{MQA}, \mathrm{MVA}\}$ \\
        $C^h$         & SSM readout coefficient      & $\mathbb{R}^{N \times L}$        & Per head $h$; shared if $p \in \{\mathrm{MKA}, \mathrm{MVA}\}$ \\
        
        $P_\mathrm{out}$ & Output projection & $\mathbb{R}^{D \times D}$ \\ 
        $W_x$ & Head contraction weights & $p$-dependent & Used in $\mathrm{contract}(\cdot, W_x)$ 
        \\
        $W_y$ & Head aggregation weights & $p$-dependent & Used in $\mathrm{aggregate}(\cdot, W_y)$ 
        \\

        \hline
    
    \end{tabular}
\end{table}

\renewcommand{\arraystretch}{1.00}

\color{black}

\FloatBarrier
\FloatBarrier
\section{\texorpdfstring{Ablation Studies and Generalizations}{Ablation Studies and Generalizations}}
\label{app:additional-experiments}

\subsection{\texorpdfstring{Recall Circuit Minimality Ablation}{Recall Circuit Minimality Ablation}}
\label{app:mqar_ablation_impl_details}

\paragraph{Circuit minimality experiments.}
To find the critical components of Mamba's recall circuit, we ablate the components of a single-layer Mamba model in the MQAR task. Each row in the table removes an additional component from the Mamba block. The \textit{MQAR Accuracy} column shows the performance of the model, averaged after training over 5 different seeds. Also see Fig.~\ref{app-fig:circuit-minimality}. 
Besides the SSM itself, we find that the convolution block is the most critical component of Mamba's recall circuit. Notably, removing the convolution (E) does not collapse accuracy to zero but to $\approx 1/N_f$: without the conv the model cannot bind keys to values, yet its state still retains the bag of in-context values, and guessing among the $N_f=16$ of them scores $1/N_f \approx 0.06$.

    \begin{table}[ht]
        \centering
        \caption{\textbf{Circuit minimality experiment results.} See Fig.~\ref{app-fig:circuit-minimality}.}
        \begin{tabular}{clc}
        \toprule
        \textbf{Model ID} & \textbf{Description} & \textbf{MQAR Accuracy} \\
        \midrule
        Base & 1-Layer Mamba model w/o layer normalization & $1.00 \pm 0.00$ \\
        A & Base model with $\bar{A_t} = I$ & $1.00 \pm 0.00$\\
        B & Model A w/o gate & $1.00 \pm 0.00$\\
        C & Model B w/o activation function (post-Conv1D) & $1.00 \pm 0.00$ \\
        D & Model C with $d_{\text{conv}} = 2$ & $1.00 \pm 0.00$\\
        E & 
        Model C with $d_{\text{conv}} = 1$ (effectively no Conv1D) & $0.06 \pm 0.00$\\
        \bottomrule
        \end{tabular}
    \label{fig:mqar_ablation}
        
    \end{table}

\begin{figure}[ht]
\centering
\includegraphics[width=0.5\textwidth]{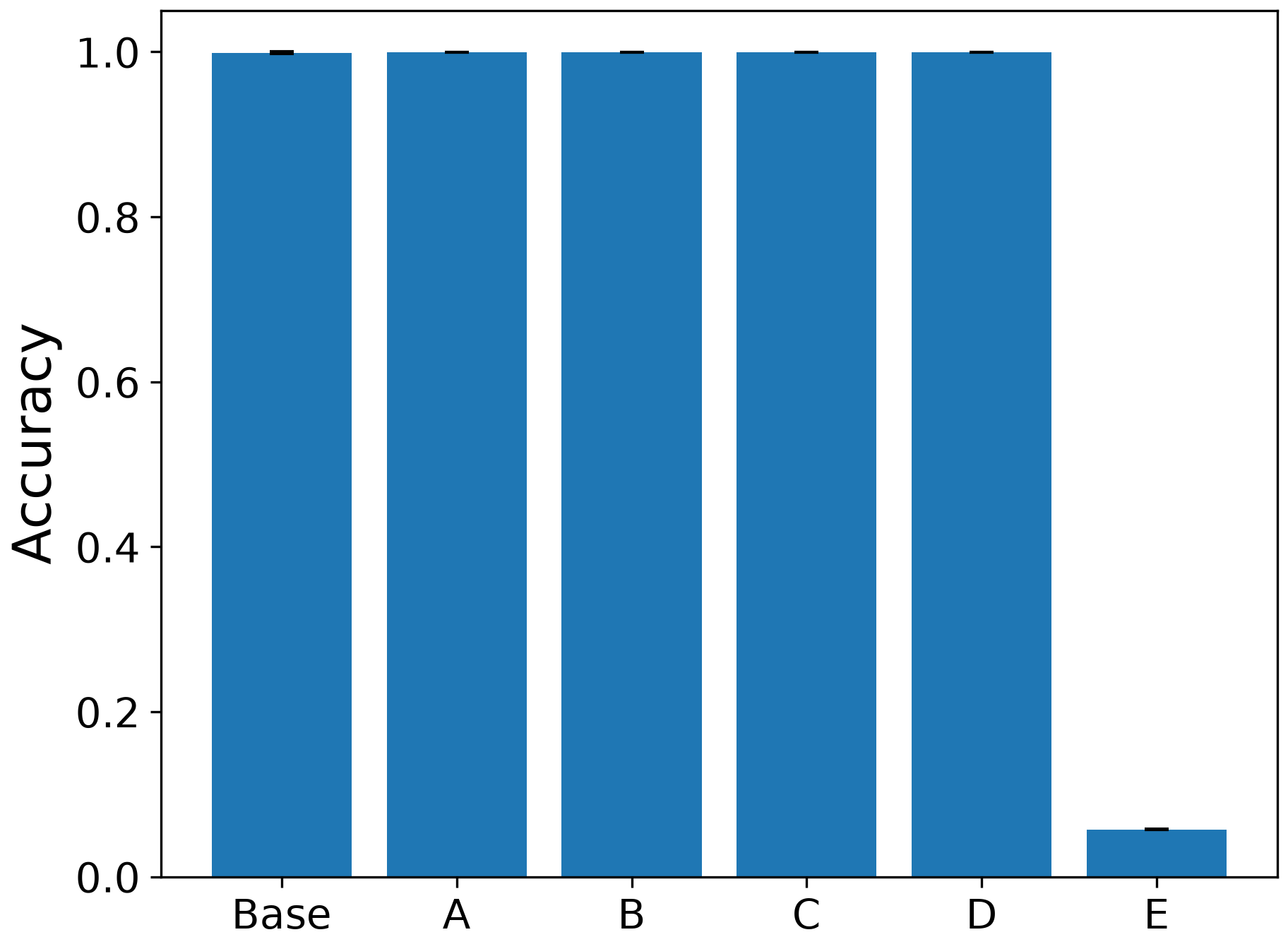}
\caption{\textbf{Circuit minimality experiment results.} The MQAR accuracy of Table~\ref{fig:mqar_ablation}'s ablated models.}
\label{app-fig:circuit-minimality}
\end{figure}

\paragraph{Implementation details.}
We use a vocabulary size of $V = 128$; each sequence is $L = 64$ tokens long and contains $N_f = 16$ key-value pairs. The model has a single layer with $D = 64$ channels and a state dimension of $N = 16$.
Training follows the setup of App.~\ref{app:impl-details}, with a $100/1400/2500$-step warmup/flat/decay schedule and at most $4000$ steps, early stopping at accuracy $0.999$, weight decay $0.1$, and gradient clipping $1.5$.
Each model is trained with $5$ different seeds, and we report the mean and standard deviation. In models Base, A, B, and C, $d_{\text{conv}} = 4$, as in the default Mamba block. In model D, we use $d_{\text{conv}} = 2$, which is the paper default (the minimal value that enables recall). In model E, $d_{\text{conv}} = 1$: a size-one kernel cannot mix a token with its predecessor, effectively removing the Conv1D.

\FloatBarrier
\subsection{\texorpdfstring{MQAR with Repeated Keys}{MQAR with Repeated Keys}}
\label{app-subsec:repeated-keys-queries}

\paragraph{Repeated Keys.}
We study a generalization of MQAR where each key appears $N_k$ times, each time paired with a different value, and the model must retrieve the value of the key's last occurrence, as in the following highlighted example ($N_k=2$):
\begin{center}
\small
\textbf{$x$} \;\;
\texttt{
\f{A} \w{2} \f{B} \w{9} \f{C} \w{4}
\f{A} \w{6} \f{B} \w{3} \f{C} \w{7} \q{B} \h{2} \h{5} \h{0} \h{9} \q{C} \h{4} \q{A} \h{8} \h{1}}
\end{center}
\begin{center}
\small
\textbf{$y$} \;\;
\texttt{
\z{*} \z{*} \z{*} \z{*} \z{*} \z{*}
\z{*} \z{*} \z{*} \z{*} \z{*} \z{*} \w{3} \z{*} \z{*} \z{*} \z{*} \w{7} \z{*} \w{6} \z{*} \z{*}}
\end{center}

\paragraph{Experiment results.}
Results are summarized in Fig.~\ref{fig:mqar-keys}. We see that the single-layer model performs no better than a mere guess among the $N_k$ stored values: its accuracy is approximately $1/N_k$. We conclude that every occurrence writes its value into the same hidden-state slot, which ends up holding a roughly equal mix of all $N_k$ values instead of just the last one. Depth resolves this: with two layers, the first layer can read back the currently stored value, and the second layer writes in the difference between the new value and the old one, overwriting the stale entry instead of blending with it.

\begin{figure*}[!ht]
\centering
\includegraphics[width=0.54\textwidth]{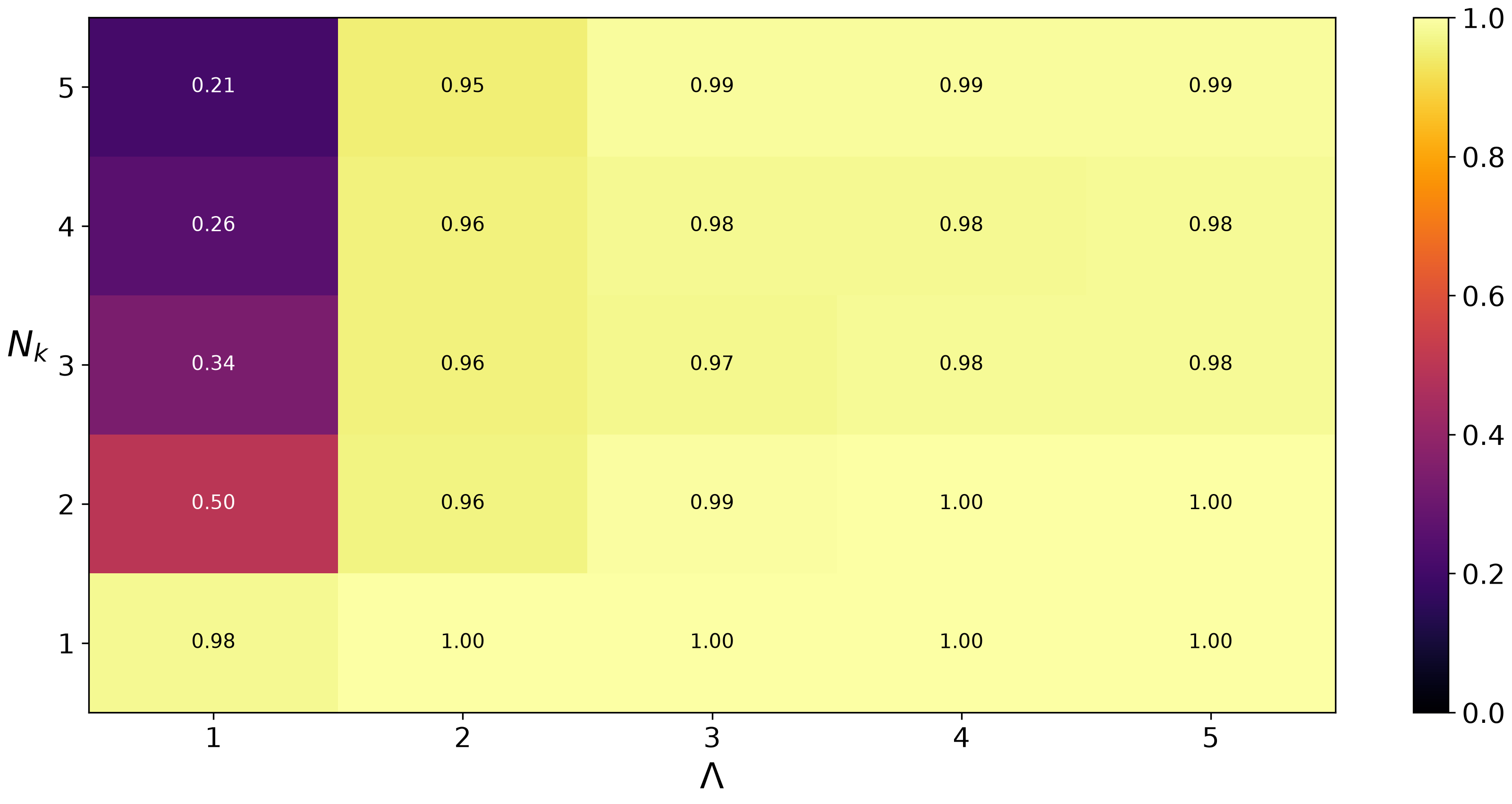}
\hspace{0.02\textwidth}
\includegraphics[width=0.39\textwidth]{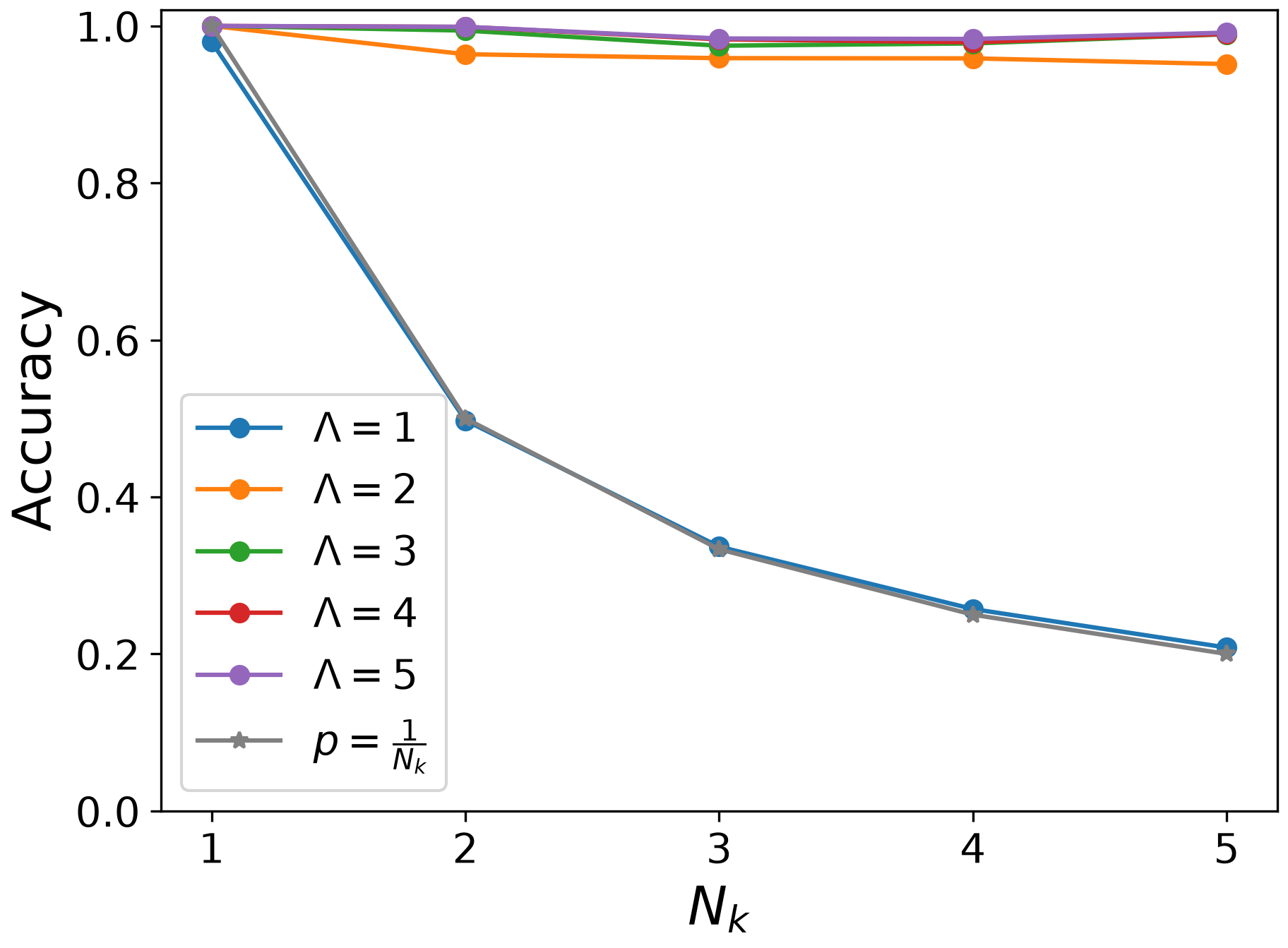}
\caption{\textbf{MQAR with repeated keys.} Recall accuracy of trained linear models. In this experiment $V = 512$, $L = 128$, $N_f = 10$, $D = 100$, and $N_\mathrm{eff} = \Lambda N = 60$ fixed (see Thm.~\ref{thm:multi-layer-recall-scaling-laws}), in order to keep total state size constant. \textbf{Left:} accuracy over the $\Lambda$-$N_k$ grid. \textbf{Right:} accuracy against $N_k$ per depth $\Lambda$, with a $1/N_k$ guessing baseline (gray).}
\label{fig:mqar-keys}
\end{figure*}

\FloatBarrier

\section{Similarities and Differences with \textit{Understanding Input Selectivity in Mamba}}
\label{app:sim_and_diff_huang}

This section provides a methodological and technical comparison between our work and \citet{huang2025understanding}.

Starting with a bird-eye view, both works focus on Mamba model variants and their ability to perform associative recall. In addition, both works offer a detailed theoretical perspective, supported by experiments. Our study, however, proves \textit{tighter bounds} for the dimensions required for the Mamba recall circuit. These tighter bounds enable both \textit{mechanistic investigation} and theoretical derivation of highly-predictive AR \textit{scaling laws} for Mamba, which are \textit{empirically verified} - capabilities not possible with the bounds established by \citet{huang2025understanding}.  

The main points of difference lie in both contribution and methodology:
\begin{enumerate}
    \item We theoretically show that Mamba can learn a parameter-efficient solution for MQAR, i.e. a solution which requires feasible model dimensions, satisfying $ND=O(\log{V})$ given a fixed number of facts $N_f$. These are \textbf{tighter bounds} compared to the bounds proved by \citet{huang2025understanding} for Mamba, which are $D=O(V_k+\log{V_v})$ and $N=V_k$.  
    \item Our suggested recall circuit is validated by \textbf{mechanistic interpretability} experiments, based on trained model weights and activations. In other words, our interest lies in the mechanism the model \textit{truly learns} during training, as opposed to those it \textit{could, in principle, learn} based solely on its architectural capacity. 
    \item Our contribution includes the theoretical derivation of \textbf{scaling laws} for recall in Mamba, as well as empirical validation. The exact scaling laws (Theorem~\ref{thm:prob-scaling-law} and Remark~\ref{rem:dim-scaling-law}), which align with experimental result, cannot be derived from \citet{huang2025understanding} Theorem 1, since it is based on a different circuit.
    \item In addition to single-layer models, our study generalizes to \textbf{multi-layer} Mamba models, aiming to contribute a further understanding of multi-layer recall circuits (Theorem~\ref{thm:multi-layer-recall-scaling-laws}). Our conclusions are validated through empirical examination in Section~\ref{subsec:multi-layer-experiments}. 
    \item While \citet{huang2025understanding} focus on single-head SSM mixers for Mamba-2, our work further contributes the study of \textbf{multi-head} recall circuits, as enabled by the Mamba-2 architecture (Theorem~\ref{thm:multi-head-recall}). The theoretical insights are further verified by experimental analysis in Section~\ref{subsec:multi-head-experiments}. 
    \item By methodology, we seek a \textbf{minimal recall circuit}, in order to isolate factors and find the most significant component for recall (see Appendix~\ref{app:mqar_ablation_impl_details}). Upon removal of model components, we identify linear compression limits as the main factor, represented by the model dimensions $D$ and $N$. Neither gating nor nonlinearities drastically overcome this inherent bottleneck (as shown in Figure~\ref{fig:scaling-laws-2d-dims}).
    
\end{enumerate}

Let us now carefully dive into a technical comparison between the works.

(Note that by \textit{Mamba}, we specifically refer to \textit{"Mamba-1"} architecture, as opposed to \textit{Mamba-2}.)

\paragraph{Non-compressive recall circuits.}
Lemma 3 in \citet{huang2025understanding} proves that a 1-layer \textit{Mamba-$\Delta^\top$} model can solve the Induction Heads task, if model dimensions are $D,N=O(V)$. Theorem~\ref{thm:non-compressive-recall} in our work leverages a similar construction for a 1-layer \textit{Mamba}, showing it can solve MQAR, given $D,N=O(V)$. Importantly, however, this Theorem is not the key novelty of our work, but rather serves as an intermediate logical step towards Theorem~\ref{thm:compressive-recall}, which addresses a realistic scenario where $ND=O(\log{V})<O(V^2)$ (given a fixed number of facts $N_f$), thus applicable to actual MQAR tasks.

\paragraph{Compressive recall circuits.}
Theorem 2 in \citet{huang2025understanding} shows that a 1-layer \textit{Mamba} with dimensions $D=O(V_k+\log{V_v})$ and $N=V_k$ can solve MQAR, by applying JL lemma to the values (but not to the keys, which remain uncompressed). 

Firstly, Theorem~\ref{thm:compressive-recall} in our work improves the \textbf{efficiency} of the solution, by showing that a 1-layer \textit{Mamba} model can solve MQAR with smaller, realistic dimensions, satisfying $ND=O(\log{V})$ given a fixed number of facts $N_f$, by using a different construction, which allows us to apply JL lemma \textit{twice}, to both values and keys. (Technically, we show \textit{Mamba} can compress its inputs twice: first compression via embedding into $Ex_t\in \mathbb{R}^{D}$, used for values; then a second compression via $S_B,S_C$ into $FEx_t \in \mathbb{R}^{N}$, used for keys and queries.) This difference is of major importance for realistic Mamba models and MQAR tasks, where $N=V_k$ and $D=O(V_k)$ are impractical and cannot be used due to memory limitation. A parameter-efficient solution is critical in practice; see, for instance, Figure~\ref{app-fig:scaling-laws-2d-dims}, row (iii), where we empirically show that a 1-layer Mamba can perfectly solve an MQAR task of $V_k=\frac{V}{2}=512$ given relatively small dimensions $N=32<<V_k$ and $D=64<<V_k$.

Secondly and most importantly, our \textbf{mechanistic analysis} in Section~\ref{subsec:mech-interp} (also see Figure~\ref{fig:reversing-hidden}) 
highlights the difference between the two constructions. The solution suggested by \citet{huang2025understanding} implies that \textit{each (hidden state) column corresponds to a specific key}: each embedded value $Ev_\tau$ is stored in a \textbf{single}, distinct key coordinate of the hidden state $h_t$ (equivalent to $H'_t$ in Section~\ref{subsec:mech-interp}), out of $N=V_k$ coordinates in total. In contrast, our parameter-efficient construction suggests each embedded value $Ev_\tau$ is stored in a \textbf{distributed} manner, across a linear mixture of \textbf{all} $N \ll V_k$ key coordinates. The mechanistic experiments we conduct in Section~\ref{subsec:mech-interp} support our suggested circuit. It shows that \textit{Mamba} actually learns a solution where writing a single key-value pair into $H'_t$ affects \textbf{all} key coordinates, leading to the pattern observed for $H'_t$, where all entries are nonzero (Figure~\ref{fig:reversing-hidden}). If, conversely, the learned mechanism would have used a \textbf{single} key entry per value $Ev_\tau$, one could have inspected the hidden state by decompressing value coordinates \textbf{only}, via $H''_t=E^\top H'_t$, since key coordinates would not interfere with one another. However, we experimentally show that a \textbf{double}-decompression is rather required, $H''_t=E^\top H'_t \,\tilde{E}$ (see Equation~\ref{eq:h-double-prime} and Figure~\ref{fig:reversing-hidden}), supporting our theory. 

Thirdly, our theoretical and empirical \textbf{scaling laws} (Theorem~\ref{thm:prob-scaling-law} and Figure~\ref{fig:scaling-laws-2d-dims} respectively) indicate that \textit{Mamba} recall capacity strongly depends on both dimensions $N$ and $D$. The exact form of this dependence, where performance improves upon increasing $\frac{ND}{N_f}$, arises due to \textit{noisy retrieval} of values out of the hidden state $H'_t$. Importantly, while the $\frac{1}{D}$ factor is caused by value reconstruction noise $v_i''=E^\top E v_i+O(\frac{1}{D})$, the $\frac{N_f}{N}$ factor arises from noisy query-key matching $\sum_p ^{N_f} (\tilde{E}k_p)^\top (\tilde{E}k_m)=O(\frac{N_f}{N})$. This is only the case if values are stored as $H'_t=\sum_p ^{N_f} (Ev_p) (\tilde{E}k_p)^\top$, where keys are compressed into $\tilde{E}k_p\in \mathbb{R}^N$. Without relying on such a \textit{double-compressive} circuit (as in Theorem~\ref{thm:compressive-recall}, as opposed to Theorem 1 in \citet{huang2025understanding}), the derived scaling law would lack this important $\frac{N_f}{N}$ factor, which is clearly demonstrated in practice (Figure~\ref{fig:scaling-laws-2d-dims}).

\paragraph{Minimal recall circuits.}

Our work focuses on finding a \textbf{minimal} Mamba model that can solve MQAR, in order to isolate the contributing factors. In our perspective, this approach is important in order to understand potential \textbf{bottlenecks }in Mamba architecture. Since Mamba have a variety of components (including discretization, gating, nonlinearities, normalization layers, input selectivity, etc.), there may \textit{exist} many possible constructions combining them that can \textit{possibly} enable Mamba to solve MQAR. Minimality is required if we aim to find a critical component.

Firstly, our ablation study (Appendix~\ref{app:mqar_ablation_impl_details}; see table entry B) clearly shows that a Mamba model without gating is still able to achieve perfect accuracy on MQAR. This hints that, while probably participating in the recall process, the gate-based mechanism is empirically not a bottleneck for Mamba ability to perform recall.

This insight is further reinforced by the results presented in Figure~\ref{fig:scaling-laws-2d-dims}. The conducted experiment compares between our gate-free simplified model and the full gated model, both trained on MQAR. Our findings show the models perform quite similarly on MQAR. This demonstrates that the main bottleneck in Mamba recall capacity is the model dimensions $D$ and $N$ (responsible for reliable compression-decompression schemes of keys and values), rather than the existence of a gate.

\color{black}

\section{Training and Implementation Details}\label{app:impl-details}

This appendix provides all training and implementation details needed to reproduce our results.
Also see \textit{reproducibility statement}, App.~\ref{app:reproducibility}.

\subsection{Datasets}
We use the original MQAR dataset as implemented by~\citet{arora2023zoology}, vendored with minor adaptations in the released code.
To avoid generalization error and overfit phenomena, instead of generating large training datasets statically, we sample new batches from MQAR in any training step.

\subsection{Models}
As for the full Mamba model, we use \texttt{mamba-ssm}~\citep{gu2023mamba}, unmodified.
Our simplified linear Mamba model is implemented based on \texttt{mamba-tiny}~\citep{mambatiny2024}, vendored with local modifications.
We use PyTorch \citep{paszke2019pytorch} for all experiments.
Asset licenses are listed in App.~\ref{app:licenses}.

\subsection{Training Hyperparameters}
Unless stated otherwise, we use the hyperparameters in Table~\ref{tab:hyperparams}.

\begin{table}[ht]
\centering
\small
\setlength{\tabcolsep}{2.5pt}
\renewcommand{\arraystretch}{1.2}
\caption{Training hyperparameters per experiment family.}
\label{tab:hyperparams}
\begin{tabularx}{\textwidth}{@{}l|*{6}{>{\centering\arraybackslash\hspace{0pt}}X|}l@{}}
\toprule
\makecell[l]{Model \\ Task} & Linear MQAR & Linear AR & Full MQAR & Multi-layer MQAR & Multi-layer AR & Multi-head MQAR & Notes \\
\midrule
Batch size & 128 & 128 & 128 & 128 & $750$ & 128 & \\
\grayline
Optimizer & \multicolumn{6}{c|}{AdamW, $\beta_1=0.9$, $\beta_2=0.95$} & \\
\grayline
Base learning rate & \multicolumn{6}{c|}{$0.01$} & \\
\grayline
Weight decay & \multicolumn{6}{c|}{$0.0$} & \\
\grayline
Label smoothing & \multicolumn{6}{c|}{\num{0.1}} & \\
\grayline
Total steps & $2000$ & $2000$ & $20000$ & $4000$ & $10000$ & $4000$ & \\
\hspace{1em}Warmup steps & $100$ & $100$ & $100$ & $100$ & $250$ & $100$ & $0$ to base LR \\
\hspace{1em}Flat-LR steps & $400$ & $400$ & $5900$ & $1400$ & $2250$ & $1400$ & at base LR \\
\hspace{1em}Decay steps & $1500$ & $1500$ & $14000$ & $2500$ & $7500$ & $2500$ & cosine decay to $0$ \\
\grayline
Early stop at accuracy & $1.0$ & $1.0$ & $0.999$ & $1.0$ & $1.0$ & $1.0$ & \\
\grayline
Gradient clipping & $1.5$ & $1.5$ & \num{0.75} & $2.5$ & $2.5$ & $5$ & (global norm) \\
\grayline
Seeds per grid point & $3$ & $3$ & $5$ & $3$ & $3$ & $3$ & \\
\midrule
Per-grid GPU time & $4.6$ h & $11.4$ h & $54.9$ h & $2.3$ h & $8.7$ h & $1.6$ h & mean across grids \\
\bottomrule
\end{tabularx}
\end{table}

\subsection{Compute and Environment}
All experiments were run on 4$\times$GPU and 8$\times$GPU nodes with Nvidia B200 GPUs, with CUDA~12.8 and PyTorch~2.10.0, using multiprocessing of several training workers per GPU. Each grid was run on a single GPU (the multi-layer AR grid on two), and independent grids ran in parallel within a node.
We fix \texttt{torch}, \texttt{mamba-ssm}, and \texttt{python} versions in the released setup script, and all remaining direct dependencies via a \texttt{constraints.txt}.

Each accuracy-grid figure or sub-figure (e.g.\ Fig.~\ref{fig:scaling-laws-2d-dims}, \ref{fig:multi-layer-scaling-laws}, \ref{fig:multi-head-scaling-laws_}) corresponds to roughly 500-1000 grid points (each trained with 3-5 seeds), each individually small.
Simplified linear model grids take $8.0 \pm 5.5$ hours wall-clock time per grid (single GPU); full Mamba model grids take $2.3 \pm 1.0$ days wall-clock time per grid (see Table~\ref{tab:hyperparams}).
Across all experiments in the paper (single-layer, multi-layer, multi-head, and ablation studies), the total compute amounts to roughly $10$ GPU-days.

\subsection{Seeds Aggregation}\label{app-subsec:seeds}
We briefly describe our considerations choosing best-of-seeds accuracy.
In this work we investigate Mamba's recall \textit{capacity}, therefore retaining the maximal result achieved by each configuration is the most relevant quantity for our analysis. In other words,  we are interested in the \textit{existence} of a recall solution, not the \textit{probability} of achieving it by optimization. Therefore, it makes more sense to take into account only the trials where optimization succeeds. To demonstrate the effect of seed aggregation, in Figs.~\ref{app-fig:seed-grids} and \ref{app-fig:seed-reliability} we further provide a detailed analysis of a single full-model grid: the per-seed accuracy grids against the best-of-seeds grid, and the per-cell convergence rates near the capacity boundary.

\begin{figure*}[!ht]
\centering
\small
\begin{tabular}{@{\hskip 0.02in}c@{\hskip 0.02in}c@{\hskip 0.02in}c@{\hskip 0.02in}c@{\hskip 0.02in}c@{\hskip 0.06in}c@{\hskip 0.02in}c}
  seed${=}0$ & seed${=}1$ & seed${=}2$ & seed${=}3$ & seed${=}4$ & Best Seed & \\
  \raisebox{-0.5\height}{\includegraphics[width=0.145\textwidth]{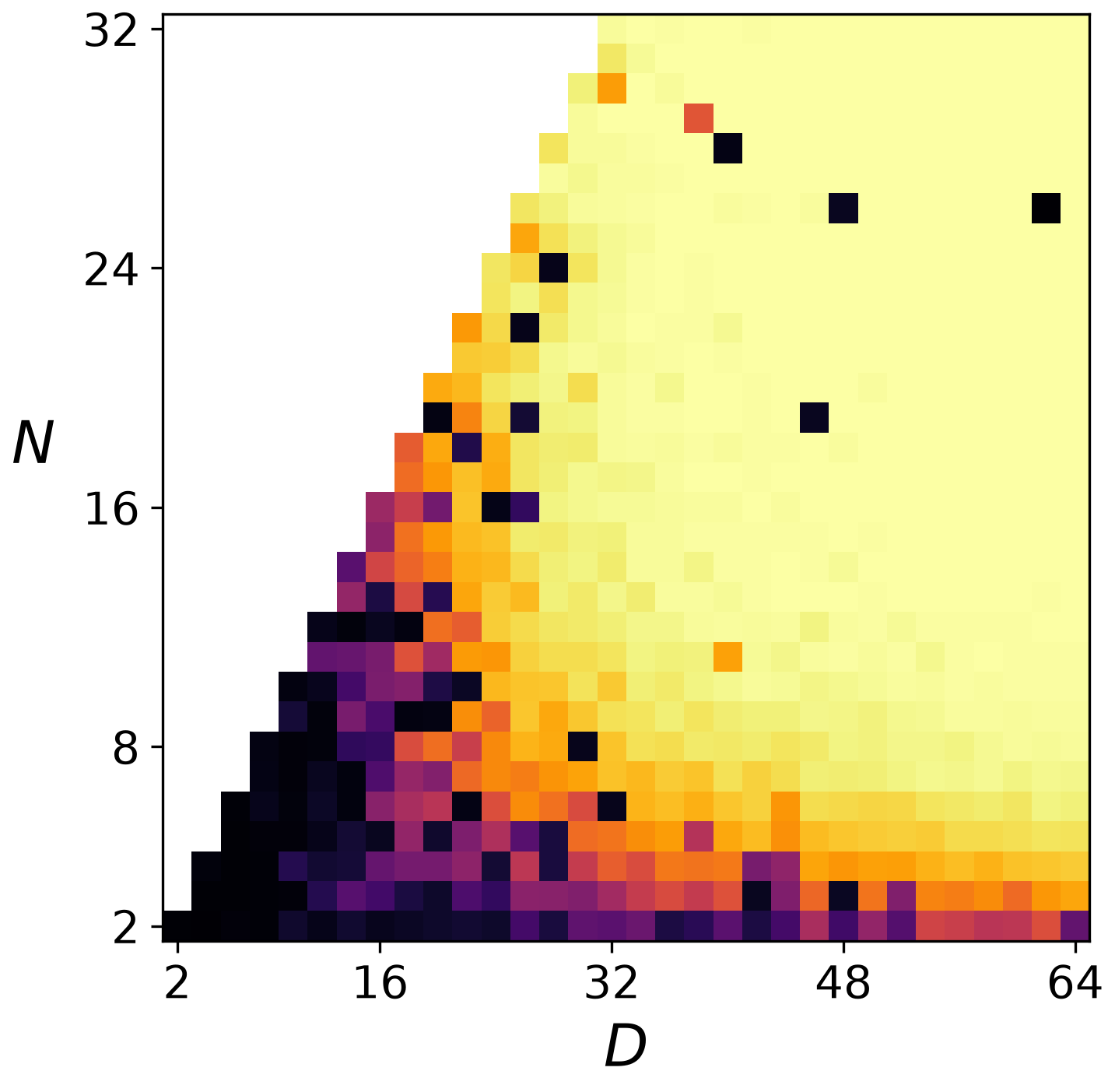}} &
  \raisebox{-0.5\height}{\includegraphics[width=0.145\textwidth]{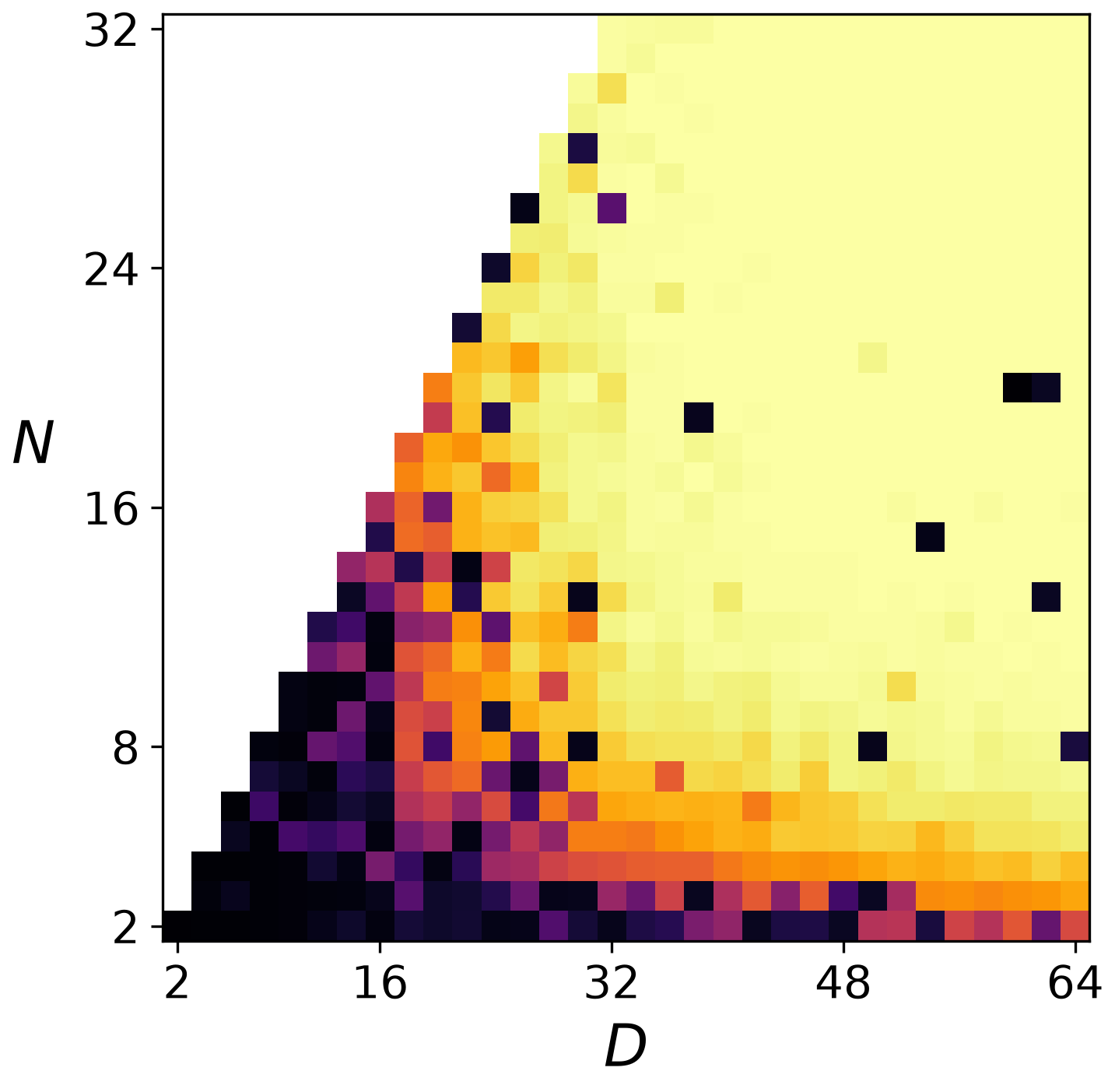}} &
  \raisebox{-0.5\height}{\includegraphics[width=0.145\textwidth]{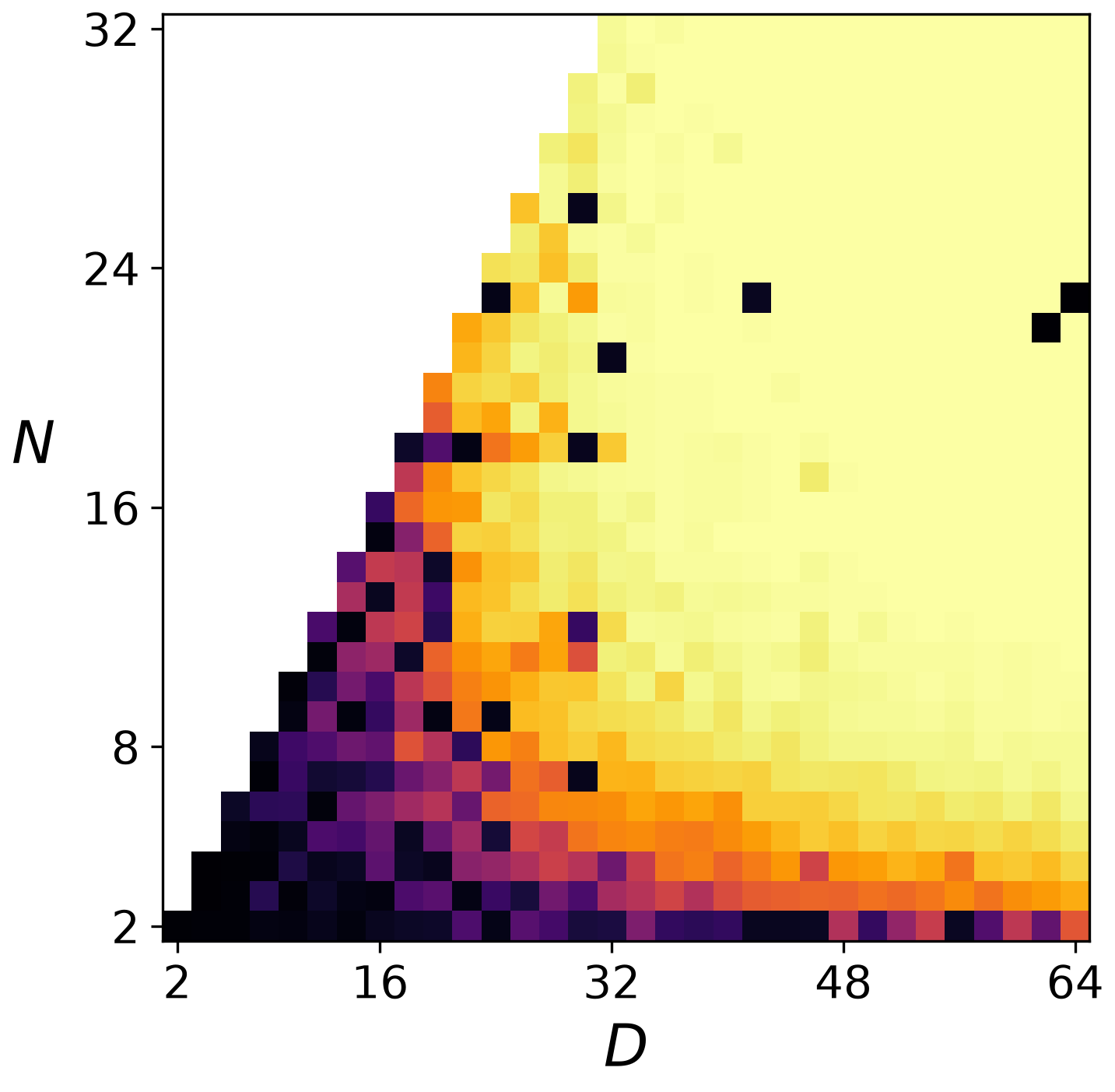}} &
  \raisebox{-0.5\height}{\includegraphics[width=0.145\textwidth]{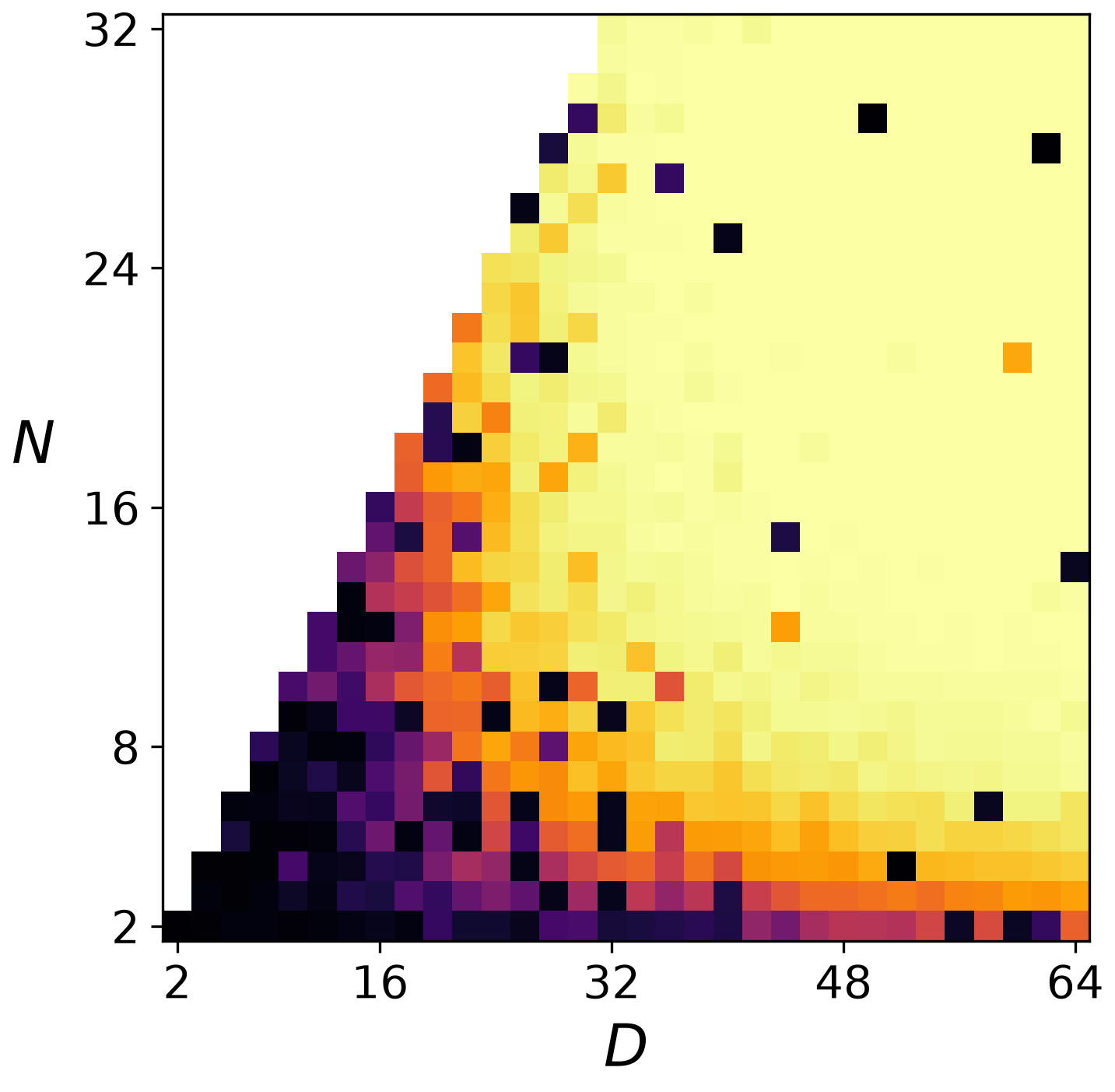}} &
  \raisebox{-0.5\height}{\includegraphics[width=0.145\textwidth]{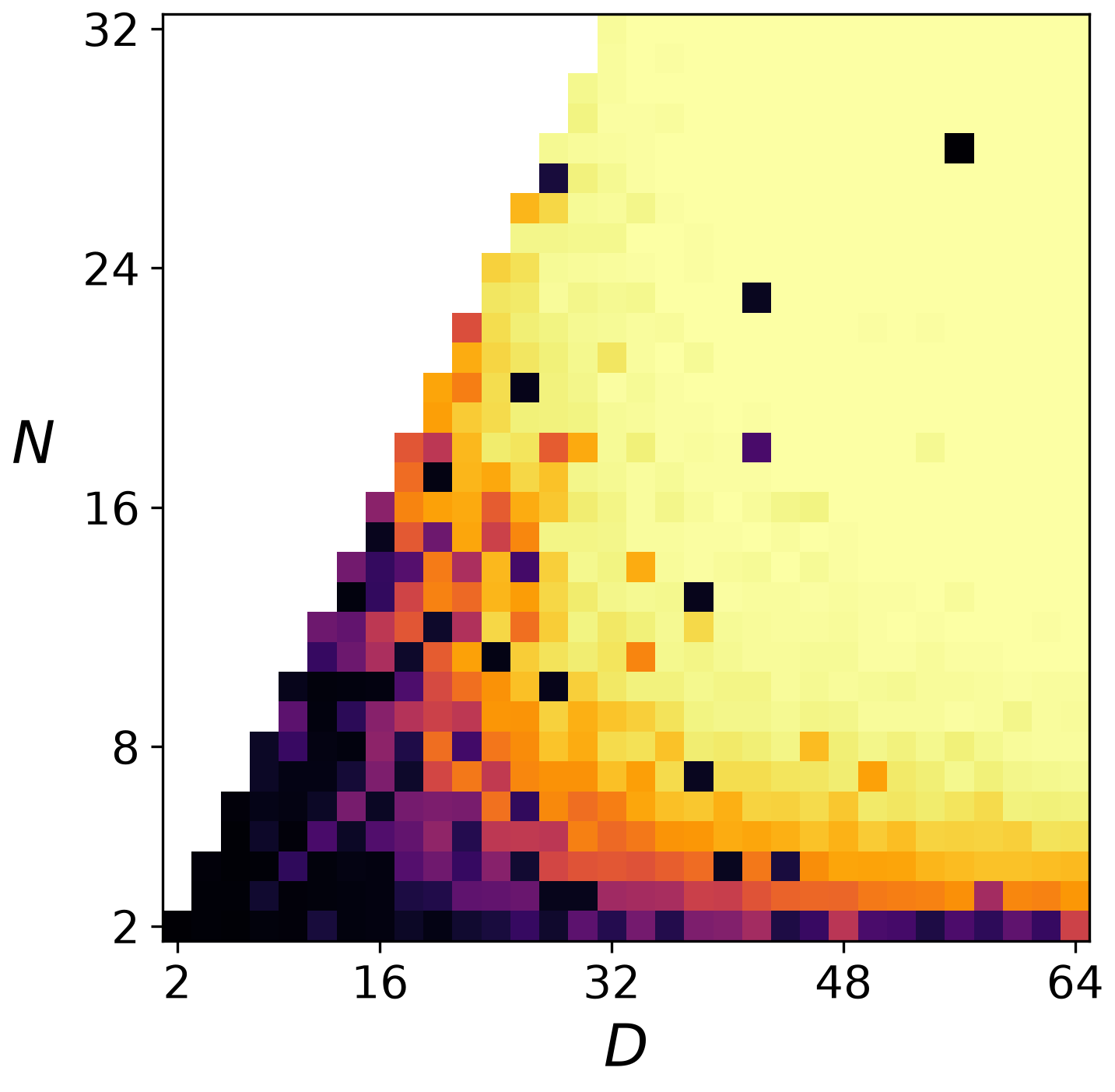}} &
  \raisebox{-0.5\height}{\includegraphics[width=0.145\textwidth]{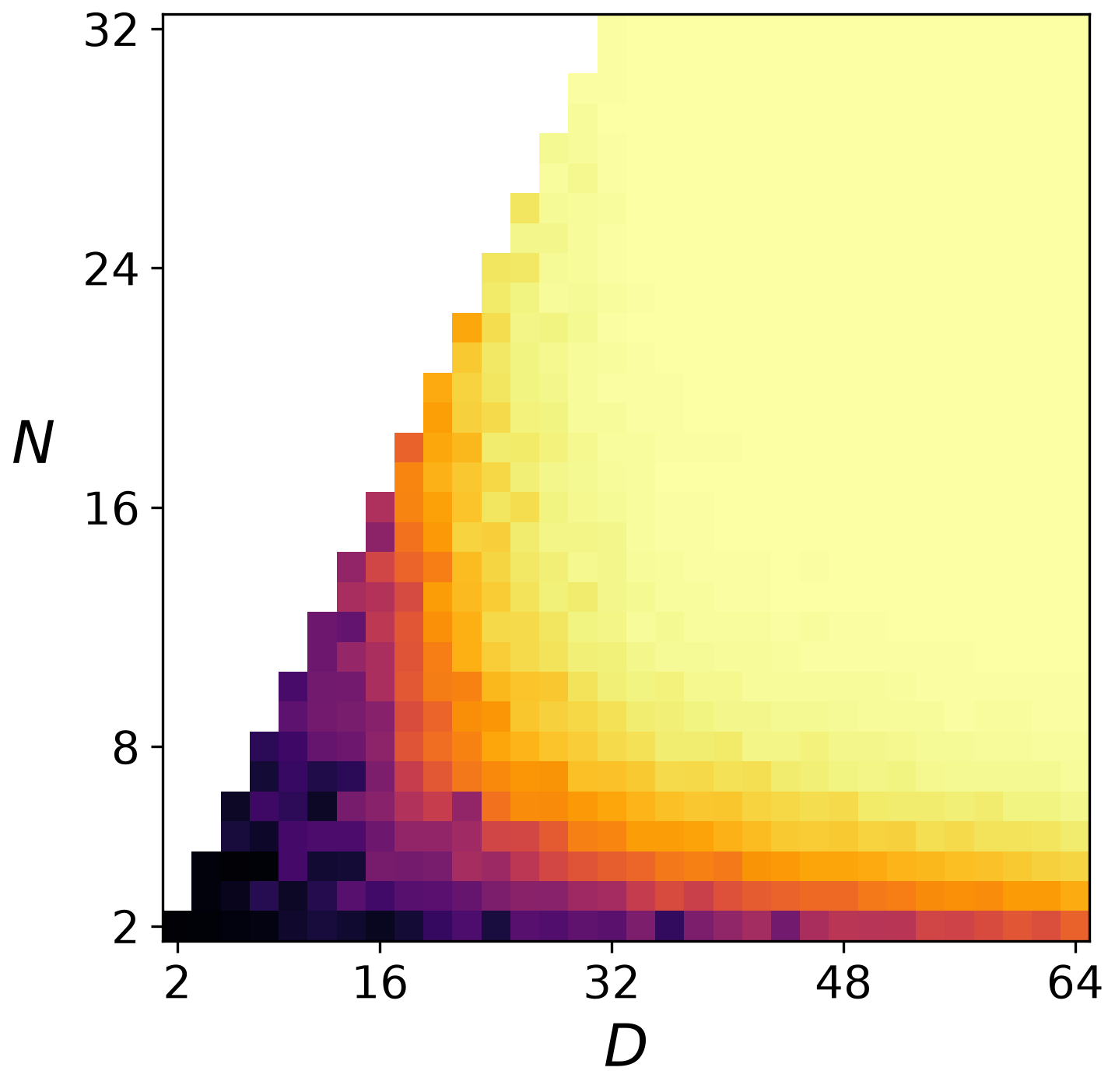}} &
  \raisebox{-0.5\height}{\includegraphics[width=0.022\textwidth]{figures/external/colorbar.png}} \\
\end{tabular}
\caption{\textbf{Per-seed accuracy grids.} Accuracy of the full model over the $D$-$N$ grid of regime (iii) of Fig.~\ref{app-fig:scaling-laws-2d-dims}, shown separately for each of the $5$ seeds, next to the best-of-seeds grid used throughout the paper. The recall phase boundary is consistent across seeds; seeds differ mainly in isolated optimization failures (analyzed in Fig.~\ref{app-fig:seed-reliability}).}
\label{app-fig:seed-grids}
\end{figure*}

\begin{figure*}[!ht]
\centering
\includegraphics[width=0.35\textwidth]{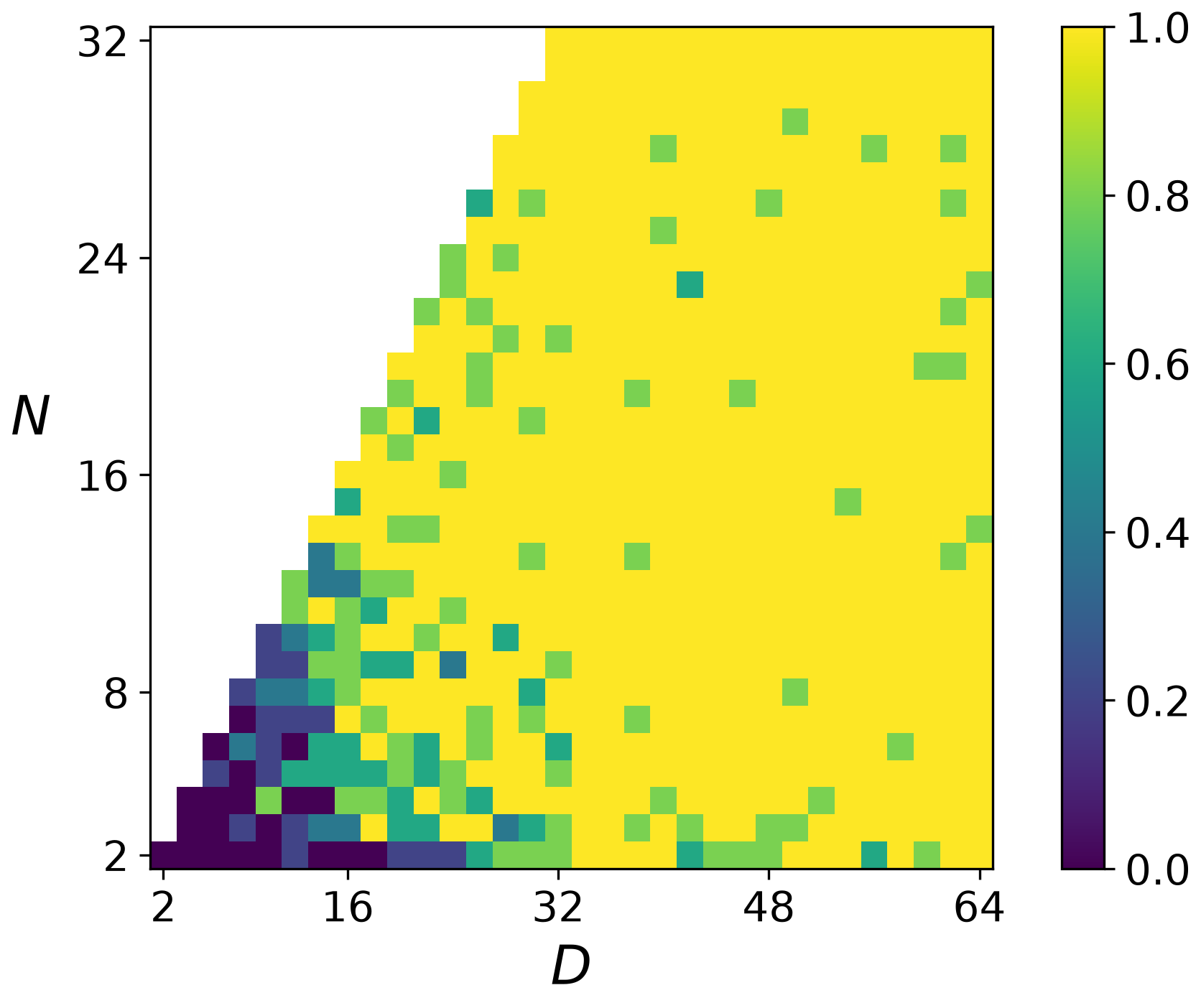}
\hspace{0.02\textwidth}
\includegraphics[width=0.48\textwidth]{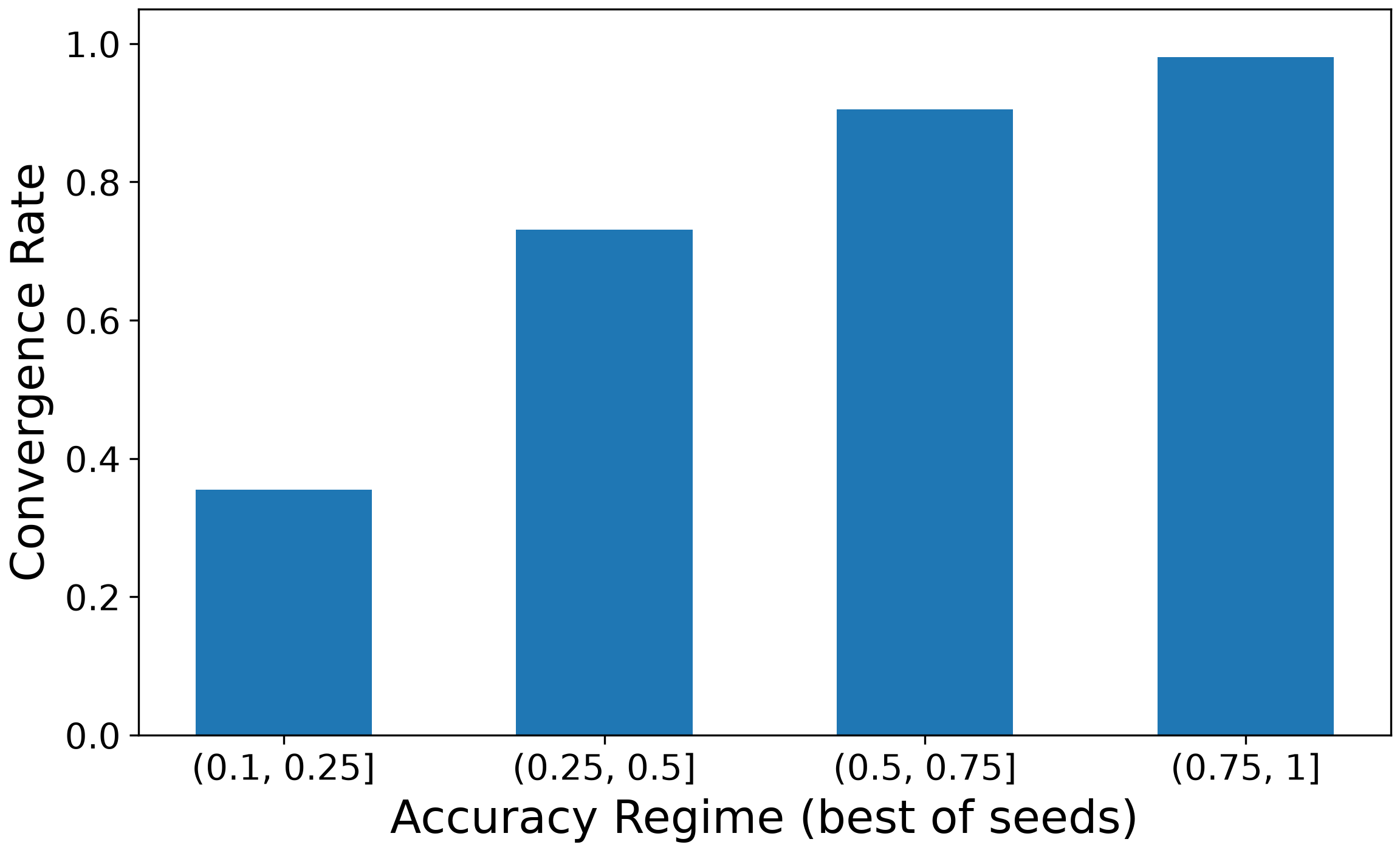}
\caption{\textbf{Seed aggregation analysis.} The full model over the $D$-$N$ grid of regime (iii) of Fig.~\ref{app-fig:scaling-laws-2d-dims}, trained with $5$ seeds per grid point; a seed is counted as converged if it reaches accuracy $\geq 0.1$. \textbf{Left:} per-cell fraction of converged seeds. \textbf{Right:} per-seed convergence rate, binned by the cell's best-of-seeds accuracy.}
\label{app-fig:seed-reliability}
\end{figure*}

\section{\texorpdfstring{Statements and Declarations}{Statements and Declarations}}
\label{app:statements}

\subsection{Impact Statement}
This work analyzes the recall capabilities of Mamba models and linear RNNs, which are crucial for deploying recurrent LLMs in real-world systems. Improving memory efficiency and retrieval understanding can lead to more reliable and controllable architectures, with broad benefits for AI applications. We do not anticipate any direct negative societal impacts beyond those generally associated with LLMs.

\subsection{Reproducibility Statement}\label{app:reproducibility}
We provide the source code for the main experiments and include detailed descriptions of the configurations in the publicly released repository, linked on the first page.
All datasets, models and metrics used in each experiment are explicitly stated in the paper to ensure full transparency and reproducibility. See also implementation details in App.~\ref{app:impl-details}.

\subsection{Licenses}\label{app:licenses}

\begin{table}[h]
\centering
\caption{Open-source assets used in this work and their licenses.}
\begin{tabular}{@{}llll@{}}
\toprule
\textbf{Asset} & \textbf{Reference} & \textbf{License} & \textbf{URL} \\
\midrule
\texttt{mamba-ssm} & \citet{gu2023mamba}       & Apache 2.0  & \href{https://github.com/state-spaces/mamba}{\texttt{github.com/state-spaces/mamba}} \\
Zoology / MQAR     & \citet{arora2023zoology}  & Apache 2.0  & \href{https://github.com/HazyResearch/zoology}{\texttt{github.com/HazyResearch/zoology}} \\
\texttt{mamba-tiny}& \citet{mambatiny2024}     & Apache 2.0  & \href{https://github.com/PeaBrane/mamba-tiny}{\texttt{github.com/PeaBrane/mamba-tiny}} \\
PyTorch            & \citet{paszke2019pytorch} & BSD 3-Clause & \href{https://github.com/pytorch/pytorch}{\texttt{github.com/pytorch/pytorch}} \\
\bottomrule
\end{tabular}
\label{tab:licenses}
\end{table}

\subsection{The Use of Large Language Models (LLMs)}
ChatGPT \citep{chatgpt2025} and Claude \citep{claude2025} were used as general-purpose writing and editing assistants in the preparation of this manuscript. Their role was limited to helping polish the presentation: rephrasing sentences for clarity, improving grammar and flow, suggesting alternative wording, and condensing or expanding explanations when asked. Importantly, they were not involved in the research ideation, design, implementation, analysis, or interpretation of results.
Additionally, ChatGPT and Claude were used for coding assistance.

\end{document}